\documentclass[3p,number,times,fleqn,preprint]{elsarticle}

\usepackage{enumitem}
\usepackage{latexsym}
\usepackage{amsmath,amssymb,amsthm}
\usepackage{pifont} 
\usepackage{hyperref}
\usepackage{booktabs}
\usepackage{multirow} 

\usepackage{xcolor}

\usepackage{tikz}
\usetikzlibrary{positioning,arrows,fit,calc}
\usetikzlibrary{decorations.pathmorphing}
\usetikzlibrary{shapes}
\usepackage{pgfplots}
\pgfplotsset{compat=newest} 

\newtheorem{theorem}{Theorem}

\newtheorem{claim}{Claim}[theorem]
\newtheorem{texample}{Example}[theorem]
\newtheorem{tlemma}[texample]{Lemma}
\newtheorem{corollary}[theorem]{Corollary}

\newtheorem{example}[theorem]{Example}

\newtheorem{tremark}[texample]{Remark}

\newcommand{\be}{b_{\q}}
\newcommand{\en}{e_{\q}}

\newcommand{\seqr}[1]{\langle {#1}\rangle}

\newcommand{\Apsi}{\A_{\q,\psi}}
\newcommand{\lit}{\boldsymbol{\ell}}

\newcommand{\bno}{\boldsymbol{b}}
\newcommand{\bnok}{\boldsymbol{b}^-}

\newcommand{\atm}{\boldsymbol{M}}

\newcommand{\tBox}{t_\Box}  
\newcommand{\tDiamond}{t_\Diamond}   

\newcommand{\initial}{initial}
\newcommand{\final}{final}
\newcommand{\cdist}{contact-distance}

\newcommand{\Kmin}{\mathfrak K_{\omq}^{\textit{min}}}
\newcommand{\SigmaQ}{\Sigma_\omq}
\newcommand{\SigmaQP}{\Sigma_\omq^e}

\newcommand{\TT}{\dis_A}

\newcommand{\shomo}{sub\-homomorphism}

\newcommand{\snode}{\mathfrak{s}}
\newcommand{\tnode}{\mathfrak{t}}
\newcommand{\unode}{\mathfrak{u}}
\newcommand{\vnode}{\mathfrak{v}}
\newcommand{\znode}{\mathfrak{z}}

\newcommand{\mh}[1]{H{#1}}

\newcommand{\neighbourhood}[1]{{$#1$}-neighbourhood}
\newcommand{\neighbourhoods}[1]{{$#1$}-neighbourhoods}
\newcommand{\connection}[1]{{$#1$}-connection}
\newcommand{\connections}[1]{{$#1$}-connections}
\newcommand{\Ia}{\I_{\!\mathfrak a}}
\newcommand{\spect}{special triple}

\newcommand{\uq}{F}
\newcommand{\dq}{T\!}
\newcommand{\qq}[1]{{}^{#1\!}\q}
\newcommand{\qqq}[2]{{}^{#1\!}\q^{\,#2}}
\newcommand{\ctt}[1]{\ct^{\,#1}}
\newcommand{\cff}[1]{\cf^{\,#1}}
\newcommand{\cttt}[2]{{}^{#1}\ct^{\,#2}}
\newcommand{\cfff}[2]{{}^{#1}\cf^{\,#2}}
\newcommand{\bnode}[2]{{}^{#2\!}{#1}}
\newcommand{\nttt}[3]{{}^{#1\!}t^{\,#2}_{#3}}
\newcommand{\nfff}[3]{{}^{#1\!}f^{\,#2}_{#3}}
\newcommand{\last}{\textit{last}}
\newcommand{\lbo}{\textit{last}-1}
\newcommand{\tStar}{t_\triangle}

\newcommand{\ct}{\mathsf{t}}
\newcommand{\cf}{\mathsf{f}}
\newcommand{\cc}{\mathsf{c}}
\newcommand{\lc}[1]{\mathsf{v}_{#1}}
\newcommand{\lcc}[1]{\mathsf{v}^{c}_{#1}}
\newcommand{\lcco}[1]{\mathsf{v}^{c_1}_{#1}}
\newcommand{\lcct}[1]{\mathsf{v}^{c_2}_{#1}}
\newcommand{\shift}{g_\leftarrow}
\newcommand{\shiftR}{g_\to}

\newcommand{\tfirst}{t_1}
\newcommand{\tsec}{t_2}

\newcommand{\tkth}{t_k}
\newcommand{\fkth}{f_k}
\newcommand{\tlast}{t_{\textit{last}}}

\newcommand{\tlbo}{t_{\textit{last}-1}}

\newcommand{\ffirst}{f_1}
\newcommand{\fsec}{f_2}

\newcommand{\worm}{\mathcal H}
\newcommand{\wheel}{cog\-wheel}
\newcommand{\wheels}{cog\-wheels}

\newcommand{\legW}{cog}

\newcommand{\bike}{bike}
\newcommand{\bikes}{bikes}

\newcommand{\pw}{\bullet}
\newcommand{\mw}{\circ}

\usepackage[ruled]{algorithm2e}

\SetAlFnt{\small}
\SetAlCapFnt{\small}
\SetAlCapNameFnt{\small}
\SetAlCapHSkip{0pt}
\IncMargin{-\parindent}

\SetKw{Guess}{guess}%
\SetKw{Check}{check}%
\SetKw{Choose}{choose}%
\SetKw{KwAnd}{and}%
\SetKw{KwSuchThat}{such that}%
\SetKwFunction{canMap}{canMap}%
\SetKwFunction{canMapTail}{canMapTail}%
\SetKwFunction{isGenerated}{isGenerated}%
\SetKwProg{func}{Function}{}{}%

\newfont{\midmathxx}{cmsy10 scaled 1440}
\newfont{\bigmathxx}{cmsy10 scaled 1440}
\newfont{\smallmathxx}{cmsy10 scaled 720}

\newcommand{\OWL}{\textsl{OWL\,2}}
\newcommand{\OWLQL}{\textsl{OWL\,2\,QL}}
\newcommand{\OWLEL}{\textsl{OWL\,2\,EL}}

\newcommand{\NL}{\textsc{NL}}

\newcommand{\DL}{\textsl{DL-Lite}}

\newcommand{\DLb}{\textsl{DL-Lite}_\textit{bool}}
\newcommand{\DLk}{\textsl{DL-Lite}_\textit{krom}}

\newcommand{\q}{{\boldsymbol{q}}}

\newcommand{\G}{{\boldsymbol{G}}}

\newcommand{\ACz}{{\ensuremath{\textsc{AC}^0}}}

\newcommand{\LogSpace}{\textsc{L}}

\newcommand{\coNP}{\textsc{coNP}}
\newcommand{\NP}{\textsc{NP}}
\newcommand{\PTime}{\textsc{P}}
\newcommand{\ExpTime}{\textsc{ExpTime}}
\newcommand{\NExpTime}{\textsc{NExpTime}}

\newcommand{\PSpace}{\textsc{PSpace}}

\newcommand{\ind}{\mathsf{ind}}

\newcommand{\cir}{\boldsymbol{C}}

\newcommand{\A}{\mathcal{A}}
\newcommand{\C}{\mathcal{C}}
\newcommand{\T}{\mathcal{O}}

\newcommand{\I}{\mathcal{I}}

\newcommand{\Tf}{\mathcal{C}}

\newcommand{\x}{\mathsf{x}}

\newcommand{\avec}[1]{\boldsymbol{#1}}

\def\omq{{\boldsymbol{Q}}}

\newcommand{\rank}{{\boldsymbol{b}\boldsymbol{r}}}

\newcommand{\dis}{\mathsf{cov}}

\def\C{{\cal C}}

\newcommand{\st}{{\boldsymbol{st}}}

\def\s{{\mathfrak s}}

\def\Af{{\mathfrak A}}

\tikzstyle{or-gate}=[rectangle,draw,inner sep=4pt,thick]
\tikzstyle{and-gate}=[rectangle,draw,inner sep=4pt,thick]
\tikzstyle{input}=[circle,draw,minimum size=3mm]

\tikzset{>=latex, 
	point/.style = {circle,draw,thick,minimum size=2mm,inner sep=0pt},
	point1/.style = {circle,draw,thick,minimum size=6mm,inner sep=0pt},
	hm/.style = {dotted,semithick},
	role/.style = {thick},
	tree/.style = {rounded corners=10pt, dashed, fill opacity=0.5, fill=nullscolour},
	wiggly/.style={thick,
	},
	query/.style={thick},
}

\begin{document}

\begin{frontmatter}



\title{A Tetrachotomy of Ontology-Mediated Queries with a Covering Axiom}


\author[hse]{Olga Gerasimova}
\ead{ogerasimova@hse.ru}

\author[lm]{Stanislav Kikot}
\ead{staskikotx@gmail.com}

\author[kcl]{Agi Kurucz}
\ead{agi.kurucz@kcl.ac.uk}

\author[smi,hse]{Vladimir Podolskii}
\ead{podolskii@mi.ras.ru}

\author[bbk]{Michael Zakharyaschev}
\ead{michael@dcs.bbk.ac.uk}

\address[hse]{HSE University, Moscow, Russia}
\address[lm]{Institute for Information Transmission Problems, Moscow, Russia}
\address[kcl]{Department of Informatics, King's College London, U.K.}
\address[smi]{Steklov Mathematical Institute, Moscow, Russia}
\address[bbk]{Department of Computer Science and Information Systems, Birkbeck, University of London, U.K.}


\begin{abstract}
Our concern is the problem of efficiently determining the data complexity of answering queries mediated by description logic ontologies and constructing their optimal rewritings to standard database queries. Originated in ontology-based data access and datalog optimisation, this problem is known to be computationally very complex in general, with no explicit syntactic characterisations available. In this article, aiming to understand the fundamental roots of this difficulty, we strip the problem to the bare bones and focus on Boolean conjunctive queries mediated by a simple  covering axiom stating that one class is covered by the union of two other classes. 
We show that, on the one hand, these rudimentary ontology-mediated queries, called disjunctive sirups (or d-sirups), capture many features and difficulties of the general case. For example, answering d-sirups is $\Pi^p_2$-complete for combined complexity and can be in \ACz{} or \LogSpace-, \NL-, \PTime-, or \coNP-complete for data complexity (with the problem of recognising FO-rewritability of d-sirups being 2\ExpTime-hard); some d-sirups only have exponential-size resolution proofs, some only double-exponential-size positive existential FO-rewritings and single-exponential-size nonrecursive datalog rewritings.  
On the other hand, we prove a few partial sufficient and necessary conditions of FO- and (symmetric/linear-) datalog rewritability of d-sirups. Our main technical result is a complete and transparent syntactic $\ACz$\,/\,\NL\,/\,\PTime\,/\,\coNP{} tetrachotomy of d-sirups with disjoint covering classes and a path-shaped Boolean conjunctive query. To obtain this tetrachotomy, we develop new techniques for establishing \PTime- and \coNP-hardness of answering non-Horn ontology-mediated queries as well as showing that they can be answered in \NL{}.
\end{abstract}



\end{frontmatter}


\section{Introduction}\label{intro}

\subsection{The ultimate question}

The general research problem we are concerned with in this article can be formulated as follows: for any given ontology-mediated query (OMQ, for short) $\omq = (\T,\q)$ with a description logic ontology $\T$ and a conjunctive query $\q$, 
\begin{description}
\item[(data complexity)] determine the computational complexity of answering $\omq$ over any input data instance $\A$ under the open world semantics and, if possible,

\item[(rewritability)] reduce the task of finding certain answers to $\omq$ over any input $\A$ to the task of evaluating a conventional database query $\omq'$  with optimal data complexity directly over $\A$ (the query $\omq'$ is then called a \emph{rewriting} of the OMQ $\omq$).
\end{description}
\paragraph{Ontology-based data access}
Answering queries mediated by a description logic (DL) ontology has been known as an important reasoning problem in knowledge representation since the early 1990s~\cite{Schaerf93}. The proliferation of DLs and their applications~\cite{BaaderCalvaneseEtAl2007,DBLP:books/daglib/0041477}, the development of the (DL-underpinned) Web Ontology Language OWL\footnote{\url{https://www.w3.org/TR/owl2-overview/}}\!, and especially the paradigm of ontology-based data access (OBDA)~\cite{PLCD*08,CDLLR07,DBLP:conf/ijcai/XiaoCKLPRZ18} (proposed in the mid 2000s and recently rebranded to the virtual knowledge graph (VKG) paradigm~\cite{DBLP:journals/dint/XiaoDCC19}), have made theory and practice of answering ontology-mediated queries (OMQs) a hot research area lying at the crossroads of Knowledge Representation and Reasoning, Semantic Technologies and the Semantic Web, Knowledge Graphs, and Database Theory and Technologies. 

In a nutshell, the idea underlying OBDA is as follows. The users of an OBDA system (such as Mastro\footnote{\url{https://www.obdasystems.com}} or Ontop\footnote{\url{https://ontopic.biz}}) may assume that the data they want to query is given in the form of a directed graph whose nodes are labelled with concepts (unary predicates or classes) and whose edges are labelled with roles (binary predicates or properties)---even though, in reality, the data can be physically stored in different and possibly heterogeneous data sources---hence the moniker VKG. The concept and role labels come from an ontology, designed by a domain expert, and should be familiar to the intended users who, on the other hand, do not have to know anything about the real data sources. Apart from providing a user-friendly vocabulary for queries and a high-level conceptual view of the data, an important role of the ontology is to enrich possibly incomplete data with background knowledge. To illustrate, imagine that we are interested in the life of `scientists' and would like to satisfy our curiosity by querying the data available on the Web (it may come from the universities' databases, publishing companies, personal web pages, social networks, etc.). An ontology $\mathcal{O}$ about scientists, provided by an OBDA system, might contain the following `axioms' (given, for readability, both as DL concept inclusions and first-order sentences):
\begin{align}\label{ax1}
& \textit{BritishScientist} ~\sqsubseteq~ \exists\, \textit{affiliatedWith}. \textit{UniversityInUK}\\\notag
& \mbox{}\hspace*{3cm} {\small \forall x \, [\textit{BritishScientist}(x) \to \exists y \, (\textit{affiliatedWith}(x,y) \land \textit{UniversityInUK}(y))]}\\\label{ax2}
& \exists\, \textit{worksOnProject} ~\sqsubseteq~ \textit{Scientist}\\\notag
& \mbox{}\hspace*{3cm} {\small \forall x \, [\exists y \, \textit{worksOnProject}(x,y) \to \textit{Scientist}(x)]}\\\label{ax3}
& \textit{Scientist} \sqcap \exists\, \textit{affiliatedWith}. \textit{UniversityInUK} ~\sqsubseteq~ \textit{BritishScientist}\\\notag
& \mbox{}\hspace*{3cm} {\small \forall x \, [(\textit{Scientist}(x) \land \exists y\, (\textit{affiliatedWith}(x,y) \land \textit{UniversityInUK}(y))) \to  \textit{BritishScientist}(x)]}\\\label{ax4}
& \textit{BritishScientist} ~\sqsubseteq~ \textit{Brexiteer} \sqcup \textit{Remainer}\\\notag
& \mbox{}\hspace*{3cm} {\small \forall x \, [\textit{BritishScientist}(x) \to  (\textit{Brexiteer}(x) \lor \textit{Remainer}(x))]}
\end{align}
Now, to find, for example, British scientists, we could execute a simple OMQ $\omq(x) = (\mathcal{O},\q(x))$ with the query 
\begin{equation*}\label{que1}
\q(x) ~=~ \textit{BritishScientist}(x)
\end{equation*}
mediated by the ontology $\mathcal{O}$. The OBDA system is expected to return the members of the concept \textit{BritishScientist} that are extracted from the original datasets by `mappings' (database queries connecting the data with the ontology vocabulary and virtually populating its concepts and roles) and also deduced from the data and axioms in $\mathcal{O}$ such as~\eqref{ax3}. It is this latter reasoning task that makes OMQ answering non-trivial and potentially intractable both in practice and from the complexity-theoretic point of view. 

\paragraph{Uniform approach} 
To ensure theoretical and practical tractability, the OBDA paradigm presupposes that the users' OMQs are reformu\-lated---or rewritten---by the OBDA system into conventional database queries over the original data sources, which have proved to be  quite efficiently evaluated by the existing database management systems. Whether or not such a rewriting is possible and into which query language naturally depends on the OMQ in question. One way to \emph{uniformly} guarantee the desired rewritability is to delimit the language for OMQ ontologies and queries. Thus, the \DL{} family of description logics~\cite{CDLLR07} and the \OWLQL{} profile\footnote{\url{https://www.w3.org/TR/owl2-profiles/}} of \OWL{} were designed so as to guarantee rewritability of \emph{all} OMQs with a \DL{} ontology and a  conjunctive query (CQ) into first-order (FO) queries, that is, essentially SQL queries~\cite{Abitebouletal95}. In complexity-theoretic terms, FO-rewritability of an OMQ means that it can be answered in \textsc{LogTime} uniform \ACz{}, one of the smallest complexity classes~\cite{Immerman99}. In our example above, only axioms \eqref{ax1} and \eqref{ax2} are allowed by \OWLQL. 
Various dialects of tuple-generating dependencies (tgds), aka datalog$^\pm$ or existential rules, that admit FO-rewritability and extend \OWLQL{} have also been identified; see, e.g.,~\cite{DBLP:conf/datalog/CiviliR12,DBLP:journals/tods/GottlobOP14,DBLP:journals/ai/BagetLMS11,DBLP:journals/semweb/KonigLMT15}. 

Any OMQ with an $\mathcal{EL}$, \OWLEL{} or $\textsl{Horn}\mathcal{SHIQ}$ ontology is datalog-rewritable~\cite{DBLP:conf/ijcai/HustadtMS05,DBLP:conf/dlog/Rosati07,DBLP:journals/japll/Perez-UrbinaMH10,DBLP:conf/aaai/EiterOSTX12}, and so can be answered in \PTime---polynomial time in the size of data---using various datalog engines, say GraphDB\footnote{\url{https://graphdb.ontotext.com}}\!\!, LogicBlox\footnote{\url{https://developer.logicblox.com}} or RDFox\footnote{\url{https://www.oxfordsemantic.tech}}\!\!. Axioms \eqref{ax1}--\eqref{ax3} are admitted by the $\mathcal{EL}$ syntax. 
On the other hand, OMQs with an $\mathcal{ALC}$ (a notational variant of the multimodal logic {\bf K}$_n$~\cite{Gabbayetal03}) ontology and a CQ are in general \coNP-complete~\cite{Schaerf93}, and so often regarded as intractable and not suitable for OBDA, though they can be rewritten to disjunctive datalog~\cite{DBLP:phd/de/Motik2006,DBLP:journals/jar/HustadtMS07,DBLP:conf/ijcai/GrauMSH13} supported by systems such as DLV\footnote{\url{http://www.dlvsystem.com}} or clasp\footnote{\url{https://potassco.org/clasp/}}\!\!. For example, \coNP-complete is the OMQ $(\{\eqref{ax4}\},\q_1)$ with the CQ  
\begin{multline*}
\q_1 = \exists w,x,y,z \, [ \textit{Brexiteer}(w) \land \textit{hasCoAuthor}(w,x) \land{}   \textit{Remainer}(x) \land{} \\ \textit{hasCoAuthor}(x,y) \land \textit{Brexiteer}(y) \land \textit{hasCoAuthor}(y,z) \land  \textit{Remainer}(z)]
\end{multline*}
(see also the representation of $\q_1$ as a labelled graph below). It might be of interest to note that by making the role \textit{hasCoAuthor} symmetric using, for example, the role inclusion axiom 
\begin{align}\label{symm}
& \textit{hasCoAuthor} ~\sqsubseteq~ \textit{hasCoAuthor}^-\\\notag
& \mbox{}\hspace*{3cm} {\small \forall x ,y\, [\textit{hasCoAuthor}(x,y) \to \textit{hasCoAuthor}(y,x)]} 
\end{align}
we obtain the OMQ $(\{\eqref{ax4},\eqref{symm}\},\q_1)$, which is rewritable to a symmetric datalog query, and so can be answered by a highly parallelisable algorithm in the complexity class \LogSpace{} (logarithmic space). 

For various reasons, many existing ontologies do not comply with the restrictions imposed by the standard languages for OBDA. Notable examples include the large-scale medical ontology SNOMED CT\footnote{\url{https://bioportal.bioontology.org/ontologies/SNOMEDCT}}\!, which is mostly but not entirely in $\mathcal{EL}$, and the oil and gas NPD FactPages\footnote{\url{https://factpages.npd.no}} ontology and the Subsurface Exploration Ontology~\cite{DBLP:conf/semweb/HovlandKSWZ17}, both of which fall outside \OWLQL{} by a whisker, in particular because of \emph{covering axioms} like~\eqref{ax4} that are quite typical in conceptual modelling. One way to (partially) resolve this issue is to compute an approximation of a given ontology within the required ontology language, which is an interesting and challenging reasoning problem by itself;  see, e.g.,~\cite{DBLP:conf/cade/MartinezFGHH14,DBLP:journals/jair/ZhouGNKH15,DBLP:conf/aaai/BotoevaCSSSX16,DBLP:conf/ijcai/BotcherLW19} and references therein. 
In practice, the non-complying axioms are often simply omitted from the ontology in the hope that not too many answers to OMQs will be lost. An attempt to figure out whether it was indeed the case for the OMQs with the Subsurface Exploration Ontology and  geologists' queries from~\cite{DBLP:conf/semweb/HovlandKSWZ17} was the starting point of research that led to this article.

\paragraph{Non-uniform approach} 
An ideal alternative to the uniform approach to OBDA discussed above would be to admit 
OMQs in a sufficiently expressive language and supply the OBDA system with an algorithm that recognises the data complexity of each given OMQ and rewrites it to a database query in the corresponding target language.
For example, while answering the OMQ $(\{\eqref{ax4}\}, \q_1)$ is \coNP-complete, we shall see later on in this paper that $(\{\eqref{ax4}\}, \q_2)$ with the same ontology and the CQ $\q_2$ shown in the picture below is \PTime-complete and datalog-rewritable, $(\{\eqref{ax4}\}, \q_3)$ is \NL- (non-deterministic logarithmic space) complete and linear-datalog-rewritable, $(\{\eqref{ax4}\}, \q_4)$ is \LogSpace-complete and symmetric-datalog-rewritable, while $(\{\eqref{ax4}\}, \q_5)$ is in \ACz{} and FO-rewritable. In the picture, $F(u)$ stands for $\textit{Brexiteer}(u)$, $T(u)$ for\\[-10pt]
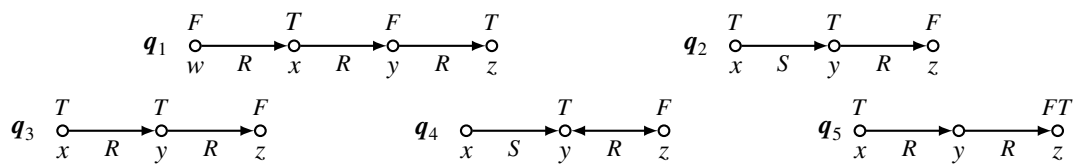
\begin{center}
\begin{tikzpicture}[>=latex,line width=0.8pt, rounded corners,scale = 1.3]
		\node (0) at (-0.4,0) {$\q_1$};
		\node[point,scale=0.7,label=above:{\small $F$},label=below:{$w$}] (1) at (0,0) {};
		\node[point,scale=0.7,label=above:{$T$},label=below:{$x$}] (m) at (1,0) {};
		\node[point,scale=0.7,label=above:{\small $F$},label=below:{$y$}] (2) at (2,0) {};
		\node[point,scale=0.7,label=above:{\small $T$},label=below:{$z$}] (3) at (3,0) {};
		\draw[->,right] (1) to node[below] {\small $R$}  (m);
		\draw[->,right] (m) to node[below] {\small $R$} (2);
		\draw[->,right] (2) to node[below] {\small $R$} (3);
		\end{tikzpicture}
\hspace*{2cm}
		\begin{tikzpicture}[>=latex,line width=0.8pt, rounded corners,scale = 1.3]
		\node (0) at (-0.4,0) {$\q_2$};
		\node[point,scale=0.7,label=above:{\small $T$},label=below:{ $x$}] (1) at (0,0) {};
		\node[point,scale=0.7,label=above:\small$T$,label=below:{ $y$}] (m) at (1,0) {};
		\node[point,scale=0.7,label=above:\small$F$,label=below:{ $z$}] (2) at (2,0) {};
		\draw[->,right] (1) to node[below] {\small $S$}  (m);
		\draw[->,right] (m) to node[below] {\small $R$} (2);
		\end{tikzpicture}\\
\begin{tikzpicture}[>=latex,line width=0.8pt, rounded corners,scale = 1.3]
		\node (0) at (-0.4,0) {$\q_3$};
		\node[point,scale=0.7,label=above:{\small $T$},label=below:{ $x$}] (1) at (0,0) {};
		\node[point,scale=0.7,label=above:\small$T$,label=below:{ $y$}] (m) at (1,0) {};
		\node[point,scale=0.7,label=above:\small$F$,label=below:{ $z$}] (2) at (2,0) {};
		\draw[->,right] (1) to node[below] {\small $R$}  (m);
		\draw[->,right] (m) to node[below] {\small $R$} (2);
		\end{tikzpicture}
\hspace*{1.5cm}
\begin{tikzpicture}[>=latex,line width=0.8pt, rounded corners, scale = 1.3]
\node (0) at (-0.4,0) {$\q_4$};
		\node[point,scale=0.7,label=above:{},label=below:{ $x$}] (1) at (0,0) {};
		\node[point,scale=0.7,label=above:{\small $T$},label=below:{ $y$}] (m) at (1,0) {};
		\node[point,scale=0.7,label=above:{\small $F$},label=below:{ $z$}] (2) at (2,0) {};
		\draw[->,right] (1) to node[below] {\small $S$}  (m);
		\draw[<->,right] (m) to node[below] {\small $R$} (2);
		\end{tikzpicture}
\hspace*{1.5cm}
\begin{tikzpicture}[>=latex,line width=0.8pt, rounded corners, scale = 1.3]
\node (0) at (-0.3,0) {$\q_5$};
		\node[point,scale=0.7,label=above:{\small $T$},label=below:{ $x$}] (1) at (0,0) {};
		\node[point,scale=0.7,label=above:{},label=below:{ $y$}] (m) at (1,0) {};
		\node[point,scale=0.7,label=above:{\small $FT$},label=below:{ $z$}] (2) at (2,0) {};
		\draw[->,right] (1) to node[below] {\small $R$}  (m);
		\draw[->,right] (m) to node[below] {\small $R$} (2);
		\end{tikzpicture}		
\end{center}
%
$\textit{Remainer}(u)$, $R(u,v)$ for $\textit{hasCoAuthor}(u,v)$, $S(u,v)$ for $\textit{hasBoss}(x,y)$, and all of the variables $w$, $x$, $y$, $z$ are assumed to be existentially quantified.
Another example is the experiments with the NPD FactPages and Subsurface Exploration ontologies used for testing OBDA in industry~\cite{DBLP:conf/dlog/Rosati07,DBLP:conf/semweb/HovlandKSWZ17,DBLP:journals/ws/KharlamovHSBJXS17}. Although the ontologies contain  covering axioms of the form $A \sqsubseteq B_1 \sqcup \dots\sqcup B_n$ not allowed in \OWLQL{}, one can show that the concrete queries provided by the end-users do not `feel' those dangerous axioms and are FO-rewritable. 
Note also the experiments in~\cite{DBLP:journals/ai/KaminskiNG16} showing that  rewriting non-Horn OMQs to datalog can significantly improve the efficiency of  answering by means of existing engines. 

Is it possible to efficiently recognise the data complexity of answering any given OMQ and construct its optimal rewriting? The database community has been investigating these questions in the context of datalog optimisation since the 1980s; see Section~\ref{related} for details and references. For various families of DLs, a complexity-theoretic analysis of the {\bf (data complexity)} problem was launched by Lutz and Wolter~\cite{DBLP:conf/kr/LutzW12} and Bienvenu et al.~\cite{DBLP:journals/tods/BienvenuCLW14}. Incidentally, the latter discovered a close connection with another important and rapidly growing area of Computer Science and AI: constraint satisfaction problems (CSPs), for which a P/NP-dichotomy, conjectured by Feder and Vardi~\cite{DBLP:journals/siamcomp/FederV98}, has recently been established~\cite{DBLP:conf/focs/Bulatov17,DBLP:conf/focs/Zhuk17}. We briefly survey the current state of the art in Section~\ref{related} below. Here, it suffices to say  that recognising FO-rewritability is \ExpTime-complete for OMQs with a `lightweight' $\mathcal{EL}$ ontology~\cite{DBLP:conf/ijcai/LutzS17,DBLP:journals/corr/abs-1904-12533} and 2\NExpTime-complete for OMQs with a `full-fledged' $\mathcal{ALC}$ ontology~\cite{DBLP:journals/lmcs/FeierKL19}. In either case, the problem seems to be too complex for a universal algorithmic solution, although  experiments in~\cite{DBLP:conf/ijcai/LutzHSW15} demonstrated that many real-life atomic OMQs in $\mathcal{EL}$ can be efficiently rewritten to non-recursive datalog by the \ExpTime{} algorithm. 

A more practical take on the {\bf (rewritability)} problem, started by  Motik~\cite{DBLP:phd/de/Motik2006}, exploits the datalog connection mentioned above. In a nutshell, the idea is as follows. OMQs with a Horn DL ontology are rewritten to datalog queries, which could further be treated by the datalog optimisation techniques for removing or linearising recursion or partial FO-rewriting algorithms such as~\cite{DBLP:conf/dlog/KaminskiG13}. Non-Horn OMQs are transformed to (possibly exponential-size~\cite{DBLP:journals/jar/HustadtMS07}) disjunctive datalog queries to which partial datalog rewriting algorithms such as the ones in~\cite{DBLP:journals/ai/KaminskiNG16} can be applied. It is to be emphasised, however, that tractable datalog optimisation and rewriting techniques cannot be complete.

In this article, we propose to approach the ultimate question from a different, bottom-up direction. In order to see the wood for the trees, we isolate some major sources of difficulty with {\bf (data complexity)} and {\bf (rewritability)} within a syntactically simple yet highly non-trivial class of OMQs. Apart from unearthing the fundamental roots of high complexity, this will allow us to obtain explicit syntactic rewritability conditions and even complete classifications of OMQs according to their data complexity and rewritability type. (Note that similar approaches were taken for analysing datalog programs and CSPs; see Sections~\ref{ourcontribution} and~\ref{related}.)

\subsection{Our contribution}\label{ourcontribution}

We investigate the {\bf (data complexity)} and {\bf (rewritability)} problems for OMQs $\omq$ of a very simple form:
\begin{description}
\item[(d-sirup)] $\omq = (\dis_A,\q)$, where $\dis_A = \{\,A \sqsubseteq F \sqcup T\,\}$ and $\q$ is a Boolean CQ with unary predicates $F$, $T$ and arbitrary binary predicates. 
\end{description}
Our ultimate aim is to understand how the interplay between the \emph{covering axiom} $A \sqsubseteq F \sqcup T$ and the structure of $\q$ determines the complexity and rewritability properties of $\omq$. By regarding $\q$ and data instances as labelled directed graphs (like in the picture above), we can formulate the problem of answering $\omq$ in plain graph-theoretic terms: 
\begin{description}
\item[\rm \quad \textsc{Instance:}] any labelled directed graph (digraph, for short) $\A$; 

\item[\rm \quad \textsc{Problem:}] decide whether each digraph obtained by labelling every $A$-node in $\A$ with either $F$ or $T$ contains a homomorphic image of $\q$ (in which case the certain answer to $\omq$ over $\A$ is `yes'). 
\end{description}
By definition (see, e.g.,~\cite{Arora&Barak09}), this can be done in \coNP{} as $\q$ is \emph{fixed}, and so the existence of a homomorphism from $\q$ to any labelling of $\A$ can be checked in polynomial time by inspecting all possible $|\A|^{|\q|}$-many maps from $\q$ to $\A$. In practice, we could try to solve this problem using, say, a resolution-based prover (see Example~\ref{ex-resolution}) or by evaluating the disjunctive datalog program $\{\eqref{d-sirup1},\eqref{d-sirup2}\}$ below over $\A$, both of which would require finding proofs of exponential size in general (see Theorem~\ref{thm:resolution}). The {\bf (data complexity)} and {\bf (rewritability)} problems ask whether there exists a more efficient algorithmic solution for the given $\omq$ in principle and whether it can be realised as a standard (linear, symmetric) datalog or FO-query evaluated over the input graphs $\A$.

The OMQ $\omq = (\dis_A,\q)$ is equivalent to the \emph{monadic disjunctive datalog query}
\begin{align}\label{d-sirup1}
T(x) \lor F(x) & \leftarrow A(x)\\\label{d-sirup2}
\boldsymbol{G} & \leftarrow \q
\end{align}
with a nullary goal predicate $\G$. In the 1980s, trying to understand boundedness (FO-rewritability) and linearisability (linear-datalog-rewritability) of datalog queries, the database community introduced the notion of \emph{sirup}---standing for `\emph{datalog query with a single recursive rule}'~\cite{DBLP:conf/pods/CosmadakisK86,DBLP:conf/pods/Vardi88,DBLP:journals/iandc/GottlobP03}---which was thought to be crucial for understanding datalog recursion and optimising datalog programs~\cite[Problem 4.2.10]{DBLP:books/el/leeuwen90/Kanellakis90}. Our   OMQs $\omq$ or disjunctive datalog queries $(\{\eqref{d-sirup1},\eqref{d-sirup2}\},\boldsymbol{G})$---which henceforth are referred to as (\emph{monadic}) \emph{disjunctive sirups} or simply \emph{d-sirups}---play the same fundamental role for understanding OMQs with expressive ontologies and monadic disjunctive datalog queries. 

Looking pretty trivial syntactically, d-sirups form a very sophisticated class of OMQs. For example, deciding FO-rewritability of d-sirups (even those of them that are equivalent to monadic datalog sirups) turns out to be 2\ExpTime-hard~\cite{PODS21}---as complex as deciding program boundedness of arbitrary monadic datalog programs~\cite{DBLP:conf/stoc/CosmadakisGKV88,DBLP:conf/lics/BenediktCCB15}.
%
%
Interestingly, one of the sources of this unexpectedly high complexity is `twin' $FT$-labels of nodes in CQs like $\q_5$ above. We can eliminate this source by imposing the standard \emph{disjointness constraint} $F \sqcap T \sqsubseteq \bot$ (or $\bot \leftarrow F(x), T(x)$ in datalog parlance), often used in ontologies and conceptual modelling. Thus, we arrive to \emph{dd-sirups} of the form  
\begin{description}
\item[(dd-sirup)] $\omq = (\dis_A^\bot,\q)$, where $\dis_A^\bot = \{\,A \sqsubseteq F \sqcup T,\ F \sqcap T \sqsubseteq \bot\,\}$. 
\end{description}
The complexity and rewritability of both d- and dd-sirups only depend on the structure of the CQs $\q$, which suggests a research programme of classifying (d)d-sirups by the type of the graph underlying $\q$---directed path, tree, their undirected variants, etc.---and characterising the data complexity and rewritability of OMQs in the resulting classes. 
Thus, in the context of datalog sirups, Afrati and Papadimitriou~\cite{DBLP:journals/jacm/AfratiP93} gave a complete characterisation of \emph{binary chain} sirups that are computable in NC, and so parallelisable. Actually, according to~\cite{DBLP:journals/jacm/AfratiP93}, Kanellakis and Papadimitriou `have investigated the case of unary sirups, and have made progress towards a complete characterization'\!. Unfortunately, that work has never been  published\footnote{\url{https://en.wikipedia.org/wiki/Paris_Kanellakis}}\!\!. (As shown later on in this article and~\cite{PODS21}, unary datalog sirups are closely connected to d-sirups.)    

The main achievement of this article is a complete characterisation of dd-sirups with a path-shaped CQ (like $\q_1$--$\q_3$ and $\q_5$ above). 
Syntactically, the obtained characterisation, a tetrachotomy, is transparent and easily checkable: for any dd-sirup $\omq = (\dis_A^\bot,\q)$ with a path-shaped CQ $\q$, 
%
\begin{description}
\item[($\ACz$)] $\omq$ is FO-rewritable and can be answered in $\ACz$ iff $\q$ contains an $FT$-twin or has no $F$-nodes or no $T$-nodes; 
\end{description}
otherwise,
\begin{description}			
\item[(\NL)] $\omq$ is linear-datalog-rewritable and answering it is \NL-complete if $\q$ is a `periodic' CQ with a single $F$-node or a single $T$-node;
			
\item[(\PTime)] $\omq$ is datalog-rewritable and answering it is \PTime-complete if $\q$ is an `aperiodic' CQ with a single $F$- or $T$-node;
			
\item[(\coNP)] answering $\omq$ is \coNP-complete if $\q$ has at least two $F$-nodes and at least two $T$-nodes.
\end{description}
%
%
%
%
%
(Assuming that $\NL \ne \PTime \ne \coNP$, the three `if' above can be replaced by `iff'\!.)
From the technical point of view, however, to establish this first complete syntactic characterisation of OMQs with disjunctive axioms, we require an adaptation of known methods from description logic~\cite{DBLP:conf/ijcai/LutzS17,DBLP:journals/corr/abs-1904-12533} and datalog~\cite{DBLP:conf/stoc/CosmadakisGKV88,DBLP:conf/lics/BenediktCCB15} as well as developing novel techniques for proving \PTime- and especially \coNP-hardness. As a (cruel) exercise, the reader might be tempted to consider the dd-sirup with $\q_1$ above and then permute the $F$s and $T$s in it. 
The known techniques of encoding NP-complete problems such as 2+2CNF or graph 3-colouring in terms of OMQ answering are not applicable in this case as $\dis_A^\bot$ is not capable of any reasoning bar \emph{binary} case distinction and $\q_1$ has only one binary relation. (To compare, the first \coNP-hard d-sirup found by Schaerf~\cite{Schaerf93} has five roles that are used to encode clauses and their literals in 2+2-CNFs.) An even harder problem is to find a unified construction for arbitrary path-shaped CQs as different types of them require different treatment. 
%
%
%
\paragraph{Structure of the article}
In the remainder of this section, we briefly review the related work. Section~\ref{prelims} contains the necessary background definitions. It also shows (by reduction of the mutilated chessboard problem~\cite{DBLP:conf/focs/DantchevR01,DBLP:journals/tcs/Alekhnovich04}) that answering d-sirups using resolution-based provers requires finding proofs of exponential size in general. 

In Section~\ref{sec:gen}, we make an initial scan of the `battleground' and obtain a few relatively simple complexity and rewritability results for arbitrary (not necessarily path-shaped) d- and dd-sirups. First, we show (by reduction of $\forall\exists$SAT) that answering (d)d-sirups is $\Pi^p_2$-complete for combined complexity (in the size of $\q$ and $\A$), that is, harder than answering \DL{} and $\mathcal{EL}$ OMQs~\cite{DBLP:journals/jacm/BienvenuKKPZ18,DBLP:conf/dlog/Rosati07} (unless $\NP = \Pi_2^p$, and so $\NP = \PSpace$). This result is an improvement on $\Pi^p_2$-hardness of answering OMQs with a Schema.org ontology~\cite{DBLP:conf/ijcai/HernichLOW15}, which are more expressive than (d)d-sirups. 
Then we start classifying d-sirups in terms of occurrences of $F$ and $T$ in the CQs $\q$. Those without occurrences of a solitary $F$ (like $\q_5$) or a solitary $T$ are readily seen to be FO-rewritable. All other twinless d-sirups are shown to be L-hard, with certain symmetric d-sirups with one solitary $F$ and one solitary $T$ being rewritable to symmetric datalog, and so L-complete. D-sirups with a single solitary $F$ or a single solitary $T$ (and possibly with twins) are shown to be rewritable to monadic datalog queries (which also follows from~\cite{DBLP:journals/ai/KaminskiNG16}). This observation allows us to use datalog expansions~\cite{DBLP:books/cs/Ullman89} (called \emph{cactuses} in our context) and automata-theoretic techniques~\cite{DBLP:conf/stoc/CosmadakisGKV88} to analyse FO- and linear-datalog-rewritability of the corresponding d-sirups. In~\cite{PODS21}, we used the criterion of FO-rewritability in terms of cactuses to prove that deciding FO-rewritability of d-sirups with a single solitary $F$ or $T$ as well as that of monadic datalog sirups is 2\ExpTime-complete. Here, we show that nonrecursive datalog, positive existential and UCQ-rewritings of such d-sirups are of at least  single-, double- and triple-exponential size in the worst case, respectively. 


As far as we are aware, there is no known semantic or syntactic criterion distinguishing between datalog programs in \NL{} and \PTime, though Lutz and Sabellek~\cite{DBLP:conf/ijcai/LutzS17,DBLP:journals/corr/abs-1904-12533} gave a nice semantic characterisation of OMQs with an $\mathcal{EL}$ ontology. In Section~\ref{boundedLin}, we combine their ideas with the automata-theoretic technique of Cosmadakis et al.~\cite{DBLP:conf/stoc/CosmadakisGKV88} and prove a useful graph-theoretic sufficient condition for d-sirups to be linear-datalog-rewritable (and so in \NL). Note that every d-sirup whose CQ $\q$ is a ditree with a single solitary $F$ (or $T$) at the root can be rewritten to an atomic OMQ in $\mathcal{EL}$, to which the \ExpTime-complete trichotomy of~\cite{DBLP:conf/ijcai/LutzS17} is applicable.

Finally, in Sections~\ref{sec:tetra} and~\ref{monster}, we obtain the tetrachotomy of the  path-shaped dd-sirups discussed above. Items {\bf ($\ACz$)} and {\bf (\NL)} and the upper bound in {\bf (\PTime)} follow from the previous sections. By far the hardest part of the tetrachotomy is establishing P- and \coNP-hardness. To prove the former, we assemble AND- and OR-gates from copies of a given aperiodic CQ and then use those gates to construct ABoxes that `compute' arbitrary monotone Boolean circuits, which is known to be P-complete. The structure of the gates and circuits is uniform for each type of aperiodicity. 
In the proof of \coNP-hardness, building 3CNFs from copies of a given CQ $\q$ is not uniform as various parts of the construction subtly depend on the order of and the  distances between the $F$- and $T$-nodes in $\q$. We are not aware of any even remotely similar methods in the literature, and believe that our novel `bike technique' can be used for showing \coNP-hardness of many other classes of OMQs.  
(It might be of interest to note that the \coNP-hardness result in our tetrachotomy implies completeness of the datalog rewriting algorithm from~\cite{DBLP:journals/ai/KaminskiNG16}  for path-shaped dd-sirups.)

In Section~\ref{conclude}, we summarise the obtained results and formulate a few open problems for future research.

An extended abstract~\cite{DBLP:conf/kr/GerasimovaKKPZ20} with some of the results from this article has been presented at the 17th International Conference on Principles of Knowledge Representation and Reasoning.


\subsection{Related work}\label{related}

There have been two big waves of research related to {\bf (data complexity)} and {\bf (rewritability)} of ontology-mediated queries. The first one started in the mid 1980s, when the database community was working on optimisation and  parallelisation of datalog programs, which was hoped to be done by `intelligent compilers' (see, e.g.,~\cite{DBLP:conf/pods/Naughton86,DBLP:conf/pods/CosmadakisK86,DBLP:journals/algorithmica/UllmanG88,DBLP:journals/jacm/Naughton89,ramakrishnan1989proof,DBLP:conf/pods/Saraiya89,DBLP:conf/icdt/Wang95}, surveys~\cite{DBLP:books/el/leeuwen90/Kanellakis90,DBLP:journals/csur/DantsinEGV01} and references therein). One of the  fundamental problems considered was to decide whether the depth of recursion required to evaluate a given datalog query could be bounded independently of the input data,  which implies FO-rewritability of the datalog query. Boundedness was shown to be decidable in P for some classes of linear programs~\cite{DBLP:conf/vldb/Ioannidis85,DBLP:conf/pods/Naughton86}, \NP-complete for linear monadic and dyadic single rule programs~\cite{DBLP:conf/pods/Vardi88}, \PSpace-complete for linear mon\-adic programs~\cite{DBLP:conf/stoc/CosmadakisGKV88,DBLP:journals/ijfcs/Meyden00}, and  2\ExpTime-complete for arbitrary monadic programs~\cite{DBLP:conf/stoc/CosmadakisGKV88,DBLP:conf/lics/BenediktCCB15}; see also~\cite{DBLP:conf/pods/NaughtonS87}. On the other hand, boundedness of linear datalog queries with binary predicates and of ternary linear datalog queries with a single recursive rule was proved to be undecidable~\cite{DBLP:journals/jlp/HillebrandKMV95,DBLP:journals/siamcomp/Marcinkowski99} along with many other semantic properties of datalog programs including linearisability, being in L or being in NC~\cite{DBLP:journals/jacm/GaifmanMSV93}. 
The computational complexity of evaluating datalog sirups (of arbitrary arity) as well as their descriptive complexity were studied in~\cite{DBLP:journals/iandc/GottlobP03}. 

The second wave was largely caused by the apparent success story of the DL-underpinned  Web Ontology Language OWL and the OBDA paradigm, both in theory and practice. On the one hand, as we mentioned earlier, large families of DLs that guarantee FO-rewritability~\cite{CDLLR07,ACKZ09} (the \DL-family) and datalog-rewritability~\cite{DBLP:conf/ijcai/BaaderBL05,DBLP:conf/dlog/BaaderLS06,BaaderEtAl-OWLED08DC} (the $\mathcal{EL}$-family) and~\cite{DBLP:conf/ijcai/HustadtMS05,DBLP:journals/jar/HustadtMS07} (the Horn DL-family) were designed and investigated. Other types of rule-based languages with FO-rewritability have also been identified~\cite{DBLP:journals/ws/CaliGL12,DBLP:journals/ai/CaliGP12,DBLP:journals/tods/GottlobOP14,DBLP:journals/ai/BagetLMS11,DBLP:journals/semweb/KonigLMT15}. On the other hand, various methods for rewriting  expressive OMQs to (disjunctive) datalog were suggested and implemented~\cite{DBLP:journals/japll/Perez-UrbinaMH10,DBLP:conf/aaai/EiterOSTX12,DBLP:conf/aaai/KaminskiNG14,DBLP:journals/ai/KaminskiNG16,DBLP:journals/ws/TrivelaSCS15,DBLP:journals/kais/TrivelaSCS20}. For example, the PAGOdA system combines the datalog reasoner RDFox and the OWL 2 reasoner HermiT~\cite{DBLP:journals/jair/ZhouGNKH15}. A partial FO-rewriting algorithm for OMQs with an $\mathcal{ELU}$ ontology (allowing disjunction in $\mathcal{EL}$) and an atomic query was suggested in~\cite{DBLP:conf/dlog/KaminskiG13}, and  
a sound and complete but not necessarily terminating algorithm for OMQs with existential rules in~\cite{DBLP:journals/semweb/KonigLMT15}.

Complexity-theoretic investigations of the {\bf (data complexity)} and {\bf (rewritability)} problems for DL OMQs fall into two categories depending on whether the ontology language is Horn or not. 
FO-rewritability of OMQs with an ontology given in a Horn DL between $\mathcal{EL}$ and \textsl{Horn}$\mathcal{SHIF}$ was studied in~\cite{DBLP:conf/ijcai/BienvenuLW13,DBLP:conf/ijcai/LutzHSW15,DBLP:conf/ijcai/Bienvenu0LW16}, which provided semantic characterisations and established, using automata-theoretic techniques, the complexity of deciding FO-rewritability ranging from \ExpTime{} via \NExpTime{} to 2\ExpTime. A complete characterisation of  OMQs with an $\mathcal{EL}$-ontology was obtained in~\cite{DBLP:conf/ijcai/LutzS17,DBLP:journals/corr/abs-1904-12533}, establishing an \ACz/NL/P data complexity trichotomy, which corresponds to an FO-/linear-datalog-/datalog-rewritability trichotomy. Deciding this trichotomy was shown to be \ExpTime-complete. FO-rewritability of OMQs whose ontology is a set of (frontier-)guarded existential rules was investigated  in~\cite{DBLP:conf/ijcai/BarceloBLP18}.

For non-Horn ontology languages (allowing disjunctive axioms), a crucial step in understanding {\bf (data complexity)} and {\bf (rewritability)} was the discovery  in~\cite{DBLP:journals/tods/BienvenuCLW14,DBLP:journals/lmcs/FeierKL19} of a connection between OMQs and non-uniform constraint satisfaction problems (CSPs) with a fixed template via MMSNP of~\cite{DBLP:journals/siamcomp/FederV98}. It was used  to show that deciding FO- and datalog-rewritability of OMQs with an ontology in any DL between $\mathcal{ALC}$  and $\mathcal{SHIU}$ and an atomic query is \mbox{\textsc{NExpTime}}-complete. The Feder-Vardi dichotomy of CSPs~\cite{DBLP:conf/focs/Bulatov17,DBLP:conf/focs/Zhuk17} implies a P/\textsc{coNP} dichotomy of such OMQs, which is decidable in \textsc{NExpTime}. For monadic disjunctive datalog and OMQs with an $\mathcal{ALCI}$ ontology (that is, $\mathcal{ALC}$ with inverse roles) and a CQ, deciding FO-rewritability rises and becomes 2\NExpTime-complete; deciding whether such an OMQ is rewritable to monadic datalog is between 2\NExpTime{} and 3\NExpTime{}~\cite{DBLP:conf/kr/BourhisL16,DBLP:journals/lmcs/FeierKL19}.
Deciding FO-rewritability of OMQs with a Schema.org\footnote{\url{https://schema.org}} ontology (which admits inclusions between concept and role names as well as covering axioms for role domains and ranges) and a union of CQs (UCQ) is \PSpace-hard; for acyclic UCQs, it can be done in \NExpTime~\cite{DBLP:conf/ijcai/HernichLOW15}.
The data complexity and rewritability of OMQs whose ontology is given in the guarded fragment of first-order logic were considered in~\cite{DBLP:journals/tocl/HernichLPW20}. 

Despite the discovery of general algebraic, automata- and graph-theoretic and semantic characterisations of data complexity and rewritability---which are usually very hard to check---there are very few explicit and easily checkable, possibly partial and applicable to limited OMQ families, sufficient and/or necessary conditions let alone complete classifications. Notable examples include (non-)linearisability conditions for chain datalog queries~\cite{DBLP:conf/pods/Saraiya90,DBLP:journals/tods/ZhangYT90,DBLP:journals/tcs/AfratiGT03}, the markability condition of datalog rewritability for disjunctive datalog programs and DL ontologies~\cite{DBLP:journals/ai/KaminskiNG16}, and explicit NC/P-dichotomy of datalog chain sirups~\cite{DBLP:journals/jacm/AfratiP93}. 
Classifications and dichotomies of various CSPs have been intensively investigated since Schaefer's classification theorem~\cite{DBLP:conf/stoc/Schaefer78}; see, e.g.,~\cite{DBLP:journals/siamcomp/BulatovJK05,DBLP:journals/lmcs/LaroseLT07,DBLP:journals/csr/HellN08,DBLP:journals/toct/ChenL17,DBLP:conf/focs/Bulatov17,DBLP:conf/focs/Zhuk17} and references therein. 


The natural idea~\cite{DBLP:journals/tods/BienvenuCLW14} of translating OMQs to CSPs and then using the algorithms and techniques developed for checking their complexity looks hardly viable in general: for instance, as reported in~\cite{DBLP:conf/birthday/GerasimovaKZ19}, the Polyanna program~\cite{gault2004implementing}, designed to check tractability of CSPs, failed to recognise \coNP-hardness of the very simple OMQ obtained by swapping the last $F$- and $T$-labels in $(\{\eqref{ax4}\}, \q_1)$ above because the CSP translation is unavoidably exponential.

	
\section{Preliminaries}\label{prelims}
	
Using the standard description logic syntax and semantics~\cite{DBLP:books/daglib/0041477}, we consider \emph{ontology-mediated queries} (OMQs) of the form $\omq = (\T,\q)$, where $\T$ is one of the two ontologies
$$
\dis_A ~=~ \{\, A \sqsubseteq F \sqcup T \,\}, \qquad\qquad \dis_A^\bot ~=~ \{\, A \sqsubseteq F \sqcup T,\ \ F \sqcap T \sqsubseteq \bot\,\}
$$
and $\q$ is a \emph{Boolean conjunctive query} (\emph{CQ}, for short): 
an FO-sentence $\q = \exists \avec{x}\, \varphi(\avec{x})$, in which $\varphi$ is a conjunction of (constant- and function-free) atoms with variables from $\avec{x}$. We often think of $\q$ as the \emph{set} of its atoms.  In the context of this paper, CQs may only contain two unary predicates $F$, $T$ and arbitrary binary predicates. 
%
%
As in the previous section, OMQs $\omq = (\dis\!_A,\q)$ are also called \emph{d-sirups} and $\omq = (\dis_A^\bot,\q)$ \emph{dd-sirups}. 

Occasionally, we set $A = \top$, in which case $A \sqsubseteq F \sqcup T$ becomes the \emph{total covering axiom} $F \sqcup T$. It is to be noted that this axiom is \emph{domain dependent}~\cite{Abitebouletal95}, and so regarded to be \emph{unsafe} and disallowed in disjunctive datalog. In general, answering a d-sirup $(\dis_A,\q)$ could be harder than answering the corresponding OMQ $(\dis_\top,\q)$ as shown by Example~\ref{ex:cactus}. We only consider OMQs $(\dis_\top,\q)$ in examples and when proving some lower complexity bounds. 

An \emph{ABox} (data instance), $\A$, is a finite set of ground atoms with unary or binary predicates. We denote by $\ind(\A)$ the set of constants (individuals) in $\A$. An \emph{interpretation} is a structure of the form $\I = (\Delta^\I,\cdot^\I)$ with a domain $\Delta^\I \ne\emptyset$ and an interpretation function $\cdot^\I$ such that $a^\I \in \Delta^\I$ for any constant $a$, $\top^\I = \Delta^\I$, $\bot^\I = \emptyset$, $P^\I \subseteq \Delta^\I$ for any unary predicate $P$, and $P^\I \subseteq \Delta^\I \times \Delta^\I$ for any binary $P$. 
The truth-relation $\I \models \q$, for any CQ $\q$, is defined as usual in first-order logic.
The interpretation $\I$ is a \emph{model} of $\T$ if $A^\I \subseteq F^\I \cup T^\I$ and, for $\T = \dis_A^\bot$, also $F^\I \cap T^\I = \emptyset$; it is a \emph{model} of $\A$ if $P(a) \in \A$ implies $a^\I \in P^\I$ and $P(a,b) \in \A$ implies $(a^\I,b^\I) \in P^\I$. 

The \emph{certain answer} to an OMQ $\omq = (\T,\q)$ over an ABox $\A$ is `yes' if $\I \models \q$ for all models $\I$ of $\T$ and $\A$---in which case we write $\T,\A \models \q$---and `no' otherwise.
A model of $\T$ and $\A$ is \emph{minimal} if, for each \emph{undecided $A$-individual} $a$, for which $A(a)$ is in $\A$ but neither $F(a)$ nor $T(a)$ is,  exactly one of $a^\I\in F^\I$ or $a^\I\in T^\I$ holds. Clearly, $\T, \A \models \q$ iff $\I \models \q$ for every minimal model $\I$ of $\T$ and $\A$. So, from now on, `model' always means `minimal model'\!.

It is often convenient to regard CQs, ABoxes and interpretations as digraphs with labelled edges and partially labelled nodes (by $F$, $T$ in CQs and $F$, $T$, $A$ in ABoxes and interpretations). It is straightforward to see that the truth-relation $\I \models \q$ is equivalent to the existence of a digraph homomorphism $h \colon \q \to \I$ preserving the labels of nodes and edges. 
Without loss of generality, we assume that CQs are \emph{connected} as undirected graphs. 
This graph-theoretic perspective allows us to consider special classes of CQs such as \emph{tree-shaped} CQs, in which the underlying \emph{undirected} graph is a tree, or \emph{ditree-shaped} CQs, in which the underlying directed graph is a tree with all edges pointing away from the root, or dag-shaped CQs, which contain no directed cycles, etc. In particular, by a \emph{path-shaped CQ} $\q$ (or \emph{path CQ}, for short) we mean a (simple) directed path each of whose edges is labelled by one binary predicate. In other words, the  binary atoms in $\q$ form a sequence $R_1(x_1,x_2),R_2(x_2,x_3),\dots,R_n(x_n,x_{n+1})$, where the $x_i$ are all pairwise distinct variables in $\q$ (the $R_i$ are not necessarily distinct). 




We illustrate the reasoning required to find the certain answer to a d-sirup over an ABox both on intuitive and formal levels. Our first example shows that, unsurprisingly, answering d-sirups can be done by a plain proof by cases.

\begin{example}\em 
Consider the OMQ $\omq = (\dis_\top, \q)$ with $A = \top$ and the path CQ $\q$ shown in the picture below:\\
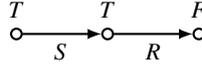
\centerline{
\begin{tikzpicture}[line width=0.8pt,scale=.8]
\node[point,scale=0.7,label=above:{\small $T$}] (1) at (-3,0) {};
\node[point,scale=0.7,label=above:{\small $T$}] (2) at (-1.5,0) {};
\node[point,scale=0.7,label=above:{\small $F$}] (3) at (0,0) {};
\draw[->,right] (1) to node[below] {\small $S$} (2);
\draw[->,right] (2) to node[below] {\small $R$} (3);
\end{tikzpicture}}\\[4pt]  
By analysing the four possible cases for $a,b \in F^\I,T^\I$ in an arbitrary model $\I$ of $\dis_\top$ and the ABox below, one can readily see that each of them contains $\q$ as a subgraph, and so the certain answer to $\omq$ over this ABox is `yes'\!.\\
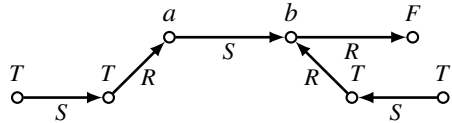
\centerline{
\begin{tikzpicture}[>=latex,line width=1pt, rounded corners,scale=.8]
			\node[point,scale=0.7,label=above:$a$] (1) at (0,0) {};
			\node[point,scale=0.7,label=above:\small$T$] (1d) at (-1,-1) {};
			\node[point,scale=0.7,label=above:\small$T$] (1dd) at (-2.5,-1) {};
			\node[point,scale=0.7,label=above:$b$] (m) at (2,0) {};
			\node[point,scale=0.7,label=above :\small$\ \ T$] (md) at (3,-1) {};
			\node[point,scale=0.7,label=above:\small$T$] (mdd) at (4.5,-1) {};
			\node[point,scale=0.7,label=above:\small$F$] (2) at (4,0) {};
			\draw[->,right] (1) to node[below = -.5mm] {\small $S$}  (m);
			\draw[->,right] (1d) to node[ below right = -1.5mm] {\small $R$}  (1);
			\draw[->,right] (1dd) to node[below = -.5mm] {\small $S$}  (1d);
			\draw[->,right] (m) to node[below = -.5mm] {\small $R$} (2);
			\draw[->,right] (md) to node[ below left = -1.5mm] {\small $R$}  (m);
			\draw[->,right] (mdd) to node[below = -.5mm] {\small $S$}  (md);
\end{tikzpicture}
}\\
Indeed, if $a^\I \in F^\I$, then $\q$ is homomorphically embeddable into the $S$--$R$ path on the left-hand side of $\I$. Otherwise $a^\I \in T^\I$. If $b^\I \in F^\I$, then $\q$ is homomorphically embeddable into the bottom $S$--$R$ path on the right-hand side of $\I$. In the remaining case $b^\I \in T^\I$, there is a homomorphism from $\q$ into the $S$--$R$ path on the top of $\I$.
\end{example}

Such proofs can be given as formal resolution refutations (derivations of the empty clause) in clausal logic. 

\begin{example}\label{ex-resolution}\em 
The certain answer to a d-sirup $\omq = (\dis_A,\q)$ over an ABox $\A$ is `yes' iff the following set $\mathcal{S}_{\omq,\A}$ of clauses is unsatisfiable:
$$
\mathcal{S}_{\omq,\A} ~=~ \big\{\,\neg A(y) \lor F(y) \lor T(y), \  \bigvee_{P(\avec{x}) \in \q} \neg P(\avec{x}) \,\big\} \ \cup \ \A .
$$
(For a dd-sirup $\omq = (\dis_A^\bot,\q)$, the set $\mathcal{S}_{\omq,\A}$ also contains the clause $\neg F(z) \lor \neg T(z)$.) In other words, the certain answer to $\omq$ over $\A$ is `yes' iff there is a derivation of the empty clause from $\mathcal{S}_{\omq,\A}$ in classical first-order resolution calculus~\cite{ChangLee}. By grounding $\mathcal{S}_{\omq,\A}$, that is, by uniformly substituting individuals in $\ind(\A)$ for the variables $\avec{x}$, $y$ and $z$ in the first two clauses of $\mathcal{S}_{\omq,\A}$, we obtain a set $\bar{\mathcal{S}}_{\omq,\A}$ of essentially propositional clauses with $|\bar{\mathcal{S}}_{\omq,\A}|$ polynomial in $|\A|$. Again, the certain answer to $\omq$ over $\A$ is `yes' iff there is a derivation of the empty clause from $\bar{\mathcal{S}}_{\omq,\A}$ using propositional resolution. 
%
We now show that, in general, such derivations are of exponential size in $|\A|$. 
\end{example}
\begin{theorem}\label{thm:resolution}
There exist a CQ $\q$ and a sequence $\A_n$, $n >0$, of ABoxes such that $|\A_n|$ is polynomial in $n$ and any resolution refutation of $\mathcal{S}_{\omq,\A_n}$ or $\bar{\mathcal{S}}_{\omq,\A_n}$, for $\omq = (\dis_A,\q)$, is of size $2^{\Omega(n)}$.
\end{theorem}

\begin{proof}
We show that the \emph{mutilated chessboard problem} can be solved by answering a d-sirup $\omq$ over certain ABoxes $\A_n$. The problem is as follows: given a chessboard of size $2n \times 2n$, for $n > 0$, with two white corner squares  removed, prove that it \emph{cannot} be covered by domino tiles (rectangles with two squares). 
This problem was encoded as a set of propositional clauses of size linear in $n$~\cite{DBLP:conf/focs/DantchevR01,DBLP:journals/tcs/Alekhnovich04}, any resolution proof of which is of size $2^{\Omega(n)}$~\cite[Theorem~2.1]{DBLP:conf/focs/DantchevR01}. 
%
%
We encode the same problem by the set $\mathcal{S}_{\omq,\A_n}$, for some
d-sirup $\omq = (\dis_A,\q)$ and ABox $\A_n$. Since our encoding and the one  in~\cite{DBLP:conf/focs/DantchevR01} are `locally' translatable to each other and in view of~\cite[Proposition~3.4]{DBLP:journals/tcs/Alekhnovich04}), any resolution refutation of $\mathcal{S}_{\omq,\A_n}$ or $\bar{\mathcal{S}}_{\omq,\A_n}$ is also of size $2^{\Omega(n)}$.

Our CQ $\q$ is shown on the left-hand side of Fig.~\ref{f:chess-mutilated}. The mutilated $2n \times 2n$
chessboard is turned to an ABox $\A_n$ by replacing each of its squares with the pattern shown in the middle of  Fig.~\ref{f:chess-mutilated}. The encodings of the different squares are connected via their four \emph{contacts\/},
depicted as $\circ$-nodes. 
Each of these contacts is labelled by $A$ or $F$, depending on whether it is in-between two squares,
or at the boundary of the board; see the right-hand side of Fig.~\ref{f:chess-mutilated}.
All of the binary edges in $\q$ and $\A_n$ are assumed to be labelled by $R$. Labels $w,x,y,z$ are just pointers and not parts of $\q$ or $\A_n$. 



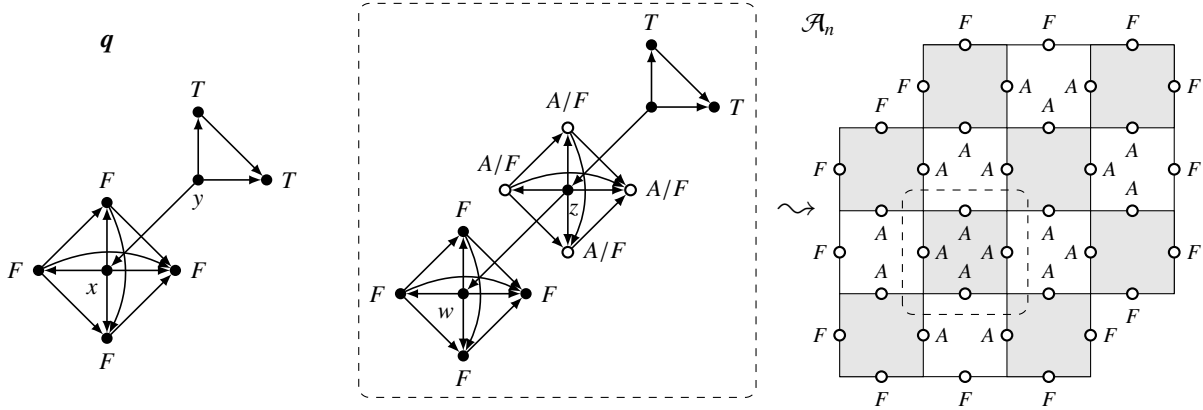
\begin{figure}[t]
\begin{tikzpicture}[>=latex,line width=0.6pt,scale = .6]
\node[]  at (0,5) {$\q$};
\node[]  at (0,-3) {\ };
\node[point,fill=black,scale = 0.6,label=below left:{\small $x$\!}] (0) at (0,0) {};
\node[point,fill=black,scale = 0.6,label=left:{\small $F$}] (1) at (-1.5,0) {};
\node[point,fill=black,scale = 0.6,label=above:{\small $F$}] (2) at (0,1.5) {};
\node[point,fill=black,scale = 0.6,label=right:{\small $F$}] (3) at (1.5,0) {};
\node[point,fill=black,scale = 0.6,label=below:{\small $F$}] (4) at (0,-1.5) {};
\draw[->] (0) to (1);
\draw[->] (0) to (2);
\draw[->] (0) to (3);
\draw[->] (0) to  (4);
\draw[->] (1) to (2);
\draw[->] (2) to (3);
\draw[->] (1) to (4);
\draw[->] (4) to (3);
\draw[->,bend left = 25] (2) to (4);
\draw[->,bend left = 25] (1) to (3);
\node[point,fill=black,scale = 0.6,label=below:{\small $y$}] (t0) at (2,2) {};
\node[point,fill=black,scale = 0.6,label=right:{\small $T$}] (t1) at (3.5,2) {};
\node[point,fill=black,scale = 0.6,label=above:{\small $T$}] (t2) at (2,3.5) {};
\draw[->] (t0) to (0);
\draw[->] (t0) to (t1);
\draw[->] (t0) to (t2);
\draw[->] (t2) to (t1);
\end{tikzpicture}
\hspace*{0.5cm}
\begin{tikzpicture}[>=latex,line width=0.6pt,scale = .55]
\node[point,fill=black,scale = 0.6,label=below left:{\small $w$\!}] (0) at (0,0) {};
\node[point,fill=black,scale = 0.6,label=left:{\small $F$}] (1) at (-1.5,0) {};
\node[point,fill=black,scale = 0.6,label=above:{\small $F$}] (2) at (0,1.5) {};
\node[point,fill=black,scale = 0.6,label=right:{\small $F$}] (3) at (1.5,0) {};
\node[point,fill=black,scale = 0.6,label=below:{\small $F$}] (4) at (0,-1.5) {};
\draw[->] (0) to (1);
\draw[->] (0) to (2);
\draw[->] (0) to (3);
\draw[->] (0) to  (4);
\draw[->] (1) to (2);
\draw[->] (2) to (3);
\draw[->] (1) to (4);
\draw[->] (4) to (3);
\draw[->,bend left = 25] (2) to (4);
\draw[->,bend left = 25] (1) to (3);
\node[point,fill=black,scale = 0.6,label=below right:{\small \!\!\!$z$}] (a0) at (2.5,2.5) {};
\node[point,scale = 0.7,label=right:{\small $A/F$}] (a1) at (4,2.5) {};
\node[point,scale = 0.7,label=above:{\small $A/F$}] (a2) at (2.5,4) {};
\node[point,scale = 0.7,label=above:{\small $A/F$\ \ \ }] (a3) at (1,2.5) {};
\node[point,scale = 0.7,label=right:{\small $A/F$}] (a4) at (2.5,1) {};
\draw[->] (a0) to (0);
\draw[->] (a0) to (a1);
\draw[->] (a0) to (a2);
\draw[->] (a0) to (a3);
\draw[->] (a0) to  (a4);
\draw[->] (a2) to (a1);
\draw[->] (a3) to (a2);
\draw[->] (a4) to (a1);
\draw[->] (a3) to (a4);
\draw[->,bend left = 25] (a2) to (a4);
\draw[->,bend left = 25] (a3) to (a1);
\node[point,fill=black,scale = 0.6] (t0) at (4.5,4.5) {};
\node[point,fill=black,scale = 0.6,label=above:{\small $T$}] (t1) at (4.5,6) {};
\node[point,fill=black,scale = 0.6,label=right:{\small $T$}] (t2) at (6,4.5) {};
\draw[->] (t0) to (a0);
\draw[->] (t0) to (t1);
\draw[->] (t0) to (t2);
\draw[->] (t1) to (t2);
\draw[thin,dashed,rounded corners] (-2.5,-2.5) rectangle (7,7);
\node[]  at (8,2) {\bf\Large $\leadsto$};

\draw[-,thin] (9,-2) -- (15,-2);
\draw[-,thin] (9,-2) -- (9,4);
\draw[-,thin] (17,6) -- (17,0);
\draw[-,thin] (17,6) -- (11,6);
\draw[thin,fill=gray!20] (9,-2) rectangle (11,0);
\node[point,fill=white,scale = 0.7,label=above:{\footnotesize $A$}] (00u) at (10,0) {};
\node[point,fill=white,scale = 0.7,label=right:{\footnotesize \!$A$}] (00r) at (11,-1) {};
\node[point,fill=white,scale = 0.7,label=below:{\footnotesize $F$}] (00d) at (10,-2) {};
\node[point,fill=white,scale = 0.7,label=left:{\footnotesize $F$\!}] (00l) at (9,-1) {};
\draw[thin,fill=gray!20] (11,0) rectangle (13,2);
\node[point,fill=white,scale = 0.7,label=below:{\footnotesize $A$}] (11u) at (12,2) {};
\node[point,fill=white,scale = 0.7,label=left:{\footnotesize $A$\!}] (11r) at (13,1) {};
\node[point,fill=white,scale = 0.7,label=above:{\footnotesize $A$}] (11d) at (12,0) {};
\node[point,fill=white,scale = 0.7,label=right:{\footnotesize \!$A$}] (11l) at (11,1) {};
\draw[thin,fill=gray!20] (13,2) rectangle (15,4);
\node[point,fill=white,scale = 0.7,label=above:{\footnotesize $A$}] (22u) at (14,4) {};
\node[point,fill=white,scale = 0.7,label=right:{\footnotesize \!$A$}] (22r) at (15,3) {};
\node[point,fill=white,scale = 0.7,label=below:{\footnotesize $A$}] (22d) at (14,2) {};
\node[point,fill=white,scale = 0.7,label=left:{\footnotesize $A$\!}] (22l) at (13,3) {};
\draw[thin,fill=gray!20] (15,4) rectangle (17,6);
\node[point,fill=white,scale = 0.7,label=above:{\footnotesize $F$}] (33u) at (16,6) {};
\node[point,fill=white,scale = 0.7,label=right:{\footnotesize \!$F$}] (33r) at (17,5) {};
\node[point,fill=white,scale = 0.7,label=below:{\footnotesize $A$}] (33d) at (16,4) {};
\node[point,fill=white,scale = 0.7,label=left:{\footnotesize $A$\!}] (33l) at (15,5) {};
\draw[thin,fill=gray!20] (13,-2) rectangle (15,0);
\node[point,fill=white,scale = 0.7,label=above:{\footnotesize $A$}] (20u) at (14,0) {};
\node[point,fill=white,scale = 0.7,label=right:{\footnotesize \!$F$}] (20r) at (15,-1) {};
\node[point,fill=white,scale = 0.7,label=below:{\footnotesize $F$}] (20d) at (14,-2) {};
\node[point,fill=white,scale = 0.7,label=left:{\footnotesize $A$\!}] (20l) at (13,-1) {};
\draw[thin,fill=gray!20] (15,0) rectangle (17,2);
\node[point,fill=white,scale = 0.7,label=above:{\footnotesize $A$}] (31u) at (16,2) {};
\node[point,fill=white,scale = 0.7,label=right:{\footnotesize \!$F$}] (31r) at (17,1) {};
\node[point,fill=white,scale = 0.7,label=below:{\footnotesize $F$}] (31d) at (16,0) {};
\node[point,fill=white,scale = 0.7,label=left:{\footnotesize $A$\!}] (31l) at (15,1) {};
\draw[thin,fill=gray!20] (9,2) rectangle (11,4);
\node[point,fill=white,scale = 0.7,label=above:{\footnotesize $F$}] (02u) at (10,4) {};
\node[point,fill=white,scale = 0.7,label=right:{\footnotesize \!$A$}] (02r) at (11,3) {};
\node[point,fill=white,scale = 0.7,label=below:{\footnotesize $A$}] (02d) at (10,2) {};
\node[point,fill=white,scale = 0.7,label=left:{\footnotesize $F$\!}] (02l) at (9,3) {};
\draw[thin,fill=gray!20] (11,4) rectangle (13,6);
\node[point,fill=white,scale = 0.7,label=above:{\footnotesize $F$}] (13u) at (12,6) {};
\node[point,fill=white,scale = 0.7,label=right:{\footnotesize \!$A$}] (13r) at (13,5) {};
\node[point,fill=white,scale = 0.7,label=below:{\footnotesize $A$}] (13d) at (12,4) {};
\node[point,fill=white,scale = 0.7,label=left:{\footnotesize $F$\!}] (13l) at (11,5) {};
\node[point,fill=white,scale = 0.7,label=below:{\footnotesize $F$}] (10d) at (12,-2) {};
\node[point,fill=white,scale = 0.7,label=left:{\footnotesize $F$\!}] (01l) at (9,1) {};
\node[point,fill=white,scale = 0.7,label=right:{\footnotesize \!$F$}] (32r) at (17,3) {};
\node[point,fill=white,scale = 0.7,label=above:{\footnotesize $F$}] (23u) at (14,6) {};
\draw[thin,dashed,rounded corners] (10.5,-.5) rectangle (13.5,2.5);
\node[]  at (8.5,6.5) {$\A_n$};
\end{tikzpicture}
\caption{Encoding the mutilated chessboard problem as $\mathcal{S}_{(\dis_A,\q),\A_n}$.}\label{f:chess-mutilated}
\end{figure}	
	
    
We call a model $I$ of $\dis_A$ and $\A_n$ \emph{covering} if exactly one contact in the  encoding of each square is in $T^\I$.
Covering models are clearly in one-to-one correspondence with domino-coverings (with each contact being in $T^\I$ iff it is between two squares covered by the same domino).
We show that a model $\I$ of $\dis_A$ and $\A_n$ is covering iff $\I\not\models\q$.
%
%
This implies the correctness of our encoding: the $2n\times 2n$ mutilated chessboard cannot be covered by dominos iff the answer to $\omq$ over $\A_n$ is `yes'\!, that is, there exist resolution refutations of both 
$\mathcal{S}_{\omq,\A_n}$ and $\bar{\mathcal{S}}_{\omq,\A_n}$. 
    
$(\Rightarrow)$ 
If $\I$ is covering, then at least one of the four contacts of each square is not labelled by $F$. Thus, node $x$ of $\q$ can be homomorphically mapped only to node $w$ of the encoding of some square, and so $y$ should be mapped to $z$. But then the two $T$-nodes in $\q$ should be mapped to two different contacts of the same square, contrary to the fact that in covering models only one such contact is labelled by $T$. Therefore, there is no $\q\to\I$ homomorphism.
$(\Leftarrow)$ If $\I$ is not covering because there is a square none of whose contacts is labelled by $T$, then
by mapping $x$ to $z$ in the encoding of that square we can obtain a $\q\to\I$ homomorphism. 
And if there is a square such that at least two of its contacts are labelled by $T$, we can obtain a $\q\to\I$ homomorphism by mapping $x$ to $w$ and $y$ to $z$. 
%
\end{proof}

Our concern in the remainder of this article is the \emph{combined} and \emph{data complexity} of deciding, for a given (d)d-sirup $\omq = (\T,\q)$ and an ABox $\A$, whether $\T,\A \models \q$. In the former case, both $\q$ and $\A$ are regarded as input; in the latter one, $\q$ is fixed. It should be clear  that $\Pi^p_2 = \coNP^\NP$ is an upper bound for the combined complexity of our problem, which amounts to checking that, for every model $\I$ of $\T$ and $\A$, there exists a homomorphism $\q \to \I$, with the latter being \NP-complete. For data complexity, when $\q$ is fixed, checking the existence of a homomorphism $\q \to \I$ can be done in \PTime, and so the whole problem is in \coNP{}. 
	
We are also interested in various types of rewritability of (d)d-sirups.  
An OMQ $\omq = (\T,\q)$ is called \emph{FO-rewritable} if there is an FO-sentence $\Phi$ such that $\T, \A\models \q$ iff $\Phi$ is true in $\A$ given as an FO-structure~\cite{Immerman99}. In terms of circuit complexity, FO-rewritability is equivalent to answering $\omq$ in logtime-uniform $\ACz$~\cite{Immerman99}.

Recall from, say~\cite{Abitebouletal95}, that a \emph{datalog program}, $\Pi$, is a finite set of \emph{rules} of the form
$\forall \avec{x}\, (\gamma_0 \leftarrow \gamma_1 \land \dots \land \gamma_m)$,
where each $\gamma_i$ is a (constant- and function-free) atom $Q(\avec{y})$ with $\avec{y} \subseteq  \avec{x}$. As usual, we omit $\forall \avec{x}$. The atom $\gamma_0$ is the \emph{head} of the rule, and $\gamma_1,\dots,\gamma_m$ its  \emph{body}. All of the variables in the head must occur in the body. The predicates in the heads of rules are called \emph{IDB predicates}, the rest \emph{EDB predicates}. The \emph{arity} of $\Pi$ is the maximum arity of its IDB predicates; 1-ary $\Pi$ is called \emph{monadic}.  
A \emph{datalog query} in this article takes the form $(\Pi,\G)$ with a 0-ary (goal) atom $\G$. The \emph{answer} to $(\Pi,\G)$ over an ABox $\A$ is `yes' if $\G$ is true in the structure $\Pi(\A)$ obtained by closing $\A$ under the rules in $\Pi$, in which case we write $\Pi,\A \models \G$. We call $(\Pi,\G)$ a \emph{datalog-rewriting} of an OMQ $\omq = (\T,\q)$ in case $\T, \A\models \q$ iff  $\Pi,\A \models \G$, for any ABox $\A$ containing EDB predicates of $\Pi$ only. If $\omq$ is datalog-rewritable, then it can be answered in \PTime{} for data complexity~\cite{DBLP:journals/csur/DantsinEGV01}. 
 
If there is a rewriting of $\omq$ to a $(\Pi,\G)$ with a \emph{linear} program $\Pi$, having at most one IDB predicate in the body of each of its rules, then $\omq$ can be   answered in \NL{} (non-deterministic logarithmic space).  
The \NL{} upper bound also holds for datalog queries with a linear-stratified program, which is defined as follows.
A \emph{stratified} program~\cite{Abitebouletal95} is a sequence $\Pi = (\Pi_0, \dots, \Pi_n)$ of datalog programs, called the \emph{strata} of $\Pi$, such that each predicate  in $\Pi$ can occur in the head of a rule only in one stratum $\Pi_i$ and can occur in the body of a rule only in strata $\Pi_j$ with $j \ge i$. 
If, in addition, the body of each rule in $\Pi$ contains at most one occurrence of a head predicate from the same stratum, $\Pi$ is called \emph{linear-stratified}. Every linear-stratified program can be converted to an equivalent linear datalog program~\cite{DBLP:journals/tcs/AfratiGT03}, and so datalog queries with a linear-stratified program can be answered in \NL{} for data complexity. 

A linear program $\Pi$ is \emph{symmetric} if, for any recursive rule $I(\avec{x})\leftarrow J(\avec{y}) \land E(\avec{z})$ in $\Pi$ (except the goal rules), where $J$ is an IDB predicate and $E(\avec{z})$ is the conjunction of the EDBs of the rule, its symmetric counterpart  $J(\avec{y})\leftarrow I(\avec{x}) \land E(\avec{z})$ is also a rule in $\Pi$. It is known (see, e.g.,~\cite{egri2007symmetric}) that symmetric programs can be evaluated in \LogSpace{} (deterministic logarithmic space) for data complexity. Thus, if $\omq$ is rewritable to a symmetric datalog query, it can be answered in \LogSpace.

The complexity classes we deal with in this article form the chain
$$
\ACz \quad \subsetneq \quad \LogSpace \quad \subseteq \quad \NL \quad \subseteq \quad \PTime \quad \subseteq \quad \coNP \quad \subseteq \quad \Pi_2^p
$$
(whether any of the inclusions $\subseteq$ is strict is a major open problem in complexity theory). The P/NP dichotomy for CSPs~\cite{DBLP:conf/focs/Bulatov17,DBLP:conf/focs/Zhuk17} and the reductions from~\cite{DBLP:journals/tods/BienvenuCLW14,DBLP:journals/lmcs/FeierKL19} imply that every (d)d-sirup is either in P or \coNP-complete. However, as far as we know, at the moment there are no other established dichotomies for OMQs with disjunctive axioms. 
On the other hand, as mentioned above, OMQ answering in \ACz{} is equivalent to FO-rewritability, but whether OMQ answering in L (NL or P) implies symmetric-datalog-rewritability (respectively, linear-datalog- or datalog-rewritability) also remains open.

	
\section{Initial observations}\label{sec:gen}
	
In this section, we obtain a number of relatively simple complexity and rewritability results that are applicable to arbitrary (not necessarily path) d- and dd-sirups.
	

\subsection{Combined complexity}

Our first result pushes to the limit~\cite[Theorem 5]{DBLP:conf/ijcai/HernichLOW15} according to which answering OMQs with a Schema.org ontology is $\Pi^p_2$-complete for combined complexity (the hardness proof of that theorem uses  an ontology with an enumeration definition $E = \{0,1\}$ and additional concept names, i.e., unary predicates, none of which is available in our case).
	
\begin{theorem}\label{combi}
$(i)$ Answering d-sirups $(\dis_A,\q)$ and dd-sirups $(\dis_A^\bot,\q)$ is $\Pi^p_2$-complete for combined complexity.   
		
$(ii)$ Answering d- and dd-sirups with a tree-shaped 
CQ $\q$ is \coNP-complete for combined complexity.
\end{theorem}
\begin{proof}
As mentioned in Section~\ref{prelims}, deciding whether $\T,\A \models \q$, given a (d)d-sirup $\omq = (\T,\q)$ and an ABox $\A$, can be done by a \coNP{} Turing machine (checking all models $\I$ of $\T$ and $\A$) with an \NP-oracle (checking the existence of a homomorphism $h\colon \q \to \I$); for tree-shaped $\q$, a \PTime-oracle is enough (see, e.g.,~\cite{DBLP:journals/jacm/Grohe07} and further references therein). The lower bound in $(ii)$ follows from Theorem~\ref{t:coNPhard}. 

For $(i)$, we prove it by reduction of $\Pi^p_2$-complete $\forall\exists \text{3SAT}$~\cite{DBLP:journals/tcs/Stockmeyer76}. We remind the reader that  a propositional formula $\psi(\avec{x},\avec{y})$ with tuples $\avec{x}$ and $\avec{y}$ of propositional variables is a \emph{3CNF} if it is a conjunction of 
\emph{clauses} of the form $\lit_1 \lor \lit_2 \lor \lit_3$, where each $\lit_i$ is a \emph{literal}
	(a propositional variable or a negation thereof). The decision problem $\forall\exists \text{3SAT}$ asks whether the fully quantified
propositional formula $\varphi = \forall \avec{x} \exists \avec{y} \, \psi(\avec{x},\avec{y})$ is true, for any given 3CNF $\psi$.
We may assume that each clause contains each variable at most once. 
 Denote by $\q_\varphi$ the CQ that, for each clause $c = \lit_1 \lor \lit_2 \lor \lit_3$ in $\psi$, contains atoms $R^c_i(z^c,u_i^c)$, $i=1,2,3$, with $u_i^c = y$ if $y \in \avec{y}$ occurs in $\lit_i$ and $u_i^c = x^c$ if $x \in \avec{x}$ occurs in $\lit_i$; in the latter case, $\q_\varphi$ also contains $T(x^c)$ if $\lit_i = x$ and $F(x^c)$ if $\lit_i = \neg x$. For example, clauses $c_1 = x_1 \lor \neg x_2 \lor y_1$ and $c_2 = \neg y_1 \lor x_2 \lor y_2$ contribute the following atoms to $\q_\varphi$:\\
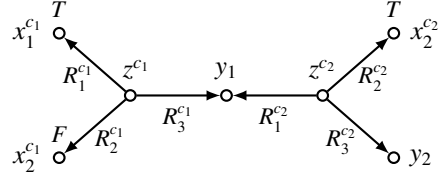
\centerline{ 
			\begin{tikzpicture}[decoration={brace,mirror,amplitude=7},line width=0.8pt, scale = 1.1]
			\node[point,scale=0.7, label=above:{\small $T$},label=left:{$x_1^{c_1}$}] (1) at (0,-0.25) {};
			\node[point,scale=0.7,label=above:{$\ \ z^{c_1}$}] (2) at (.85,-1) {};
			\node[point,scale=0.7, label=above:{\small $F$},label=left:{$x_2^{c_1}$}] (3) at (0,-1.75) {};
			\node[point,scale=0.7,label=above:{$y_1$}] (4) at (2,-1) {};
			\draw[->,right] (2) to node[below left = -1.8mm] {\small $R_1^{c_1}$} (1);
			\draw[->,right] (2) to node[below right = -1.8mm] {\small $R_2^{c_1}$} (3);
			\draw[->,right] (2) to node[below] {\small $R_3^{c_1}$} (4);
			\node[point,scale=0.7,label=above:{$z^{c_2}$}] (5) at (3.15,-1) {};
			\node[point,scale=0.7,label=right:{$y_2$}] (6) at (4,-1.75) {};
			\node[point,scale=0.7,label=above:{\small $T$},label=right:{$x_2^{c_2}$}] (7) at (4,-0.25) {};
			\draw[->,right] (5) to node[below] {\small $R_1^{c_2}$} (4);
			\draw[->,right] (5) to node[below left = -1.7mm] {\small $R_3^{c_2}$} (6);
			\draw[->,right] (5) to node[below right = -1.8mm] {\small $R_2^{c_2}$} (7);
			\node[point,scale=0.7,draw=white] (8) at (2,-1.75) {};
\end{tikzpicture}
}\\
For $\T = \dis_A$, the ABox $\A_\varphi$ is defined as follows. For $x \in \avec{x}$, we take individuals $a^*_x$ and $a^\circ_x$ and, for $y \in \avec{y}$, individuals $b^F_y$ and $b^T_y$. $\A_\varphi$ 
contains the atoms $A(a^*_x)$, $F(a^\circ_x)$, $T(a^\circ_x)$, for $x \in \avec{x}$. For each $c = \lit_1 \lor \lit_2 \lor \lit_3$, we define a set $E^c$ of triples of the above individuals: $(e_1,e_2,e_3)\in E^c$ iff
$(i)$ for $i =1,2,3$, $e_i \in\{ a^*_x,a^\circ_x\}$  whenever $x \in \avec{x}$ occurs in $\lit_i$,
$(ii)$  for $i =1,2,3$, $e_i \in\{ b^F_y,b^T_y\}$ whenever $y \in \avec{y}$ occurs in $\lit_i$, and $(iii)$ there is $i\in\{1,2,3\}$ such that either $e_i = a^*_x$, or  $e_i = b^\nu_y$ and the assignment $y = \nu$ makes $\lit_i$ true.
Now, for any $c$ and $(e_1,e_2,e_3)$ in $E^c$, we take a fresh individual $d_{(e_1,e_2,e_3)}^c$---the \emph{centre} of the pair $\bigl(c,(e_1,e_2,e_3)\bigr)$---and add three atoms  $R^c_i(d_{(e_1,e_2,e_3)}^c,e_i)$, $i=1,2,3$, to $\A_\varphi$. 
To illustrate, for $c_1 = x_1 \lor \neg x_2 \lor y_1$, the set $E^{c_1}$ contains all triples of the form $(a_{x_1}^{\mu_1},a_{x_2}^{\mu_2},b_{y_1}^\nu)$ except $(a_{x_1}^\circ,a_{x_2}^\circ,b_{y_1}^F)$ and gives the following fragment of $\A_\varphi$:\\
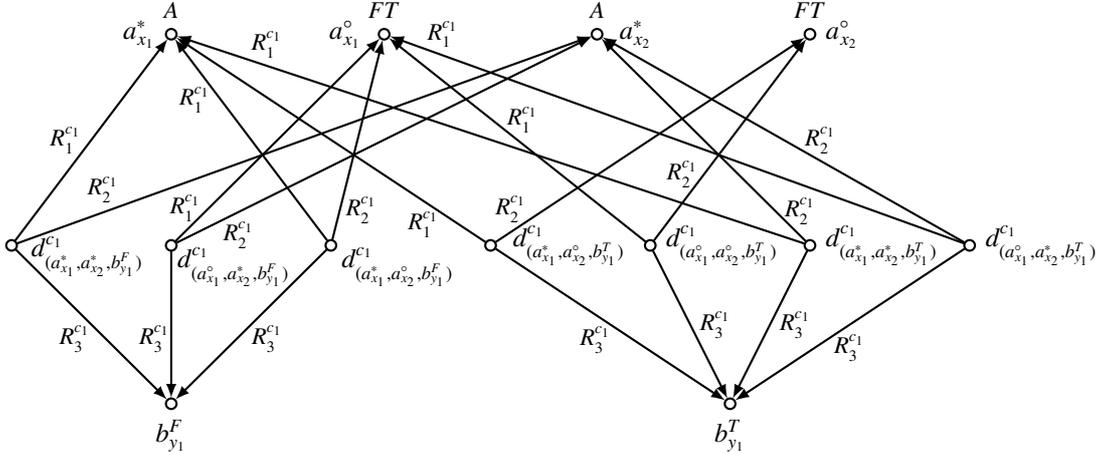
\centerline{ 
\begin{tikzpicture}[line width=0.8pt, scale =0.7]
\node[point,scale=0.7] (d1) at (0,0) {};
\node[label=right:{$d^{c_1}_{(a^*_{x_1},a^*_{x_2},b^F_{y_1})}$}] (dd1) at (0,-.2) {};
\node[point,scale=0.7] (d2) at (3,0) {};
\node[label=right:{\!\!\!$d^{c_1}_{(a^\circ_{x_1},a^*_{x_2},b^F_{y_1})}$}] (dd2) at (3,-.4) {};
\node[point,scale=0.7] (d3) at (6,0) {};
\node[label=right:{\!\!$d^{c_1}_{(a^*_{x_1},a^\circ_{x_2},b^F_{y_1})}$}] (dd3) at (6,-.4) {};
\node[point,scale=0.7,label=right:{\ $d^{c_1}_{(a^*_{x_1},a^\circ_{x_2},b^T_{y_1})}$}] (d4) at (9,0) {};
\node[point,scale=0.7,label=right:{$d^{c_1}_{(a^\circ_{x_1},a^\circ_{x_2},b^T_{y_1})}$}] (d5) at (12,0) {};
\node[point,scale=0.7,label=right:{$d^{c_1}_{(a^*_{x_1},a^*_{x_2},b^T_{y_1})}$}] (d6) at (15,0) {};
\node[point,scale=0.7,label=right:{$d^{c_1}_{(a^\circ_{x_1},a^*_{x_2},b^T_{y_1})}$}] (d7) at (18,0) {};
\node[point,scale=0.7,label=left:{$a^*_{x_1}$},label=above:{\small $A$}] (a1) at (3,4) {};
\node[point,scale=0.7,label=left:{$a^\circ_{x_1}\ $},label=above:{\small $FT$}] (a2) at (7,4) {};
\node[point,scale=0.7,label=right:{ $\ a^*_{x_2}$},label=above:{\small $A$}] (a3) at (11,4) {};
\node[point,scale=0.7,label=right:{$a^\circ_{x_2}$},label=above:{\small $FT$}] (a4) at (15,4) {};
\node[point,scale=0.7,label=below:{$b^F_{y_1}$}] (b1) at (3,-3) {};
\node[point,scale=0.7,label=below:{$b^T_{y_1}$}] (b2) at (13.5,-3) {};
\draw[->] (d1) to node[left] {\small $R_1^{c_1}$} (a1);
\draw[->] (d1) to node[above,pos=.15] {\small $R_2^{c_1}$} (a3);
\draw[->] (d1) to node[below left = -1.7mm] {\small $R_3^{c_1}$} (b1);
\draw[->] (d2) to node[above,pos=.05] {\small $R_1^{c_1}$} (a2);
\draw[->] (d2) to node[below,pos=.15] {\small $R_2^{c_1}$} (a3);
\draw[->] (d2) to node[below left = -1.7mm] {\small $R_3^{c_1}$} (b1);
\draw[->] (d3) to node[left,pos=.7] {\small $R_1^{c_1}$} (a1);
\draw[->] (d3) to node[right,pos=.15] {\small \!$R_2^{c_1}$} (a2);
\draw[->] (d3) to node[below right = -1.7mm] {\small $R_3^{c_1}$} (b1);
\draw[->] (d4) to node[below,pos=.2] {\small $R_1^{c_1}$} (a1);
\draw[->] (d4) to node[above,pos=.05] {\small $R_2^{c_1}$} (a4);
\draw[->] (d4) to node[left,pos=.6] {\small $R_3^{c_1}\ \ $} (b2);
\draw[->] (d5) to node[right,pos=.6] {\small $\ R_1^{c_1}$} (a2);
\draw[->] (d5) to node[left,pos=.33] {\small $R_2^{c_1}\!$} (a4);
\draw[->] (d5) to node[right] {\small $R_3^{c_1}$} (b2);
\draw[->] (d6) to node[above,pos=.85] {\small \!\!$R_1^{c_1}$} (a1);
\draw[->] (d6) to node[right,pos=.15] {\small $R_2^{c_1}$} (a3);
\draw[->] (d6) to node[right] {\small $R_3^{c_1}$} (b2);
\draw[->] (d7) to node[above,pos=.9] {\small \!\!$R_1^{c_1}$} (a2);
\draw[->] (d7) to node[right,pos=.5] {\small \ \ $R_2^{c_1}$} (a3);
\draw[->] (d7) to node[below] {\small $R_3^{c_1}$} (b2);
\end{tikzpicture}
	}\\
For $\T = \dis^\bot_A$, we take $a^F_x$ and $a^T_x$ instead of each $a^\circ_x$, 		add the atoms $F(a^F_x)$, $T(a^T_x)$ instead of $F(a^\circ_x)$, $T(a^\circ_x)$, and replace item $(i)$ in the definition of $E^c$ with $(i)'$ for $i =1,2,3$, $e_i \in\{ a^*_x,a^F_x,a^T_x\}$  whenever $x \in \avec{x}$ occurs in $\lit_i$. 

The number of atoms in $\A_\varphi$ is polynomial in the size of $\varphi$. 
		
\begin{claim}\label{l:AE}
Suppose $\mathfrak a \colon \avec{x} \to \{F,T\}$ is any assignment and 
$$
\A^{\mathfrak a}_\varphi ~=~ \A_\varphi \cup \{\, T(a^*_x) \mid \mathfrak a(x) = T,\ x\in\avec{x} \,\} \cup \{\, F(a^*_x) \mid \mathfrak a(x) = F ,\ x\in\avec{x}\,\}.
$$
There exists an assignment $\mathfrak b \colon \avec{y} \to \{F,T\}$ that makes $\psi(\mathfrak a(\avec{x}),\mathfrak b(\avec{y}))$ true iff $\A^{\mathfrak a}_\varphi \models \q_\varphi$. 
\end{claim}
\begin{proof}
$(\Rightarrow)$ Suppose $\mathfrak b$ is such that $\psi(\mathfrak a(\avec{x}),\mathfrak b(\avec{y}))$ is true. We need to show that there is a homomorphism $h \colon \q_\varphi \to \A^{\mathfrak a}_\varphi$.

\emph{Case} $\dis\!_A$: 
For any clause $c = \lit_1 \lor \lit_2 \lor \lit_3$ in $\psi$ and for any $i=1,2,3$, we define $e_i^c$ as follows. We let $(i)$ $e_i^c=a^*_x$ if $x \in \avec{x}$ occurs  in $\lit_i$ and $\mathfrak a$ makes $\lit_i$ true,
			$(ii)$ $e_i^c=a^\circ_x$ if $x \in \avec{x}$ occurs in $\lit_i$ and $\mathfrak a$ makes $\lit_i$ false,
			and $(iii)$ $e_i^c=b_y^{\mathfrak b(y)}$ if $y \in \avec{y}$ occurs in $\lit_i$.
			As $\psi(\mathfrak a(\avec{x}),\mathfrak b(\avec{y}))$ is true, $(e_1^c,e_2^c,e_3^c)$ is in $E^c$.
			Then we define a map $h$ by taking $h(z^c)$ to be the centre of $\bigl(c,(e_1^c,e_2^c,e_3^c)\bigr)$ and $h(u_i^c)=e_i^c$.
			It follows from the construction that $h$ is well-defined and 
			a homomorphism from $\q_\varphi$ to $\A_\varphi$ with respect to the binary atoms. We show that it 
			preserves the unary atoms as well. Indeed, for each $c$ and each $x \in \avec{x}$ occurring in 
			some literal $\lit_i$ in $c$, 
			there are two cases: 
			$(1)$ If $x^c$ is labelled by $T$ in $\q_\varphi$, then $\lit_i=x$. So if $\mathfrak a$
			makes $\lit_i$ true, then $e_i^c=a^*_x$ is labelled by $T$ in $\A^{\mathfrak a}_\varphi$.
			And if $\mathfrak a$ makes $\lit_i$ false, then $e_i^c=a^\circ_x$ is labelled by both $T$ and $F$ in $\A^{\mathfrak a}_\varphi$.
			$(2)$ If $x^c$ is labelled by $F$ in $\q_\varphi$, then $\lit_i=\neg x$. So if $\mathfrak a$
			makes $\lit_i$ true,   then $e_i^c=a^*_x$ is labelled by $F$ in $\A^{\mathfrak a}_\varphi$.
			And if $\mathfrak a$ makes $\lit_i$ false, then $e_i^c=a^\circ_x$ is labelled by both $T$ and $F$ in $\A^{\mathfrak a}_\varphi$.
			
			\emph{Case} $\dis_A^\bot$: 
			In the definition of $e_i^c$, we replace $(ii)$ with
			$(ii)'$ $e_i^c=a^T_x$ if $\lit_i=x$ for some $x \in \avec{x}$ and $\mathfrak a(x)=F$, and
			$(ii)''$ $e_i^c=a^F_x$ if $\lit_i=\neg x$ for some $x \in \avec{x}$ and $\mathfrak a(x)=T$.
			Again, we claim that $h$ as defined above preserves the unary atoms. Indeed, for each $c$ and for each 
			$x \in \avec{x}$ occurring in some literal $\lit_i$ in $c$, 
			there are two cases: 
			$(1)$ If $x^c$ is labelled by $T$ in $\q_\varphi$, then $\lit_i=x$. So if $\mathfrak a$
			makes $\lit_i$ true, then $e_i^c=a^*_x$ is labelled by $T$ in $\A^{\mathfrak a}_\varphi$.
			And if $\mathfrak a$ makes $\lit_i$ false, then $e_i^c=a^T_x$ is labelled by $T$ in $\A^{\mathfrak a}_\varphi$.
			$(2)$ If $x^c$ is labelled by $F$ in $\q_\varphi$, then $\lit_i=\neg x$. So if $\mathfrak a$
			makes $\lit_i$ true,   then $e_i^c=a^*_x$ is labelled by $F$ in $\A^{\mathfrak a}_\varphi$.
			And if $\mathfrak a$ makes $\lit_i$ false, then $e_i^c=a^F_x$ is labelled by $F$ in $\A^{\mathfrak a}_\varphi$.
			
			$(\Leftarrow)$ Suppose $h \colon \q_\varphi \to \A^{\mathfrak a}_\varphi$. Then, for any 
			$y \in \avec{y}$, we have $h(y) = b_y^\nu$ for some $\nu \in \{F,T\}$.
			We then set $\mathfrak b(y) = \nu$. We claim that $\psi(\mathfrak a(\avec{x}),\mathfrak b(\avec{y}))$ is true.
			Indeed, for every clause $c= \lit_1 \lor \lit_2 \lor \lit_3$ in $\psi$, there is $(e_1,e_2,e_3)\in E^c$ such that $h$ maps the `contribution' of $c$ in 
			$\q_\varphi$ onto the `star' with centre $d_{(e_1,e_2,e_3)}^c$. 
			If $(e_1,e_2,e_3)$ is in $E^c$ because $e_i=a^*_x$, for some $i\in\{1,2,3\}$,
			$x \in \avec{x}$, then the label of $a^*_x$ in $\A^{\mathfrak a}_\varphi$ is $\mathfrak a(x)$. As $h$ is a homomorphism,
			the label of $x^c$ in $\q_\varphi$ is also $\mathfrak a(x)$, and so $\mathfrak a$ makes $c$ true by
			the definition of $\q_\varphi$.
			And if $(e_1,e_2,e_3)$ is in $E^c$ because $e_i=b_y^{\mathfrak b(y)}$, for some $i\in\{1,2,3\}$,
			$y \in \avec{y}$ with $\mathfrak b(y)$ making $\lit_i$ true, then $c$ is clearly true as well.
		\end{proof}
		
		Finally, we prove that $\varphi$ is satisfiable iff $\T, \A_\varphi \models \q_\varphi$
		iff $\I\models\q_\varphi$ for every model $\I$ of $\T$ and $\A_\varphi$.
		$(\Rightarrow)$ 
		Given $\I$, define an assignment $\mathfrak a_\I \colon \avec{x} \to \{F,T\}$ by taking
		$\mathfrak a_\I(x) = T$ if $a_x^* \in T^\I$ and $\mathfrak a_\I(x) = F$ if $a_x^* \in F^\I$. Then
		$\I = \A^{\mathfrak a_\I}_\varphi$, and so we are done by Claim~\ref{l:AE}. The implication $(\Leftarrow)$ also follows from Claim~\ref{l:AE}, as $\A^{\mathfrak a}_\varphi$ is a model of $\T$ and $\A_\varphi$, for every
		assignment $\mathfrak a \colon \avec{x} \to \{F,T\}$.
\end{proof}
	

\subsection{Data complexity: $\ACz$ and \LogSpace}\label{s:dataAC}
	
From now on, we focus on the \emph{data} complexity of answering d-sirups $(\dis_A,\q)$ and dd-sirups $(\dis_A^\bot,\q)$. 
We start classifying (d)d-sirups in terms of occurrences of $F$ and $T$ in the CQs $\q$. 
Atoms $F(x),T(x) \in \q$, for some variable $x$, are referred to as \emph{$FT$-twins} in $\q$. If $\q$ does not contain $FT$-twins, we call it \emph{twinless\/}, and similarly for ABoxes.   
By a \emph{solitary} $F$ or $T$ we mean a non-twin $F$- or, respectively, $T$-node.  

Observe that answering $(\dis\!_A,\q)$ is not easier than answering $(\dis_A^\bot,\q)$: 
If a given ABox $\A$ contains $FT$-twins, then there is no model of $\dis_A^\bot$ and $\A$, and so 
$\dis_A^\bot,\A\models\q$. Also, 
\begin{equation}\label{twinless}
\mbox{$\dis_A,\A\models\q$ \quad iff \quad $\dis_A^\bot,\A\models\q$, \quad for any twinless ABox $\A$.}
\end{equation}
If $\q$ contains $FT$-twins, then $\exists x\, \bigl(F(x) \land T(x)\bigr)$ is an FO-rewriting of $(\dis_A^\bot,\q)$. In general, any rewriting of $(\dis_A,\q)$ can be converted to a rewriting of $(\dis_A^\bot,\q)$ into the same language. For example, if $(\Pi,\G)$ is a (symmetric/linear) datalog rewriting of $(\dis\!_A,\q)$, then $(\Pi \cup \{\bot \leftarrow F(x), T(x)\},\G)$ is a (symmetric/linear) datalog rewriting of $(\dis^\bot_A,\q)$.

%

Aiming to identify FO-rewritable (d)d-sirups, we consider first those CQs that do not have a solitary $F$ or a solitary $T$, calling them $0$-\emph{CQs}. Queries of this type are quite common, only asking about one of the covering predicates (as in `are there any undergraduate students who take symbolic
AI courses?' and `what about the postgraduate ones?' provided that students are either undergraduate or postgraduate). The following theorem establishes a complexity dichotomy between 0-CQs and non-0-CQs, which contain occurrences of both covering predicates (as in `are there both undergraduate and postgraduate students in the College's University Challenge team?'). 


\begin{theorem}\label{ac0}
$(i)$ If $\q$ is a $0$-CQ, then both $(\dis_A,\q)$ and $(\dis_A^\bot,\q)$ can be answered in $\ACz$,
with $\q$ being an FO-rewriting of each of them.
		
$(ii)$ If $\q$ is twinless and contains at least one solitary $F$ and at least one solitary $T$, then answering $(\dis_\top, \q)$ and $(\dis_\top^\bot, \q)$, and so $(\dis\!_A, \q)$ and $(\dis_A^\bot, \q)$ is \LogSpace-hard. 
\end{theorem}
\begin{proof}
$(i)$ Let $\T$ be one of $\dis_A$ or $\dis_A^\bot$. We show that $\T, \A \models \q$ iff $\A \models \q$, and so $\q$ is an FO-rewriting of $(\T, \q)$.  
$(\Rightarrow)$ Suppose $\A \not\models \q$ and $\q$ has no solitary $F$ (the other case is similar). Let $\A'$ be the result of adding a label $F$ to every undecided $A$-node in $\A$. Clearly, $\A'$ is a model of $\T$ and $\A$ with $\A' \not \models \q$. $(\Leftarrow)$ is trivial. 
		
$(ii)$ The proof is by an FO-reduction of the \LogSpace-complete reachability problem for undirected graphs.
Denote by $\q'$ the CQ obtained by gluing together all the $T$-nodes and by gluing together all the $F$-nodes in $\q$. Thus, $\q'$ contains a single $T$-node, $x$, and a single $F$-node, $y$. Clearly, there is a homomorphism $h \colon \q \to \q'$.
Let $\q'' = \q' \setminus \{T(x), F(y)\}$. 
		
Suppose $G = (V,E)$ is a graph with $\snode,\tnode \in V$. We regard $G$ as a directed graph such that $(\unode,\vnode) \in E$ iff $(\vnode,\unode) \in E$, for any $\unode,\vnode \in V$. Construct 
a twinless ABox $\A_G$ from $G$ in the following way. Replace each edge $e = (\unode,\vnode) \in E$ by a copy $\q''_e$ of $\q''$
		such that, in $\q''_e$, node $x$ is renamed to $\unode$, $y$ to $\vnode$, and all other nodes $z$ to some fresh copy $\znode_e$. 
		Then $\A_G$ comprises all such $\q''_e$, for $e \in E$, as well as atoms $T(\snode)$ and $F(\tnode)$.
		We show that there is a path from $\snode$ to $\tnode$ in $G$ ($\snode \to_G \tnode$, in symbols) iff  $\dis_\top,\A_G \models \q$ iff  $\dis_\top^\bot,\A_G \models \q$ (cf.\ \eqref{twinless}).
		
$(\Rightarrow)$ Suppose there is a path $\snode=\vnode_0,  \dots, \vnode_n = \tnode$ in $G$ with $e_i =(\vnode_i,\vnode_{i+1}) \in E$, for $i < n$.  
		Consider an arbitrary model $\I$ of $\dis_\top$ and $\A_G$. Since $\I \models \dis_\top$, and $T(\snode)$ and $F(\tnode)$ are in $\A_G$, we can find some $i < n$ such that  
$\vnode_i\in T^\I$ and $\vnode_{i+1}\in F^\I$. 
As $\q''_{e_i}$ is an isomorphic copy of $\q''$, we obtain $\I \models \q''$, and so $\I \models \q$.
		
$(\Leftarrow)$ Suppose $\snode \not\to_G \tnode$. 
		Then, by the construction, $\tnode$ is not reachable from $\snode$ in $\A_G$ (not even via an undirected path).
		Define a model $\I$ of $\dis_\top^\bot$ and $\A_G$ by taking $T^\I$ to be the set of nodes in $\A_G$ that are reachable from $\snode$ (via an undirected path) and $F^\I$ its complement. Clearly, no connected component of $\A_G$ (as undirected graph) contains both $T^\I$- and $F^\I$ nodes.
		Since $\q$ is connected and contains at least one $T$ and at least one  $F$, it follows that $\I \not \models \q$.
\end{proof}

As $\ACz \subsetneq  \LogSpace$ and $\exists x\, (F(x) \land T(x))$ is an FO-rewriting of $(\dis_A^\bot, \q)$ in which $\q$ contains a twin, Theorem~\ref{ac0} gives a sufficient and necessary criterion of FO-rewritability for dd-sirups:
	
\begin{corollary}\label{cor:ac0}
A dd-sirup $\omq = (\dis_A^\bot, \q)$ can be answered in $\ACz$ iff $\q$ is a $0$-CQ or contains a twin. 
\end{corollary}
%
	
Characterising FO-rewritable d-sirups with a CQ containing twins turns out to be a much harder problem, which will be discussed in Section~\ref{pspaceFO}. 
	
\begin{example}\label{loop}\em
Meanwhile, the reader is invited to show that the d-sirups with the CQs below are FO-rewritable (see also Example~\ref{loop1}). 
%
Note that each of these CQs is \emph{minimal}, that is, not equivalent to any of its proper sub-CQs.\\
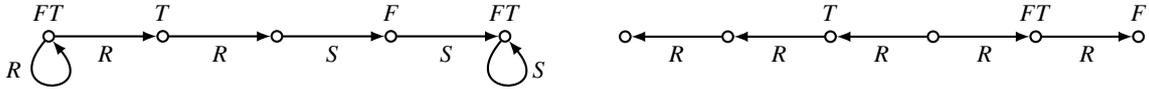
\centerline{ 
\begin{tikzpicture}[line width=0.8pt]
		\node[label=left:{\small $R$}] (0) at (-3.1,-.45) {};
		\node[label=right:{\small $S$}] (00) at (3.1,-.45) {};
		\node[point,scale=0.7,label=above:{\small $FT$}] (1) at (-3,0) {};
		\node[point,scale=0.7,label=above:{\small $T$}] (2) at (-1.5,0) {};
		\node[point,scale=0.7] (3) at (0,0) {};
		\node[point,scale=0.7,label=above:{\small $F$}] (4) at (1.5,0) {};
		\node[point,scale=0.7,label=above:{\small $FT$}] (5) at (3,0) {};
		\draw[->,right] (1) to node[below] {\small $R$} (2);
		\draw[->,right] (2) to node[below] {\small $R$} (3);
		\draw[->,right] (3) to node[below] {\small $S$} (4);
		\draw[->,right] (4) to node[below] {\small $S$} (5);
		\draw[->,right, scale =2] (1) to [out=-130,in=-50,loop] (1);
		\draw[->,right, scale =2] (5) to [out=-130,in=-50,loop]  (5);
\end{tikzpicture}
\qquad
\begin{tikzpicture}[>=latex,line width=0.8pt,rounded corners, scale = 0.9]
\node (b) at (0,-.8) {$\ $};
\node[point,scale = 0.7] (0) at (-1.5,0) {};
\node[point,scale = 0.7] (1) at (0,0) {};
\node[point,scale = 0.7,label=above:{\small $T$}] (m) at (1.5,0) {};
\node[point,scale = 0.7] (2) at (3,0) {};
\node[point,scale = 0.7,label=above:{\small $FT$}] (3) at (4.5,0) {};
\node[point,scale = 0.7,label=above:{\small $F$}] (4) at (6,0) {};
\draw[<-,right] (0) to node[below] {\small $R$}  (1);
\draw[<-,right] (1) to node[below] {\small $R$}  (m);
\draw[<-,right] (m) to node[below] {\small $R$} (2);
\draw[->,right] (2) to node[below] {\small $R$} (3);
\draw[->,right] (3) to node[below] {\small $R$} (4);
\end{tikzpicture}		
		}	
\end{example}

The lower bound result in Theorem~\ref{ac0}~$(ii)$ is complemented by the following simple sufficient condition. To formulate it, we require non-Boolean CQs $\q(\avec{x})$ that apart from existentially quantified variables
may also contain \emph{free} variables $\avec{x}$ called \emph{answer} or \emph{distinguished variables}. Such a CQ $\q'(x,y)$ is \emph{symmetric} if, for any ABox $\mathcal{A}$ and any $a,b \in \ind(\A)$, we have $\A \models \q'(a,b)$  iff $\A \models \q'(b,a)$, where $\A$ is regarded as an FO-structure and $\models$ is the usual first-order truth relation. 
	
\begin{theorem}\label{thm:symL}
Let $\T$ be one of $\dis_A$ or $\dis_A^\bot$ and let $\q$ be a Boolean CQ that is  equivalent to 
$$
 \exists x, y \, \big(F(x) \land \q'_1(x) \land \q'(x,y)\land \q'_2(y) \land T(y)\big),
$$ 
where
$(a)$ CQs $\q'_1(x)$, $\q'(x,y)$ and $\q'_2(y)$ do not contain solitary $T$ and $F$, 
$(b)$ $\q'(x,y)$ is symmetric, and
$(c)$ $\q'_1(x)$ and $\q'_2(y)$ are disjoint, with $x$ and $y$ being their only common variables with $\q'(x,y)$. Then $(\T,\q)$ is rewritable to a symmetric datalog program, and so can be answered in \LogSpace. 
\end{theorem}

\begin{proof}
Suppose $\T = \dis_A$. 
We claim that $\T,\A \models \q$ iff there exist $n \ge 1$ and $v_0, v_1, \dots, v_n \in \ind(\A)$ such that  
\begin{itemize}
\item[(S1)] $F(v_0), A(v_1), \dots, A(v_{n-1}), T(v_n) \in \A$,
\item[(S2)] $\A\models\q'(v_i, v_{i+1})$, for $0 \le i < n$,
\item[(S3)] $\A\models\q'_1(v_i)$, for $0 \le i < n$,
\item[(S4)] $\A\models\q'_2(v_i)$, for $1 \le i \le n$.
\end{itemize}
Indeed, suppose there are $v_0, v_1, \dots, v_n \in \ind(\A)$ such that (S1)--(S4) hold. Consider any model $\I$ of $\T$ and $\A$. By (S1), there is $i<n$ with $v_i\in F^\I$ and $v_{i+1}\in T^\I$. Then (S2)--(S4) guarantee that $\I \models \q$. 
Conversely, suppose $\T,\A \models \q$ for some ABox $\A$. For $P\in\{F,T,A\}$,
let $P^\A=\{a\in\ind(\A)\mid P(a)\in\A\}$.
Define inductively sets $F_j$ and $F'_j$, for $j \ge 0$, by setting $F_0 = F^\A$, 
$F'_{j} = \{ b \mid \A\models\q'_1(a)\land\q'(a,b) \land \q'_2(b) \mbox{ for some } a \in F_{j}\}$ and 
$F_{j+1} = A^\A\cap F'_j$. Let $\I$ be a model of $\T$ and $\A$ with 
$F^\I = \bigcup_{j = 0}^{\infty} F_j$ and 
$T^\I = T^\A \cup \bigl(A^\A \setminus \bigcup_{j = 1}^{\infty} F_j\bigr)$. 
By our assumption, there is a homomorphism $h \colon \q \to \I$. Thus, $h(x)\in F_j$ and $h(y)\in F'_j$ for some $j$.
Then $h(y)\in T^\A$, for otherwise  $h(y)\in F_{j+1}$, contrary to $h(y)\in T^\I$. Now, let $v_{n-1}=h(x)$ and $v_n=h(y)$. If $j=0$ then we are done with $n=1$. If $j>0$ then $h(x)\in A^\A\cap F'_{j-1}$, and so there is $v_{n-2}\in F_{j-1}$ such that
$\A\models \q'_1(v_{n-2})\land \q'(v_{n-2},v_{n-1})\land \q'_2(v_{n-1})$.
 By iterating this process, we obtain $v_0, v_1, \dots, v_n \in \ind(\A)$ as required.

It remains to observe that checking whether there are $v_0, v_1, \dots, v_n \in \ind(\A)$ such that (S1)--(S4) hold can be done by the following symmetric datalog program,         in which $B(x) = A(x)\land \q'_1(x) \land \q'_2(x)$:
\begin{align*}
\G &\leftarrow \q \\        
\G &\leftarrow F(x), \q_1'(x), \q'(x,y), P(y)\\
P(x)  &\leftarrow B(x),\q'(x,y), \q'_2(y), T(y)\\ 
P(x) &\leftarrow B(x), \q'(x,y), P(y), B(y) 
\end{align*}
where, by the symmetry of $\q'(x,y)$, the only recursive rule $P(x) \leftarrow B(x), \q'(x,y), P(y), B(y)$ is equivalent to its symmetric counterpart. 
If $\T = \dis_A^\bot$, we add the non-recursive rule $\G \leftarrow F(x), T(x)$ to the program. 
\end{proof}

	\begin{example}\em 
		By Theorems~\ref{thm:symL} and~\ref{ac0} $(ii)$, the d-sirup $(\dis_\top,\q)$ with $\q$ shown below is \LogSpace-complete.\\ 
		
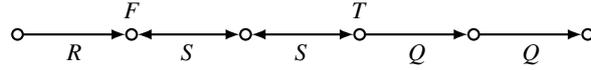
\centerline{
			\begin{tikzpicture}[>=latex,line width=0.8pt,rounded corners]
			\node[point,scale = 0.7] (0) at (-1.5,0) {};
			\node[point,scale = 0.7,label=above:{\small $F$}] (1) at (0,0) {};
			\node[point,scale = 0.7] (m) at (1.5,0) {};
			\node[point,scale = 0.7,label=above:{\small $T$}] (2) at (3,0) {};
			\node[point,scale = 0.7] (3) at (4.5,0) {};
			\node[point,scale = 0.7] (4) at (6,0) {};
			\draw[->,right] (0) to node[below] {\small $R$}  (1);
			\draw[<->,right] (1) to node[below] {\small $S$}  (m);
			\draw[<->,right] (m) to node[below] {\small $S$} (2);
			\draw[->,right] (2) to node[below] {\small $Q$} (3);
			\draw[->,right] (3) to node[below] {\small $Q$} (4);
			\end{tikzpicture}}
\end{example}


\subsection{Datalog rewritability of d-sirups with a $1$-CQ}	

In this section, we introduce a technical tool that can be used to show datalog rewritability of 
(d)d-sirups whose CQ contains exactly one solitary $F$ and at least one solitary $T$ (or exactly one solitary $T$ and at least one solitary $F$). We refer to such CQs as \emph{$1$-CQs}. 
The tool is an adaptation of the known (disjunctive) datalog technique of \emph{expansions}~\cite{DBLP:conf/pods/Naughton86,DBLP:conf/stoc/CosmadakisGKV88,DBLP:books/cs/Ullman89}. 
We use this tool to 
observe that every (d)d-sirup with a 1-CQ can be rewritten to a very simple  datalog query---nearly a sirup in the sense of~\cite{DBLP:conf/pods/CosmadakisK86,DBLP:conf/pods/Vardi88}, and so can be answered in \PTime. Note that a more general \emph{markability} technique (tracing dependencies on disjunctive predicates in the program rules) for rewriting disjunctive datalog programs into datalog was developed in~\cite{DBLP:journals/ai/KaminskiNG16}. 
In Section~\ref{pspaceFO}, we also adapt the datalog expansion technique to characterise FO-rewritability of those datalog queries semantically.

 
	
Throughout this section, we	assume that $\q$ is a $1$-CQ, with $F(x)$ and $T(y_1), \dots , T(y_n)$ being all of the solitary occurrences of $F$ and $T$ in $\q$. As before, we let $\T$ be one of $\dis_A$ or $\dis_A^\bot$. For each dd-sirup $\omq = (\T,\q)$, we define a monadic (that is, having at most unary IDB predicates) datalog program $\Pi_\omq$ with nullary goal $\G$ and four rules
\begin{align}
		\G &\leftarrow F(x), \q', P(y_1), \dots, P(y_n)\label{one}\\
		P(x)  &\leftarrow T(x)\label{two} \\ 
		P(x) &\leftarrow A(x), \q', P(y_1), \dots, P(y_n) \label{three}\\
		\G &\leftarrow F(x), T(x)\label{four}
\end{align}
where $\q' = \q \setminus \{F(x),T(y_1), \dots , T(y_n)\}$ and $P$ is a fresh predicate symbol that never occurs in ABoxes. Thus, the body of rule~\eqref{three} is obtained from $\q$ by replacing $F(x)$ with $A(x)$ and each $T(y_i)$ with $P(y_i)$. If $\T = \dis\!_A$, rule~\eqref{four} is omitted.
		
We also define by induction a class $\mathfrak K_\omq$ of ABoxes called \emph{cactuses for} $\omq$. We start by setting $\mathfrak K_\omq = \{\q\}$, regarding $\q$ as an ABox, and then recursively apply to $\mathfrak K_\omq$ the following `budding' rule:
\begin{description}
\item[(bud)] if $T(y) \in \C \in \mathfrak K_\omq$ with solitary $T(y)$, then we add to $\mathfrak K_\omq$ the ABox obtained by replacing $T(y)$ in $\C$ with the set $(\q \setminus \{ F(x) \}) \cup \{A(x)\}$, in which $x$ is renamed to $y$ and all other variables are given \emph{fresh} names.
%
\end{description}
It is straightforward to see by structural induction that
\begin{equation}\label{cactusq}
\mbox{$\T, \C\models \q$, \ \ for every $\C \in \mathfrak K_\omq$.}
\end{equation}
%
For $\C \in \mathfrak K_\omq$, we refer to the copies $\mathfrak s$ of (maximal subsets of) $\q$ comprising $\C$ as \emph{segments\/}. The \emph{skeleton} $\C^s$ of $\C$ is the ditree whose nodes are the segments $\mathfrak s$ of $\C$ and edges $(\mathfrak s, \mathfrak s')$ mean that $\mathfrak s'$ was attached to $\mathfrak s$ by budding. The \emph{depth of $\mathfrak s$ in} $\C$ is the number of edges on the branch from the root of $\C^s$ to $\mathfrak s$. The \emph{depth of} $\C$ is the maximum depth of its segments.
	%

\begin{example}\label{cac-ill}\em
In the picture below, the cactus $\C_2$ is obtained by applying {\bf (bud)} to the 1-CQ $\q$ twice. Its skeleton $\C_2^s$ with three segments 
$\mathfrak s_0,\mathfrak s_1,\mathfrak s_2$ is shown on the right-hand side of the picture.\\
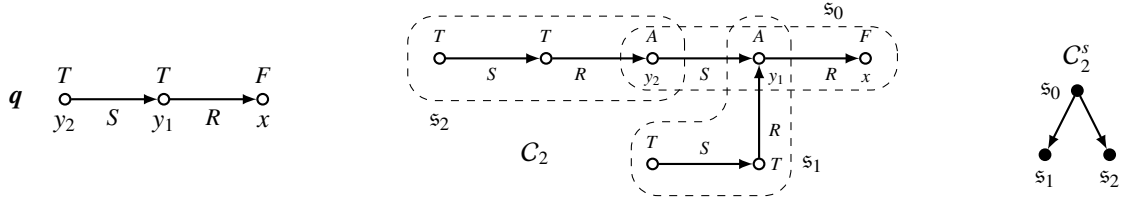
\centerline{
\begin{tikzpicture}[>=latex,line width=0.8pt, rounded corners,scale = 1.3]
\node (b) at (-0.5,-.9) {\ };
\node (0) at (-0.5,0) {$\q$};
\node[point,scale=0.7,label=above:{\small $T$},label=below:$y_2$] (1) at (0,0) {};
\node[point,scale=0.7,label=above:\small$T$,label=below:$y_1$] (m) at (1,0) {};
\node[point,scale=0.7,label=above:\small$F$,label=below:$x$] (2) at (2,0) {};
\draw[->,right] (1) to node[below] {\small $S$}  (m);
\draw[->,right] (m) to node[below] {\small $R$} (2);
\end{tikzpicture}
\hspace{1.5cm} 
\begin{tikzpicture}[line width=0.8pt,scale = 1.4]
\node (d) at (0.9,-0.9) {$\C_2$};
\node[point,scale=0.7,label=above:{\scriptsize $T$}] (1) at (0,0) {};
\node[point,scale=0.7,label=above:\scriptsize$T$] (2) at (1,0) {};
\node[point,scale=0.7,label=above:\scriptsize$A$,label=below:\scriptsize$y_2$] (3) at (2,0) {};
\node[point,scale=0.7,label=above:{\scriptsize $A$},label=below right:\!\scriptsize$y_1$] (4) at (3,0) {};
\node[point,scale=0.7,label=above:\scriptsize$F$,label=below:\scriptsize$x$] (5) at (4,0) {};
\node[point,scale=0.7,label=right:\!\scriptsize$T$] (6) at (3,-1) {};
\node[point,scale=0.7,label=above:\scriptsize$T$] (7) at (2,-1) {};
\draw[->,right] (1) to node[below] {\scriptsize $S$}  (2);
\draw[->,right] (2) to node[below,pos=.3] {\scriptsize $R$}  (3);
\draw[->,right] (3) to node[below] {\scriptsize $S$}  (4);
\draw[->,right] (4) to node[below,pos=.7] {\scriptsize $R$} (5);
\draw[->,right] (6) to node[right,pos=.3] {\scriptsize $R$}  (4);
\draw[->,right] (7) to node[above] {\scriptsize $S$}  (6);
\draw[thin,dashed,rounded corners=10] (4.3,.3) -- (4.3,-.3) -- (1.7,-.3) -- (1.7,.3) -- cycle;
\node[] at (3.7,.45) {$\mathfrak s_0$}; 
\draw[thin,dashed,rounded corners=10] (2.3,.4) -- (2.3,-.4) -- (-.3,-.4) -- (-.3,.4) -- cycle;
\node[] at (0,-.6) {$\mathfrak s_2$}; 
\draw[thin,dashed,rounded corners=10] (3.3,.4) -- (3.3,-1.3) -- (1.8,-1.3) -- (1.8,-.6) -- (2.7,-.6) -- (2.7,.4) -- cycle;
\node[] at (3.5,-1) {$\mathfrak s_1$}; 
\end{tikzpicture}
\hspace{1.5cm} 
\begin{tikzpicture}[line width=0.8pt,scale = 0.85]
\node (d) at (0,.5) {$\C_2^s$};
\node[point,scale=0.7,fill=black,label=below:{$\mathfrak s_1$}] (2) at (-.5,-1) {};
\node[point,scale=0.7,fill=black,label=below:{$\mathfrak s_2$}] (4) at (.5,-1) {};
\node[point,scale=0.7,fill=black,label=left:{$\mathfrak s_0$}] (7) at (0,0) {};
\draw[->,right] (7) to node[below] {} (2);
\draw[->,right] (7) to node[below] {} (4);
\end{tikzpicture}
}
\end{example}

\begin{theorem}\label{t:1cqchar}
For any $($d$)$d-sirup $\omq = (\T,\q)$ with a $1$-CQ $\q$ and any ABox $\A$, the following conditions are equivalent\textup{:}
\begin{itemize}
\item[$(i)$] $\T,\A \models \q$,

\item[$(ii)$] $\Pi_\omq, \A \models \G$,

\item[$(iii)$] there exists a homomorphism $h \colon \C \to \A$, for some $\C \in \mathfrak K_\omq$, or $\T=\dis_A^\bot$ and $\A$ contains an $FT$-twin.
\end{itemize}
\end{theorem}
\begin{proof}
We show the implications $(i)\Rightarrow (ii)\Rightarrow (iii)\Rightarrow (i)$.
		
$(i)\Rightarrow (ii)$		
		If $\T=\dis\!_A^\bot$ and $\A$ contains a node labelled by both $T$ and $F$, then $\G$ holds in the closure $\Pi_\omq(\A)$ of $\A$ under $\Pi_\omq$ by rule \eqref{four}. In any other case,
		we define a model $\I$ based on $\A$ by labelling each `undecided' $A$-node $a$ by $T$ if $P(a)$ holds in $\Pi_\omq(\A)$, and by $F$ otherwise. 
		As $\I$ is a model of $\T$ and $\A$, there is a homomorphism $h\colon\q\to\I$. Then $h(y_i)\in T^\I$,  and so $P\bigl(h(y_i)\bigr)$ holds in $\Pi_\omq(\A)$, for every $i\leq n$ (by rule \eqref{two} and the definition of $\I$). 
		We claim that $h(x)$ is an $F$-node in $\Pi_\omq(\A)$, and so $\G$ holds in $\Pi_\omq(\A)$
		by rule \eqref{one}. Indeed, otherwise by $h(x)\in F^\I$ and the definition of $\I$, $h(x)$ is an $A$-node but not a $P$-node
		in $\Pi_\omq(\A)$, contrary to rule \eqref{three}.

$(ii)\Rightarrow (iii)$	
Suppose $\T=\dis\!_A$ or $\A$ does not contain a node labelled by both $T$ and $F$.
Then rule  \eqref{four} is either not in $\Pi_\omq$ or not used.
We define inductively (on the applications of rule \eqref{three} in the derivation of $\G$) a cactus $\C \in \mathfrak K_\omq$ and a homomorphism $h\colon\C\to\A$. 
To begin with, there are objects $x^a,y^a_1,\dots,y^a_n$ for which rule \eqref{one} was triggered. Thus, $x^a$ is an $F$-node in $\Pi_\omq(\A)$, and so it is an $F$-node in $\A$. Take a function
$h_0\colon\q\to\A$ that preserves binary predicates, with $h_0(x)=x^a$ and $h_0(y_i)=y^a_i$ for $i\le n$. 
If $y^a_i$ is a $T$-node in $\A$ for every $i\leq n$, then $h=h_0$ is the required homomorphism from $\q\in \mathfrak K_\omq$ to $\A$.
		If $y^a_i$ is not a $T$-node in $\A$, for some $i$, then $y^a_i$ is a $P$-node in $\Pi_\omq(\A)$ obtained by rule \eqref{three},
		and so $y^a_i$ is an $A$-node in $\A$. Also, there are $x^b=y^a_i$ and $y^b_1,\dots,y^b_n$ such that rule \eqref{three} was triggered for 
		$x^b,y^b_1,\dots,y^b_n$. 		
		Let $\C$ be the cactus obtained from $\q$ by budding at $y_i$. We extend $h_0$ to a function $h_1\colon\C\to\A$ such that it preserves binary predicates and $h_1(y_j^{\mathfrak s})=y^b_j$ for all $T$-nodes  $y_j^{\mathfrak s}$ of the new segment $\mathfrak s$.
If $y^b_j$ is a $T$-node in $\A$ for every $j\leq n$, then $h=h_1$ is the required homomorphism from $\C\in \mathfrak K_\omq$ to $\A$. Otherwise, we bud $\C$ again and repeat the above argument. As the derivation of $\G$ from $\A$ using $\Pi_\omq$ is finite, sooner or later the procedure stops with a cactus and a homomorphism.
		
$(iii)\Rightarrow (i)$ 
If $\T=\dis\!_A^\bot$ and $\A$ contains a node labelled by both $T$ and $F$, then $\T,\A \models \q$
obviously holds. Otherwise,
take an arbitrary model $\I$ of $\T$ and $\A$. We define a model $\I^+$ of $\T$ and $\C$ by `pulling back $\I$' via the homomorphism $h$: for every node $x$ in $\C$, $x\in A^{\I^+}$ iff $h(x)\in A^{\I}$.
By \eqref{cactusq}, there is a homomorphism $g\colon\q\to\I^+$. Thus, the composition of $g$ and $h$ is
a $\q\to\I$ homomorphism, as required.
\end{proof}

\begin{corollary}\label{1cq}
Any \textup{(}d\textup{)}d-sirup $(\T,\q)$ with a $1$-CQ $\q$ is datalog-rewritable, and so can be answered in \PTime. 
\end{corollary}

As mentioned in the introduction, the problems of FO-rewritability (aka boundedness in the datalog literature) and linear-datalog-rewritability (aka linearisability) of datalog queries have been thoroughly investigated since the 1980s. 
In Sections~\ref{pspaceFO} and \ref{boundedLin}, we discuss these questions for (d)d-sirups with a 1-CQ. 


\subsection{Deciding FO-rewritability of d- and dd-sirups with a $1$-CQ}\label{pspaceFO}

A key to understanding FO-rewritability of d- and dd-sirups with a $1$-CQ is the following semantic criterion, which is well-known in the datalog setting; see, e.g.,~\cite{DBLP:conf/pods/Naughton86,DBLP:conf/stoc/CosmadakisGKV88}: 

\begin{theorem}\label{cac-ac0}
A \textup{(}d\textup{)}d-sirup $\omq = (\T,\q$) with a 1-CQ $\q$ is FO-rewritable iff 
there exists a $d<\omega$ such that every cactus $\C\in \mathfrak K_{\omq}$ contains a homomorphic image of some cactus $\C^-\in \mathfrak K_{\omq}$ of depth $\le d$, in which case a disjunction of the cactuses of depth $\le d$, regarded as Boolean CQs, is an FO-rewriting of $\omq$.
\end{theorem}
\begin{proof}
$(\Rightarrow)$ 
By~\cite[Proposition~5.9]{DBLP:journals/tods/BienvenuCLW14}, 
$\omq$ has an FO-rewriting of the form $\q_1 \lor \dots \lor \q_n$, where the $\q_i$ are CQs. Treating the $\q_i$ as ABoxes, we obviously have $\T,\q_i \models \q$, and so, by Theorem~\ref{t:1cqchar}, there is a homomorphism from some $\C_i \in \mathfrak K_\omq$ to $\q_i$. Now let $d$ be the maximum of the depths of the $\C_i$. Consider any $\C \in \mathfrak K_\omq$ of depth $>d$. Then there are homomorphisms $\C_i \to \q_i \to \C$, for some $i$, $1 \le i \le n$, as required. 
 
 $(\Leftarrow)$ Given $d<\omega$, we take all of the cactuses $\C_1,\dots,\C_n$ of depth $\le d$ (up to isomorphism). Now we consider each $\C_i$ as a CQ. Then $\C_1 \lor \dots \lor \C_n$ is an FO-rewriting of $\omq$. Indeed, if $\T,\A \models \q$ then there are homomorphisms $\C_i \to \C \to \A$, for some $\C$ and $i$, again by Theorem~\ref{t:1cqchar}. 
\end{proof}

\begin{example}\label{loop1}\em 
Let $\omq_1$ be the d-sirup with the first CQ from Example~\ref{loop}. It is not hard to verify that every cactus for $\omq_1$ contains a homomorphic image of this CQ, which is therefore an FO-rewriting of $\omq_1$. Now, let $\omq_2$ be the d-sirup with the second CQ from Example~\ref{loop}. Let $\C_k$ be the cactus obtained by applying {\bf (bud)} $k$-times to the original cactus $\C_0$ (isomorphic to the given CQ). There are homomorphisms $h\colon \C_1 \to \C_k$, for $k \ge 2$, and so $\omq_2$ is rewritable to $\C_0 \lor \C_1$. 
\end{example}

As follows from~\cite{DBLP:conf/stoc/CosmadakisGKV88}, which considered arbitrary monadic datalog queries, checking the criterion of Theorem~\ref{cac-ac0} can be done in 2\textsc{ExpTime}. A matching lower bound for monadic datalog queries with multiple recursive rules was established in~\cite{DBLP:conf/lics/BenediktCCB15}. It has recently been shown that already deciding FO-rewritability of monadic datalog sirups of the form $(\{\eqref{two},\eqref{three}\},P(x))$ and also of d-sirups with a 1-CQ is 2\textsc{ExpTime}-hard~\cite{PODS21}. Thus, we obtain:

\begin{theorem}[\cite{DBLP:conf/stoc/CosmadakisGKV88,DBLP:journals/lmcs/FeierKL19,PODS21}]
Deciding FO-rewritability of d-sirups can be done in 2\NExpTime. Deciding FO-rewritability of d-sirups with a 1-CQ is 2\ExpTime-complete. It follows that deciding FO-rewritability of CQs mediated by a Schema.org or $\DLb$~\textup{\cite{ACKZ09}} ontology can be done in 2\NExpTime{}, and is 2\ExpTime-hard.
\end{theorem}

The exact complexity of deciding FO-rewritability of d-sirups (2\NExpTime{} or 2\ExpTime) remains an open problem. 
Another important issue for OBDA and datalog optimisation is the \emph{succinctness problem} for FO-rewritings~\cite{DBLP:journals/ai/GottlobKKPSZ14,DBLP:journals/jacm/BienvenuKKPZ18}. It is not known if every FO-rewritable d-sirup has a polynomial-size FO-rewriting. However, we can show that this is not the case for the UCQ-, PE- and NDL-rewritings, which are standard in OBDA systems. We remind the reader (see~\cite{DBLP:journals/jacm/BienvenuKKPZ18} for details and further references) that a \emph{UCQ-rewriting} takes the form of disjunction (union) of CQs, while a \emph{positive existential} (\emph{PE}) \emph{rewriting} is built from atoms  using $\exists$, $\land$ and $\lor$ in an arbitrary way. A \emph{nonrecursive datalog} (\emph{NDL}) \emph{rewriting} is a datalog query $(\Pi,\G)$ such that 
the dependency digraph of 
$\Pi$ is acyclic, where a predicate $P$ \emph{depends} on a predicate $P'$ in $\Pi$ if $\Pi$ has a clause with $P$ in the head and $P'$ in the body.  

\begin{theorem}\label{hugeUCQ}
There is a sequence of FO-rewritable d-sirups $\omq_n = (\dis_A, \q_n)$ of polynomial size in $n > 0$ such that any UCQ-, PE- and NDL-rewritings of $\omq_n$ are of at least triple, double and single exponential size in $n$, respectively.
\end{theorem}
\begin{proof}
Consider an alternating Turing machine (ATM) $\atm_n$ that works as follows on any input of length $\le n$. Its tape of size exponential in $n$ is used as a counter from 0 to $2^{2^n}$. The tape also has two extra cells $a$ and $b$. $\atm_n$ begins in a $\lor$-state by writing 0 and 1 in cell $a$ in two alternative branches of the full computation space. Then $\atm_n$ continues, in a $\land$-state, by writing 0 and 1 in cell $b$ in two alternative branches of the full computation space. If the bits in $a$ and $b$ in a given branch of the tree are distinct, $\atm_n$ enters an accepting state. Otherwise, the counter is increased by 1 and the ATM repeats the previous two steps. If the counter exceeds $2^{2^n}$, $\atm_n$ enters a rejecting state. Thus, $\atm_n$ rejects every input. Moreover, given any input $\boldsymbol{w}$, every computation tree of $\atm_n$ on $\boldsymbol{w}$ contains exactly one rejecting configuration, which is the leaf of a branch of length double-exponential  in $n$.

We now use the ATMs $\atm_n$ and any input $\boldsymbol{w}$ of length $\le n$ to construct, as described in~\cite{PODS21}, polynomial-size 1-CQs $\q_n$. Then, by the $(\Rightarrow)$ direction of ~\cite[Lemma 4]{PODS21}, the d-sirups $\omq_n = (\dis_A, \q_n)$ are FO-rewritable. On the other hand, one can show similarly to the proof of the $(\Leftarrow)$ direction of~\cite[Lemma 4]{PODS21} that any computation tree of $\atm_n$ on $\boldsymbol{w}$ corresponds to a cactus $\C\in\mathfrak K_{\omq_n}$ of triple-exponential size in $n$ such that no smaller cactus is homomorphically embeddable to $\C$. 
It follows that \emph{any} UCQ-rewritings of $\omq_n$ must be of at least triple-exponential size. 

Any PE-rewritings of $\omq_n$ are of at least double-exponential size. Indeed, given a PE-rewriting in prenex form of size $s$ (the number of atoms in the formula), we can  transform its matrix (the quantifier-free part) to DNF and obtain a UCQ rewriting of size $\le s 2^s$. So if $s$ were sub-double-exponential, then the size of this UCQ-rewriting would be less than triple-exponential. A similar argument shows that it is impossible to obtain NDL-rewritings  of subexponential size because otherwise we could transform them to sub-double-exponential PE-rewritings.
\end{proof}

The proof above does not provide us with any lower bound on the size of FO-rewritings because, by a result of Gurevich and Shelah~\cite{DBLP:journals/jacm/Rossman08}, there is a potentially non-elementary blow-up in length from a homomorphism invariant FO-sentence to its shortest equivalent PE-sentence.
We illustrate Theorem~\ref{hugeUCQ} by a simple example of an FO-rewritable d-sirup whose UCQ-rewritings are of at least double-exponential size.

\begin{example}\label{bigUCQ}\em
Consider the d-sirups $\omq_n = (\dis_A, \q_n)$, where $\q_n$, for $n \ge 2$, is the 1-CQ depicted below, with the omitted labels on the edges being all $R$ (and $r,a_i,b_i,c_i$ being pointers rather than labels in $\q_n$).\\
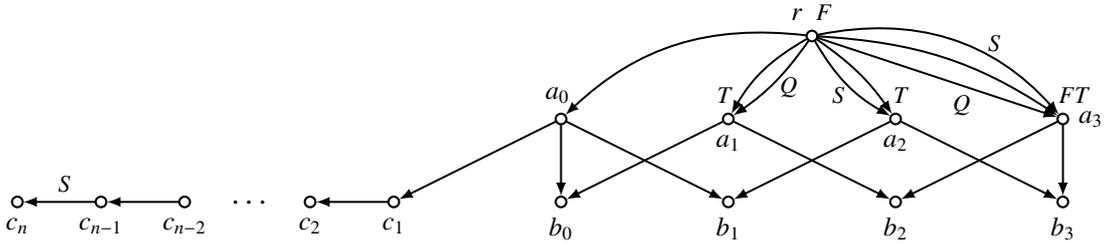
\centerline{\begin{tikzpicture}[>=latex,line width=.75pt, rounded corners,scale=1.1]
\node[point,scale=0.7,label=above:{$r$\ \ {\small $F$}}] (r) at (3,1) {}; 
\node[point,scale=0.7,label=above:{$a_0$\ \ }] (a0) at (0,0) {}; 
\node[point,scale=0.7,label=below:{$a_1$},label=above:{\small $T$\ \ }] (a1) at (2,0) {}; 
\node[point,scale=0.7,label=below:{$a_2$},label=above:{\small \ \ $T$}] (a2) at (4,0) {}; 
\node[point,scale=0.7,label=right:{$a_3$},label=above:{\small \quad $FT$}] (a3) at (6,0) {}; 
\node[point,scale=0.7,label=below:{$b_0$}] (b0) at (0,-1) {}; 
\node[point,scale=0.7,label=below:{$b_1$}] (b1) at (2,-1) {}; 
\node[point,scale=0.7,label=below:{$b_2$}] (b2) at (4,-1) {}; 
\node[point,scale=0.7,label=below:{$b_3$}] (b3) at (6,-1) {}; 
\node[point,scale=0.7,label=below:{$c_1$}] (c1) at (-2,-1) {}; 
\node[point,scale=0.7,label=below:{$c_2$}] (c2) at (-3,-1) {}; 
\node[point,scale=0.7,label=below:{$c_{n-2}$}] (cn1) at (-4.5,-1) {}; 
\node[point,scale=0.7,label=below:{$c_{n-1}$}] (cn) at (-5.5,-1) {}; 
\node[point,scale=0.7,label=below:{$c_{n}$}] (cnn) at (-6.5,-1) {}; 
\draw[->,bend right = 25] (r) to (a0);
\draw[->,bend right = 15] (r) to  (a1);
\draw[->,bend left = 10] (r) to node[below,pos=.3] {\small $Q$} (a1);
\draw[->,bend right = 15] (r) to node[below,pos=.4] {\small $S$} (a2);
\draw[->,bend left = 10] (r) to (a2);
\draw[->] (r) to node[below,pos=.6] {\small $Q$} (a3);
\draw[->,bend left = 15] (r) to (a3);
\draw[->,bend left = 30] (r) to node[above,pos=.7] {\small $S$} (a3);
\draw[->] (a0) to (b0);
\draw[->] (a0) to (b1);
\draw[->] (a1) to (b0);
\draw[->] (a1) to (b2);
\draw[->] (a2) to (b1);
\draw[->] (a2) to (b3);
\draw[->] (a3) to (b2);
\draw[->] (a3) to (b3);
\draw[->] (a0) to (c1);
\draw[->] (c1) to (c2);
\draw[->] (cn1) to (cn);
\draw[->] (cn) to node[above,pos=.4] {\small $S$} (cnn);
\node (d) at (-3.7,-1) {\large $\dots$};
\end{tikzpicture}}
\\
For any cactus $\C \in \mathfrak K_{\omq_n}$ and any node $x$ in $\q_n$, let $x^{\C}$ denote the copy of $x$ in the root segment of $\C$.
Observe that $\C$ is of depth $\geq n$ iff $\C$ contains an $R$-path $\pi$ that starts at $r^{\C}$ and has $\ge n$ $A$-nodes, the first of which is either $a_1^{\C}$ or $a_2^{\C}$.
We show first that if the depth of $\C$ is $\geq n$, then there is
a homomorphism $h \colon \q_n \to \C$. Indeed, if the first $A$-node of $\pi$ is $a_1^{\C}$, then
we can define $h$ by taking $h(r) = r^{\C}$, $h(a_0) = a_1^{\C}$, $h(a_i) = a_3^{\C}$ for $i=1,2,3$, 
$h(b_j) = b_2^{\C}$ for $j=0,1$,  $h(b_j) = b_3^{\C}$ for $j=2,3$, $c_1,\dots,c_{n-1}$ are $h$-mapped to the next $n-1$ $A$-nodes in $\pi$, and $h(c_n)$ is the 
$FT$-node in the segment with root $h(c_{n-1})$. 
(The case when the first $A$-node of $\pi$ is $a_2^{\C}$ is similar.)
So Theorem~\ref{cac-ac0} implies that $\omq_n$ is FO-rewritable. 

On the other hand, we claim that if $\C,\C'$ are cactuses of depths $<n$ and there is a
homomorphism $h \colon \C \to \C'$, then $\C=\C'$. 
We show this by induction on the depth of $\C$ (which cannot exceed the depth of $\C'$).
Observe first that, for any $x$ in $\q_n$, we must have $h(x^{\C})=x^{\C'}$: 
This holds
for $r^{\C}$ because the $FT$-node $a_3$ has no $S$-successors, for $a_0^{\C}$ because the depth of $\C'$ is less than $n$,
 $h(a_1^{\C}) \ne a_2^{\C'}$ because $a_2$ has no $Q$-predecessor,
  $h(a_1^{\C}) \ne a_3^{\C'}$ because $a_0$ and $a_1$ have a common successor, while $a_0$ and $a_3$ do not,
 we have a similar argument for $h(a_2^{\C})$, and then we clearly have $h(x^{\C})=x^{\C'}$ for $x=c_1,\dots,c_{n-1},a_3$.
 It follows that if $\C=\q_n$ then $\C'=\q_n$ must also hold, otherwise $h$ does not preserve $T$. If the depth of $\C$ is
 $>0$ then, for $i=1,2$,  let $\mathfrak s_i$ be the segment in $\C^s$ having $a_i^{\C}$ as its root node, and
 let $\C^-_i$ be the `subcactus' of $\C$ whose skeleton is the subtree of $\C^s$ with root $\mathfrak s_i$.
 We define $\C_i'{}^{-}$ from $\C'$ similarly. An inspection of $\q_n$ shows that we must have homomorphisms
 $h_1\colon\C^-_1\to\C'_1{}^-$ and $h_2\colon\C^-_2\to\C'_2{}^-$. Thus, we have $\C^-_1=\C'_1{}^-$ and $\C^-_2=\C'_2{}^-$ by the induction hypothesis (IH).
 Therefore, $\C=\C'$ follows, and so the UCQ rewriting $\Phi_n$ of $\omq_n$ provided by Theorem~\ref{cac-ac0} contains all different cactuses of depth $< n$, the number of which is $2^{2^{\mathcal{O}(n)}}$. 
It follows that \emph{any} UCQ-rewritings $\Phi'_n$ of $\omq_n$ have at least 
$2^{2^{\mathcal{O}(n)}}$ disjuncts. For otherwise, by the pigeonhole principle, there exist different disjuncts $\C$ and $\C'$ in $\Phi_n$ and $\C \to D$ and $\C' \to D$ homomorphisms, for some disjunct $D$ of $\Phi'_n$. On the other hand, there is a $D \to \C''$ homomorphism, for some disjunct $C''$ in $\Phi_n$, and so, as shown above, $\C = \C' = \C''$, which is a contradiction. 

One can readily transform $\Phi_n$ to an equivalent PE-rewriting of exponential size at the expense of nested $\land$ and $\lor$. But, by the proof of Theorem~\ref{hugeUCQ}, there are no PE-rewritings of subexponential size. On the other hand, the datalog program $\{\eqref{one}\text{--}\eqref{three}\}$ can be converted to an NDL-program describing cactuses of depth $<n$ and containing $O(n)$ rules.
\end{example} 

Finding an explicit syntactic characterisation of FO-rewritable d-sirups turns out to be nearly as hard as characterising FO-rewritable OMQs in fully-fledged expressive DLs and monadic disjunctive datalog queries. 
Notice, however, that the 1-CQs used in Example~\ref{bigUCQ} and the construction of~\cite{PODS21} (underlying Theorem~\ref{hugeUCQ}) are quite involved dags with multiple edges and  possibly multiple $FT$-twins. So one could hope that by restricting the shape of CQs and/or by disallowing $FT$-twins we would obtain less impenetrable yet practically useful classes of d-sirups.
Indeed, for d-sirups $\omq$ whose 1-CQ is a \emph{ditree} with its unique solitary $F$-node as root, the program $\Pi_\omq$ can be reformulated as an $\mathcal{EL}$-ontology, and so one can use the $\ACz$/\NL/\PTime{} trichotomy of~\cite{DBLP:conf/ijcai/LutzS17,DBLP:journals/corr/abs-1904-12533}, which is checkable in \textsc{ExpTime}.
	
\begin{example}\label{withEL}\em
		To illustrate, consider the $1$-CQ $\q$ below:\\ 
		
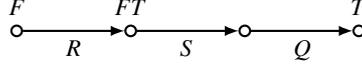
\centerline{
			\begin{tikzpicture}[line width=0.8pt]
			\node[point,scale=0.7,label=above:{\small $F$}] (1) at (-3,0) {};
			\node[point,scale=0.7,label=above:{\small $FT$}] (2) at (-1.5,0) {};
			\node[point,scale=0.7] (3) at (0,0) {};
			\node[point,scale=0.7,label=above:{\small $T$}] (4) at (1.5,0) {};
			%
			\draw[->,right] (1) to node[below] {\small $R$} (2);
			\draw[->,right] (2) to node[below] {\small $S$} (3);
			\draw[->,right] (3) to node[below] {\small $Q$} (4);
			\end{tikzpicture}}\\ 
		We have
		$\dis\!_A,\A \models \q$ iff $\mathcal{E},\A \models \exists x\, B(x)$, where $\mathcal{E}$ is the $\mathcal{EL}$ TBox  
		$\{F \sqcap C_\q \sqsubseteq B, \, T \sqsubseteq P, \, A \sqcap C_\q \sqsubseteq P\}$ with  
		$C_\q = \exists R. (F \sqcap T \sqcap \exists S.\exists Q.P)$ and the DL syntax illustrated in terms of first-order logic by~\eqref{ax1}--\eqref{ax4} in Section~\ref{intro}. 
	\end{example}


Further, as shown in~\cite{PODS21}, any d-sirup with a ditree 1-CQ, not necessarily having an $F$-labelled root, is either FO-rewritable or \LogSpace-hard, and deciding this dichotomy is fixed-parameter tractable if we regard the number of solitary $T$-nodes as a  parameter. 
Moreover, for dd-sirups with an arbitrary ditree CQ, there is an explicit syntactic trichotomy: each of them is either FO-rewritable or \LogSpace-complete, or \NL-hard. On the other hand, there is no readily available machinery for explicitly characterising \NL-completeness, \PTime- and \coNP-hardness of (d)d-sirups (let alone more general types of OMQs). We are going to fill in this gap to some extent in the remainder of the article. To begin with, we combine some ideas from~\cite{DBLP:conf/stoc/CosmadakisGKV88,DBLP:conf/ijcai/LutzS17} to prove a general sufficient condition of linearisability for d-sirups with a 1-CQ.


\section{Linear-datalog-rewritability of d- and dd-sirups with a $1$-CQ}\label{boundedLin}

We require a few new definitions, assuming as before that $\T$ is one of $\dis_A$ or $\dis_A^\bot$. First, we extend  the class $\mathfrak K_\omq$ of \emph{cactuses} for any (d)d-sirup $\omq = (\T,\q)$ 
to a wider class $\mathfrak K^+_\omq$ by adding another inductive rule to its definition. We define $\mathfrak K^+_\omq$ as the class of structures obtained from $\q$ by recursively applying {\bf (bud)} and 
the following `pruning' rule: 
%
\begin{description}
\item[(prune)] if $\C \in \mathfrak K^+_\omq$ and $\T, \C^- \models \q$, where $\C^- = \C \setminus \{T(y)\}$, for some solitary $T(y)$ in $\C$, then we add $\C^-$ to $\mathfrak K^+_\omq$.   
\end{description}
If $\C^-$ is obtaining from $\C$ by {\bf (prune)}, we define the skeleton $(\C^-)^s$ of $\C^-$ to be $\C^s$.
We continue to call members of $\mathfrak K^+_\omq$ \emph{cactuses\/}. 
%
We write $\C' \subseteq \C$ to say that, when regarded as ABoxes (sets of atoms), the cactus $\C'$ is (isomorphic to) a 
subset of the cactus $\C$. A cactus $\C\in\mathfrak K^+_\omq$ is \emph{minimal} if, for every  $\C'\in \mathfrak K^+_\omq$, $\C' \subseteq \C$ implies $\C'=\C$.
The class of minimal cactuses in $\mathfrak K^+_\omq$ is denoted by $\Kmin$. 
It should be clear that 
\eqref{cactusq} 
holds for $\Kmin$ in place of $\mathfrak K_\omq$.

\begin{example}\label{ex:prune}\em
Consider the d-sirup $\omq = (\dis_A,\q)$ with $\q$ shown on the left-hand side of the picture below (the $R$-labels on the edges are omitted). The cactus $\C_1$ is obtained by budding $y_1$ in $\q$, and $\C_2$ is obtained by budding $y_2$ in $\C_1$.\\
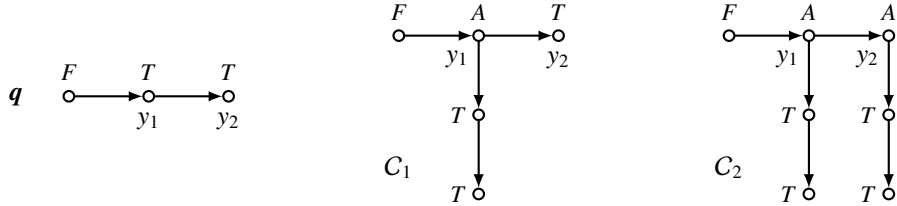
\centerline{ 
\begin{tikzpicture}[decoration={brace,mirror,amplitude=7},line width=0.8pt,scale=.7]
\node[point,scale=0.7, label=above:{\small $F$}] (1) at (0,0) {};
\node[point,scale=0.7,label=above:{\small $T$},label=below:{$y_1$}] (2) at (1.5,0) {};
\node[point,scale=0.7,label=above:{\small $T$},label=below:{$y_2$}] (3) at (3,0) {};
\draw[->,right] (1) to node[below] { } (2);
\draw[->,right] (2) to node[below] { } (3);
\node (d) at (-1,0) {$\q$};
\node (dd) at (-1,-2) {\ };
\end{tikzpicture}
\hspace*{1.5cm}
\begin{tikzpicture}[decoration={brace,mirror,amplitude=7},line width=0.8pt,scale=.7]
\node[point,scale=0.7,label=above:{\small $F$}] (1) at (0,0) {};
\node[point,scale=0.7,label=above:{\small $A$},label=below left:{$y_1$\!}] (2) at (1.5,0) {};
\node[point,scale=0.7,label=above:{\small $T$},label=below:{$y_2$}] (3) at (3,0) {};
\draw[->,right] (1) to node[below] { } (2);
\draw[->,right] (2) to node[below] { } (3);
\node[point,scale=0.7,label=left:{\small $T$\!}] (5) at (1.5,-1.5) {};
\node[point,scale=0.7,label=left:{\small $T$\!}] (6) at (1.5,-3) {};
\draw[->,right] (2) to node[below] { } (5);
\draw[->,right] (5) to node[below] { } (6);
\node (d) at (0,-2.5) {$\C_1$};
\end{tikzpicture}
\hspace*{1.5cm}
\begin{tikzpicture}[decoration={brace,mirror,amplitude=7},line width=0.8pt,scale=.7]
\node[point,scale=0.7,label=above:{\small $F$}] (1) at (0,0) {};
\node[point,scale=0.7,label=above:{\small $A$},label=below left:{$y_1$\!}] (2) at (1.5,0) {};
\node[point,scale=0.7,label=above:{\small $A$},label=below left:{$y_2$\!}] (3) at (3,0) {};
\draw[->,right] (1) to node[below] { } (2);
\draw[->,right] (2) to node[below] { } (3);
\node[point,scale=0.7,label=left:{\small $T$\!}] (5) at (1.5,-1.5) {};
\node[point,scale=0.7,label=left:{\small $T$\!}] (6) at (1.5,-3) {};
\draw[->,right] (2) to node[below] { } (5);
\draw[->,right] (5) to node[below] { } (6);
\node[point,scale=0.7,label=left:{\small $T$\!}] (7) at (3,-1.5) {};
\node[point,scale=0.7,label=left:{\small $T$\!}] (8) at (3,-3) {};
\draw[->,right] (3) to node[below] { } (7);
\draw[->,right] (7) to node[below] { } (8);
\node (d) at (0,-2.5) {$\C_2$};
\end{tikzpicture}
}
Let $\C_1^-$ be the result of removing $T(y_2)$ from $\C_1$. Then $\dis_A,\C_1^- \models \q$, the pruned cactus $\C_1^-$ is minimal, while $\C_2 \supset \C_1^-$ is not. Based on this observation, one can show that the skeleton of each cactus in $\Kmin$ has only one branch.
\end{example}

The branching number~\cite{DBLP:conf/ijcai/LutzS17} of a rooted tree $\mathfrak T$ is defined as follows.	
For any node $u$ in $\mathfrak T$, we compute inductively its \emph{branching rank} $\rank(u)$ by taking 
$\rank(u)=0$ if $u$ is a leaf and, for a non-leaf $u$,
\begin{equation}\label{rankdef}
\rank(u) = \begin{cases}
m + 1, & \text{if $u$ has $\ge 2$ children of branching rank $m$;}\\
m, & \text{otherwise,}
\end{cases} 
\end{equation}
where $m$ is the maximum of the branching ranks of $u$'s children. 
The \emph{branching number of\/} $\mathfrak T$ is the branching rank of its root node.
(In other words, the branching number of $\mathfrak T$ is $\bno$ if the largest full binary tree that is a minor of $\mathfrak T$ is of depth $\bno$.)
The \emph{branching number of a cactus\/} $\C\in\mathfrak K^+_\omq$ is the branching number of $\C^s$.
We call $\Kmin$ \emph{boundedly branching} if there is some 
	  $\bno<\omega$ such that $\Kmin$ contains a cactus with branching number $\bno$ but no cactus of greater branching number. Otherwise, we call $\Kmin$ \emph{unboundedly branching\/}.

\begin{example}\label{ex:cactus}\em
The branching number of each cactus in $\Kmin$ from Example~\ref{ex:prune} is 0; however, there are cactuses in $\mathfrak K_\omq$ with an arbitrarily large branching number $\bno < \omega$. 
As another instructive example, consider the 1-CQ $\q$ depicted on the left-hand side below. The cactus $\C_2$ for $\omq=(\dis_\top,\q)$ on the right-hand side, obtained by first budding $y_2$ and then \\ 
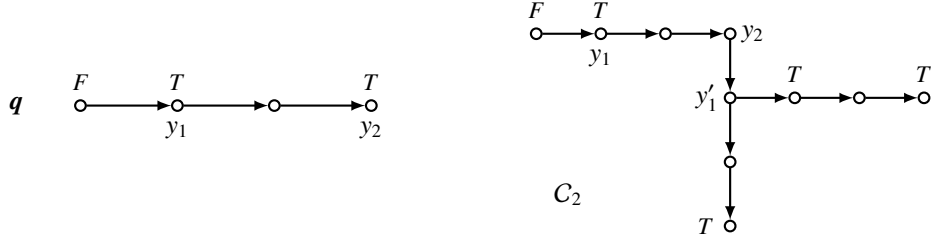
\centerline{ 
\begin{tikzpicture}[decoration={brace,mirror,amplitude=7},line width=0.8pt,scale=.85]
\node[point,scale=0.7, label=above:{\small $F$}] (1) at (0,0) {};
\node[point,scale=0.7,label=above:{\small $T$},label=below:{$y_1$}] (2) at (1.5,0) {};
\node[point,scale=0.7] (3) at (3,0) {};
\node[point,scale=0.7,label=above:{\small $T$},label=below:{$y_2$}] (4) at (4.5,0) {};
\draw[->,right] (1) to node[below] { } (2);
\draw[->,right] (2) to node[below] { } (3);
\draw[->,right] (3) to node[below] { } (4);
\node (d) at (-1,0) {$\q$};
\node (dd) at (-1,-2) {\ };
\end{tikzpicture}
\hspace*{1.5cm}
\begin{tikzpicture}[decoration={brace,mirror,amplitude=7},line width=0.8pt,scale=.85]
\node[point,scale=0.7, label=above:{\small $F$}] (1) at (0,0) {};
\node[point,scale=0.7,label=above:{\small $T$},label=below:{$y_1$}] (2) at (1,0) {};
\node[point,scale=0.7] (3) at (2,0) {};
\node[point,scale=0.7,label=right:{\!$y_2$}] (4) at (3,0) {};
\draw[->,right] (1) to node[below] { } (2);
\draw[->,right] (2) to node[below] { } (3);
			\draw[->,right] (3) to node[below] { } (4);
			\node[point,scale=0.7,label=left:{},label=left:{$y_1'$\!}] (5) at (3,-1) {};
			\node[point,scale=0.7] (6) at (3,-2) {};
			\node[point,scale=0.7,label=left:{\small $T$}] (7) at (3,-3) {};
			\draw[->,right] (4) to node[below] { } (5);
			\draw[->,right] (5) to node[below] { } (6);
			\draw[->,right] (6) to node[below] { } (7);
			\node[point,scale=0.7,label=above:{\small $T$}] (8) at (4,-1) {};
			\node[point,scale=0.7] (9) at (5,-1) {};
			\node[point,scale=0.7,label=above:{\small $T$}] (10) at (6,-1) {};
			\draw[->,right] (5) to node[below] { } (8);
			\draw[->,right] (8) to node[below] { } (9);
			\draw[->,right] (9) to node[below] { } (10);
\node (d) at (0.5,-2.5) {$\C_2$};
\end{tikzpicture}
}\\
$y_1'$, can  be pruned at $y_1$ by removing $T(y_1)$ (since \emph{every} node in a model of $\dis_\top$ is labelled by $F$ or $T$). Using this observation, one can show that every cactus in $\Kmin$ has branching number $\le 1$. On the other hand, if $\omq = (\dis_A,\q)$, then $\Kmin$ is unboundedly branching as follows from Theorems~\ref{thm:semanticCR} and~\ref{P-cond}.
\end{example}
	
\begin{theorem}\label{thm:semanticCR}
Every \textup{(}d\textup{)}d-sirup $\omq = (\T,\q)$ with a $1$-CQ $\q$ and boundedly branching $\Kmin$ is linear-datalog-rewritable, and so can be answered in \NL. 
\end{theorem}
\begin{proof}
Similarly to~\cite{DBLP:conf/stoc/CosmadakisGKV88}, 
we represent cactus-like ABoxes as terms of a tree alphabet and construct a tree automaton 
		$\mathfrak A_\omq$ such that 
		$(i)$ cactuses in $\Kmin$ are accepted by $\mathfrak A_\omq$, and $(ii)$ for every ABox $\A$ accepted by $\mathfrak A_\omq$, we have $\T,\A\models\q$.
Then, using ideas of \cite{DBLP:conf/ijcai/LutzS17}, we show that 
if $\Kmin$ is boundedly branching, then 
		the automaton $\mathfrak A_\omq$ can be transformed into a (monadic) linear-stratified datalog rewriting of $\omq$. As shown in \cite{DBLP:journals/tcs/AfratiGT03}, such a rewriting can further be converted into a linear datalog rewriting (at the expense of increasing the arity of IDB predicates in the program).
	
We only consider the case $\mathcal{O} = \dis_A$ leaving a similar proof for $\dis_A^\bot$ to the reader. As before, we assume that $\q$ is a $1$-CQ such that $F(x)$ and $T(y_1), \dots , T(y_n)$ are all of the solitary occurrences of $F$ and $T$ in $\q$.


Recall from \cite{tata2007} that a \emph{tree alphabet} is a finite set $\Sigma$ of symbols, each of which is associated with a natural number, its \emph{arity\/}. A $\Sigma$-\emph{tree} is any ground term built up inductively, using the symbols of $\Sigma$ as functions: $0$-ary symbols in $\Sigma$ are $\Sigma$-trees and, for any $k$-ary $\mathfrak a$ in $\Sigma$ and $\Sigma$-trees $\Tf_1,\dots,\Tf_k$, the term $\mathfrak a(\Tf_1,\dots,\Tf_k)$ is a $\Sigma$-tree. The \emph{branching number of a $\Sigma$-tree} is that of its parse tree. 
We define a tree alphabet $\SigmaQ$ as follows.
Consider cactus-like ABoxes that are built from $\q$ using {\bf (bud)} and {\bf (prune)}, with applications of the latter also allowed when $\dis_A,\C^-\not\models\q$ for the resulting ABox $\C^-$, and extend the notions of skeleton, branching number and segments to these in the natural way. 
The symbols of $\SigmaQ$ are the segments $\mathfrak s$ of such ABoxes, with the arity of $\mathfrak s$ being the number of its budding nodes, and with the $x$-node of $\mathfrak s$ being either labelled by $F$ or not.
Then each cactus in $\mathfrak K^+_\omq$ can be encoded by some $\SigmaQ$-tree; see Fig.~\ref{f:Stree} for an example.  On the other hand, every $\SigmaQ$-tree represents some cactus-like ABox. So, with a slight abuse of terminology, from now on by a $\SigmaQ$-tree we mean either the corresponding term or ABox. 
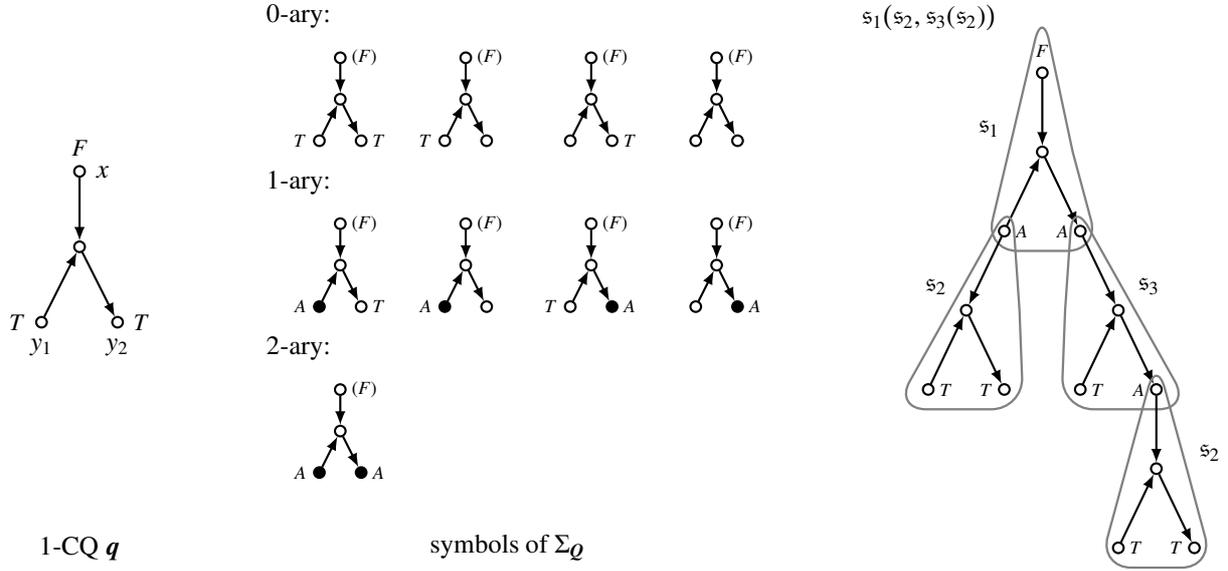
\begin{figure}[t]
		\centering 
\begin{tikzpicture}[decoration={brace,mirror,amplitude=7},line width=0.8pt]
			\node[point,scale=0.7, label=above:{\small$F$},label=right:{$x$}] (1) at (0,0) {};
			\node[point,scale=0.7] (2) at (0,-1) {};
			\node[point,scale=0.7,label=left:{\small $T$},label=below:{$y_1$}] (3) at (-.5,-2) {};
			\node[point,scale=0.7,label=right:{\small $T$},label=below:{$y_2$}] (4) at (.5,-2) {};
			\draw[->,right] (1) to node[below] { } (2);
			\draw[->,right] (3) to  (2);
			\draw[->,right] (2) to node[below] { } (4);
			\node[] at (0,-5) {$1$-CQ $\q$};
\end{tikzpicture}
\hspace*{1.2cm}
\begin{tikzpicture}[decoration={brace,mirror,amplitude=7},line width=0.8pt,scale=.55]			
			\node[point,draw=white] (t0) at (-1,2) {$0$-ary:};
			\node[point,scale=0.7, label=right:\!{\scriptsize $(F)$}] (1) at (0,1) {};
			\node[point,scale=0.7] (2) at (0,0) {};
			\node[point,scale=0.7,label=left:{\scriptsize $T$\!}] (3) at (-.5,-1) {};
			\node[point,scale=0.7,label=right:\!{\scriptsize $T$}] (4) at (.5,-1) {};
			\draw[->,right] (1) to node[below] { } (2);
			\draw[->,right] (3) to  (2);
			\draw[->,right] (2) to node[below] { } (4);
			\node[point,scale=0.7, label=right:\!{\scriptsize $(F)$}] (11) at (3,1) {};
			\node[point,scale=0.7] (12) at (3,0) {};
			\node[point,scale=0.7,label=left:{\scriptsize $T$\!}] (13) at (2.5,-1) {};
			\node[point,scale=0.7] (14) at (3.5,-1) {};
			\draw[->,right] (11) to node[below] { } (12);
			\draw[->,right] (13) to  (12);
			\draw[->,right] (12) to node[below] { } (14);
			\node[point,scale=0.7, label=right:\!{\scriptsize $(F)$}] (21) at (6,1) {};
			\node[point,scale=0.7] (22) at (6,0) {};
			\node[point,scale=0.7] (23) at (5.5,-1) {};
			\node[point,scale=0.7,label=right:\!{\scriptsize $T$}] (24) at (6.5,-1) {};
			\draw[->,right] (21) to node[below] { } (22);
			\draw[->,right] (23) to  (22);
			\draw[->,right] (22) to node[below] { } (24);
			\node[point,scale=0.7, label=right:\!{\scriptsize $(F)$}] (31) at (9,1) {};
			\node[point,scale=0.7] (32) at (9,0) {};
			\node[point,scale=0.7] (33) at (8.5,-1) {};
			\node[point,scale=0.7] (34) at (9.5,-1) {};
			\draw[->,right] (31) to node[below] { } (32);
			\draw[->,right] (33) to  (32);
			\draw[->,right] (32) to node[below] { } (34);
			\node[point,draw=white] (t1) at (-1,-2) {$1$-ary:};
			\node[point,scale=0.7, label=right:\!{\scriptsize $(F)$}] (o1) at (0,-3) {};
			\node[point,scale=0.7] (o2) at (0,-4) {};
			\node[point,fill,scale=0.7,label=left:{\scriptsize $A$\!}] (o3) at (-.5,-5) {};
			\node[point,scale=0.7,label=right:\!{\scriptsize $T$}] (o4) at (.5,-5) {};
			\draw[->,right] (o1) to node[below] { } (o2);
			\draw[->,right] (o3) to  (o2);
			\draw[->,right] (o2) to node[below] { } (o4);
			\node[point,scale=0.7, label=right:\!{\scriptsize $(F)$}] (o11) at (3,-3) {};
			\node[point,scale=0.7] (o12) at (3,-4) {};
			\node[point,fill,scale=0.7,label=left:{\scriptsize $A$\!}] (o13) at (2.5,-5) {};
			\node[point,scale=0.7] (o14) at (3.5,-5) {};
			\draw[->,right] (o11) to node[below] { } (o12);
			\draw[->,right] (o13) to  (o12);
			\draw[->,right] (o12) to node[below] { } (o14);
			\node[point,scale=0.7, label=right:\!{\scriptsize $(F)$}] (o21) at (6,-3) {};
			\node[point,scale=0.7] (o22) at (6,-4) {};
			\node[point,scale=0.7,label=left:{\scriptsize $T$\!}] (o23) at (5.5,-5) {};
			\node[point,fill,scale=0.7,label=right:\!{\scriptsize $A$}] (o24) at (6.5,-5) {};
			\draw[->,right] (o21) to node[below] { } (o22);
			\draw[->,right] (o23) to  (o22);
			\draw[->,right] (o22) to node[below] { } (o24);
			\node[point,scale=0.7, label=right:\!{\scriptsize $(F)$}] (o31) at (9,-3) {};
			\node[point,scale=0.7] (o32) at (9,-4) {};
			\node[point,scale=0.7] (o33) at (8.5,-5) {};
			\node[point,fill,scale=0.7,label=right:\!{\scriptsize $A$}] (o34) at (9.5,-5) {};
			\draw[->,right] (o31) to node[below] { } (o32);
			\draw[->,right] (o33) to  (o32);
			\draw[->,right] (o32) to node[below] { } (o34);
			\node[point,draw=white] (t2) at (-1,-6) {$2$-ary:};
			\node[] (caption) at (4,-10.8) {symbols of  $\SigmaQ$};
			\node[point,scale=0.7, label=right:\!{\scriptsize $(F)$}] (k1) at (0,-7) {};
			\node[point,scale=0.7] (k2) at (0,-8) {};
			\node[point,fill,scale=0.7,label=left:{\scriptsize $A$\!}] (k3) at (-.5,-9) {};
			\node[point,fill,scale=0.7,label=right:\!{\scriptsize $A$}] (k4) at (.5,-9) {};
			\draw[->,right] (k1) to node[below] { } (k2);
			\draw[->,right] (k3) to  (k2);
			\draw[->,right] (k2) to node[below] { } (k4);
\end{tikzpicture}	
\hspace*{1cm}
\begin{tikzpicture}[decoration={brace,mirror,amplitude=7},line width=0.8pt,yscale=1.05]
			\node[point,draw=white] (s) at (-1.5,.7) {$\mathfrak s_1\bigl(\mathfrak s_2,\mathfrak s_3(\mathfrak s_2)\bigr)$};
			\node[point,draw=white] (s1) at (-.7,-.7) {$\mathfrak s_1$};
			\node[point,scale=0.7, label=above:{\scriptsize $F$}] (1) at (0,0) {};
			\node[point,scale=0.7] (2) at (0,-1) {};
			\node[point,scale=0.7,label=right:\!{\scriptsize $A$}] (3) at (-.5,-2) {};
			\node[point,scale=0.7,label=left:{\scriptsize $A$}\!] (4) at (.5,-2) {};
			\draw[->,right] (1) to node[below] { } (2);
			\draw[->,right] (3) to (2);
			\draw[->,right] (2) to node[below] { } (4);
			\draw[ line width=.3mm,gray,rounded corners=12pt] (.4,-1) -- (0.75,-2.25) -- (-.75, -2.25) -- (0,0.8) --(.4,-1);
			\node[point,draw=white] (s2) at (-1.4,-2.7) {$\mathfrak s_2$};
			\node[point,scale=0.7] (22) at (-1,-3) {};
			\node[point,scale=0.7,label=right:\!{\scriptsize $T$}] (23) at (-1.5,-4) {};
			\node[point,scale=0.7,label=left:{\scriptsize $T$}\!] (24) at (-.5,-4) {};
			\draw[->,right] (3) to node[below] { } (22);
			\draw[->,right] (23) to (22);
			\draw[->,right] (22) to node[below] { } (24);
			\draw[ line width=.3mm,gray,rounded corners=12pt] (-.3,-3) -- (-0.25,-4.25) -- (-1.95, -4.25) -- (-.4,-1.6) --(-.3,-3);
			\node[point,draw=white] (s3) at (1.4,-2.7) {$\mathfrak s_3$};
			\node[point,scale=0.7] (32) at (1,-3) {};
			\node[point,scale=0.7,label=right:\!{\scriptsize $T$}] (33) at (.5,-4) {};
			\node[point,scale=0.7,label=left:{\scriptsize $A$}\!] (34) at (1.5,-4) {};
			\draw[->,right] (4) to node[below] { } (32);
			\draw[->,right] (33) to (32);
			\draw[->,right] (32) to node[below] { } (34);
			\draw[ line width=.3mm,gray,rounded corners=12pt] (.3,-3)-- (0.25, -4.25)  -- (1.95,-4.25) -- (.4,-1.6) --(.3,-3);
			\node[point,draw=white] (s22) at (2.2,-4.8) {$\mathfrak s_2$};
			\node[point,scale=0.7] (42) at (1.5,-5) {};
			\node[point,scale=0.7,label=right:\!{\scriptsize $T$}] (43) at (1,-6) {};
			\node[point,scale=0.7,label=left:{\scriptsize $T$}\!] (44) at (2,-6) {};
			\draw[->,right] (34) to node[below] { } (42);
			\draw[->,right] (43) to (42);
			\draw[->,right] (42) to node[below] { } (44);
			\draw[ line width=.3mm,gray,rounded corners=12pt] (1.9,-5) -- (2.25,-6.25) -- (.75, -6.25) -- (1.5,-3.6) --(1.9,-5);
			\end{tikzpicture}
\caption{An example of a tree alphabet $\SigmaQ$ and a cactus as a $\SigmaQ$-tree.}\label{f:Stree}
\end{figure}					

However, such an ABox $\C$ is not necessarily a cactus in $\mathfrak K^+_\omq$ for two possible reasons: either $\dis_A,\C\not\models\q$ or $\C$ having $F$-nodes in
some `wrong' segments (every cactus has a unique $F$-node, viz., the $x$-node of its root segment).
%
We are interested in those $\SigmaQ$-trees $\C$ for which $\dis_A,\C \models \q$. 
To capture them, we use tree automata \cite{tata2007}.
A \emph{nondeterministic finite tree automaton} (NTA) over a tree alphabet $\Sigma$ is a quadruple $\mathfrak A = (Q,Q_f,\Delta,\Sigma)$, where 
		%
$Q$ is a finite set of \emph{states},
%
$Q_f \subseteq Q$ is a set of \emph{final} states, and
%
$\Delta$ is a set of \emph{transitions} of the form $q_1,\dots,q_{k}\Rightarrow^{\mathfrak a}q$, where $k\ge 1$ is the arity of $\mathfrak a \in \Sigma$ and $q_1,\dots,q_k,q\in Q$; for symbols $\mathfrak a$ of arity $0$, we have \emph{initial transitions} in $\Delta$ of the form $\Rightarrow^{\mathfrak a}q$.
		%
		A \emph{run} of $\mathfrak A$ on a $\Sigma$-tree $\Tf$ is a labelling function $r$ from the subterms of $\Tf$ to $Q$ satisfying the following condition: for any subterm $\C^- = \mathfrak a(\C_1, \dots, \C_k)$ of $\Tf$, there is a transition $q_1, \dots, q_k\Rightarrow^\mathfrak a q$ in $\Delta$ such
		that $r(\C_1)=q_1,\dots,r(\C_k)=q_k$ and $r(\C^-)=q$ (in which case we say that the transition is \emph{used in} $r$).
		A $\Sigma$-tree $\C$ is \emph{accepted} by $\mathfrak A$ if there is a run of $\mathfrak A$ on $\C$ that labels $\C$ with a final state. Let $L(\mathfrak A)$ be the set of all $\Sigma$-trees accepted by $\mathfrak A$.
A set $L$ of $\Sigma$-trees is called a \emph{regular tree language} if $L=L(\mathfrak A)$, for some NTA $\mathfrak A$ over $\Sigma$.
		
\begin{claim}\label{l:wreg}
$L_\omq=\{\C\mid \C$ is a $\SigmaQ$-tree with $\dis_A,\C\models\q\}$ is a regular tree language.
\end{claim}		
\begin{proof}
We proceed via a series of steps.  
In the construction, we use Theorem~\ref{t:1cqchar} for describing $L_\omq$ by means of  the datalog program $\Pi_\omq = \{\eqref{one},\eqref{two},\eqref{three}\}$.
We extend the tree alphabet $\SigmaQ$ to a tree alphabet $\SigmaQP$ as follows.
For each symbol $\mathfrak s$ in $\SigmaQ$, we label some (possibly none) of the nodes in
segment $\mathfrak s$ by $P$. We call each resulting `segment' $\mathfrak s^e$ an \emph{extension of} 
$\mathfrak s$. (Each symbol in $\SigmaQ$ might have several extensions, and each of them has
the same arity as $\mathfrak s$.) 
Let $\SigmaQP$ consist of all possible extensions of every $\mathfrak s$ in $\SigmaQ$.
We say that a $\SigmaQP$-tree $\C^e$ is an \emph{extension} of a $\SigmaQ$-tree $\C$
if they have isomorphic tree structures, and each symbol $\mathfrak s^e$ in $\C^e$ is an extension
of the corresponding symbol $\mathfrak s$ in $\C$. 
For example, the closure $\Pi_\omq(\C)$ of any $\SigmaQ$-tree $\C$ under $\Pi_\omq$ is an extension of $\C$.

For any $\SigmaQP$-tree $\C^e$, we write $\C^e\models\G$, for the goal predicate $\G$ of $\Pi_\omq$,
if there is a homomorphism from $\q^e$ to $\C^e$, where
$\q^e=\q\setminus\{T(y_1),\dots,T(y_n)\}\cup\{P(y_1),\dots,P(y_n)\}$.
We claim that each of the following is a regular tree language:
\begin{enumerate}
\item[(a)]
the set of $\SigmaQP$-trees $\C^e$ with $\C^e\ne\Pi_\omq(\C^e)$;
\item[(b)]
the set of $\SigmaQP$-trees $\C^e$ with $\C^e\models\G$;
\item[(c)]
the set of $\SigmaQ$-trees $\C$ that have some extension $\C^e$ with $\C^e=\Pi_\omq(\C^e)$ and
$\C^e\not\models\G$;
\item[(d)]
the set of $\SigmaQ$-trees $\C$ with $\Pi_\omq,\C\models\G$.
\end{enumerate}
Indeed,
to show (a), we need an NTA `detecting a pattern' in the ABox $\C^e$ falsifying one of rules \eqref{two}--\eqref{three} in $\Pi_\omq$.
Similarly, to show (b), we need an NTA `detecting a pattern' in $\C^e$ corresponding to an application of rule \eqref{one} in $\Pi_\omq$.
Now, (c) follows from (a), (b) and the fact that regular tree languages are closed under taking complements, intersections and linear homomorphisms \cite{tata2007}
(as the `forgetting' function substituting $\mathfrak s$ for each $\mathfrak s^e$ is a linear tree homomorphism
from $\SigmaQP$-trees to $\SigmaQ$-trees, mapping any extension $\C^e$ to $\C$). 
To show (d), take the complement of (c), and observe that $\Pi_\omq,\C\models\G$ iff, for every extension $\C^e$ of $\C$, whenever $\C^e=\Pi_\omq(\C^e)$ then $\C^e\models\G$.

Finally, it follows from (d) and Theorem~\ref{t:1cqchar} that $L_\omq$ is a regular tree language.
\end{proof}

An NTA $\mathfrak A= (Q,Q_f,\Delta,\Sigma)$ is \emph{linear-stratified} if there is a function $\st \colon Q \to \omega$ such that, for any transition $q_1,\dots,q_{k}\Rightarrow^{\mathfrak a}q$ in $\Delta$, 
\begin{itemize}
\item[--] 
$\st (q_i) \le \st(q)$, for every $i$, $1 \le i \le k$, and
\item[--] 
there is at most one $i$ such that  $1 \le i \le k$ and $\st(q_i) = \st(q)$.
\end{itemize}
		
\begin{claim}\label{le:stra} 
For any NTA $\mathfrak A$ and any $\bno < \omega$, there is a linear-stratified NTA $\mathfrak A^s$ such that 
\begin{equation}\label{automaton-inclusions}
	\{ \Tf \in L(\mathfrak A) \mid \text{the branching number of $\Tf$ is} \le \bno \}~\subseteq~	L(\mathfrak A^s) ~\subseteq~ L(\mathfrak A).
\end{equation}		
\end{claim}
\begin{proof}
Suppose $\mathfrak A= (Q,Q_f,\Delta,\Sigma)$.		
			We define $\mathfrak A^s = (Q^s,Q^s_f,\Delta^s,\Sigma)$ as follows. First, set 
$Q^s = Q \times \{0,\dots,\bno\}$ and  $Q_f = Q_f \times \{0,\dots,\bno\}$.
Then, for any transition of the form $\Rightarrow^{\mathfrak a}q$ in $\Delta$, we add  the transition 
$\Rightarrow^{\mathfrak a}(q,0)$ to $\Delta^s$.
 For any transition $q_1,\dots,q_{k}\Rightarrow^{\mathfrak a}q$ in $\Delta$ and any $m \le \bno$, we add to
 $\Delta^s$ all transitions $(q_1,m_1), \dots, (q_k,m_{k})\Rightarrow^{\mathfrak a} (q,m)$ such that 
\begin{itemize}
\item[--] either $m_1,\dots,m_{k} < m$ and $m_i = m_j = m-1$, for some $i \ne j$;
				
\item[--] or $m_i = m$, for some $i$, and $m_j < m$, for all $j \ne i$.
\end{itemize}
$\mathfrak A^s$ is linear-stratified as one can set $\st\bigl((q,m)\bigr) = m$, for $q\in Q$, $m\le \bno$. To show \eqref{automaton-inclusions}, observe that $L(\mathfrak A^s)\subseteq L(\mathfrak A)$ since from every run $r$ of $\mathfrak A^s$ on $\C$ we obtain a run of $\mathfrak A$ on $\C$
by replacing each  $(q_1,m_1), \dots, (q_k,m_{k})\Rightarrow^{\mathfrak a} (q,m)$ used in $r$
with $q_1,\dots,q_{k}\Rightarrow^{\mathfrak a}q$. For the other inclusion, given a run $r$ of $\mathfrak A$ 
on some $\C$ with branching number $\le \bno$, we obtain a run of $\mathfrak A^s$ on $\C$ by labelling each subterm $\C^-$
of $\C$ with state $\bigl(r(\C^-),\bnok\bigr)$, where $\bnok$ is the branching number of $\C^-$.			
\end{proof}
		
We can now complete the proof of Theorem~\ref{thm:semanticCR}.	
Indeed, suppose that every cactus in $\Kmin$ has branching number at most $\bno <\omega$.
		By Claims~\ref{l:wreg} and \ref{le:stra}, there is a linear-stratified NTA $\Af = (Q,Q_f,\Delta,\Sigma_\omq)$ 
such that
\[
 \{ \Tf \in L_\omq \mid \text{the branching number of $\Tf$ is at most }  \bno \} ~\subseteq~		L(\Af) ~\subseteq~ L_\omq.
 \]
 %
		Using $\Af$, we construct a (monadic) linear-stratified program $\Pi_\Af$ with goal predicate $\G_\Af$ as follows.
		For every $q \in Q$, we introduce a fresh unary predicate $P_q$. For every final state $q \in Q_f$, the program $\Pi_\Af$ contains the rule 
\begin{equation}\label{finalrule}		
\G_\Af \leftarrow P_q(x).
\end{equation}
 For every transition $q_1, \dots, q_k\Rightarrow^{\mathfrak s}  q$ in $\Delta$,  
		where the budding nodes in the $k$-ary segment $\mathfrak s$ are $y_{i_1}, \dots, y_{i_k}$,
		 $\Pi_\Af$ contains 
\begin{equation}\label{stratrule}
		P_q(x) \leftarrow \s, P_{q_1}(y_{i_1}),\dots, P_{q_k}(y_{i_k}).
\end{equation}
		%
		%
		As $\Af$ is linear-stratified, it is easy to see that the program $\Pi_\Af$ is linear-stratified.
		We claim that $(\Pi_\Af,\G_\Af)$ is a datalog-rewriting of $\omq$, that is, for any ABox $\A$ (without the $P_q$), we have 
		$\Pi_\Af, \A \models \G_\Af$ iff $\dis_A,\A \models \q$.
		
($\Leftarrow$) 
By Theorem~\ref{t:1cqchar},
there is a homomorphism $h \colon \C \to \A$, for some $\C \in \mathfrak K_\omq$. As $\C$ always contains some $\C'\in\Kmin$, 
we may assume that $\C\in\Kmin$, and so $\C$ has branching number $\le \bno$.
As $\dis_A,\C\models\q$ clearly holds for every $\C\in\Kmin$, 
it follows that $\C\in L_\omq$, and so $\C\in L(\Af)$.
Let $r$ be an accepting run of $\mathfrak A$ on $\C$. We construct a derivation of $\G_\Af$ in $\Pi_\Af(\A)$ by induction on $\C$ as a $\SigmaQ$-tree, moving from leaves to the root. 
For every segment $\mathfrak s$ in $\C$, if the transition  $q_1, \dots, q_k\Rightarrow^{\mathfrak s}  q$ 
is used in $r$ then
we apply \eqref{stratrule} with the substitution of $h(z)$ for any node $z$ in $\mathfrak s$.
Also, if $r(\C)=q$, for some final state $q$ of $\Af$, then we apply \eqref{finalrule}
with the substitution $h(x_{{\mathfrak s}_0})$, where $x_{{\mathfrak s}_0}$ is the $x$-node of the root segment 
$\mathfrak s_0$ in $\C$. It follows that $\Pi_\Af, \A \models \G_\Af$. 
		
($\Rightarrow$) By induction on a derivation of $\G_\Af$, we construct a $\SigmaQ$-tree $\mathcal{B}$, an accepting run $r$ of $\Af$ on $\mathcal{B}$,
and a homomorphism $f\colon\mathcal{B}\to\A$.
To begin with, there is an object $x^a$ for which \eqref{finalrule} was triggered for some $q\in Q_f$.
Then $P_q(x^a)$ was deduced by an application of \eqref{stratrule} for some $\mathfrak s$.
If this $\mathfrak s$ is $0$-ary, then $s$ is a $\SigmaQ$-tree (of depth $0$), the function $r$ labelling $\mathfrak s$ with $q$ is an accepting run on $\mathfrak s$, and the substitution $f_0$ used in \eqref{stratrule} is a homomorphism from $\mathfrak s$ to $\A$.
If $\mathfrak s$ is $k$-ary, for some $k>0$, then there are $y_{i_1}^a,\dots,y_{i_k}^a$ for which \eqref{stratrule} was triggered. 
For each $j=1,\dots,k$, consider the rule 
\[
P_{q_j}(x) \leftarrow \mathfrak s^j, P_{q_1^j}(y_{i_1}),\dots, P_{q_{k_j}^j}(y_{i_{k_j}}) 
\]
by which $P_{q_j}(y_{i_j}^a)$ was deduced. Take the ABox $\mathcal B$ built up by glueing the $x$ node of each segment $\mathfrak s^j$ to the $y_{i_j}$ node of $\mathfrak s$, extend $r$ by labelling each $\mathfrak s^j$ with $q_j$, and extend $f_0$ to a $\mathcal{B}\to\A$ homomorphism by taking the substitutions used in the rules. Now, if every $\mathfrak s^j$ is $0$-ary, then $\mathcal B$ is a $\SigmaQ$-tree and we are done. Otherwise, repeat the above procedure for the `arguments' of each $\mathfrak s^j$ of arity $>0$. As the derivation of $\G_\Af$ is finite, sooner or later the procedure stops, as required.

As $\mathcal{B}\in L(\Af)\subseteq L_\omq$, by Theorem~\ref{t:1cqchar} there exists  a homomorphism $h \colon \C \to \mathcal{B}$, for some cactus $\C \in \mathfrak K_\omq$. Then the composition of $h$ and $f$ is a homomorphism from $\C$ to $\A$, and so $\dis_A,\A \models \q$ by Theorem~\ref{t:1cqchar}, as required. 
\end{proof}

We do not know if the sufficient condition in Theorem~\ref{thm:semanticCR}  of linear-datalog-rewritability of (d)d-sirups with a 1-CQ is also a necessary one. As follows from~\cite{DBLP:conf/ijcai/LutzS17,DBLP:journals/corr/abs-1904-12533}, this is so for ditree 1-CQs with root labelled by $F$; see Example~\ref{withEL}. 
%
We use Theorem~\ref{thm:semanticCR} in the next section to show that answering path-shaped dd-sirups with a certain periodic structure can be done in \NL. 

	
\section{$\ACz\,/\,\NL\,/\,\PTime\,/\,\coNP$-tetrachotomy of dd-sirups with a path CQ}\label{sec:tetra}
	
We now obtain the main result of this article: a complete syntactic classification of dd-sirups $(\dis_A^\bot, \q)$ with a path-shaped CQ $\q$ according to their data complexity and rewritability type. While the $\ACz/\NL$ part of this $\ACz/\NL/\PTime/\coNP$-tetrachotomy follows from our earlier results, proving \PTime- and especially \coNP-hardness turns out to be tough and requires the development of novel machinery.

From now on, we only consider path CQs $\q$ (whose digraph is path-shaped).  
Solitary $F$- and $T$-nodes  will simply be called $F$- and $T$-\emph{nodes}, respectively. 
We denote the first (root) node in $\q$ by $\be$ and the last (leaf) node by $\en$.  Given 
 nodes $x$ and $y$, we write $x \prec y$ to say that there is a directed path from $x$ to $y$ 
 in $\q$;
 as usual, $x\preceq y$ means $x\prec y$ or $x= y$. 
 For $x \preceq y$,
 the set $[x,y]$  comprises those atoms in $\q$ whose variables are in the interval $\{z \mid x \preceq z \preceq y\}$ and $(x,y) = [x,y] \setminus \{T(x), F(x), T(y), F(y)\}$. For $\boldsymbol{i}=(x,y)$, we let $|\boldsymbol{i}|$ be 
 the length of the path from $x$ to $y$, and $|\q|=|(\be,\en)|$. 

We divide path CQs into three disjoint classes: the $0$-CQs and the $1$-CQs defined earlier, and the \emph{$2$-CQs} that contain at least two $F$-nodes and at least two $T$-nodes. 
As we saw in Section~\ref{s:dataAC}, 
dd-sirups $(\dis_A^\bot,\q)$ with $\q$ containing $FT$-twins are always FO-rewritable. 
We split twinless $1$-CQs further into \emph{periodic} and \emph{aperiodic} ones, only considering $1$-CQs with a single $F$-node and at least one $T$-node (as the case with a single $T$-node and at least one $F$-node is symmetric). Given such a twinless $1$-CQ $\q$ and natural numbers $l,r$ with $l+r\geq 1$, we write $\q=\q_{lr}$ to say that $\q$ has $l$-many $T$-nodes $x_{-l}\prec\dots\prec x_{-1}$ that $\prec$-precede its only $F$-node $x_0$, and $r$-many $T$-nodes $x_{1}\prec\dots\prec x_{r}$ that $\prec$-succeed $x_0$.
%
%
For every $i$ with $-l\le i\le r+1$, we define a set $\boldsymbol{r}_i$ of binary atoms by taking
$\boldsymbol{r}_i=(x_{i-1},x_i)$, where $x_{-l-1}=\be$ and $x_{r+1}=\en$.
%
%
Note that $\boldsymbol{r}_i\ne\emptyset$ for $-l< i< r+1$, but 
$\boldsymbol{r}_{-l}= \emptyset$ if $\be=x_{-l}$ and $\boldsymbol{r}_{r+1}= \emptyset$ if $x_r=\en$.\\
%
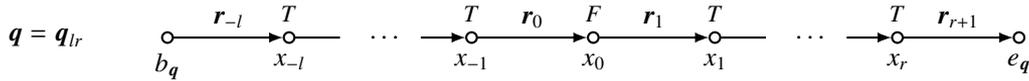
\centerline{ 
\begin{tikzpicture}[decoration={brace,mirror,amplitude=7},line width=0.8pt,scale=1.6]
                         \node (q) at (-2,0) {$\q=\q_{lr}$};
			\node[point,scale=0.7,label=below:{$\be$}] (0) at (-1,0) {};
			\node[point,scale=0.7,label=above:{\small $T$},label=below:{$x_{-l}$}] (1) at (0,0) {};
			\node (2) at (.5,0) {};
			\node (2b) at (.8,0) {$\dots$};
			\node (2c) at (1,0) {};
			\node[point,scale=0.7,label=above:{\small $T$},label=below:{$x_{-1}$}] (3) at (1.5,0) {};
			\node[point,scale=0.7,label=above:{\small $F$},label=below:{$x_0$}] (4) at (2.5,0) {};
			\node[point,scale=0.7,label=above:{\small $T$},label=below:{$x_1$}] (5) at (3.5,0) {};
			\node (5b) at (4,0) {};
			\node (5a) at (4.3,0) {$\dots$};
			\node (5c) at (4.5,0) {};
			\node[point,scale=0.7,label=above:{\small $T$},label=below:{$x_r$}] (8) at (5,0) {};
			\node[point,scale=0.7,label=below:$\en$] (9) at (6,0) {};
			\draw[->,right] (0) to node[above] {$\boldsymbol{r}_{-l}$} (1);
			\draw[-,right] (1) to (2);
			\draw[->,right] (2c) to (3);
			\draw[->,right] (3) to node[above] {$\boldsymbol{r}_{0}$} (4);
			\draw[->,right] (4) to node[above] {$\boldsymbol{r}_{1}$} (5);
			\draw[-,right] (5) to (5b);
			\draw[->,right] (5c) to (8);
			\draw[->,right] (8) to node[above] {$\boldsymbol{r}_{r+1}$} (9);
			%
			%
			\end{tikzpicture}	
			}\\
Each $\boldsymbol{r}_i$ determines a finite sequence $\seqr{\boldsymbol{r}_i}$ of binary predicate symbols.
We call $\q$ \emph{right-periodic} if $\q=\q_{0r}$  and either $r=1$ or
$\seqr{\boldsymbol{r}_i}=\seqr{\boldsymbol{r}_1}$ for all $i=1,\dots,r$ and 
$\seqr{\boldsymbol{r}_{r+1}}=\seqr{\boldsymbol{r}_1}^\ast\lambda$ for some (possibly empty) prefix $\lambda$ of 
$\seqr{\boldsymbol{r}_1}$.
	By taking a mirror image of this definition, we obtain the notion of  \emph{left-periodic} 1-CQ, in which case 
$\q=\q_{l\,0}$ and either $l=1$ or $\seqr{\boldsymbol{r}_{-i}}=\seqr{\boldsymbol{r}_0}$ for all $i=1,\dots,l-1$ and 
$\seqr{\boldsymbol{r}_{-l}}=\lambda\seqr{\boldsymbol{r}_0}^\ast$ for some (possibly empty) suffix $\lambda$ of 
$\seqr{\boldsymbol{r}_0}$.
A twinless 1-CQ $\q$ is called \emph{periodic} if it is either right- or left-periodic, and \emph{aperiodic} otherwise. 

\begin{theorem}[{\bf tetrachotomy}]\label{t:tetra}
Let $\omq$ be any d-sirup with a twinless path CQ $\q$ or any dd-sirup with a path CQ $\q$. Then the following tetrachotomy holds \textup{(}where the three `if' can be replaced by `iff' provided that $\NL \ne \PTime \ne \coNP$\textup{)}\textup{:}
\begin{description}
\item[($\ACz$)] $\omq$ is FO-rewritable and can be answered in $\ACz$ iff $\q$ is a $0$-CQ or contains an $FT$-twin\textup{;} otherwise,
			
\item[(\NL)] $\omq$ is linear-datalog-rewritable and answering it is \NL-complete if $\q$ is a periodic $1$-CQ\textup{;}
			
\item[(\PTime)] $\omq$ is datalog-rewritable and answering it is \PTime-complete if $\q$ is an aperiodic $1$-CQ\textup{;}
			
\item[(\coNP)] answering $\omq$ is \coNP-complete if $\q$ is a $2$-CQ.
\end{description}
\end{theorem}
The first item follows from Theorem~\ref{ac0} and the fact that $\ACz \ne \LogSpace$. The upper bounds in the remaining three are given by Theorem~\ref{NL-cond}, Corollary~\ref{1cq}, and Theorem~\ref{combi}, respectively.
The matching lower bounds are established by Theorems~\ref{thm:nl-hardness}, \ref{P-cond} and \ref{t:coNPhard} to be proved below. 
We begin with  the following criterion: 	
	
\begin{theorem}\label{thm:nl-hardness}
If $\q$ is a twinless path $1$-CQ, then answering $(\dis_A, \q)$ and $(\dis_A^\bot, \q)$ is \NL-hard.
\end{theorem}
\begin{proof}
The proof is by an FO-reduction of the \NL-complete reachability problem for directed acyclic graphs (dags).
We assume that there exist a $T$-node $x$ and an $F$-node $y$ in $\q$ with $x\prec y$ (the other case is symmetric)  
and without any $F$- or $T$-nodes between them. 
Given a dag $G = (V,E)$ with nodes $\snode,\tnode \in V$, we construct a twinless 
 ABox $\A_G$ as follows.
Replace each edge $e = (\unode,\vnode) \in E$ by a fresh copy $\q^e$ of $\q$
such that node $x$ in $\q^e$ is renamed to $\unode$ with $T(\unode)$ replaced by $A(\unode)$, 
and node $y$ is renamed to $\vnode$ with $F(\vnode)$ replaced by $A(\vnode)$.
Then $\A_G$ comprises all such $\q^e$, for $e \in E$, as well as $T(\snode)$ and $F(\tnode)$.
We show that $\snode \to_G \tnode$ iff  $\TT,\A_G \models \q$ iff $\dis_A^\bot,\A_G\models \q$ (cf.\ \eqref{twinless}). 

$(\Rightarrow)$ Suppose there is a path $\snode=\vnode_0,  \dots, \vnode_n = \tnode$ in $G$ with $e_i =(\vnode_i,\vnode_{i+1}) \in E$, for $i < n$.  Then, for any model $\I$ of $\TT$ and $\A_G$, there is some $i < n$ such that $\vnode_i\in T^\I $ and $\vnode_{i+1}\in F^\I$. Thus, the isomorphism mapping from $\q$ to its copy $\q^{e_i}$ is a $\q\to\I$ homomorphism, and so $\I \models \q$. 

$(\Leftarrow)$ Suppose $\snode \nrightarrow_G \tnode$. Define a model $\I$ of $\TT$ and $\A_G$ by labelling with $T$ the undecided $A$-nodes in $\A_G$ that are reachable from $\snode$ (via a directed path) and with $F$ the remaining ones. 
By excluding all possible locations where the $T$-node $x$ could be mapped, we see that there is no homomorphism $h\colon\q\to\I$. Indeed, suppose $h(x)$ is in a copy $\q^{e}$ for some edge $e = (\unode,\vnode) \in E$. Then $h(x)$ cannot precede $\unode$ or succeed $\vnode$ in $\q^{e}$, otherwise there is not enough room in $\q^{e}$ 
for the rest of $\q$ to be mapped. And $h(x)$ cannot be between $\unode$ and $\vnode$ either, as $\q$ is twinless and
there are no $T$-nodes between $x$ and $y$ in $\q$. If $h(x)=\unode$, then we exclude all options where the $F$-node $y$ could be mapped: $h(y)$ cannot succeed $\unode$ in $\q^{(\unode',\unode)}$ for any edge $(\unode',\unode)$ because of the lack of room in $\q^{(\unode',\unode)}$, and $h(y)=\vnode'$ cannot hold in $\q^{(\unode,\vnode')}$ for any edge 
$(\unode,\vnode')\in E$ because such a $\vnode'$ is labelled by $T$ in $\I$.
For similar reasons, $h(x)=\vnode$ cannot happen either.
\end{proof} 

A generalisation of this theorem to d-sirups with ditree-shaped 1-CQs possibly containing $FT$-twins has been proved in~\cite{PODS21} using a much more involved construction; see also Example~\ref{NL-examples} below.



By Corollary~\ref{1cq}, all (d)d-sirups with a 1-CQ are datalog-rewritable and can be answered in \PTime. Our next task is to establish an \NL/\PTime{} dichotomy for d-sirups with a twinless path 1-CQ. 
	
\begin{theorem}\label{NL-cond}
If $\q$ is a periodic twinless path $1$-CQ, then both $(\TT,\q)$ and $(\dis_A^\bot, \q)$ are linear-datalog-rewritable, and so can be answered in \NL. 
\end{theorem}
\begin{proof}
We use the notation above, and only consider the case when $\omq = (\TT,\q)$ and $\q = \q_{0r}$ is a right-periodic twinless path $1$-CQ with a single $F$-node $x_0$ and $T$-nodes $x_1,\dots,x_r$.
We show that every cactus in $\Kmin$ has branching number at most $1$ and use Theorem~\ref{thm:semanticCR}. If $r =1$, then the cactuses in $\Kmin$ have branching number $0$. 

So suppose $r\geq 2$ and $\C \in \Kmin$. For nodes $u,v$ in $\C$, we write $u \prec_\C v$ to say that there is a directed path from $u$ to $v$ in the (acyclic) digraph $\C$.
We call a node in $\C$ a \emph{$T$-copy} if it is a copy of a $T$-node $x_i$ of $\q$ for some
$i=1,\dots,r$. There can be three kinds of $T$-copies: those that were budded while constructing $\C$ are labelled by $A$, those that were pruned have no label, and the rest are labelled by $T$.
Observe first that 
\begin{equation}\label{noTbelow}
\mbox{if some $T$-copy $u$ is unlabelled in $\C$, then there is no $T$-copy $v$ such that $u \prec_\C v$ and $v$ is labelled by $T$.}
\end{equation}
Indeed, consider any model $\I$ of $\TT$ and $\C$ in which all $A$-nodes $u'$ with $u\prec_\C u'$ are in $T^\I$. As  
$\dis_A,\C\models\q$, there is a homomorphism $h \colon \q \to \I$. As $x_0$ is an $F$-node in $\q$ and its copy $x_0^{{\mathfrak s}_0}$ in the root segment $\mathfrak s_0$ of $\C$ is the only $F$-node in $\C$, it follows from our assumption on $\I$ that $u\nprec_\C h(x_0)$. We show that $u \nprec_\C h(x_i)$, for any $i \le r$.
Indeed, this is clear if $h(x_0)\nprec_\C u$. So suppose $h(x_0)\prec_Cu \prec_\C h(x_r)$.
Then, either $h(x_0)=x_0^{{\mathfrak s}_0}$ or $h(x_0)$ is a budded $T$-copy $\prec_C$-preceding $u$.
By $\q$ being right-periodic, every $T$-copy on the path from $h(x_0)$ to $h(x_r)$ in $\C$ different from $h(x_0)$ must be labelled by $T$. However, this is not the case for $u$, which is a contradiction. 
As $u \nprec_\C h(x_i)$ for any $i \le r$, by using {\bf (bud)} and {\bf (prune)} we can construct a cactus $\C_1\in\mathfrak K^+_\omq$ that is the same as $\C$ apart from all $T$-labelled $T$-copies $u'$ with $u\prec_\C u'$ being unlabelled.
Then $\C_1\subseteq\C$, and so $\C \in \Kmin$ implies that $\C=\C_1$, proving \eqref{noTbelow}. 
 
Next, consider any branch $\mathfrak s_0, \dots,\mathfrak s_{n-1},\mathfrak s_n$ in the skeleton $\C^s$ of $\C$ such that there are no $A$-nodes
in the segment $\mathfrak s_{n-1}$ $\prec_\C$-succeeding the $A$-labelled $T$-copy $w$ that has been budded to obtain the leaf segment $\mathfrak s_n$. 
Let $\pi$ be the path in $\C$ from the root node of $\mathfrak s_0$ to the leaf node of $\mathfrak s_n$.
We claim that 
\begin{equation}\label{allTkept}
\mbox{all $T$-copies in $\pi$ are labelled by either $T$ or $A$.}
\end{equation}
Indeed, by \eqref{noTbelow}, it is enough to show that all $T$-copies in $\mathfrak s_n$ are labelled by $T$.
Suppose on the contrary that at least one of them is not. Let $\I$ be a model of $\dis_A$ and $\C$ where 
the $A$-node $w$ is labelled by $F$.
Then there is a homomorphism $h\colon\q\to \I$ such that $h(x_i)\ne w$ for any $i=0,\dots,r$.
Thus, by using {\bf (bud)} and {\bf (prune)} we can construct a cactus $\C_2\in\mathfrak K^+_\omq$ that is
the same as $\C$ apart from $w$ in $\mathfrak s_{n-1}$ not being budded but pruned (and so $w$ is unlabelled in $\C_2$ and $\mathfrak s_n$ is not
a segment in $\C_2^s$). 
Then $\C_2\subseteq\C$ but $\C_2\ne\C$, contrary to $\C\in\Kmin$. 

Now \eqref{allTkept} and $\q$ being right-periodic imply that, 
for any model $\I$ of $\dis_A$ and $\C$, there is a homomorphism $h\colon\q\to \I$ mapping
$x_0$ and $(x_0,\en)$ into $\pi$. Indeed, take $h(x_0)=z$ where $z$ is the $\prec_\C$-last $F$-labelled $A$-node in $\pi$ if there is such, and $x_0^{{\mathfrak s}_0}$ otherwise. By the definition of budding, the part of $\q$ $\prec$-preceding $x_0$ can be mapped to $\I$ as well, possibly covering some parts of $\C$ not in $\pi$ but in a child-segment of some of the $\mathfrak s_i$. Therefore, by using {\bf (bud)} and {\bf (prune)} we can construct a cactus $\C_3\in\mathfrak K^+_\omq$
whose skeleton consists of the branch $\mathfrak s_0, \dots,\mathfrak s_n$ and
all other children of the segments $\mathfrak s_i$ for $i=0,\dots,n-1$, the $T$-copies that were labelled by $A$ and budded further in $\C$ in some of these children are unlabelled in $\C_3$, and all other $T$-copies are the same in $\C_3$ and $\C$.
Then $\C_3\subseteq\C$ and the branching number of $\C_3$ is at most $1$. As $\C\in\Kmin$, $\C=\C_3$ follows. 
\end{proof}

One can generalise the proof of Theorem~\ref{NL-cond} to various path $1$-CQs \emph{with} $FT$-twins. Here are some examples. 

\begin{example}\label{NL-examples}\em 
We invite the reader to show that answering the d-sirups with the following 1-CQs is \NL-complete:\\[2pt]
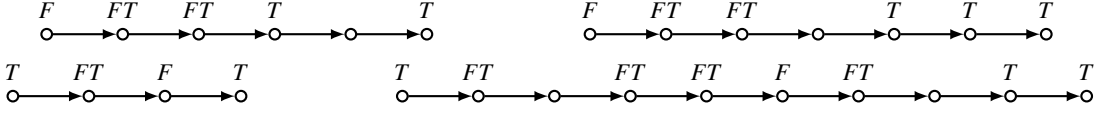
\centerline{
\begin{tikzpicture}[>=latex, rounded corners, scale = 0.5,line width=0.8pt]
\node[point,scale = 0.7,label=above:{\small $F$}] (1) at (0,0) {};
\node[point,scale = 0.7,label=above:{\small $FT$}] (2) at (2,0) {};
\node[point,scale = 0.7,label=above:{\small $FT$}] (3) at (4,0) {};
\node[point,scale = 0.7,label=above:{\small $T$}] (4) at (6,0) {};
\node[point,scale = 0.7,label=above:$ $] (5) at (8,0) {};
\node[point,scale = 0.7,label=above:{\small $T$}] (6) at (10,0) {};
\draw[->,right] (1) to node[below] { }  (2);
\draw[->,right] (2) to node[below] { } (3);
\draw[->,right] (3) to node[below] { } (4);;
\draw[->,right] (4) to node[below] { } (5);
\draw[->,right] (5) to node[below] { } (6);
\end{tikzpicture}\\
\quad 
\begin{tikzpicture}[>=latex, rounded corners, scale = 0.5,line width=0.8pt]
\node[point,scale = 0.7,label=above:{\small $F$}] (1) at (0,0) {};
\node[point,scale = 0.7,label=above:{\small $FT$}] (2) at (2,0) {};
\node[point,scale = 0.7,label=above:{\small $FT$}] (3) at (4,0) {};
\node[point,scale = 0.7,label=above:$ $] (4) at (6,0) {};
\node[point,scale = 0.7,label=above:{\small $T$}] (5) at (8,0) {};
\node[point,scale = 0.7,label=above:{\small $T$}] (6) at (10,0) {};
\node[point,scale = 0.7,label=above:{\small $T$}] (7) at (12,0) {};
\draw[->,right] (1) to node[below] { }  (2);
\draw[->,right] (2) to node[below] { } (3);
\draw[->,right] (3) to node[below] { } (4);;
\draw[->,right] (4) to node[below] { } (5);
\draw[->,right] (5) to node[below] { } (6);
\draw[->,right] (6) to node[below] { } (7);
\end{tikzpicture}
}\\
\centerline{
\begin{tikzpicture}[>=latex, rounded corners, scale = 0.5,line width=0.8pt]
\node[point,scale = 0.7,label=above:{\small $T$}] (1) at (0,0) {};
\node[point,scale = 0.7,label=above:{\small $FT$}] (2) at (2,0) {};
\node[point,scale = 0.7,label=above:{\small $F$}] (3) at (4,0) {};
\node[point,scale = 0.7,label=above:{\small $T$}] (4) at (6,0) {};
\draw[->,right] (1) to node[below] { }  (2);
\draw[->,right] (2) to node[below] { } (3);
\draw[->,right] (3) to node[below] { } (4);
\end{tikzpicture} 
\hspace*{1.5cm}
\begin{tikzpicture}[>=latex, rounded corners, scale = 0.5,line width=0.8pt]
\node[point,scale = 0.7,label=above:{\small $T$}] (05) at (-10,0) {};
\node[point,scale = 0.7,label=above:{\small $FT$}] (04) at (-8,0) {};
\node[point,scale = 0.7] (03) at (-6,0) {};
\node[point,scale = 0.7,label=above:{\small $FT$}] (02) at (-4,0) {};
\node[point,scale = 0.7,label=above:{\small $FT$}] (01) at (-2,0) {};
\node[point,scale = 0.7,label=above:{\small $F$}] (1) at (0,0) {};
\node[point,scale = 0.7,label=above:{\small $FT$}] (2) at (2,0) {};
\node[point,scale = 0.7] (3) at (4,0) {};
\node[point,scale = 0.7,label=above:{\small $T$}] (4) at (6,0) {};
\node[point,scale = 0.7,label=above:{\small $T$}] (5) at (8,0) {};
\draw[->,right] (05) to node[below] { }  (04);
\draw[->,right] (04) to node[below] { } (03);
\draw[->,right] (03) to node[below] { } (02);
\draw[->,right] (02) to node[below] { } (01);
\draw[->,right] (01) to node[below] { }  (1);
\draw[->,right] (1) to node[below] { }  (2);
\draw[->,right] (2) to node[below] { } (3);
\draw[->,right] (3) to node[below] { } (4);;
\draw[->,right] (4) to node[below] { } (5);
\end{tikzpicture}}
%
\end{example}

We next show that answering any d-sirup with twinless path 1-CQs not covered by Theorem~\ref{NL-cond} is \PTime-hard. 
	
\begin{theorem}\label{P-cond}
If $\q$ is an aperiodic twinless path $1$-CQ, then answering both $(\TT, \q)$ and $(\TT^\bot, \q)$ is \PTime-hard.  
\end{theorem}
\begin{proof}
The theorem is proved by an FO-reduction of the monotone circuit evaluation problem, which is known to be \PTime-complete~\cite{Papadimitriou94}. 	We remind the reader that a \emph{monotone Boolean circuit} is a directed acyclic graph $\cir$ whose vertices are called \emph{gates}. Gates with in-degree~$0$ are \emph{input gates}.  Each non-input gate $g$ is either an AND-gate or an OR-gate, and has in-degree~$2$ (with the two edges coming in from gates we call the \emph{inputs of} $g$). 
One of the non-input gates is distinguished as the \emph{output gate}. 	Given an assignment $\alpha$ of $F$ and $T$ to the input gates of $\cir$, we compute the value of each gate in $\cir$ under $\alpha$ as usual in Boolean logic. The \emph{output $\cir(\alpha)$ of $\cir$ on} $\alpha$ is the truth-value of the output gate.
For every monotone Boolean circuit $\cir$ and every assignment $\alpha$, we construct a 
twinless ABox $\A_{\cir,\alpha}$ whose size is polynomial in the sizes of $\q$ and $\cir$, and then show that $\cir(\alpha) = T$ iff $\TT,\A_{\cir,\alpha} \models \q$ iff  $\dis_A^\bot,\A_{\cir,\alpha} \models \q$ (cf.\ \eqref{twinless}).

We prove the theorem for aperiodic $1$-CQs with a single $F$-node 
(the other case is symmetric). Suppose $\q=\q_{lr}$ for some $l,r$ with $l+r\geq 1$. Then there can be three reasons for $\q$ being aperiodic: 
$(i)$ $l=0$ and $\q$ is not right-periodic, 
$(ii)$ $r=0$ and $\q$ is not left-periodic,
or $(iii)$ $l, r\ge 1$.
In each of the three cases  $(i)$--$(iii)$, we give a different reduction. 

$(i)$ If $\q=\q_{0r}$ and $\q$ is not right-periodic, then $r\geq 2$.
We let
\[
n=\left\{
\begin{array}{ll}
r, & \mbox{if $\seqr{\boldsymbol{r}_1}=\seqr{\boldsymbol{r}_2}=\dots =\seqr{\boldsymbol{r}_r}$;}\\ 
\min\,\{i\mid 1<i\le r$ and $\seqr{\boldsymbol{r}_i}\ne\seqr{\boldsymbol{r}_1}\}, & \mbox{otherwise.}
\end{array}
\right.
\]
Then $n\geq 2$.
The ABox 	$\A_{\cir,\alpha}$ is built up from isomorphic copies of the following intervals: 
$\boldsymbol{l} = (\be,x_0)$, 
$\boldsymbol{r}_1=(x_0,x_{1})$, 
$\boldsymbol{r}=(x_1,x_{n-1})$, 
$\boldsymbol{s} =(x_{n-1},x_n)$, and 
$\boldsymbol{t} = (x_n,\en)$.
Note that $\boldsymbol{s}$ is nonempty and has no $T$-nodes. On the other hand, $\boldsymbol{l}$ can be empty when $\be=x_0$,
$\boldsymbol{r}$ can be empty when $n=2$,
and $\boldsymbol{t}$ can be empty when $x_n=\en$.\\
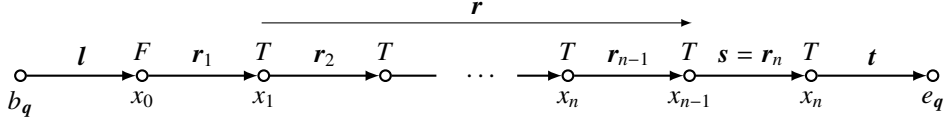
\centerline{ 
			\begin{tikzpicture}[decoration={brace,mirror,amplitude=7},line width=0.8pt,xscale=1.6]
			\node[point,scale=0.7,label=below:{$\be$}] (0) at (-1,0) {};
			\node[point,scale=0.7,label=above:{$F$},label=below:{$x_0$}] (1) at (0,0) {};
			\node[point,scale=0.7,label=above:{$T$},label=below:{$x_1$}] (2) at (1,0) {};
			\node[point,scale=0.7,label=above:{$T$}] (3) at (2,0) {};
			\node (3a) at (2.5,0) {};
			\node (3b) at (2.8,0) {$\dots$};
			\node (3c) at (3,0) {};
			\node[point,scale=0.7,label=above:{$T$},label=below:{$x_{n}$}] (4) at (3.5,0) {};
			\node[point,scale=0.7,label=above:{$T$},label=below:{$x_{n-1}$}] (5) at (4.5,0) {};
			\node[point,scale=0.7,label=above:{$T$},label=below:{$x_n$}] (6) at (5.5,0) {};
			\node[point,scale=0.7,label=below:{$\en$}] (7) at (6.5,0) {};
			\draw[->,right] (0) to node[above] {$\boldsymbol{l}$} (1);
			\draw[->,right] (1) to node[above] {$\boldsymbol{r}_1$} (2);
			\draw[->,right] (2) to node[above] {$\boldsymbol{r}_2$} (3);
			\draw[-,right] (3) to (3a);
			\draw[->,right] (3c) to (4);
			\draw[->,right] (4) to node[above] {$\boldsymbol{r}_{n-1}$} (5);
			\draw[->,right] (5) to node[above] {$\boldsymbol{s}=\boldsymbol{r}_{n}$} (6);
			\draw[->,right] (6) to node[above] {$\boldsymbol{t}$} (7);
			\node (u1) at (.9,.7) {};
			\node (ue) at (4.6,.7) {};
			\draw[->,thin] (u1) to node[above] {$\boldsymbol{r}$} (ue);
			\end{tikzpicture}
		}

We use the gadgets in Fig.~\ref{f:gategadgets1} to simulate the AND- and OR-gates.		
For AND-gates, we distinguish between two cases $|\boldsymbol{s}| > |\boldsymbol{r}_1|$ 
		and $|\boldsymbol{s}| \le |\boldsymbol{r}_1|$, while the gadget for OR-gates is the same in both cases. 
Throughout, in our pictures of ABoxes, lower case letters like $a,b,z,\dots$ are just pointers, not actual labels of nodes. In Fig.~\ref{f:gategadgets1}, if $\boldsymbol{r}=\emptyset$ then $z=a'$ and $z'=b$ are labelled only by $A$. 
		%
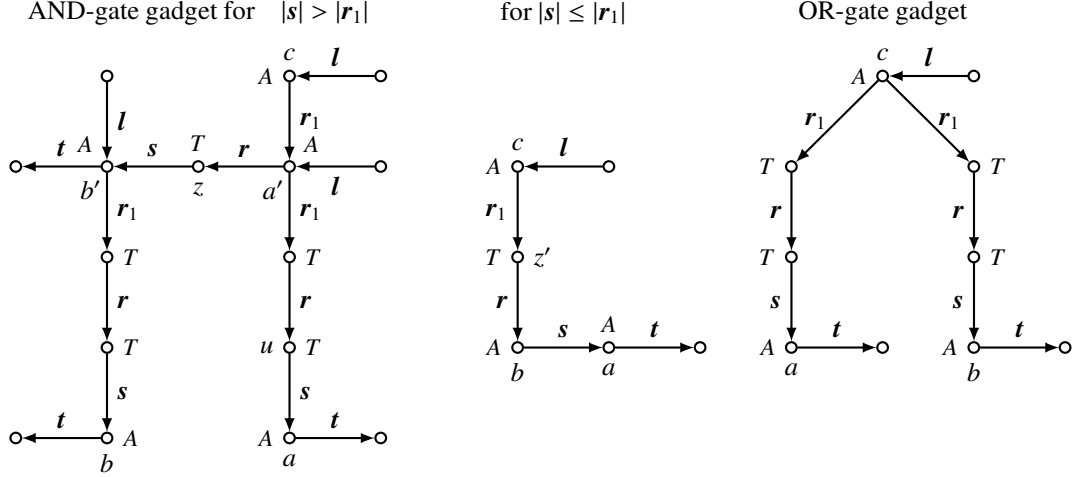
\begin{figure}[th]		
\centering
				\begin{tikzpicture}[decoration={brace,mirror,amplitude=7},line width=0.8pt,scale=1.2]
				\node[] (r0) at (-1,1.7) {AND-gate gadget for\quad $|\boldsymbol{s}| > |\boldsymbol{r}_1|$};
				\node[point,scale=0.7] (b1) at (1,1) {};
				\node[point,scale=0.7] (b2) at (1,0) {};
				\node[point,scale=0.7] (b3) at (-2,1) {};
				\node[point,scale=0.7,label=above:{$c$},label=left:{\small $A$}] (r1) at (0,1) {};
				\node[point,scale=0.7,label=above right:{\small $A$},label=below left:{$a'$\!\!}] (r2) at (0,0) {};
				\node[point,scale=0.7,label=right:{\small $T$}] (r3) at (0,-1) {};
				\node[point,scale=0.7,label=right:{\small $T$},label=left:{$u$}] (r4) at (0,-2) {};
				\node[point,scale=0.7,label=below:{$a$},label=left:{\small $A$}] (r5) at (0,-3) {};
				\node[point,scale=0.7,label=above:{\small $T$},label=below:{$z$}] (c) at (-1,0) {};
				\node[point,scale=0.7,label=above left:{\small $A$},label=below left:{$b'$\!\!}] (l2) at (-2,0) {};
				\node[point,scale=0.7,label=right:{\small $T$}] (l3) at (-2,-1) {};
				\node[point,scale=0.7,label=right:{\small $T$}] (l4) at (-2,-2) {};
				\node[point,scale=0.7,label=below:{$b$},label=right:{\small $A$}] (l5) at (-2,-3) {};
				\node[point,scale=0.7] (z1) at (-3,0) {};
				\node[point,scale=0.7] (z2) at (-3,-3) {};
				\node[point,scale=0.7] (z3) at (1,-3) {};
				\draw[->] (b1) to node[above] {$\boldsymbol{l}$} (r1);
				\draw[->] (b2) to node[below] {$\boldsymbol{l}$} (r2);
				\draw[->] (b3) to node[right] {$\boldsymbol{l}$} (l2);
				\draw[->,right] (r1) to node[right] {$\boldsymbol{r}_1$} (r2);
				\draw[->,right] (r2) to node[right] {$\boldsymbol{r}_1$} (r3);
				\draw[->,right] (r3) to node[right] {$\boldsymbol{r}$} (r4);
				\draw[->,right] (r4) to node[right] {$\boldsymbol{s}$} (r5);
				\draw[->,right] (r5) to node[above] {$\boldsymbol{t}$} (z3);
				\draw[->,right] (l2) to node[right] {$\boldsymbol{r}_1$} (l3);
				\draw[->,right] (l3) to node[right] {$\boldsymbol{r}$} (l4);
				\draw[->,right] (l4) to node[right] {$\boldsymbol{s}$} (l5);
				\draw[->,right] (l5) to node[above] {$\boldsymbol{t}$} (z2);
				\draw[->,right] (r2) to node[above] {$\boldsymbol{r}$} (c);
				\draw[->,right] (c) to node[above] {$\boldsymbol{s}$} (l2);
				\draw[->,right] (l2) to node[above] {$\boldsymbol{t}$} (z1);
				%
				\node[] (rr0) at (3,1.7) {for $|\boldsymbol{s}| \le |\boldsymbol{r}_1|$};
				\node[point,scale=0.7] (l) at (3.5,0) {};
				\node[point,scale=0.7,label=left:{\small $A$},label=above:{$c$}] (1) at (2.5,0) {};
				\node[point,scale=0.7,label=left:{\small $T$},label=right:{$z'$}] (2) at (2.5,-1) {};
				\node[point,scale=0.7,label=left:{\small $A$},label=below:{$b$}] (3) at (2.5,-2) {};
				\node[point,scale=0.7,label=above:{\small $A$}  ,label=below:{$a$}] (4) at (3.5,-2) {};
				\node[point,scale=0.7] (5) at (4.5,-2) {};
				\draw[->,right] (l) to node[above] {$\boldsymbol{l}$} (1);
				\draw[->,right] (1) to node[left] {$\boldsymbol{r}_1$} (2);
				\draw[->,right] (2) to node[left] {$\boldsymbol{r}$} (3);
				\draw[->,right] (3) to node[above] {$\boldsymbol{s}$} (4);
				\draw[->,right] (4) to node[above] {$\boldsymbol{t}$} (5);
				%
				\node[] (orr0) at (6.5,1.7) {OR-gate gadget};
				\node[point,scale=0.7] (ol) at (7.5,1) {};
				\node[point,scale=0.7,label=left:{\small $A$},label=above:{$c$}] (oc) at (6.5,1) {};
				\node[point,scale=0.7,label=left:{\small $T$}] (ol1) at (5.5,0) {};
				\node[point,scale=0.7,label=left:{\small $T$}] (ol2) at (5.5,-1) {};
				\node[point,scale=0.7,label=below:{$a$},label=left:{\small $A$}] (ol3) at (5.5,-2) {};
				\node[point,scale=0.7] (ol4) at (6.5,-2) {};
				\node[point,scale=0.7,label=right:{\small $T$}] (or1) at (7.5,0) {};
				\node[point,scale=0.7,label=right:{\small $T$}] (or2) at (7.5,-1) {};
				\node[point,scale=0.7,label=below:{$b$},label=left:{\small $A$}] (or3) at (7.5,-2) {};
				\node[point,scale=0.7] (or4) at (8.5,-2) {};
				\draw[->,right] (ol) to node[above] {$\boldsymbol{l}$} (oc);
				\draw[->,right] (oc)  to node[left] {$\boldsymbol{r}_1$} (ol1);
				\draw[->,right] (ol1) to node[left] {$\boldsymbol{r}$} (ol2);
				\draw[->,right] (ol2) to node[left] {$\boldsymbol{s}$} (ol3);
				\draw[->,right] (ol3) to node[above] {$\boldsymbol{t}$} (ol4);
				\draw[->,right] (oc)  to node[right] {$\boldsymbol{r}_1$} (or1);
				\draw[->,right] (or1) to node[left] {$\boldsymbol{r}$} (or2);
				\draw[->,right] (or2) to node[left] {$\boldsymbol{s}$} (or3);
				\draw[->,right] (or3) to node[above] {$\boldsymbol{t}$} (or4);
				\end{tikzpicture}
\caption{Gate gadgets in case $(i)$.}\label{f:gategadgets1}
\end{figure}

		Given a monotone circuit $\cir$ and an assignment $\alpha$, we construct $\A_{\cir,\alpha}$ as follows.
		With each non-input gate $g$ we associate a fresh copy of its gadget.
		When the inputs of $g$ are gates $g_a$ and $g_b$ then, for each $i=a,b$, if $g_i$ is a non-input gate,  then we merge node $c$ of the gadget for $g_i$ with node $i$ in the gadget for $g$;
		and if $g_i$ is an input gate, we replace the label $A$ of $i$ and $i'$ (if available) in the gadget for $g$ with $\alpha(g_i)$. Finally, we replace the label $A$ of node $c$ in the gadget for the output gate with $F$. 
		We claim that $\TT,\A_{\cir,\alpha} \models \q$ iff $\cir(\alpha) = T$. 
		
		$(\Leftarrow)$ is proved by induction on the number of non-input gates in $\cir$. The basis is obvious. For the induction step, suppose the output 
		gate $g$ in $\cir$ is an AND-gate with inputs $g_a$ and $g_b$, at least one of which is a non-input gate.  Let $\I$ be an arbitrary model of $\TT$ and $\A_{\cir,\alpha}$. If both $a$ and $b$ in
		the gadget for $g$ are in $T^\I$, then it is easy to check that we always have a $\q\to\I$ homomorphism,
		no matter what the labels of $a'$ and $b'$ (if available) are.
		It remains to consider the case when either $a$ or $b$ is in $F^\I$, and so the corresponding $g_i$ is not an input gate. Take the subcircuit $\cir^-$ of $\cir$ whose output gate is $g_i$. Then $\A_{\cir^-,\alpha}$
		is the sub-ABox of $\A_{\cir,\alpha}$ with node $c$ in the gadget for $g_i$ as its topmost node, and $A(c)$ replaced by $F(c)$. Now, if $\I^-$ is the restriction of $\I$ to $\A_{\cir^-,\alpha}$ (and so $c\in F^{\I^-}$), then by IH there is
		a $\q\to \I^-$ homomorphism, and so $\I \models \q$ as well. 
		%
		%
		The case when the output gate $g$ in $\cir$ is an OR-gate is similar.
		
$(\Rightarrow)$ Suppose $\cir(\alpha) = F$. To show $\TT,\A_{\cir,\alpha} \not\models \q$,  we define a model $\I$ of $\TT$ and $\A_{\cir,\alpha}$ inductively by labelling the $A$-nodes in the gadget for each non-input gate $g$ of $\cir$ by $F$ or $T$ as follows: 
node $c$ is labelled by the the truth-value of $g$ under $\alpha$, while node $i$ (and node $i'$ if applicable), for $i=a,b$, is labelled by the truth-value of $g_i$ under $\alpha$, where  $g_a$ and $g_b$ are the inputs of $g$. 		
%
Suppose, on the contrary, that there is a homomorphism $h \colon \q \to \I$. We exclude all options for the image $h(\q)$ of $\q$. To this end, we track possible locations for $h(x_0) \in F^\I$. 
Let the non-input gate $g$ be such that $h(x_0)$ is  in the gadget for $g$ and the inputs of $g$ are gates
$g_a$ and $g_b$. We may assume that $h(x_0)$ is different from nodes $a$ and $b$, because if $h(x_0)=i$ for $i\in\{a,b\}$
then $g_i$ must be a non-input gate (otherwise there is no room for $h(\q)$ in $\A_{\cir,\alpha}$), and so $h(x_0)=c$ in the gadget for $g_i$.

		%
		Suppose first that  $g$ is an AND-gate and $|\boldsymbol{s}| > |\boldsymbol{r}_1|$.
We have the following cases:
\begin{description}
\item[$a,a' \in T^\I,\, b,b',c \in F^\I${\rm :}] 
 If $h(x_0) = c$, then $h(x_1) = a'$ and, since $b' \in F^\I$, $h(\q)$ cannot continue `horizontally' towards $b'$.
But then, since $|\boldsymbol{s}| > |\boldsymbol{r}_1|$, the node $h(x_n)$ must be strictly between $u$ and $a$ 
		which is impossible because there are no $T$-nodes in $\boldsymbol{s}$. We cannot have $h(x_0) = b'$ because $b \in F^\I$ and there is no room for $h(\q)$ in $\boldsymbol{t}$.
\item[$a,a',c \in F^\I,\, b,b' \in T^\I${\rm :}] 
 If $h(x_0) = a'$ then, since $a \in F^\I$, $h(\q)$ cannot continue `vertically' towards $a$. Then $h(x_1)$ is the central $T$-node. But then, since $|\boldsymbol{s}| > |\boldsymbol{r}_1|$, the 
		node $h(x_{n-1})$ must be strictly between $b'$ and $z$,
		which is impossible because there are no $T$-nodes in $\boldsymbol{s}$. 
%
%
\item[$a,a',b,b',c \in F^\I${\rm :}] 
This case is covered by the previous ones.
\end{description}
Suppose next that $g$ is an AND-gate and $|\boldsymbol{s}| \le |\boldsymbol{r}_1|$.
Then $h(x_0) = c$ and $h(x_{n-1}) = b$, provided that $b \in T^\I$ (otherwise such $h$ is impossible), which means that $a \in F^\I$, and so $h(x_n)$ is located in some other gadget $g'$ whose node $c$ is merged with the current $b=h(x_{n-1})$. However, this is impossible because of the following.
		In every gadget, the `edges' leaving node $c$ are all labelled by $\boldsymbol{r}_1$.
As $x_{n-1}$ is `$\boldsymbol{s}$-connected' to $x_n$ and $h(x_{n-1})=c$ in the gadget for $g'$, 
if $|\boldsymbol{s}| < |\boldsymbol{r}_1|$ then $h(x_n)$ must be strictly between $c$
		and the end-node of an  $\boldsymbol{r}_1$-edge, but there are no $T$-nodes there.
So suppose $|\boldsymbol{s}| = |\boldsymbol{r}_1|$. Then $h(x_n)$ is the end-node of an  $\boldsymbol{r}_1$-edge in the gadget for $g'$, and so
 $\seqr{\boldsymbol{s}}=\seqr{\boldsymbol{r}_1}$. Now
it follows from the definition of $n$ and $\boldsymbol{s}$ that $n=r$ and
$\seqr{\boldsymbol{r}_1}=\dots=\seqr{\boldsymbol{r}_r}=\seqr{\boldsymbol{s}}$.
%
As $h(x_r)=h(x_n)$ is the end-node of an  $\boldsymbol{r}_1$-edge starting at $c$ in the gadget for $g'$, an inspection of the gate-gadgets shows that
$\boldsymbol{r}_{r+1}=(x_r,\en)=\boldsymbol{t}$ must be mapped to a non-empty sequence of $\boldsymbol{r}_1$-intervals followed by $\boldsymbol{t}$ (either in the gadget for $g'$, or in some subsequent gadgets).
So $\seqr{\boldsymbol{r}_{r+1}}$ must be a possibly empty sequence of $\seqr{\boldsymbol{r}_1}$s, possibly followed by a non-empty proper prefix of $\seqr{\boldsymbol{r}_1}$,
contrary to  $\q$ being not right-periodic. 

				
Finally, if $g$ is an OR-gate and $h(x_0) = c$ in the gadget for $g$, then both $a$ and $b$ of the gadget are in $F^I$, and so $h(x_n)\in F^\I$, which is a contradiction. 

The proof of $(ii)$ is a mirror image of the previous one.


$(iii)$ If $\q=\q_{lr}$ and $l, r\ge 1$, then $\A_{\cir,\alpha}$ is built up from isomorphic copies of the following intervals: 
$\boldsymbol{l} =  (\be,x_{-1})$, $\boldsymbol{r} = (x_{-1},x_0)$, $\boldsymbol{s} = (x_0,x_r)$, and $\boldsymbol{t} = (x_r,\en)$. Note that $\boldsymbol{r}$ is not empty and has no $T$-nodes, while $\boldsymbol{l}$ and $\boldsymbol{t}$ may be empty.\\
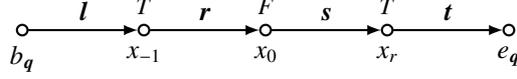
\centerline{ 
\begin{tikzpicture}[decoration={brace,mirror,amplitude=7},line width=0.8pt,xscale=1.6]
			\node[point,scale=0.7,label=below:{$\be$}] (0) at (-1,0) {};
			\node[point,scale=0.7,label=above:{\small $T$},label=below:{$x_{-1}$}] (1) at (0,0) {};
			\node[point,scale=0.7,label=above:{\small $F$},label=below:{$x_0$}] (2) at (1,0) {};
			\node[point,scale=0.7,label=above:{\small $T$},label=below:{$x_r$}] (4) at (2,0) {};
			\node[point,scale=0.7,label=below:{$\en$}] (5) at (3,0) {};
			\draw[->,right] (0) to node[above] {$\boldsymbol{l}$} (1);
			\draw[->,right] (1) to node[above] {$\boldsymbol{r}$} (2);
			\draw[->,right] (2) to node[above] {$\boldsymbol{s}$} (4);
			\draw[->,right] (4) to node[above] {$\boldsymbol{t}$} (5);
\end{tikzpicture}
}\\
We use the gadgets in Fig.~\ref{f:gategadgets2} to simulate the AND- and OR-gates. The number of $A$-nodes in the gadget for a non-output AND-gate exceeds $|\q|+2$.
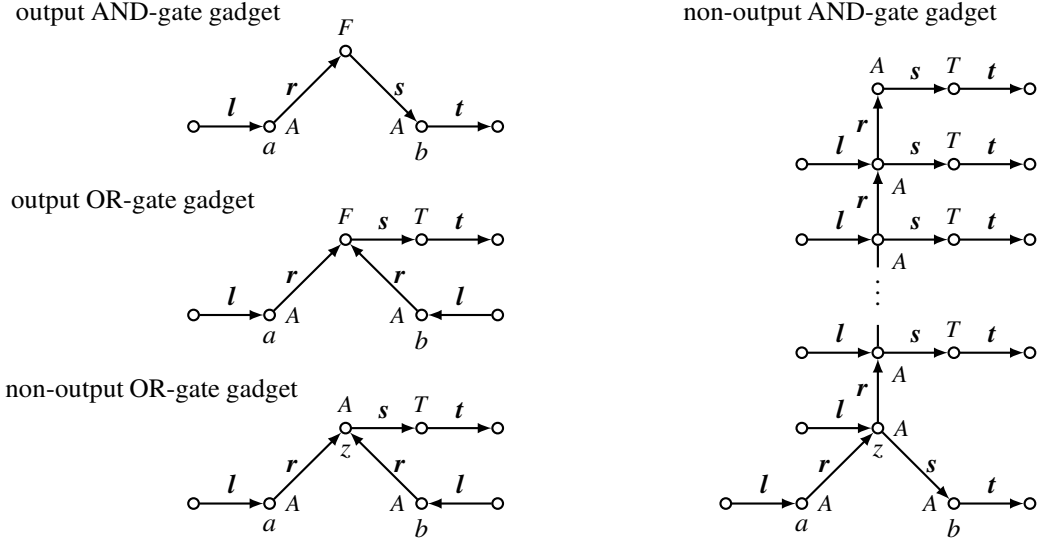
\begin{figure}[ht]
\centering
\begin{tikzpicture}[decoration={brace,mirror,amplitude=7},line width=0.8pt,scale=1]
			\node[] (oand) at (-7.6,7) {output AND-gate gadget};
			\node[point,scale=0.7] (1) at (-7,5.5) {};
			\node[point,scale=0.7,label=below:{$a$},label=right:{\small $A$}] (2) at (-6,5.5) {};
			\node[point,scale=0.7,label=below:{$b$},label=left:{\small $A$}] (3) at (-4,5.5) {};
			\node[point,scale=0.7] (4) at (-3,5.5) {};
			\node[point,scale=0.7,label=above:{\small $F$}] (5) at (-5,6.5) {};
			\draw[->] (1) to node[above] {$\boldsymbol{l}$} (2);
			\draw[->] (2) to node[left] {$\boldsymbol{r}$} (5);
			\draw[->] (5) to node[right] {$\boldsymbol{s}$} (3);
			\draw[->] (3) to node[above] {$\boldsymbol{t}$} (4);
			\node[] (oor) at (-7.8,4.5) {output OR-gate gadget};
			\node[point,scale=0.7] (1) at (-7,3) {};
			\node[point,scale=0.7,label=below:{$a$},label=right:{\small $A$}] (2) at (-6,3) {};
			\node[point,scale=0.7,label=below:{$b$},label=left:{\small $A$}] (3) at (-4,3) {};
			\node[point,scale=0.7] (4) at (-3,3) {};
			\node[point,scale=0.7,label=above:{\small $F$}] (5) at (-5,4) {};
			\node[point,scale=0.7,label=above:{\small $T$}] (6) at (-4,4) {};
			\node[point,scale=0.7] (7) at (-3,4) {};
			\draw[->] (1) to node[above] {$\boldsymbol{l}$} (2);
			\draw[->] (2) to node[left] {$\boldsymbol{r}$} (5);
			\draw[->] (3) to node[right] {$\boldsymbol{r}$} (5);
			\draw[->] (4) to node[above] {$\boldsymbol{l}$} (3);
			\draw[->] (5) to node[above] {$\boldsymbol{s}$} (6);
			\draw[->] (6) to node[above] {$\boldsymbol{t}$} (7);
			\node[left] (noor) at (-5.5,2) {non-output OR-gate gadget};
			\node[point,scale=0.7] (1) at (-7,.5) {};
			\node[point,scale=0.7,label=below:{$a$},label=right:{\small $A$}] (2) at (-6,.5) {};
			\node[point,scale=0.7,label=below:{$b$},label=left:{\small $A$}] (3) at (-4,.5) {};
			\node[point,scale=0.7] (4) at (-3,.5) {};
			\node[point,scale=0.7,label=below:{$z$},label=above:{\small $A$}] (5) at (-5,1.5) {};
			\node[point,scale=0.7,label=above:{\small $T$}] (6) at (-4,1.5) {};
			\node[point,scale=0.7] (7) at (-3,1.5) {};
			\draw[->] (1) to node[above] {$\boldsymbol{l}$} (2);
			\draw[->] (2) to node[left] {$\boldsymbol{r}$} (5);
			\draw[->] (3) to node[right] {$\boldsymbol{r}$} (5);
			\draw[->] (4) to node[above] {$\boldsymbol{l}$} (3);
			\draw[->] (5) to node[above] {$\boldsymbol{s}$} (6);
			\draw[->] (6) to node[above] {$\boldsymbol{t}$} (7);
			\node[] (orr0) at (1.5,7) {non-output AND-gate gadget};
			\node[point,scale=0.7] (1) at (0,.5) {};
			\node[point,scale=0.7,label=below:{$a$},label=right:{\small $A$}] (2) at (1,.5) {};
			\node[point,scale=0.7,label=below:{$b$},label=left:{\small $A$}] (3) at (3,.5) {};
			\node[point,scale=0.7] (4) at (4,.5) {};
			\node[point,scale=0.7] (5) at (1,1.5) {};
			\node[point,scale=0.7,label=below:{$z$},label=right:{\small $A$}] (6) at (2,1.5) {};
			\node[point,scale=0.7] (7) at (1,2.5) {};
			\node[point,scale=0.7,label=below right:{\small $A$}] (8) at (2,2.5) {};
			\node[point,scale=0.7,label=above:{\small $T$}] (9) at (3,2.5) {};
			\node[point,scale=0.7] (10) at (4,2.5) {};
			\node (11) at (2,3) {};
			\node (12) at (2,3.5) {};
			\node[point,scale=0.7] (13) at (1,4) {};
			\node[point,scale=0.7,label=below right:{\small $A$}] (14) at (2,4) {};
			\node[point,scale=0.7,label=above:{\small $T$}] (15) at (3,4) {};
			\node[point,scale=0.7] (16) at (4,4) {};
			\node[point,scale=0.7] (17) at (1,5) {};
			\node[point,scale=0.7,label=below right:{\small $A$}] (18) at (2,5) {};
			\node[point,scale=0.7,label=above:{\small $T$}] (19) at (3,5) {};
			\node[point,scale=0.7] (20) at (4,5) {};
			\node[point,scale=0.7,label=above:{\small $A$}] (21) at (2,6) {};
			\node[point,scale=0.7,label=above:{\small $T$}] (22) at (3,6) {};
			\node[point,scale=0.7] (23) at (4,6) {};
			\draw[->] (1) to node[above] {$\boldsymbol{l}$} (2);
			\draw[->] (2) to node[left] {$\boldsymbol{r}$} (6);
			\draw[->] (6) to node[right] {$\boldsymbol{s}$} (3);
			\draw[->] (3) to node[above] {$\boldsymbol{t}$} (4);
			\draw[->] (5) to node[above] {$\boldsymbol{l}$} (6);
			\draw[->] (6) to node[left] {$\boldsymbol{r}$} (8);
			\draw[->] (7) to node[above] {$\boldsymbol{l}$} (8);
			\draw[->] (8) to node[above] {$\boldsymbol{s}$} (9);
			\draw[->] (9) to node[above] {$\boldsymbol{t}$} (10);
			\draw[-] (8) to (11);
			\draw[-] (12) to (14);
			\draw[->] (13) to node[above] {$\boldsymbol{l}$} (14);
			\draw[->] (14) to node[above] {$\boldsymbol{s}$} (15);
			\draw[->] (15) to node[above] {$\boldsymbol{t}$} (16);
			\draw[->] (14) to node[left] {$\boldsymbol{r}$} (18);
			\draw[->] (17) to node[above] {$\boldsymbol{l}$} (18);
			\draw[->] (18) to node[above] {$\boldsymbol{s}$} (19);
			\draw[->] (19) to node[above] {$\boldsymbol{t}$} (20);
			\draw[->] (18) to node[left] {$\boldsymbol{r}$} (21);
			\draw[->] (21) to node[above] {$\boldsymbol{s}$} (22);
			\draw[->] (22) to node[above] {$\boldsymbol{t}$} (23);
			\node at (2,3.4) {$\vdots$};
\end{tikzpicture}
\caption{Gate gadgets in case $(iii)$.}\label{f:gategadgets2}
\end{figure}

Given a monotone circuit $\cir$ and an assignment $\alpha$, we construct $\A_{\cir,\alpha}$ as follows.
		With each non-input gate $g$ we associate a fresh copy of its gadget.
		When the inputs of $g$ are gates $g_a$ and $g_b$ then, for each $i=a,b$, if $g_i$ is a non-input gate,  then we merge the topmost $A$-node of the gadget for $g_i$ with node $i$ in the gadget for $g$;
		and if $g_i$ is an input gate, we replace the label $A$ of $i$ in the gadget for $g$ with $\alpha(g_i)$. We claim that $\TT,\A_{\cir,\alpha} \models \q$ iff $\cir(\alpha) = T$.

		$(\Leftarrow)$ is proved by induction on the number of non-input gates in $\cir$. The basis (when $\cir$ has one non-input gate) is obvious. For the induction step, suppose the output 
		gate $g$ in $\cir$ is an OR-gate with inputs $g_a$ and $g_b$, at least one of which is a non-input gate.  Let $\I$ be an arbitrary model of $\TT$ and $\A_{\cir,\alpha}$. 
		If at least one of $a$ or $b$ in the gadget for $g$ is in $T^\I$, then clearly $\I\models \q$. 
		It remains to consider the case when $a$ and $b$ are both in $F^\I$. Let $i$ be such that $g_i$ is a non-input gate. There are two cases. $(a)$ 
		If node $z$ in the gadget for $g_i$ is in $F^\I$, 
		consider the subcircuit $\cir^-$ of $\cir$ whose output gate is $g_i$. Then $\A_{\cir^-,\alpha}$
		is the sub-ABox of $\A_{\cir,\alpha}$ with $z$ as its topmost node, and $A(z)$ replaced by $F(z)$.
		Now, if $\I^-$ is the restriction of $\I$ to $\A_{\cir^-,\alpha}$ (and so $z\in F^{\I^-}$) then, by IH, there is
		a $\q\to \I^-$ homomorphism, and so $\I \models \q$ as well.
		$(b)$ If $z\in T^\I$ then $g_i$ is an AND-gate and, as the topmost $A$-node in the gadget for $g_i$ is in $F^\I$, there is an $A$-node in the gadget for $g_i$
		that is in $T^\I$ while the next $A$-node above it is in $F^I$. So we have a $\q\to \I$ homomorphism.
The case when the output gate of $\cir$ is an AND-gate is similar. 
		
		
$(\Rightarrow)$ Suppose $\cir(\alpha) = F$. To show $\TT,\A_{\cir,\alpha} \not\models \q$, we define a model $\I$ of $\TT$ and $\A_{\cir,\alpha}$ by putting the $A$-nodes of the gadget for any gate $g$ in $\cir$ to $F^\I$ (or $T^\I$) if the truth-value of $g$ under $\alpha$ is $F$ (or, respectively, $T$). Suppose, on the contrary, that there is a homomorphism $h \colon \q \to \I$. We track the possible locations of $h(x_0) \in F^\I$:
\begin{itemize}
\item[--] If the output gate is an AND-gate, then $h(x_0)$ cannot be the $F$-node of its gadget because then
			$h(x_{-1})=a$ and $h(x_r)=b$, and so at least one of them would be in $F^I$, which is a contradiction.
			
\item[--] If the output gate is an OR-gate, then $h(x_0)$ cannot be the $F$-node of its gadget because then
			either $h(x_{-1})=a$  or $h(x_{-1})=b$, and so $h(x_{-1})$ would be in $F^I$, a contradiction. 
			
\item[--]
			So suppose $h(x_0)$ is an $A$-node in a gadget for a non-input and non-output gate $g$.
			If $g$ is an OR-gate, then either $h(x_{-1})=a$ or $h(x_{-1})=b$ in the gadget for $g$, and so $h(x_{-1})$ would be in $F^I$, a contradiction.
			So suppose $g$ is an AND-gate, and consider the gadget for $g$.
			Then $h(x_0)$ cannot be any $A$-node located above $z$, because otherwise $h(x_{-1})$ would be the previous $A$-node, and so in $F^\I$, a contradiction.
			Finally, if $h(x_0)=z$  then,
			as the vertical line comprised of the $\boldsymbol{r}$ is longer than $\q$ and contains no $T$-nodes, $h(x_1)\in T^\I$ must also be in the gadget for $g$, and it must be in one of the horizontal $\boldsymbol{s}$. But this is impossible because $\boldsymbol{r}$ is non-empty, and so the distance between $z=h(x_0)$ and $h(x_1)$ in the gadget would be greater than the distance between $x_0$ and $x_1$ in $\q$. 
\end{itemize}
Thus, we cannot have a homomorphism $h \colon \q \to \I$.
\end{proof}

The proof above bears some superficial similarities to the construction of Afrati and Papadimitriou~\cite{DBLP:journals/jacm/AfratiP93} in their classification of binary chain sirups. One could also draw some parallels with the proof of P-hardness for OMQs with an $\mathcal{EL}$ ontology given by Lutz and Sabellek~\cite{DBLP:conf/ijcai/LutzS17,DBLP:journals/corr/abs-1904-12533}, who used a reduction of path systems accessibility (PSA) rather than monotone circuit evaluation.

The most difficult part of our tetrachotomy is proving \coNP-hardness of dd-sirups with 
path 2-CQs. Despite the abundance of results on algorithmic aspects of graph homomorphisms~\cite{DBLP:journals/csr/HellN08}, we failed to find any known technique applicable to our case. 
%
%
In the remainder of the article, we develop a new method for establishing \coNP-hardness of disjunctive OMQs.
		

\section{Proving \coNP-hardness: the bike technique}\label{monster}

\begin{theorem}\label{t:coNPhard}
If $\q$ is a twinless path $2$-CQ, then answering both $(\TT,\q)$ and $(\dis^\bot_A,\q)$ is \coNP-hard. 
\end{theorem}
	
We prove Theorem~\ref{t:coNPhard} by a polynomial reduction of the complement of \NP-complete 3SAT~\cite{Papadimitriou94}. 
Recall that a \emph{3CNF} is a conjunction of 
\emph{clauses} of the form $\lit_1 \lor \lit_2 \lor \lit_3$, where each $\lit_i$ is a \emph{literal}
	(a propositional variable or a negation thereof).
The decision problem 3SAT asks whether a given 3CNF $\psi$ is satisfiable. 
For any 3CNF $\psi$, we construct
	a twinless  	
	 ABox $\Apsi$ whose size is polynomial in the sizes of $\q$ and $\psi$,
	and show that $\psi$ is satisfiable iff $\TT,\Apsi\not\models\q$  iff $\dis_A^\bot,\Apsi\not\models\q$
	(cf.\ \eqref{twinless}). 
The construction, called the \emph{bike technique}, builds $\Apsi$ from many copies of $\q$ via three major steps:
\begin{enumerate}
\item
First, we represent the truth-values of the \emph{literals} in $\psi$ by gadgets called \wheels. 

\item
Next, we connect \wheels{} to represent \emph{negation} properly by gadgets called \bikes.

\item
Finally, we connect \bikes{} to represent the interaction of the \emph{clauses} in $\psi$ and obtain $\Apsi$.
\end{enumerate}
These steps will be defined and investigated in detail in Sections~\ref{s:wheels}--\ref{s:clauses}. But before that we explain the underlying ideas and illustrate them by an example.
Each \wheel{} $\mathcal{W}$ in Step~1 has many $A$-nodes (the number depends on $|\q|$ and the number of clauses in $\psi$) where the different copies of $\q$ meet. Each $\mathcal{W}$ is such that, for every model $\I$ of $\dis_A$ and $\mathcal{W}$, we have 
$\I \not\models \q$ iff the $A$-nodes in $\mathcal{W}$ are all in $T^\I$ or are all in $F^\I$. If a variable $p$ occurs in $\psi$, then in Step~2 a \bike{} $\mathcal{B}$, representing the pair
$\{p,\neg p\}$ of literals, is assembled from two disjoint \wheels{} by connecting them via $A$-nodes using two further copies of $\q$. We have pairwise disjoint \bikes{} for all variables occurring in $\psi$. Each \bike{} $\mathcal{B}$ is constructed 
%
in such a way that, for every model $\I$ of $\dis_A$ and $\mathcal{B}$, we have 
$\I \not\models \q$ iff the truth-values in $\I$ represented by the two \wheels{} of $\mathcal{B}$ are opposites of each other. 
Finally, in Step~3, for each clause $c=\lit_1 \lor \lit_2 \lor \lit_3$ in $\psi$, we use a further copy $\q^c$ of $\q$ to connect,
via three $A$-nodes, three \wheels{} from the \bikes{} representing $\{\lit_1,\neg\lit_1\}$, $\{\lit_2,\neg\lit_2\}$ and 
$\{\lit_3,\neg\lit_3\}$ in such a way that for every model $\I$ of $\dis_A$ and the resulting ABox $\Apsi$, we have $\I \not\models \q$ iff the labels of the three `$c$-connection' $A$-nodes in $\I$ define an assignment satisfying $c$.
If in Step~1 we choose the number of $A$-nodes in the \wheels{} to be large enough, then in each \wheel{} we can use different `$c$-connection' $A$-nodes for different clauses, and they can also be different from those $A$-nodes that are used for constructing the \bikes{} from the \wheels.

\begin{texample}\label{ex:TTFF}\em
Consider the d-sirup $(\TT,\q)$ with the 2-CQ $\q$ shown in the picture below (with $R$ on edges omitted).\\ 
\centerline{
\begin{tikzpicture}[>=latex,line width=1pt,rounded corners]
\node[point,scale = 0.7,label=above:{\small $T$}] (1) at (0,0) {};
\node[point,scale = 0.7,label=above:{\small $T$}] (m) at (1.5,0) {};
\node[point,scale = 0.7,label=above:{\small $F$}] (2) at (3,0) {};
\node[point,scale = 0.7,label=above:{\small $F$}] (3) at (4.5,0) {};
\draw[->,right] (1) to node[below] {}  (m);
\draw[->,right] (m) to node[below] {} (2);
\draw[->,right] (2) to node[below] {} (3);
\end{tikzpicture}}\\
Let $\psi = c_1 \land c_2 \land c_3$, where 
$c_1=\neg p\lor q\lor \neg r$,
$c_2=p\lor q\lor\neg r$, and
$c_3=p\lor\neg q\lor r$.
Fig.~\ref{f:gadgetex} shows the steps of the construction of $\Apsi$. 
In Step~1, we construct six \wheels{}, each from four copies of $\q$, representing one of
$p,\neg p, q,\neg q, r,\neg r$. Then in Step~2 we construct three \bikes, representing the pairs $\{p,\neg p\}$,
$\{q,\neg q\}$ and $\{r,\neg r\}$. Finally, we connect the \bikes{} in Step~3 to obtain $\Apsi$.
Given an assignment $\mathfrak a\colon\{p,q,r\}\to\{T,F\}$, we define a model $\Ia$ of $\TT$ and $\Apsi$  as follows: 
for $v\in\{p,q,r\}$, if $\mathfrak a(v)=T$ then the $A$-nodes in the $v$-\wheel{} are in $T^{\Ia}$ and in the $\neg v$-\wheel{} are in $F^{\Ia}$; and 
 if $\mathfrak a(v)=F$ then the $A$-nodes in the $v$-\wheel{} are in $F^{\Ia}$ and in the $\neg v$-\wheel{} are in $T^{\Ia}$. 
It is tedious but not hard to
check that $\Ia\not\models \q$ iff $\mathfrak a$ satisfies $\psi$. 
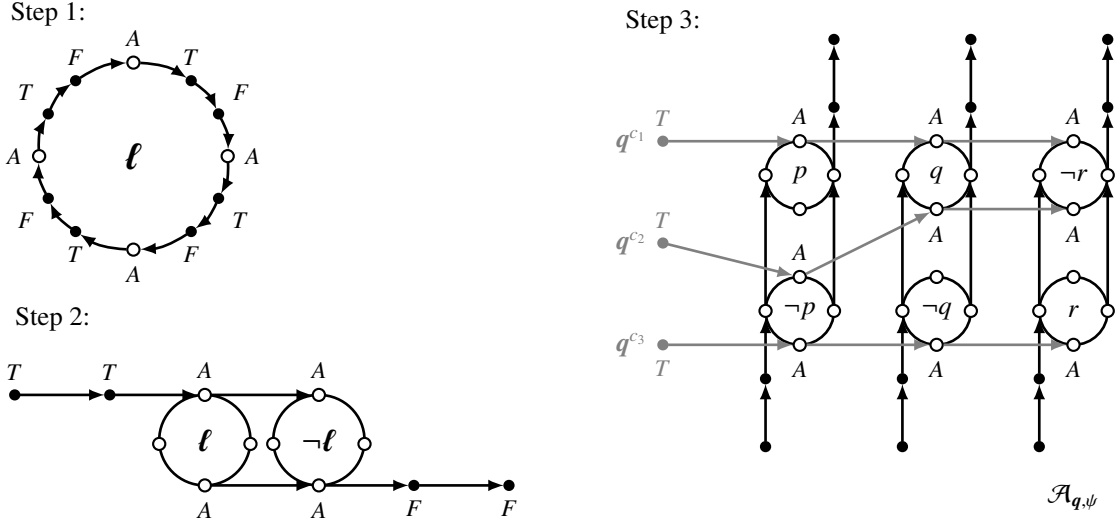
\begin{figure}[ht]
\hspace*{1cm}
\begin{tikzpicture}[>=latex,line width=1pt,rounded corners,scale=.4]
\node[point,scale = 0.8,label = above:{\small $A$}] (1) at (0,3.1) {};
\node[point,scale = 0.6,fill=black,label = above:{\small $T$}] (2) at (1.9,2.5) {};
\node[point,scale = 0.6,fill=black,label = above right:{\small $F$}] (3) at (2.8,1.4) {};
\node[point,scale = 0.8,label = right:{\small $A$}] (4) at (3.1,0) {};
\node[point,scale = 0.6,fill=black,label = below right:{\small $T$}] (5) at (2.8,-1.4) {};
\node[point,scale = 0.6,fill=black,label = below:{\small $F$}] (6) at (1.9,-2.5) {};
\node[point,scale = 0.8,label = below:{\small $A$}] (7) at (0,-3.1) {};
\node[point,scale = 0.6,fill=black,label = below:{\small $T$}] (8) at (-1.9,-2.5) {};
\node[point,scale = 0.6,fill=black,label = below left:{\small $F$}] (9) at (-2.8,-1.4) {};
\node[point,scale = 0.8,label = left:{\small $A$}] (10) at (-3.1,0) {};
\node[point,scale = 0.6,fill=black,label = above left:{\small $T$}] (11) at (-2.8,1.4) {};
\node[point,scale = 0.6,fill=black,label = above:{\small $F$}] (12) at (-1.9,2.5)  {};

\draw[->,right,bend left = 15] (1) to node[below] {} (2);
\draw[->,right,bend left = 15] (2) to node[below left] {} (3);
\draw[->,right,bend left = 15] (3) to node[below left] {} (4);
\draw[->,right,bend left = 15] (4) to node[above left] {} (5);
\draw[->,right,bend left = 15] (5) to node[above left] {} (6);
\draw[->,right,bend left = 15] (6) to node[above] {} (7);
\draw[->,right,bend left = 15] (7) to node[below] {} (8);
\draw[->,right,bend left = 15] (8) to node[below] {} (9);
\draw[->,right,bend left = 15] (9) to node[below] {} (10);
\draw[->,right,bend left = 15] (10) to node[below right] {} (11);
\draw[->,right,bend left = 15] (11) to node[below right] {} (12);
\draw[->,right,bend left = 15] (12) to node[below ] {} (1);

\node[] (0) at (0,0) {\Large $\lit$};
\node[] (s) at (-2.8,4.7) {Step~1:};
\end{tikzpicture}\\[2pt]
\hspace*{1cm}
\begin{tikzpicture}[>=latex,line width=1pt,rounded corners,scale=.5]
\draw (8,-1.2) circle [radius=1.2];
\draw (11,-1.2) circle [radius=1.2];

\node[point,scale = 0.6,fill=black,label = above:{\small $T$}] (2) at (3,0) {};
\node[point,scale = 0.6,fill=black,label = above:{\small $T$}] (3) at (5.5,0) {};
\node[point,scale = 0.8,fill=white,label = above:{\small $A$}] (4) at (8,0) {};
\node[point,scale = 0.8,fill=white,label = above:{\small $A$}] (5) at (11,0) {};

\node[point,scale = 0.8,fill=white,label = below:{\small $A$}] (22) at (8,-2.4) {};
\node[point,scale = 0.8,fill=white,label = below:{\small $A$}] (23) at (11,-2.4) {};
\node[point,scale = 0.6,fill=black,label = below:{\small $F$}] (24) at (13.5,-2.4) {};
\node[point,scale = 0.6,fill=black,label = below:{\small $F$}] (25) at (16,-2.4) {};

\node[point,scale = 0.8,fill=white] (m1) at (6.8,-1.3) {};
\node[point,scale = 0.8,fill=white] (m2) at (9.2,-1.3) {};
\node[point,scale = 0.8,fill=white] (m3) at (9.8,-1.3) {};
\node[point,scale = 0.8,fill=white] (m4) at (12.2,-1.3) {};

\draw[->] (2) to (3);
\draw[->] (3) to (4);
\draw[->] (4) to (5);

\draw[->] (22) to (23);
\draw[->] (23) to (24);
\draw[->] (24) to (25);

\node[] (01) at (8,-1.2) {\large $\lit$};
\node[] (02) at (11,-1.2) {\large $\neg \lit$};
\node[] (s) at (4,2) {Step~2:};
\end{tikzpicture}

\vspace*{-7cm}
\hspace*{9cm}
\begin{tikzpicture}[>=latex,line width=1pt,rounded corners,scale=.45]
\node[] (s) at (0,12.5) {Step~3:};
\node[] (a) at (12,-1.5) {$\Apsi$};
\draw (4,4) circle [radius=1];
\node[] (np) at (4,4) {$\neg p$};
\draw (4,8) circle [radius=1];
\node[] (p) at (4,8) {$p$};
\node[point,scale = 0.8,fill=white] (pd) at (4,7) {};
\node[point,scale = 0.6,fill=black] (l1) at (3,0) {};
\node[point,scale = 0.6,fill=black] (l2) at (3,2) {};
\node[point,scale = 0.8,fill=white] (l3) at (3,4) {};
\node[point,scale = 0.8,fill=white] (l4) at (3,8) {};
\node[point,scale = 0.8,fill=white] (l5) at (5,4) {};
\node[point,scale = 0.8,fill=white] (l6) at (5,8) {};
\node[point,scale = 0.6,fill=black] (l7) at (5,10) {};
\node[point,scale = 0.6,fill=black] (l8) at (5,12) {};
\draw[->] (l1) to (l2);
\draw[->] (l2) to (l3);
\draw[->] (l3) to (l4);
\draw[->] (l5) to (l6);
\draw[->] (l6) to (l7);
\draw[->] (l7) to (l8);

\draw (8,4) circle [radius=1];
\node[] (nq) at (8,4) {$\neg q$};
\node[point,scale = 0.8,fill=white] (nqu) at (8,5) {};
\draw (8,8) circle [radius=1];
\node[] (q) at (8,8) {$q$};
\node[point,scale = 0.6,fill=black] (m1) at (7,0) {};
\node[point,scale = 0.6,fill=black] (m2) at (7,2) {};
\node[point,scale = 0.8,fill=white] (m3) at (7,4) {};
\node[point,scale = 0.8,fill=white] (m4) at (7,8) {};
\node[point,scale = 0.8,fill=white] (m5) at (9,4) {};
\node[point,scale = 0.8,fill=white] (m6) at (9,8) {};
\node[point,scale = 0.6,fill=black] (m7) at (9,10) {};
\node[point,scale = 0.6,fill=black] (m8) at (9,12) {};
\draw[->] (m1) to (m2);
\draw[->] (m2) to (m3);
\draw[->] (m3) to (m4);
\draw[->] (m5) to (m6);
\draw[->] (m6) to (m7);
\draw[->] (m7) to (m8);

\draw (12,4) circle [radius=1];
\node[] (r) at (12,4) {$r$};
\node[point,scale = 0.8,fill=white] (ru) at (12,5) {};
\draw (12,8) circle [radius=1];
\node[] (nr) at (12,8) {$\neg r$};
\node[point,scale = 0.6,fill=black] (r1) at (11,0) {};
\node[point,scale = 0.6,fill=black] (r2) at (11,2) {};
\node[point,scale = 0.8,fill=white] (r3) at (11,4) {};
\node[point,scale = 0.8,fill=white] (r4) at (11,8) {};
\node[point,scale = 0.8,fill=white] (r5) at (13,4) {};
\node[point,scale = 0.8,fill=white] (r6) at (13,8) {};
\node[point,scale = 0.6,fill=black] (r7) at (13,10) {};
\node[point,scale = 0.6,fill=black] (r8) at (13,12) {};
\draw[->] (r1) to (r2);
\draw[->] (r2) to (r3);
\draw[->] (r3) to (r4);
\draw[->] (r5) to (r6);
\draw[->] (r6) to (r7);
\draw[->] (r7) to (r8);

\node[point,gray,scale = 0.6,fill=gray,label = above:{\small \textcolor{gray}{$T$}},label = left:{\textcolor{gray}{$\q^{c_1}$}}] (c11) at (0,9) {};
\node[point,scale = 0.8,fill=white,label = above:{\small $A$}] (c12) at (4,9) {};
\node[point,scale = 0.8,fill=white,label = above:{\small $A$}] (c13) at (8,9) {};
\node[point,scale = 0.8,fill=white,label = above:{\small $A$}] (c14) at (12,9) {};
\draw[->,gray] (c11) to (c12);
\draw[->,gray] (c12) to (c13);
\draw[->,gray] (c13) to (c14);

\node[point,gray,scale = 0.6,fill=gray,label = above:{\small \textcolor{gray}{$T$}},label = left:{\textcolor{gray}{$\q^{c_2}$}}] (c21) at (0,6) {};
\node[point,scale = 0.8,fill=white,label = above:{\small $A$}] (c22) at (4,5) {};
\node[point,scale = 0.8,fill=white,label = below:{\small $A$}] (c23) at (8,7) {};
\node[point,scale = 0.8,fill=white,label = below:{\small $A$}] (c24) at (12,7) {};
\draw[->,gray] (c21) to (c22);
\draw[->,gray] (c22) to (c23);
\draw[->,gray] (c23) to (c24);

\node[point,gray,scale = 0.6,fill=gray,label = below:{\small \textcolor{gray}{$T$}},label = left:{\textcolor{gray}{$\q^{c_3}$}}] (c31) at (0,3) {};
\node[point,scale = 0.8,fill=white,label = below:{\small $A$}] (c32) at (4,3) {};
\node[point,scale = 0.8,fill=white,label = below:{\small $A$}] (c33) at (8,3) {};
\node[point,scale = 0.8,fill=white,label = below:{\small $A$}] (c34) at (12,3) {};
\draw[->,gray] (c31) to (c32);
\draw[->,gray] (c32) to (c33);
\draw[->,gray] (c33) to (c34);
\end{tikzpicture}
%
\caption{An example of the bike-technique.}\label{f:gadgetex}
\end{figure}

%
%
%
\end{texample}

It is far from obvious what exactly are the particular properties of this construction that can be generalised to \emph{arbitrary} twinless path 2-CQs (just consider some permutations of the $F$- and $T$-nodes in $\q$ above). 
On the one hand, it is easy to identify what is needed for the `if $\dis_A,\Apsi \not\models \q$ then $\psi$ is satisfiable' direction to hold. However, the main obstacle in proving the converse implication is that, given a model $\I$ determined by an assignment satisfying $\psi$, we need to exclude \emph{all} $\q\to\I$ homomorphisms, not just those that map $\q$ onto one of its copies
in $\Apsi$.  At each of the three steps, there can be such `parasite' $\q \to \I$ homomorphisms, and there is no single, universal way of correctly assembling the $\Apsi$ for all $\q$ and $\psi$. In the remainder of the article, we show how to overcome this obstacle. 

We fix some twinless path $2$-CQ $\q$ and use the following notation. 
For any $k$, we let $\tkth$ ($\fkth$) denote the $k$th $T$-node ($F$-node) in $\q$.
In particular, $\tfirst$, $\tlbo$, and $\tlast$ denote, respectively, the first, the last but one, and the last $T$-node in $\q$.
Given any path CQ $\q'$, we write $\prec_{\q'}$ and $\preceq_{\q'}$  for the ordering of nodes in $\q'$,
and $\delta_{\q'}$ for the distance in $\q'$,
that is, $\delta_{\q'}(x,y)$ is the number of edges in the path from $x$ to $y$ whenever $x \preceq_{\q'} y$.  
As before, we omit the subscripts when $\q'=\q$, and set $|\q| = \delta(\be,\en)$,
for the first (root) node $\be$ and the last (leaf) node $\en$ in $\q$. 

Throughout, when proving statements of the form $\TT,\A\not\models\q$ for some ABox $\A$,
we use a generalisation of homomorphisms, which 
allows us to regard our CQs as if they contained a \emph{single} binary predicate only. 
Given a model $\I$ of an ABox $\A$, we call a map $h\colon\q\to\I$ a \emph{\shomo} if the following conditions hold:
\begin{itemize}
\item[--] $h(x)\in T^\I$, for every $T$-node $x$ in $\q$, and $h(x)\in F^\I$, for every $F$-node $x$ in $\q$; 

\item[--]
for any nodes $x,y$ in $\q$,
if $R(x,y)$ is in $\q$ for some $R$, then $S\bigl(h(x),h(y)\bigr)$ is in $\A$ for some $S$.	
\end{itemize}
%
%
The ABoxes $\A$ we build from copies of $\q$ will contain cycles, but these cycles will be large compared to $|\q|$.
Thus, for any \shomo{} $h$ mapping $\q$ to some model $\I$ of $\A$, $h(\q)$ can always be regarded as a path CQ,
and we have the following obvious `$h$-\emph{shift\/}' property:
\begin{equation}\label{hshift}
\mbox{$\delta(y,z)=\delta_{h(\q)}\bigl(h(y),h(z)\bigr)$, for all nodes $y$ and $z$ in $\q$.}
\end{equation}

	
\subsection{Representing the truth-values of literals by cogwheels}\label{s:wheels}

For $n \ge |\q|$, we take $n$ disjoint copies $\q^1,\dots,\q^n$ of $\q$. 
For any $j$, $1\le j \le n$, and any node $x$ in $\q$, let $x^j$ denote the copy of $x$ in $\q^j$.  
For each $j$, we pick a $T$-node $\ctt{j}$
and an $F$-node $\cff{j}$ in $\q^j$, calling the selected nodes \emph{contacts\/}. 
We replace the $T$- and $F$-labels of all the contacts with $A$,
and then glue $\cff{j}$ together with $\ctt{j+1}$ for every $j$, $1\le j\le n$, with $\pm$ being understood throughout modulo $n$. We call 
the resulting ABox $\mathcal{W}$ an $n$-\emph{\wheel{} \textup{(}for $\q$\textup{)}}; 
see Fig.~\ref{f:wheel}.
Given two contacts $\cc_1=\cff{i}=\ctt{i+1}$ and $\cc_2=\cff{j}=\ctt{j+1}$, we define the \emph{\cdist{} between $\cc_1$ and $\cc_2$ in $\mathcal {W}$} as $\min\bigl(|i-j|,n-|i-j|\bigr)$.

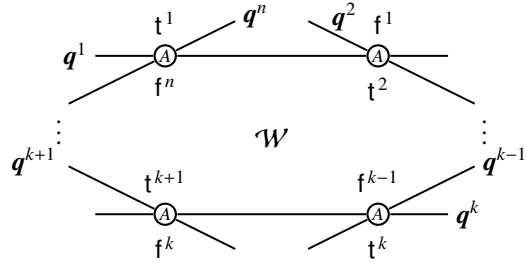
\begin{figure}[ht]
\centering
		\begin{tikzpicture}[>=latex,line width=.75pt, rounded corners,scale=.35]
		\node[label=left:{$\q^1$}\!\!\!\!\!] (s1) at (1,7) {};
		\node[point,scale=1.4,label=below:{$\cff{n}$},label=above:{$\ctt{1}$}] (t1) at (4,7) {\tiny $A$};
		\node[point,scale=1.4,label=below:{$\ctt{2}$},label=above:{$\ \ \cff{1}$}] (f1) at (12,7) {\tiny $A$};
		\node (e1) at (15,7) {};
		\node[label=right:{$\ \ \q^2$}] (s2) at (9,8.5) {};
		\node (e2) at (16,5) {};
		\node [label=right:\!\!\!\!\!{$\q^{k-1}$}] (s3) at (16,3) {};
		\node (e3) at (9,-.5) {};
		\node[point,scale=1.4,label=below:{$\ctt{k}$},label=above:{$\cff{k-1}$}] (tk) at (12,1) {\tiny $A$};
		\node[point,scale=1.4,label=below:{$\cff{k}$},label=above:{$\ctt{k+1}$}] (fk) at (4,1) {\tiny $A$};
		\node (sn) at (0,5) {};
		\node[label=right:\!\!\!\!\!{$\q^k$}] (sk) at (15,1) {};
		\node (ek) at (1,1) {};
		\node[label=right:\!\!\!\!\!{$\q^n$}] (en) at (7,8.5) {};
		\node (s4) at (7,-.5) {};
		\node[label=left:{$\q^{k+1}$}\!\!\!\!]  (e4) at (0,3) {};
		\node[label=below:$\mathcal{W}$] (w) at (8,5) {};
		\node at (0,3.7) {.};
		\node at (0,4.1) {.};
		\node at (0,4.5) {.};
		\node at (16,3.7) {.};
		\node at (16,4.1) {.};
		\node at (16,4.5) {.};
		\draw[-] (s1) to  (t1);
		\draw[-] (t1) to  (f1);
		\draw[-] (f1) to  (e1);
		\draw[-] (s2) to  (f1);
		\draw[-] (f1) to  (e2);
		\draw[-] (sn) to  (t1);
		\draw[-] (t1) to  (en);
		\draw[-] (s3) to  (tk);
		\draw[-] (tk) to  (e3);
		\draw[-] (sk) to  (tk);
		\draw[-] (tk) to  (fk);
		\draw[-] (fk) to  (ek);
		\draw[-] (s4) to  (fk);
		\draw[-] (fk) to  (e4);
		\end{tikzpicture}
\caption{An $n$-\wheel{} $\mathcal{W}$ for $\q$.}\label{f:wheel}
\end{figure}
	
As shown in Lemma~\ref{l:wheel} below, it is straightforward to see that, for any $n$-\wheel{} $\mathcal{W}$,
if $\I$ is a model of $\TT$  and $\mathcal{W}$ with $\I \not\models \q$, then 
either all contacts of $\mathcal{W}$ are in $T^\I$ or all contacts of $\mathcal{W}$ are in $F^\I$.
We want the converse implication to hold as well, in which case $\mathcal{W}$ would  `represent' a truth-value.
In order to achieve this,  we need to choose the contacts in such a way that all possible locations in $\mathcal{W}$ for the image $h(\q)$ of a potential homomorphism $h\colon\q\to\I$ are excluded. 
The following example
shows an improper choice of contacts. 

\begin{texample}\label{ex:cog}
\em
Consider the $2$-CQ shown in the picture below.\\
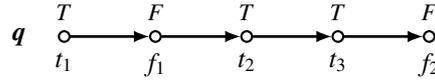
\centerline{
			\begin{tikzpicture}[>=latex,line width=1pt,rounded corners,scale=.8]
			\begin{scope}
			\node at (-.75,0) {$\q$};
			\node[point,scale = 0.7,label=above:{\small $T$},label=below:{$\tfirst$}] (1) at (0,0) {};
			\node[point,scale = 0.7,label=above:{\small $F$},label=below:{$\ffirst$}] (m) at (1.5,0) {};
			\node[point,scale = 0.7,label=above:{\small $T$},label=below:{$\tsec$}] (2) at (3,0) {};
			\node[point,scale = 0.7,label=above:{\small $T$},label=below:{$t_3$}] (3) at (4.5,0) {};
			\node[point,scale = 0.7,label=above:{\small $F$},label=below:{$\fsec$}] (4) at (6,0) {};
			\draw[->,right] (1) to node[below] {}  (m);
			\draw[->,right] (m) to node[below] {} (2);
			\draw[->,right] (2) to node[below] {} (3);
			\draw[->,right] (3) to node[below] {} (4);
			\end{scope}
			\end{tikzpicture}}\\[4pt]	
Take two copies $\q^1$ and $\q^2$ of $\q$ in $\mathcal{W}$.
If we choose the contacts $\ctt{1}=\tfirst^1$,  $\cff{1}=\ffirst^1$,  $\ctt{2}=t_3^2$, $\cff{2}=\fsec^1$,
and $\I$ is such that all contacts of $\mathcal{W}$ are in $F^\I$,
then we do have the following $h\colon\q\to\I$ homomorphism: 

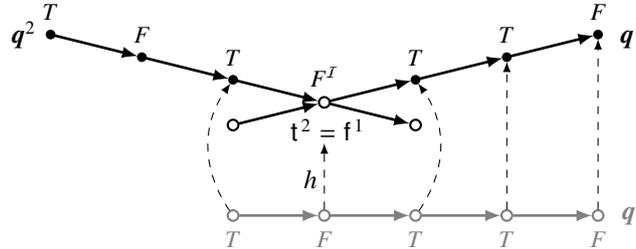
\centerline{
			\begin{tikzpicture}[>=latex,line width=1pt,rounded corners,scale=.6]
			\node[point,gray,scale = 0.7,label=below:{\textcolor{gray}{\small $T$}}] (1) at (1,-2) {};
			\node[point,gray,scale = 0.7,label=below:{\textcolor{gray}{\small $F$}}] (2) at (3,-2) {};
			\node[point,gray,scale = 0.7,label=below:{\textcolor{gray}{\small $T$}}] (3) at (5,-2) {};
			\node[point,gray,scale = 0.7,label=below:{\textcolor{gray}{\small $T$}}] (4) at (7,-2) {};
			\node[point,gray,scale = 0.7,label=right: {\ \textcolor{gray}{$\q$}},label=below:{\textcolor{gray}{\small $F$}}] (5) at (9,-2) {};
			\draw[->,gray,right] (1) to node[below] {}  (2);
			\draw[->,gray,right] (2) to node[below] {} (3);
			\draw[->,gray,right] (3) to node[below] {} (4);
			\draw[->,gray,right] (4) to node[below] {} (5);
			\node[point,scale = 0.5,fill,label=above:{\small $T$},label=left:{$\q^2$}] (u1) at (-3,2) {};
			\node[point,scale = 0.5,fill,label=above:{\small $F$}] (u2) at (-1,1.5) {};
			\node[point,scale = 0.5,fill,label=above:{\small $T$}] (u3) at (1,1) {};
			\node[point,scale = 0.7,label=above:{\small $F^{\I}$},label=below:\ {$\ctt{2}=\cff{1}$}] (u4) at (3,.5) {};
			\node[point,scale = 0.7] (u5) at (5,0) {};
			\node[point,scale = 0.7] (v1) at (1,0) {};
			\node[point,scale = 0.5,fill,label=above:{\small $T$}] (v2) at (5,1) {};
			\node[point,scale = 0.5,fill,label=above:{\small $T$}] (v3) at (7,1.5) {};
			\node[point,scale = 0.5,fill,label=above:{\small $F$},label=right:{\ $\q^1$}] (v4) at (9,2) {};
			\draw[->,right] (u1) to node[below] {}  (u2);
			\draw[->,right] (u2) to node[below] {}  (u3);
			\draw[->,right] (u3) to node[below] {}  (u4);
			\draw[->,right] (u4) to node[below] {} (u5);
			\draw[->,right] (v1) to node[below] {}  (u4);
			\draw[->,right] (u4) to node[below] {}  (v2);
			\draw[->,right] (v2) to node[below] {}  (v3);
			\draw[->,right] (v3) to node[below] {} (v4);
			\node at (2.7,-1.2) {$h$};
			\draw[->,thin,dashed,bend left=35] (1) to  (u3);
			\draw[->,thin,dashed,bend right=35] (3) to  (v2);
			\draw[->,thin,dashed] (2) to  (3,-.4);
			\draw[->,thin,dashed] (4) to  (v3);
			\draw[->,thin,dashed] (5) to  (v4);
			\end{tikzpicture}}	

\end{texample}

To make the search space for contacts smaller and exclude cases like in Example~\ref{ex:cog}, we make the following
assumptions. To begin with, we assume that 
\begin{equation}\label{tprecf}
\tfirst\prec\ffirst
\end{equation}
(as the other case is symmetric). 
We also assume that the contacts of the $n$-\wheel{} $\mathcal{W}$ have the following properties:
\begin{align}
\label{tprecfcontact}
& \mbox{$\ctt{j}\prec_{\q^j} \cff{j}$, for every $j$ with $1\le j\le n$;}\\
\label{okcontact}
& \mbox{if $\ctt{j+1}=t^{j+1}$ and $\cff{j}=f^j$, then $t\prec f$, for all $j$ with $1\le j\le n$.}
\end{align}
(Note that \eqref{okcontact} does not hold in Example~\ref{ex:cog}, as $t_3\nprec\ffirst$.) 
For each $j$, the nodes preceding $\ctt{j}$ in $\q^j$ form its \emph{\initial{} \legW}, while the nodes succeeding $\cff{j}$ in $\q^j$ form its \emph{\final{} \legW}. 
	
The following general criterion still gives us quite some flexibility in designing \wheels:
	
	\begin{tlemma}\label{l:wheel}
		Suppose $\mathcal{W}$ is an $n$-\wheel{} for some $n\geq |\q|$ 
		satisfying \eqref{tprecfcontact} and \eqref{okcontact}.	
		For any model $\I$ of $\TT$ and $\mathcal{W}$, we have 
		$\I\not\models\q$ iff the contacts in $\I$ are either all in $T^\I$ or all in $F^\I$.
	\end{tlemma}
	\begin{proof}
		$(\Rightarrow)$  
		Suppose the contact $\cff{i-1}=\ctt{i}$ is in $T^\I$. Since $\I \not\models \q$, the `clockwise next' contact $\cff{i}=\ctt{i+1}$ is also in $T^\I$. It follows by induction that all of the contacts in $\mathcal{W}$ are in $T^\I$. 
		If the contact $\cff{i-1}=\ctt{i}$ is in $F^\I$, then the `anti-clockwise next' contact $\cff{i-2}=\ctt{i-1}$ is also in $F^\I$, from which it follows by induction that all of the contacts are in $F^\I$. 
		
		$(\Leftarrow)$ 
First, suppose  that $\I$ is a model of $\TT$ and $\mathcal{W}$ such that all contacts in $\I$ are in $F^\I$.
The proof is via excluding all possible locations in $\mathcal{W}$ for the image $h(\q)$ of a potential \shomo{} $h\colon\q\to\I$.
As $n\geq |\q|$, we may consider $h(\q)$ as a path CQ, so $\preceq_{h(\q)}$ and $\delta_{h(\q)}$ are well-defined. 
Observe that if $t$ and $f$ are such that $\ctt{j}=t^j$ and  $\cff{j}=f^j$ for some $j$ then,
by the definition of the minimal model $\I$,
we clearly cannot have that $h(t)=\ctt{j}$ and $h(f)=\cff{j}$.
In particular, there is no $\q\to\I$ \shomo{} mapping $\q$ onto any of the $\q^j$, and so
$h(\q)$ must intersect with at least two copies of $\q$ in $\mathcal{W}$.
Further, by \eqref{okcontact}, there is not enough room for $h(\q)$ to start in an \initial{} \legW{} and then, after reaching a contact, to finish in a \final{} \legW{} (like in Example~\ref{ex:cog}).

So, without loss of generality, we may assume that there is some $k$ such that $2\le k < n$,
\begin{align}
\label{allm}
& \mbox{$h(\q)$ intersects each of the copies  $\q^1,\dots,\q^k$,}\\
\label{notstart}
& \be^1\prec_{\q^1} h(\be)\prec_{\q^1} \cff{1},\ \mbox{and}\\
\label{notend}
& h(\q)\cap \q^k\ne\{\ctt{k}\}.
\end{align}
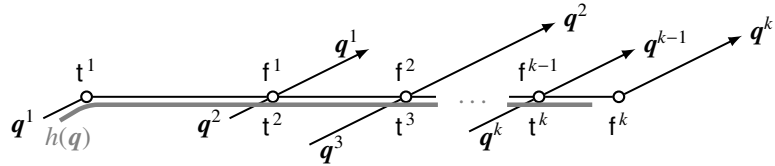
\centerline{
		\begin{tikzpicture}[>=latex,line width=.75pt, rounded corners,scale=.35]
		\node[label=left:{$\q^1$\!\!\!\!\!}] (s1) at (0,1) {};
		\node[label=left:{$\q^2$}\!\!\!\!\!] (s2) at (7,1) {};
		\node[label=right:$\q^3$] (s3) at (10,0) {};
		\node[label=right:$\q^k$] (s4) at (16,.5) {};
		\node[label=left:$\q^1$] (e1) at (13,4) {};
		\node[label=right:\!\!\!\!\!$\q^2$] (e2) at (20,5) {};
		\node[label=right:\!\!\!\!\!$\q^{k-1}$] (e3) at (23,4) {};
		\node[label=right:\!\!\!\!\!$\q^k$] (e4) at (27,4.5) {};
		\node[label=below:{\textcolor{gray}{$h(\q)$}}] (h) at (1.3,1.6) {};
		\node[point,scale=0.7,label=above:$\ctt{1}$] (t1) at (2,2) {};
		\node[point,scale=0.7,label=above:$\cff{1}$,label=below:$\ctt{2}$] (t2) at (9,2) {};
		\node[point,scale=0.7,label=above:$\cff{2}$,label=below:$\ctt{3}$] (t3) at (14,2) {};
		\node[point,scale=0.7,label=above:$\cff{k-1}$,label=below:$\ctt{k}$] (f3) at (19,2) {};
		\node[point,scale=0.7,label=below:$\cff{k}$] (f4) at (22,2) {};
		\node (u1) at (15.5,2) {};
		\node (u2) at (17.5,2) {};
		\node (uh) at (16.5,1.8) {\textcolor{gray}{$\ldots$}};
		\draw[-] (s1) to  (t1);
		\draw[-] (t1) to  (t2);
		\draw[-] (t2) to  (t3);
		\draw[-] (t3) to  (u1);
		\draw[-] (u2) to  (f3);
		\draw[-] (s2) to  (t2);
		\draw[-] (f3) to  (f4);
		\draw[->] (t2) to  (e1);
		\draw[-] (s3) to  (t3);
		\draw[-] (s4) to  (f3);
		\draw[->] (t3) to  (e2);
		\draw[->] (f3) to  (e3);
		\draw[->] (f4) to  (e4);
		\draw[line width=.6mm,gray] (1,1.1) -- (2,1.7) -- (15.2,1.7);
		\draw[line width=.6mm,gray] (17.8,1.7) -- (21,1.7);
		\end{tikzpicture}
	}\\
Now, consider the sub-ABox $\worm$ of $\mathcal{W}$ consisting of the copies $\q^1,\dots,\q^k$.
For any $\ell$ with $1\le \ell\le k$, let $\iota^\ell \colon \q^\ell \to \q$ be the isomorphism mapping each $x^\ell$ to $x$.
		We define a function $\shift\colon\ind(\worm)\to\ind(\worm)$ by taking 
		$\shift(x)=h\bigl(\iota^\ell(x)\bigr)$ whenever $x$ is a node in $\q^\ell$, where we  consider each contact $\cc=\cff{\ell}=\ctt{\ell+1}$, for $1\le \ell<k$, as a node in $\q^{\ell+1}$, that is, $\shift(\cc)=\shift(\ctt{\ell+1})=h\bigl(\iota^{\ell+1}(\ctt{\ell+1})\bigr)$.
		Throughout, we use the following property of $\shift$, which is a straightforward consequence of the $h$-shift in \eqref{hshift} and the similar property of the isomorphism $\iota^\ell$: for every $\ell$ with $1\le \ell\le k$,
\begin{multline}\label{shift}
\mbox{if  $y,z$ are both in the same copy $\q^\ell$, $y,z\ne\cff{\ell}$ whenever $\ell<k$, and $y\preceq_{\q^\ell} z$, }\\
\mbox{then } \shift(y)\preceq_{h(\q)}\shift(z)\ \mbox{and}\ \delta_{\q^\ell}(y,z)\ =\
		\delta_{h(\q)}\bigl(\shift(y),\shift(z)\bigr).
\end{multline}	
	%
%
As $\worm$ is finite, there exists a `fixpoint' of $\shift$: a node $x$ in $\worm$ and a number $N > 0$ such that $\shift^N(x)=x$. 
		We `shift this fixpoint-cycle to the left.' More precisely,  we claim that 
		%
		\begin{equation}\label{cfixpoint}
		\mbox{there is a contact $\cc$ with $\shift^N(\cc)=\cc$.}
		\end{equation}
		Indeed,
		let $y_0=x, y_1=\shift(x), y_2=\shift^2(x),\dots, y_{N-1}=\shift^{N-1}(x)$. 
		Then $\shift^{N}(y_j)=y_j$ for every $j<N$,
		and so if one of the $y_j$ is a contact, we are done with \eqref{cfixpoint}.
		So suppose otherwise. 
		We cannot have that every $y_j$ is in $\q^1$, as otherwise, by \eqref{shift} and \eqref{notstart},
		for every $j\le N$,
		%
\begin{equation}
\delta_{\q^1}(\be^1,y_j)=\delta_{h(\q)}\bigl(\shift(\be^1),\shift(y_j)\bigr)=\delta_{h(\q)}\bigl(h(\be),y_{j+1})\bigr)=
\delta_{\q^1}\bigl(h(\be),y_{j+1})\bigr)<\delta_{\q^1}\bigl(\be^1,y_{j+1})\bigr)
\end{equation}
(here $+$ is modulo $N$).		
		Therefore, by \eqref{notend}, 
		\begin{equation}\label{thereisnoloopfp}
		\mbox{there exists $j<N$ such that $y_j$ is in $\q^{\ell_j}$ and $\ctt{\ell_j}\prec_{\q^{\ell_j}} y_j$, for some $\ell_j> 1$.}
		\end{equation}
		%
(When $\ell_j=1$, such a contact $\ctt{\ell_j}$ does not necessarily exist.) For $j<N$ with $\ell_j> 1$, we set $d_j=\delta_{\q^{\ell_j}}\bigl(\ctt{\ell_j},y_j\bigr)$.
		Let $K< N$ be such that 
		\[
		d_K=\min\{ d_j\mid j< N\ \mbox{and $\ell_j> 1$}\}
		\]
		(which is well-defined by \eqref{thereisnoloopfp}), and set $\cc=\cff{\ell_K-1}=\ctt{\ell_K}$.
		By \eqref{allm}, we have $y_j\prec_{\q^{\ell_j}}\cff{\ell_j}$ whenever $1<\ell_j<k$. Thus,
		by the definition of $K$ and \eqref{shift}, for every $j\le N$,
		$\shift^j(\cc)$ belongs to the same copy $\q^{\ell_{K+ j}}$ as $y_{K+j}$,
		$\shift^j(\cc)\preceq_{\q^{\ell_{K+j}}}y_{K+j}$,
		and 
		\[
		d_K=
		\delta_{\q^{\ell_{K+j}}}\bigl(\shift^j(\cc),y_{K+j}\bigr).
		\]
		It follows, in particular, that $\shift^N(\cc)$ belongs to the same copy $\q^{\ell_{K}}$ as $y_{K}$, and 
		$\delta_{\q^{\ell_{K}}}\bigl(\shift^N(\cc),y_K\bigr)=\delta_{\q^{\ell_{K}}}(\cc,y_{K})$.
		Therefore, $\shift^N(\cc)=\cc$, as required in \eqref{cfixpoint}.
		
		It remains to show that \eqref{cfixpoint}
		leads to a contradiction. Indeed, $\cc\in F^\I$ by our assumption, and so $\cc$ cannot be in $T^\I$ by the minimality of $\I$.  On the other hand, we show by induction on $j \ge 1$ that $\shift^j(\cc)\in T^\I$, and so $\cc=\shift^N(\cc)\in T^\I$.
		If $j=1$ then $\shift(\cc)=h\bigl(\iota^\ell(\ctt{\ell})\bigr)$ for some $\ell$, and so $\shift(\cc)\in T^\I$ as $\iota^\ell(\ctt{\ell})$ is a $T$-node in $\q$ and $h$ is a \shomo. If $j>1$ then $\shift^{j-1}(\cc)\in T^\I$ by IH.
		Thus, $\shift^{j-1}(\cc)$ is not a contact
		and 
		 $\iota^\ell\bigl(\shift^{j-1}(\cc)\bigr)$ must be a $T$-node in $\q$ for some $\ell$.
		Therefore, $\shift^j(\cc)=h\bigl(\iota^\ell\bigl(\shift^{j-1}(\cc)\bigr)\bigr)$ is in $T^\I$, as $h$ is a \shomo.
		
		The case of $\I$ with contacts in $T^\I$ is similar. 
		Now we define a function $\shiftR\colon \ind(\worm)\to\ind(\worm)$ by taking again
		$\shiftR(x)=h\bigl(\iota^\ell(x)\bigr)$ whenever $x$ is a node in $\q^\ell$, but now we consider each contact $\cc=\cff{\ell}=\ctt{\ell+1}$ as a node in $\q^{\ell}$, that is, $\shiftR(\cc)=\shiftR(\cff{\ell})=h\bigl(\iota^{\ell}(\cff{\ell})\bigr)$.
		Then, in the proof of \eqref{cfixpoint} for $\shiftR$, 
		we `shift the fixpoint-cycle to the right'\!.
	\end{proof}
	
\begin{tremark}\label{r:f1}\em	
It is to be noted that if we make more specialised assumptions on the choice of contacts, then Lemma~\ref{l:wheel}
can have a more straightforward proof.  
For example, suppose that 
the $n$-\wheel{} $\mathcal{W}$ satisfies \eqref{tprecfcontact} and 
$\cff{j}=\ffirst^j$, for all $j$ with $1\le j\le n$.
Given a model $\I$ such that all contacts of $\mathcal{W}$ have the same truth-value, we can show that no \shomo{} $h\colon\q\to\I$ exists
by excluding the possible locations of $h(\ffirst)$: 
\begin{itemize}
\item[--]
$h(\ffirst)$ cannot be a contact $\cff{j}$, otherwise $h(t)$ is also a contact, for the $T$-node $t$ with $\ctt{j}=t^j$;

\item[--]
$h(\ffirst)$ cannot be in the \final{} \legW{} of some $\q^j$, otherwise there is no room for $h(\q)$ in that \legW; and

\item[--]
there are no other options for $h(\ffirst)$, as there is no $F$-node preceding $\ffirst$ in $\q$.
\end{itemize}
Unfortunately, as illustrated in Example~\ref{e:bike2} below, we cannot always assume our
$n$-\wheels{} to be that simple.
\end{tremark}

	
\subsection{Representing negation by bikes}\label{s:bikes}	
	
For each variable in the 3CNF $\psi$, we take a fresh pair of  \wheels{} 
$\mathcal{W}_\pw$ and $\mathcal{W}_\mw$ and connect them using two more fresh copies of $\q$ in a special way.	
We want to achieve that, for any model $\I$ of $\TT$ and the resulting `two-wheel' ABox, we have $\I \not\models \q$ iff 
the two \wheels{} $\mathcal{W}_\pw$ and $\mathcal{W}_\mw$ `represent' opposite truth-values:
\begin{multline}\label{pmwsame}		
\mbox{either all contacts of $\mathcal{W}_\pw$ are in $F^\I$ and all contacts of $\mathcal{W}_\mw$ are in $T^\I$,}\\
\mbox{or all contacts of $\mathcal{W}_\pw$ are in $T^\I$ and all contacts of $\mathcal{W}_\mw$ are in $F^\I$.}
\end{multline}

To this end, suppose 
	$\mathcal{W}_\pw$ and $\mathcal{W}_\mw$ are two disjoint $n$-\wheels{}, for some $n>4|\q|+2$, 
	built up from the $\q$-copies $\qqq{\pw}{1},\dots,\qqq{\pw}{n}$ and $\qqq{\mw}{1},\dots,\qqq{\mw}{n}$,
	respectively. For $i=1,\dots,n$ and $\ast=\pw,\mw$, we denote the contacts in $\qqq{\ast}{i}$ by 
	$\cttt{\ast}{i}$ and $\cfff{\ast}{i}$;
	and for any node $x$ in $\q$, we denote by $\bnode{x}{\ast}^{\,i}$ the copy of $x$ in $\qqq{\ast}{i}$.
		%
	We pick two contacts $\cfff{\pw}{i_\pw}=\cttt{\pw}{i_\pw+1}$ and $\cfff{\pw}{j_\pw}=\cttt{\pw}{j_\pw+1}$
	in $\mathcal{W}_\pw$ that are `far' from each other in the sense that 
	the \cdist{} between them in $\mathcal{W}_\pw$ is 	$>2|\q|$. Similarly,
	we pick two contacts $\cfff{\mw}{i_\mw}=\cttt{\mw}{i_\mw+1}$ and $\cfff{\mw}{j_\mw}=\cttt{\mw}{j_\mw+1}$
	in $\mathcal{W}_\mw$ such that the \cdist{} between them in $\mathcal{W}_\mw$ is $>2|\q|$.

Next,	 let $\qq{\uq}$, $\qq{\dq}$ be two more fresh and disjoint copies of $\q$.
	For $Z=\uq,\dq$ and any node $x$ in $\q$, we denote by $\bnode{x}{Z}$ the copy of $x$ in $\qq{Z}$. 
	We connect $\mathcal{W}_\pw$ and $\mathcal{W}_\mw$ via $\qq{\uq}$ and $\qq{\dq}$ as follows.
	%
	First, we pick two $F$-nodes $\cff{\pw}$ and $\cff{\mw}$ with $\cff{\pw}\prec\cff{\mw}$ in $\q$, and replace their $F$-labels by $A$. Then we glue together node $\cfff{\uq}{\pw}$ of $\qq{\uq}$ with the contact $\cfff{\pw}{i_\pw}=\cttt{\pw}{i_\pw+1}$ of $\mathcal{W}_\pw$,
and also glue together $\cfff{\uq}{\mw}$ with the contact $\cfff{\mw}{i_\mw}=\cttt{\mw}{i_\mw+1}$ of $\mathcal{W}_\mw$.
Finally, we pick two $T$-nodes $\ctt{\pw}$ and $\ctt{\mw}$ with $\ctt{\pw}\prec\ctt{\mw}$ in $\q$, and replace their $T$-labels by $A$. Then we glue together node $\cttt{\dq}{\pw}$ of $\qq{\dq}$ with the contact $\cfff{\pw}{j_\pw}=\cttt{\pw}{j_\pw+1}$ of $\mathcal{W}_\pw$,
and also glue together $\cttt{\dq}{\mw}$ with the contact $\cfff{\mw}{j_\mw}=\cttt{\mw}{j_\mw+1}$ of $\mathcal{W}_\mw$.
	The resulting ABox $\mathcal{B}$ is called an $n$-\emph{\bike{} \textup{(}for $\q$\textup{)}},
	see Fig.~\ref{f:bike}.
	We call the contacts $\cfff{\uq}{\pw}=\cfff{\pw}{i_\pw}=\cttt{\pw}{i_\pw+1}$ and $\cfff{\uq}{\mw}=\cfff{\mw}{i_\mw}=\cttt{\mw}{i_\mw+1}$ \emph{\connections{\uq}} in $\mathcal{B}$; the \emph{\neighbourhood{\uq}} of $\mathcal{B}$ consists of those contacts 
whose \cdist{} from an \connection{\uq} is $\leq |\q|$. 
Similarly, the contacts $\cttt{\dq}{\pw}=\cfff{\pw}{j_\pw}=\cttt{\pw}{j_\pw+1}$ and $\cttt{\dq}{\mw}=\cfff{\mw}{j_\mw}=\cttt{\mw}{j_\mw+1}$ \emph{\connections{\dq}} in $\mathcal{B}$, 
and the \emph{\neighbourhood{\dq}} of $\mathcal{B}$ consists of those contacts 
whose \cdist{} from a \connection{\dq} is $\leq |\q|$. 

Using  Lemma~\ref{l:wheel} and the fact that the \connections{\uq} are $F$-nodes in $\qq{\uq}$ while the \connections{\dq} are $T$-nodes in $\qq{\dq}$, it is straightforward to see that, for any $n$-\bike{} $\mathcal{B}$,
if $\I$ is a model of $\TT$  and $\mathcal{B}$ with $\I \not\models \q$, then \eqref{pmwsame} holds.

\begin{figure}[ht]
\centering
		\begin{tikzpicture}[>=latex,line width=.75pt, rounded corners,xscale=.44,yscale=.5]
		\draw[->,line width=.6mm,gray] (10.5,9) -- (11.5,7.5) -- (12,7.3) -- (17.5,8) -- (23,7.3) -- (24.5,9.5);
		\node at (18.5,8.5) {\textcolor{gray}{$\qq{\uq}$}};
		\draw[->,line width=.6mm,gray] (10,-1) -- (12,.7) -- (17.5,0) -- (23,.7) -- (24.5,-1);
		\node at (18.5,-.5) {\textcolor{gray}{$\qq{\dq}$}};
		\node[label=left:{$\qqq{\pw}{i_\pw}$}\!\!\!\!\!] (s1) at (1,7) {};
		\node[point,scale=1.4,label=below:{$\cfff{\pw}{i_\pw-1}$},label=above:{$\cttt{\pw}{i_\pw}$}] (t1) at (4,7) {\tiny $A$};
		\node[point,scale=1.4,label=below:{$\cfff{\pw}{i_\pw}$},label=above:\ \ \ {$\cttt{\pw}{i_\pw+1}$}] (f1) at (12,7) {\tiny $A$};
		\node[label=above:$\cfff{\uq}{\pw}$] (ff1) at (12,7.8) {};
		\node (e1) at (15,7) {};
		\node[label=left:{$\qqq{\pw}{i_\pw+1}$}\!\!\!\!\!\!\!\!] (s2) at (9,8.5) {};
		\node (e2) at (16,5) {};
		\node (s3) at (16,3) {};
		\node[label=left:{$\qqq{\pw}{j_\pw}$}\!\!\!\!\!\!\!]  (e3) at (9,-.5) {};
		\node[point,scale=1.4,label=above:{$\cttt{\pw}{j_\pw+1}$},label=below:{$\ \ \ \cfff{\pw}{j_\pw}$}] (tk) at (12,1) {\tiny $A$};
		\node[label=below:$\cttt{\dq}{\pw}$] (tt2) at (12,.3) {};
		\node[point,scale=1.4,label=above:{$\cfff{\pw}{j_\pw+1}$},label=below:{$\cttt{\pw}{j_\pw+2}$}] (fk) at (4,1) {\tiny $A$};
		\node (sn) at (0,5) {};
		\node (sk) at (15,1) {};
		\node[label=left:{$\qqq{\pw}{j_\pw+1}$}\!\!\!\!\!]  (ek) at (1,1) {};
		\node (en) at (7,8.5) {};
		\node (s4) at (7,-.5) {};
		\node (e4) at (0,3) {};
		\node[label=below:$\mathcal{W}_\pw$] (w) at (8,4.6) {};
		\node at (0,3.5) {.};
		\node at (0,3.9) {.};
		\node at (0,4.3) {.};
		\node at (16,3.5) {.};
		\node at (16,3.9) {.};
		\node at (16,4.3) {.};
		\draw[-] (s1) to  (t1);
		\draw[-] (t1) to  (f1);
		\draw[->] (f1) to  (e1);
		\draw[-] (s2) to  (f1);
		\draw[->] (f1) to  (e2);
		\draw[-] (sn) to  (t1);
		\draw[->] (t1) to  (en);
		\draw[-] (s3) to  (tk);
		\draw[->] (tk) to  (e3);
		\draw[-] (sk) to  (tk);
		\draw[-] (tk) to  (fk);
		\draw[->] (fk) to  (ek);
		\draw[-] (s4) to  (fk);
		\draw[->] (fk) to  (e4);
		%
		%
		\node (rs1) at (20,7) {};
		\node[point,scale=1.4,label=below:{$\cttt{\mw}{i_\mw+1}$},label=above:{$\cfff{\mw}{i_\mw}\ \ \ $}] (rt1) at (23,7) {\tiny $A$};
		\node[label=above:$\cfff{\uq}{\mw}$] (ff2) at (22.8,7.8) {};
		\node[point,scale=1.4,label=below:{$\cfff{\mw}{i_\mw+1}$},label=above:{$\cttt{\mw}{i_\mw+2}$}] (rf1) at (31,7) {\tiny $A$};
		\node[label=right:\!\!\!\!\!{$\qqq{\mw}{i_\mw+1}$}] (re1) at (34,7) {};
		\node (rs2) at (28,8.5) {};
		\node (re2) at (35,5) {};
		\node (rs3) at (35,3) {};
		\node (re3) at (28,-.5) {};
		\node[point,scale=1.4,label=above:{$\cfff{\mw}{j_\mw-1}$},label=below:{$\cttt{\mw}{j_\mw}$}] (rtk) at (31,1) {\tiny $A$};
		\node[point,scale=1.4,label=above:{$\cttt{\mw}{j_\mw+1}$},label=below:{$\cfff{\mw}{j_\mw}\ \ \ $}] (rfk) at (23,1) {\tiny $A$};
		\node[label=below:$\cttt{\dq}{\mw}$] (tt1) at (22.4,.3) {};
		\node (rsn) at (19,5) {};
		\node[label=right:\!\!\!\!\!{$\qqq{\mw}{j_\mw}$}]  (rsk) at (34,1) {};
		\node (rek) at (20,1) {};
		\node[label=right:\!\!\!\!\!\!{$\qqq{\mw}{i_\mw}$}] (ren) at (26,8.5) {};
		\node[label=right:\!\!\!\!\!\!{$\qqq{\mw}{j_\mw+1}$}]  (rs4) at (26,-.5) {};
		\node (re4) at (19,3) {};
		\node[label=below:$\mathcal{W}_\mw$] (rw) at (27,4.6) {};
		\node at (19,3.5) {.};
		\node at (19,3.9) {.};
		\node at (19,4.3) {.};
		\node at (35,3.5) {.};
		\node at (35,3.9) {.};
		\node at (35,4.3) {.};
		\draw[-] (rs1) to  (rt1);
		\draw[-] (rt1) to  (rf1);
		\draw[->] (rf1) to  (re1);
		\draw[-] (rs2) to  (rf1);
		\draw[->] (rf1) to  (re2);
		\draw[-] (rsn) to  (rt1);
		\draw[->] (rt1) to  (ren);
		\draw[-] (rs3) to  (rtk);
		\draw[->] (rtk) to  (re3);
		\draw[-] (rsk) to  (rtk);
		\draw[-] (rtk) to  (rfk);
		\draw[->] (rfk) to  (rek);
		\draw[-] (rs4) to  (rfk);
		\draw[->] (rfk) to  (re4);
		\end{tikzpicture}
\caption{An $n$-\bike{} $\mathcal{B}$ for $\q$.}\label{f:bike}
\end{figure}
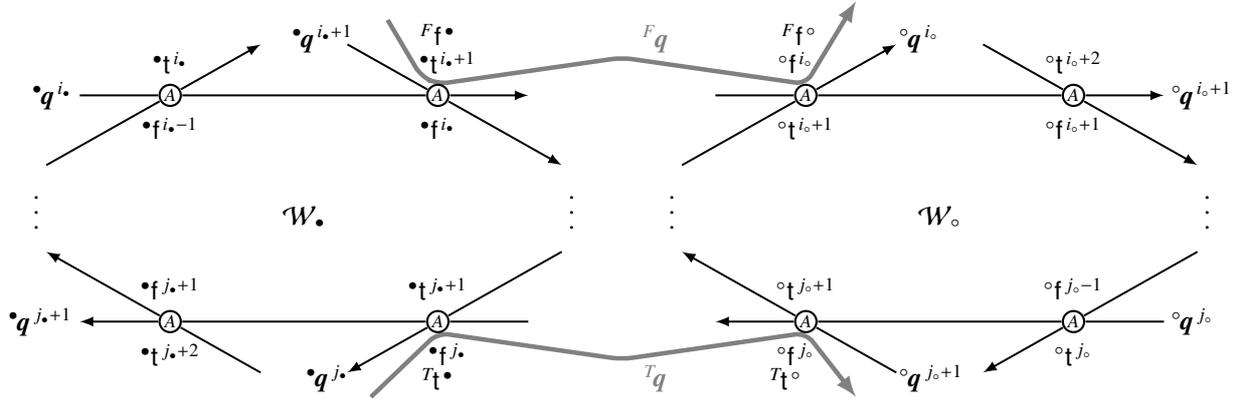		
	
However, for the converse implication to hold, we need to choose the contacts that are
%
$(a)$ the \connections{\uq} in $\qq{\uq}$,
%
$(b)$ 
the \connections{\dq} in $\qq{\dq}$, and
%
$(c)$ located in 
the $\uq$- and \neighbourhoods{\dq} in the two \wheels{} of $\mathcal{B}$
%
carefully, in such a way that all possible locations in $\mathcal{B}$ for the image $h(\q)$ of a potential homomorphism $h\colon\q\to\I$ are excluded. 
So suppose $\I$ is a model of $\TT$ and $\mathcal{B}$ satisfies \eqref{pmwsame}. 
We will again try to exclude all $h\colon\q\to\I$ \shomo s.
To begin with, as $n>4|\q|>|\q|$, we may consider the image $h(\q)$ of $\q$ in $\I$ as a path CQ.
If we choose all the contacts in $(a)$--$(c)$ above 
in such a way that  \eqref{tprecfcontact} and \eqref{okcontact} hold for both \wheels{} in $\mathcal{B}$
then, by Lemma~\ref{l:wheel}, we know that $h(\q)$ must intersect with at least one of $\qq{\uq}$ and $\qq{\dq}$.
Therefore, the intersection of $h(\q)$ with any of the two $n$-\wheels{} cannot go beyond their $\uq$- and \neighbourhoods{\dq}.
Further, it is straightforward to see that because of \eqref{pmwsame},
\begin{multline}\label{nofixhpm}
\mbox{there is no \shomo{} $h:\q\to\I$ such that $h(\cff{\pw})=\cfff{\uq}{\pw}$ and $h(\cff{\mw})=\cfff{\uq}{\mw}$, and}\\
\mbox{there is no \shomo{} $h:\q\to\I$ such that $h(\ctt{\pw})=\cttt{\dq}{\pw}$ and $h(\ctt{\mw})=\cttt{\dq}{\mw}$.}
\end{multline}
(In particular, there is no $\q\to\I$ \shomo{} mapping $\q$ onto $\qq{\uq}$ or onto $\qq{\dq}$.)
Because of this, $h(\q)$ must \emph{properly intersect} with at least one of the two $n$-\wheels{} 
$\mathcal{W}_\pw$ or $\mathcal{W}_\mw$ in the sense that the intersection of $h(\q)$ and the \wheel{} is not just a
$\dq$- or \connection{\uq}.
As the $\uq$-connections are of \cdist{} $>2|\q|$ from the
		$\dq$-connections, $h(\q)$ cannot intersect with both $\qq{\uq}$ and $\qq{\dq}$ at the same time. 
It is easy to check that, by \eqref{nofixhpm}, all options for such a $h(\q)$ are covered by the eight cases given in Fig.~\ref{f:bikecases}.

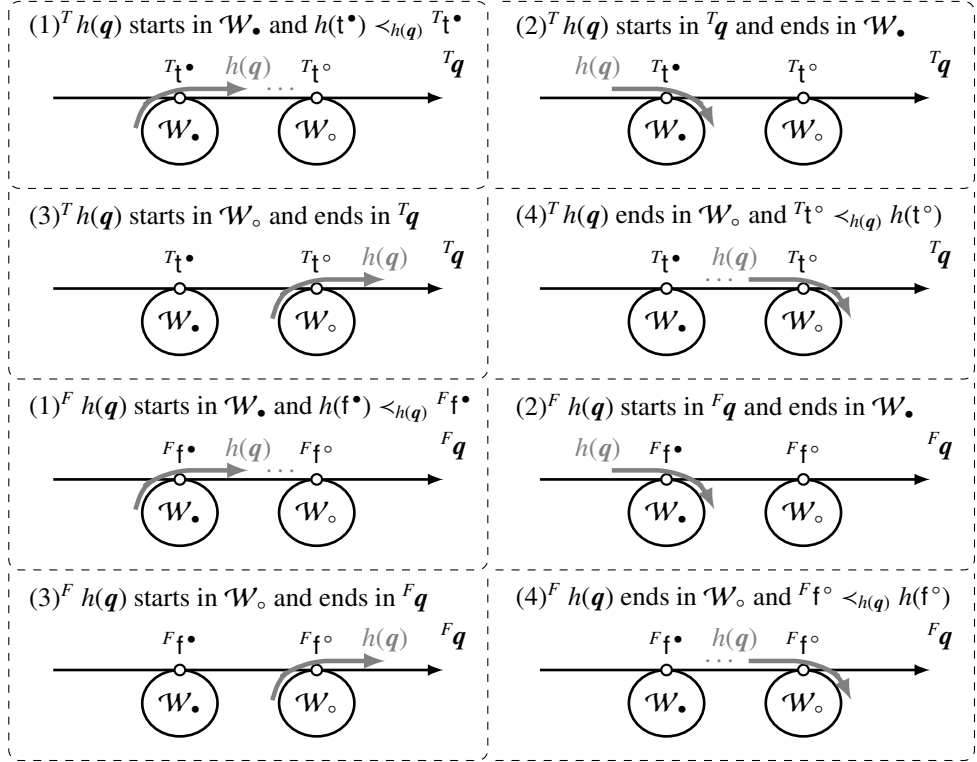
\begin{figure}[ht]
\centering
\begin{tikzpicture}[>=latex,line width=1pt,rounded corners,xscale=.9,yscale=.8]
\node[label=right:{(1)${}^\dq$ $h(\q)$ starts in $\mathcal{W}_\pw$ and $h(\ctt{\pw})\prec_{h(\q)}\cttt{\dq}{\pw}$}] (t) at (7.5,1.2) {};
\draw[-,thin,dashed] (7.5,1.6) -- (7.5,-1.5) -- (14.5,-1.5) -- (14.5,1.6) -- cycle;
	                 \draw (10,-0.55) circle [radius=.55];
	                 \draw (12,-0.55) circle [radius=.55];
			\node (1) at (8,0) {};
			\node[point,scale = 0.7,fill = white,label=above:{$\cttt{\dq}{\pw}$}] (m) at (10,0) {};
			\node[point,scale = 0.7,fill =  white,label=above:{$\cttt{\dq}{\mw}$}] (2) at (12,0) {};
			\node[label=above:{$\qq{\dq}$}] (4) at (14,0) {};
			\draw[-,right] (1) to node[below] {}  (m);
			\draw[-,right] (m) to node[below] {} (2);
			\draw[->,right] (2) to node[below] {} (4);
			\node[point, draw = white] (01) at (10,-0.55) {$\mathcal{W}_\pw$};
			\node[point, draw = white] (02) at (12,-0.55) {$\mathcal{W}_\mw$};
			\draw[->,line width=.6mm,gray] (9.35,-.5) -- (9.45,-.15) -- (9.6,0) -- (10,.15) -- (11,.15);
			\node at (11,.5) {\textcolor{gray}{$h(\q)$}};
			\node at (11.5,.15) {\textcolor{gray}{$\dots$}};
			\end{tikzpicture}	
\hspace*{-.2cm}
\begin{tikzpicture}[>=latex,line width=1pt,rounded corners,xscale=.9,yscale=.8]
\node[label=right:{(2)$^\dq$ $h(\q)$ starts in $\qq{\dq}$ and ends in $\mathcal{W}_\pw$}] (t) at (7.5,1.2) {};
\draw[-,thin,dashed]  (7.5,-1.5) -- (14.5,-1.5) -- (14.5,1.6) -- (7.5,1.6);
	                 \draw (10,-0.55) circle [radius=.55];
	                 \draw (12,-0.55) circle [radius=.55];
			\node (1) at (8,0) {};
			\node[point,scale = 0.7,fill = white,label=above:{$\cttt{\dq}{\pw}$}] (m) at (10,0) {};
			\node[point,scale = 0.7,fill =  white,label=above:{$\cttt{\dq}{\mw}$}] (2) at (12,0) {};
			\node[label=above:{$\qq{\dq}$}] (4) at (14,0) {};
			\draw[-,right] (1) to node[below] {}  (m);
			\draw[-,right] (m) to node[below] {} (2);
			\draw[->,right] (2) to node[below] {} (4);
			\node[point, draw = white] (01) at (10,-0.55) {$\mathcal{W}_\pw$};
			\node[point, draw = white] (02) at (12,-0.55) {$\mathcal{W}_\mw$};
			\draw[->,line width=.6mm,gray] (9.2,.15) -- (10,.15) -- (10.4,0) -- (10.55,-.15) -- (10.7,-.5);
			\node at (9,.5) {\textcolor{gray}{$h(\q)$}};
			\end{tikzpicture}\\	
\begin{tikzpicture}[>=latex,line width=1pt,rounded corners,xscale=.9,yscale=.8]
\node[label=right:{(3)$^\dq$ $h(\q)$ starts in $\mathcal{W}_\mw$ and ends in $\qq{\dq}$}] (t) at (7.5,1.2) {};
\draw[-,thin,dashed] (7.5,1.6) -- (7.5,-1.5) -- (14.5,-1.5) -- (14.5,1.6);
	                 \draw (10,-0.55) circle [radius=.55];
	                 \draw (12,-0.55) circle [radius=.55];
			\node (1) at (8,0) {};
			\node[point,scale = 0.7,fill = white,label=above:{$\cttt{\dq}{\pw}$}] (m) at (10,0) {};
			\node[point,scale = 0.7,fill =  white,label=above:{$\cttt{\dq}{\mw}$}] (2) at (12,0) {};
			\node[label=above:{$\qq{\dq}$}] (4) at (14,0) {};
			\draw[-,right] (1) to node[below] {}  (m);
			\draw[-,right] (m) to node[below] {} (2);
			\draw[->,right] (2) to node[below] {} (4);
			\node[point, draw = white] (01) at (10,-0.55) {$\mathcal{W}_\pw$};
			\node[point, draw = white] (02) at (12,-0.55) {$\mathcal{W}_\mw$};
			\draw[->,line width=.6mm,gray] (11.35,-.5) -- (11.45,-.15) -- (11.6,0) -- (12,.15) -- (13,.15);
			\node at (13,.5) {\textcolor{gray}{$h(\q)$}};
			\end{tikzpicture}	
\hspace*{-.2cm}
\begin{tikzpicture}[>=latex,line width=1pt,rounded corners,xscale=.9,yscale=.8]
\node[label=right:{(4)$^\dq$ $h(\q)$ ends in $\mathcal{W}_\mw$ and $\cttt{\dq}{\mw} \prec_{h(\q)}h(\ctt{\mw})$}] (t) at (7.5,1.2) {};
\draw[-,thin,dashed]  (7.5,-1.5) -- (14.5,-1.5) -- (14.5,1.6);
	                 \draw (10,-0.55) circle [radius=.55];
	                 \draw (12,-0.55) circle [radius=.55];
			\node (1) at (8,0) {};
			\node[point,scale = 0.7,fill = white,label=above:{$\cttt{\dq}{\pw}$}] (m) at (10,0) {};
			\node[point,scale = 0.7,fill =  white,label=above:{$\cttt{\dq}{\mw}$}] (2) at (12,0) {};
			\node[label=above:{$\qq{\dq}$}] (4) at (14,0) {};
			\draw[-,right] (1) to node[below] {}  (m);
			\draw[-,right] (m) to node[below] {} (2);
			\draw[->,right] (2) to node[below] {} (4);
			\node[point, draw = white] (01) at (10,-0.55) {$\mathcal{W}_\pw$};
			\node[point, draw = white] (02) at (12,-0.55) {$\mathcal{W}_\mw$};
			\draw[->,line width=.6mm,gray] (11.2,.15) -- (12,.15) -- (12.4,0) -- (12.55,-.15) -- (12.7,-.5);
			\node at (11,.5) {\textcolor{gray}{$h(\q)$}};
			\node at (10.8,.15) {\textcolor{gray}{$\dots$}};
			\end{tikzpicture}\\
\begin{tikzpicture}[>=latex,line width=1pt,rounded corners,xscale=.9,yscale=.8]
\node[label=right:{(1)$^\uq$ $h(\q)$ starts in $\mathcal{W}_\pw$ and $h(\cff{\pw})\prec_{h(\q)}\cfff{\uq}{\pw}$}] (t) at (7.5,1.2) {};
\draw[-,thin,dashed] (7.5,1.6) -- (7.5,-1.5) -- (14.5,-1.5) -- (14.5,1.6);
	                 \draw (10,-0.55) circle [radius=.55];
	                 \draw (12,-0.55) circle [radius=.55];
			\node (1) at (8,0) {};
			\node[point,scale = 0.7,fill = white,label=above:{$\cfff{\uq}{\pw}$}] (m) at (10,0) {};
			\node[point,scale = 0.7,fill =  white,label=above:{$\cfff{\uq}{\mw}$}] (2) at (12,0) {};
			\node[label=above:{$\qq{\uq}$}] (4) at (14,0) {};
			\draw[-,right] (1) to node[below] {}  (m);
			\draw[-,right] (m) to node[below] {} (2);
			\draw[->,right] (2) to node[below] {} (4);
			\node[point, draw = white] (01) at (10,-0.55) {$\mathcal{W}_\pw$};
			\node[point, draw = white] (02) at (12,-0.55) {$\mathcal{W}_\mw$};
			\draw[->,line width=.6mm,gray] (9.35,-.5) -- (9.45,-.15) -- (9.6,0) -- (10,.15) -- (11,.15);
			\node at (11,.5) {\textcolor{gray}{$h(\q)$}};
			\node at (11.5,.15) {\textcolor{gray}{$\dots$}};
			\end{tikzpicture}
\hspace*{-.2cm}
\begin{tikzpicture}[>=latex,line width=1pt,rounded corners,xscale=.9,yscale=.8]
\node[label=right:{(2)$^\uq$ $h(\q)$ starts in $\qq{\uq}$ and ends in $\mathcal{W}_\pw$}] (t) at (7.5,1.2) {};
\draw[-,thin,dashed]  (7.5,-1.5) -- (14.5,-1.5) -- (14.5,1.6);
	                 \draw (10,-0.55) circle [radius=.55];
	                 \draw (12,-0.55) circle [radius=.55];
			\node (1) at (8,0) {};
			\node[point,scale = 0.7,fill = white,label=above:{$\cfff{\uq}{\pw}$}] (m) at (10,0) {};
			\node[point,scale = 0.7,fill =  white,label=above:{$\cfff{\uq}{\mw}$}] (2) at (12,0) {};
			\node[label=above:{$\qq{\uq}$}] (4) at (14,0) {};
			\draw[-,right] (1) to node[below] {}  (m);
			\draw[-,right] (m) to node[below] {} (2);
			\draw[->,right] (2) to node[below] {} (4);
			\node[point, draw = white] (01) at (10,-0.55) {$\mathcal{W}_\pw$};
			\node[point, draw = white] (02) at (12,-0.55) {$\mathcal{W}_\mw$};
			\draw[->,line width=.6mm,gray] (9.2,.15) -- (10,.15) -- (10.4,0) -- (10.55,-.15) -- (10.7,-.5);
			\node at (9,.5) {\textcolor{gray}{$h(\q)$}};
			\end{tikzpicture}\\
\begin{tikzpicture}[>=latex,line width=1pt,rounded corners,xscale=.9,yscale=.8]
\node[label=right:{(3)$^\uq$ $h(\q)$ starts in $\mathcal{W}_\mw$ and ends in $\qq{\uq}$}] (t) at (7.5,1.2) {};
\draw[-,thin,dashed] (7.5,1.6) -- (7.5,-1.5) -- (14.5,-1.5) -- (14.5,1.6);
	                 \draw (10,-0.55) circle [radius=.55];
	                 \draw (12,-0.55) circle [radius=.55];
			\node (1) at (8,0) {};
			\node[point,scale = 0.7,fill = white,label=above:{$\cfff{\uq}{\pw}$}] (m) at (10,0) {};
			\node[point,scale = 0.7,fill =  white,label=above:{$\cfff{\uq}{\mw}$}] (2) at (12,0) {};
			\node[label=above:{$\qq{\uq}$}] (4) at (14,0) {};
			\draw[-,right] (1) to node[below] {}  (m);
			\draw[-,right] (m) to node[below] {} (2);
			\draw[->,right] (2) to node[below] {} (4);
			\node[point, draw = white] (01) at (10,-0.55) {$\mathcal{W}_\pw$};
			\node[point, draw = white] (02) at (12,-0.55) {$\mathcal{W}_\mw$};
			\draw[->,line width=.6mm,gray] (11.35,-.5) -- (11.45,-.15) -- (11.6,0) -- (12,.15) -- (13,.15);
			\node at (13,.5) {\textcolor{gray}{$h(\q)$}};
			\end{tikzpicture}
\hspace*{-.2cm}
\begin{tikzpicture}[>=latex,line width=1pt,rounded corners,xscale=.9,yscale=.8]
\node[label=right:{(4)$^\uq$ $h(\q)$ ends in $\mathcal{W}_\mw$ and $\cfff{\uq}{\mw} \prec_{h(\q)}h(\cff{\mw})$}] (t) at (7.5,1.2) {};
\draw[-,thin,dashed]  (7.5,-1.5) -- (14.5,-1.5) -- (14.5,1.6);
	                 \draw (10,-0.55) circle [radius=.55];
	                 \draw (12,-0.55) circle [radius=.55];
			\node (1) at (8,0) {};
			\node[point,scale = 0.7,fill = white,label=above:{$\cfff{\uq}{\pw}$}] (m) at (10,0) {};
			\node[point,scale = 0.7,fill =  white,label=above:{$\cfff{\uq}{\mw}$}] (2) at (12,0) {};
			\node[label=above:{$\qq{\uq}$}] (4) at (14,0) {};
			\draw[-,right] (1) to node[below] {}  (m);
			\draw[-,right] (m) to node[below] {} (2);
			\draw[->,right] (2) to node[below] {} (4);
			\node[point, draw = white] (01) at (10,-0.55) {$\mathcal{W}_\pw$};
			\node[point, draw = white] (02) at (12,-0.55) {$\mathcal{W}_\mw$};
			\draw[->,line width=.6mm,gray] (11.2,.15) -- (12,.15) -- (12.4,0) -- (12.55,-.15) -- (12.7,-.5);
			\node at (11,.5) {\textcolor{gray}{$h(\q)$}};
			\node at (10.8,.15) {\textcolor{gray}{$\dots$}};
			\end{tikzpicture}			
\caption{Possible locations for $h(\q)$ intersecting $\qq{\uq}$ or  $\qq{\dq}$.}\label{f:bikecases}
\end{figure}

We aim to show that, for every $2$-CQ, suitable contact choices always exist by actually providing an
\emph{algorithm} that, given  \emph{any} $2$-CQ $\q$, describes contact choices that are suitable for an 
$n$-\bike{} constructed from copies of $\q$.
For some $2$-CQs the suitable contact choices are straightforward (even uniquely determined by \eqref{tprecfcontact} and \eqref{okcontact}), for some others not so. In general, 
the different cases in Fig.~\ref{f:bikecases}
place different constraints on the suitable contact choices.
There might even be some further interaction among these constraints because
$h(\q)$ might intersect, say, the \neighbourhoods{\dq} of both \wheels{} in $\mathcal{B}$.
These interactions, together with constraints \eqref{tprecfcontact} and \eqref{okcontact}, make finding a general solution a tricky cat-and-mouse game.
We have tried several different ways of systematising the search for solutions,
and ended up with the following choices in our `heuristics' (with Remark~\ref{r:f1} motivating (\mh{2})):
%
%
\begin{itemize}
\item[(\mh{1})]
We try to choose all contacts in a way that results in as few cases as possible.

\item[(\mh{2})]
In excluding possible locations for $h(\q)$, we aim to track $h(\ffirst)$. So we aim to choose the contacts in such a way that
leaves as few options for $h(\ffirst)$ as possible.
\end{itemize}
In particular,
in light of (\mh{1}) and (\mh{2}), we decided to go for $\cff{\pw}=\ffirst$ and $\cff{\mw}=\fsec$ as
\connections{\uq}.
This leaves us with only two options for $h(\ffirst)$ in $\qq{\uq}$: its two contacts $\cff{\pw}$ or $\cff{\mw}$.
However, we still have to deal with case distinctions in the choices for $\cttt{\pw}{i_\pw+1}$ and $\cttt{\mw}{i_\mw+1}$ as illustrated by the following examples. 

\begin{texample}\label{e:bike}
\em
$(i)$ Consider the $2$-CQ\\[4pt]
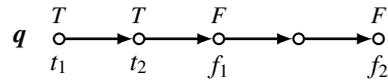
\centerline{
			\begin{tikzpicture}[>=latex,line width=1pt,rounded corners,scale=.7]
			\begin{scope}
			\node at (-.75,0) {$\q$};
			\node[point,scale = 0.7,label=above:{\small $T$},label=below:{$\tfirst$}] (1) at (0,0) {};
			\node[point,scale = 0.7,label=above:{\small $T$},label=below:{$\tsec$}] (m) at (1.5,0) {};
			\node[point,scale = 0.7,label=above:{\small $F$},label=below:{$\ffirst$}] (2) at (3,0) {};
			\node[point,scale = 0.7] (3) at (4.5,0) {};
			\node[point,scale = 0.7,label=above:{\small $F$},label=below:{$\fsec$}] (4) at (6,0) {};
			\draw[->,right] (1) to node[below] {}  (m);
			\draw[->,right] (m) to node[below] {} (2);
			\draw[->,right] (2) to node[below] {} (3);
			\draw[->,right] (3) to node[below] {} (4);
			\end{scope}
			\end{tikzpicture}}\\[4pt]	
If we choose $\cttt{\pw}{i_\pw+1}=\nttt{\pw}{i_\pw+1}{1}$ and $\cfff{\pw}{i_\pw+1}=\nfff{\pw}{i_\pw+1}{1}$, and $\I$ is such that all contacts of $\mathcal{W}_\pw$ are in $F^\I$
and all contacts of $\mathcal{W}_\mw$ are in $T^\I$,
then we do have the following $h\colon\q\to\I$ homomorphism (see case $(2)^F$ in Fig.~\ref{f:bikecases}):

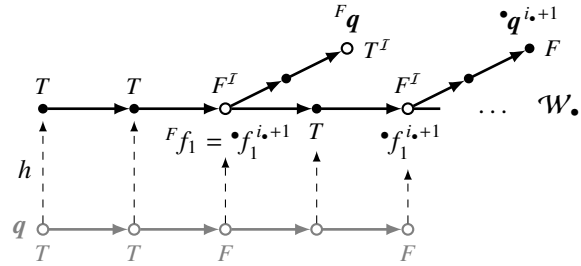
\smallskip
\centerline{
			\begin{tikzpicture}[>=latex,line width=1pt,rounded corners,scale=.8]
			\begin{scope}
			\node[point,gray,scale = 0.7,label=left: {\textcolor{gray}{$\q$}},label=below:{\textcolor{gray}{\small $T$}}] (1) at (0,-.5) {};
			\node[point,gray,scale = 0.7,label=below:{\textcolor{gray}{\small $T$}}] (2) at (1.5,-.5) {};
			\node[point,gray,scale = 0.7,label=below:{\textcolor{gray}{\small $F$}}] (3) at (3,-.5) {};
			\node[point,gray,scale = 0.7] (4) at (4.5,-.5) {};
			\node[point,gray,scale = 0.7,label=below:{\textcolor{gray}{\small $F$}}] (5) at (6,-.5) {};
			\draw[->,gray,right] (1) to node[below] {}  (2);
			\draw[->,gray,right] (2) to node[below] {} (3);
			\draw[->,gray,right] (3) to node[below] {} (4);
			\draw[->,gray,right] (4) to node[below] {} (5);
			\node[point,scale = 0.5,fill,label=above:{\small $T$}] (u1) at (0,1.5) {};
			\node[point,scale = 0.5,fill,label=above:{\small $T$}] (u2) at (1.5,1.5) {};
			\node[point,scale = 0.7,label=above:{\small $F^{\I}$},label=below:\ {$\bnode{\ffirst}{\uq}=\nfff{\pw}{i_\pw+1}{1}$}] (u3) at (3,1.5) {};
			\node[point,scale = 0.5,fill,label=below:{\small $T$}] (u4) at (4.5,1.5) {};
			\node[point,scale = 0.7,label=above:{\small $F^{\I}$},label=below:\ {$\nfff{\pw}{i_\pw+1}{1}$}] (u5) at (6,1.5) {};
			\node (u6) at (6.7,1.5) {};
			\node at (8,1.55) {$\dots\quad\mathcal{W}_\pw$};
			\node[point,scale = 0.5,fill] (s1) at (4,2) {};
			\node[point,scale = 0.7,label=above:{$\qq{\uq}$},label=right:{\small $T^{\I}$}] (s2) at (5,2.5) {};
			\node[point,scale = 0.5,fill] (s3) at (7,2) {};
			\node[point,scale = 0.5,fill,label=right:{\small $F$},label=above:{$\qqq{\pw}{i_\pw+1}$}] (s4) at (8,2.5) {};
			\draw[->,right] (u1) to node[below] {}  (u2);
			\draw[->,right] (u2) to node[below] {} (u3);
			\draw[->,right] (u3) to node[below] {} (u4);
			\draw[->,right] (u4) to node[below] {} (u5);
			\draw[-] (u5) to  (u6);
			\draw[->,right] (u3) to node[below] {} (s1);
			\draw[->,right] (s1) to node[below] {} (s2);
			\draw[->,right] (u5) to node[below] {} (s3);
			\draw[->,right] (s3) to node[below] {} (s4);
			\node at (-.3,.5) {$h$};
			\draw[->,thin,dashed] (1) to  (0,1.3);
			\draw[->,thin,dashed] (2) to  (1.5,1.3);
			\draw[->,thin,dashed] (3) to  (3,.7);
			\draw[->,thin,dashed] (4) to  (4.5,.8);
			\draw[->,thin,dashed] (5) to  (6,.5);
			\end{scope}
			\end{tikzpicture}
			}	

\noindent		
Note that choosing $\cfff{\pw}{i_\pw+1}=\nfff{\pw}{i_\pw+1}{2}$ would not help.

$(ii)$ Consider the $2$-CQ\\[4pt]
\centerline{
			\begin{tikzpicture}[>=latex,line width=1pt,rounded corners,scale=.7]
			\begin{scope}
			\node at (-.75,0) {$\q$};
			\node[point,scale = 0.7,label=above:{\small $T$},label=below:{$\tfirst$}] (1) at (0,0) {};
			\node[point,scale = 0.7,label=above:{\small $T$},label=below:{$\tsec$}] (m) at (1.5,0) {};
			\node[point,scale = 0.7,label=above:{\small $F$},label=below:{$\ffirst$}] (2) at (3,0) {};
			\node[point,scale = 0.7,label=above:{\small $T$}] (3) at (4.5,0) {};
			\node[point,scale = 0.7,label=above:{\small $F$},label=below:{$\fsec$}] (4) at (6,0) {};
			\draw[->,right] (1) to node[below] {}  (m);
			\draw[->,right] (m) to node[below] {} (2);
			\draw[->,right] (2) to node[below] {} (3);
			\draw[->,right] (3) to node[below] {} (4);
			\end{scope}
			\end{tikzpicture}}\\[4pt]	
If we choose $\cttt{\mw}{i_\mw+1}=\nttt{\mw}{i_\mw+1}{1}$ and $\cfff{\mw}{i_\mw+1}=\nfff{\mw}{i_\mw+1}{1}$, and $\I$ is such that 
all contacts of $\mathcal{W}_\pw$ are in $T^\I$
and all contacts of $\mathcal{W}_\mw$ are in $F^\I$,
then we do have the following $h\colon\q\to\I$ homomorphism (see case $(4)^F$ in Fig.~\ref{f:bikecases}):\\
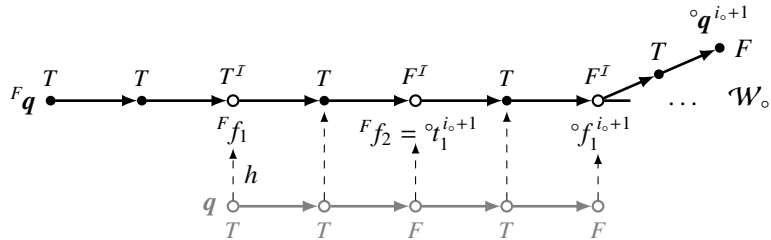
\centerline{
			\begin{tikzpicture}[>=latex,line width=1pt,rounded corners,xscale=.8,yscale=.7]
			\begin{scope}
			\node[point,gray,scale = 0.7,label=left: {\textcolor{gray}{$\q$}},label=below:{\textcolor{gray}{\small $T$}}] (1) at (0,-.5) {};
			\node[point,gray,scale = 0.7,label=below:{\textcolor{gray}{\small $T$}}] (2) at (1.5,-.5) {};
			\node[point,gray,scale = 0.7,label=below:{\textcolor{gray}{\small $F$}}] (3) at (3,-.5) {};
			\node[point,gray,scale = 0.7,label=below:{\textcolor{gray}{\small $T$}}] (4) at (4.5,-.5) {};
			\node[point,gray,scale = 0.7,label=below:{\textcolor{gray}{\small $F$}}] (5) at (6,-.5) {};
			\draw[->,gray,right] (1) to node[below] {}  (2);
			\draw[->,gray,right] (2) to node[below] {} (3);
			\draw[->,gray,right] (3) to node[below] {} (4);
			\draw[->,gray,right] (4) to node[below] {} (5);
			\node[point,scale = 0.5,fill,label=above:{\small $T$},label=left:{$\qq{\uq}$}] (uuu1) at (-3,1.5) {};
			\node[point,scale = 0.5,fill,label=above:{\small $T$}] (uu1) at (-1.5,1.5) {};
			\node[point,scale = 0.7,label=above:{\small $T^{\I}$},label=below:{$\bnode{\ffirst}{\uq}$}] (u1) at (0,1.5) {};
			\node[point,scale = 0.5,fill,label=above:{\small $T$}] (u2) at (1.5,1.5) {};
			\node[point,scale = 0.7,label=above:{\small $F^{\I}$},label=below:\ {$\bnode{\fsec}{\uq}=\nttt{\mw}{i_\mw+1}{1}$}] (u3) at (3,1.5) {};
			\node[point,scale = 0.5,fill,label=above:{\small $T$}] (u4) at (4.5,1.5) {};
			\node[point,scale = 0.7,label=above:{\small $F^{\I}$},label=below:\ {$\nfff{\mw}{i_\mw+1}{1}$}] (u5) at (6,1.5) {};
			\node (u6) at (6.7,1.5) {};
			\node at (8,1.55) {$\dots\quad\mathcal{W}_\mw$};
			\node[point,scale = 0.5,fill,label=above:{$T$}] (s3) at (7,2) {};
			\node[point,scale = 0.5,fill,label=right:{$F$},label=above:{$\qqq{\mw}{i_\mw+1}$}] (s4) at (8,2.5) {};
			\draw[->,right] (uuu1) to node[below] {}  (uu1);
			\draw[->,right] (uu1) to node[below] {}  (u1);
			\draw[->,right] (u1) to node[below] {}  (u2);
			\draw[->,right] (u2) to node[below] {} (u3);
			\draw[->,right] (u3) to node[below] {} (u4);
			\draw[->,right] (u4) to node[below] {} (u5);
			\draw[-] (u5) to  (u6);
			\draw[->,right] (u5) to node[below] {} (s3);
			\draw[->,right] (s3) to node[below] {} (s4);
			\node at (.3,.1) {$h$};
			\draw[->,thin,dashed] (1) to  (0,.6);
			\draw[->,thin,dashed] (2) to  (1.5,1.3);
			\draw[->,thin,dashed] (3) to  (3,.7);
			\draw[->,thin,dashed] (4) to  (4.5,1.3);
			\draw[->,thin,dashed] (5) to  (6,.5);
			\end{scope}
			\end{tikzpicture}}\\
Note again that choosing $\cfff{\mw}{i_\mw+1}=\nfff{\mw}{i_\mw+1}{2}$ would not help.

$(iii)$ On the other hand, as shown in Lemma~\ref{l:bike} below, the contact choices of 
$\cttt{\pw}{i_\pw+1}=\nttt{\pw}{i_\pw+1}{1}$, $\cfff{\pw}{i_\pw+1}=\nfff{\pw}{i_\pw+1}{1}$, 
$\cttt{\mw}{i_\mw+1}=\nttt{\mw}{i_\mw+1}{1}$, and $\cfff{\mw}{i_\mw+1}=\nfff{\mw}{i_\mw+1}{1}$ 
are suitable for any of the following three $2$-CQs:\\
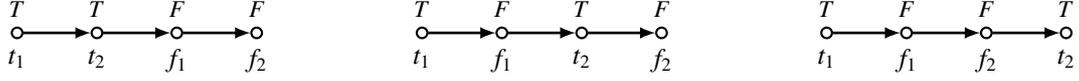
\centerline{
			\begin{tikzpicture}[>=latex,line width=1pt,rounded corners,scale=.7]
			\begin{scope}
			\node[point,scale = 0.7,label=above:{\small $T$},label=below:{$\tfirst$}] (1) at (0,0) {};
			\node[point,scale = 0.7,label=above:{\small $T$},label=below:{$\tsec$}] (m) at (1.5,0) {};
			\node[point,scale = 0.7,label=above:{\small $F$},label=below:{$\ffirst$}] (2) at (3,0) {};
			\node[point,scale = 0.7,label=above:{\small $F$},label=below:{$\fsec$}] (3) at (4.5,0) {};
			\draw[->,right] (1) to node[below] {}  (m);
			\draw[->,right] (m) to node[below] {} (2);
			\draw[->,right] (2) to node[below] {} (3);
			\end{scope}
			\end{tikzpicture}
			\hspace*{1.5cm}
			\begin{tikzpicture}[>=latex,line width=1pt,rounded corners,scale=.7]
			\begin{scope}
			\node[point,scale = 0.7,label=above:{\small $T$},label=below:{$\tfirst$}] (1) at (0,0) {};
			\node[point,scale = 0.7,label=above:{\small $F$},label=below:{$\ffirst$}] (m) at (1.5,0) {};
			\node[point,scale = 0.7,label=above:{\small $T$},label=below:{$\tsec$}] (2) at (3,0) {};
			\node[point,scale = 0.7,label=above:{\small $F$},label=below:{$\fsec$}] (3) at (4.5,0) {};
			\draw[->,right] (1) to node[below] {}  (m);
			\draw[->,right] (m) to node[below] {} (2);
			\draw[->,right] (2) to node[below] {} (3);
			\end{scope}
			\end{tikzpicture}
			\hspace*{1.5cm}
			\begin{tikzpicture}[>=latex,line width=1pt,rounded corners,scale=.7]
			\begin{scope}
			\node[point,scale = 0.7,label=above:{\small $T$},label=below:{$\tfirst$}] (1) at (0,0) {};
			\node[point,scale = 0.7,label=above:{\small $F$},label=below:{$\ffirst$}] (m) at (1.5,0) {};
			\node[point,scale = 0.7,label=above:{\small $F$},label=below:{$\fsec$}] (2) at (3,0) {};
			\node[point,scale = 0.7,label=above:{\small $T$},label=below:{$\tsec$}] (3) at (4.5,0) {};
			\draw[->,right] (1) to node[below] {}  (m);
			\draw[->,right] (m) to node[below] {} (2);
			\draw[->,right] (2) to node[below] {} (3);
			\end{scope}
			\end{tikzpicture}}
\end{texample}

There are other sources of inherent case distinctions. For example,
in light of  Remark~\ref{r:f1}, it would be tempting to try
(the copies of) $\ffirst$ as contacts throughout the \neighbourhoods{\uq} of $\mathcal{W}_\pw$ and $\mathcal{W}_\mw$. 
However, this is not always possible, as illustrated by  
the following examples.

\begin{texample}\label{e:bike2}
\em
$(i)$
Take any $2$-CQ $\q$ that contains only two $F$-nodes. Thus, we must choose $\cff{\pw}=\ffirst$. If we choose 
$\cfff{\pw}{i_\pw}=\nfff{\pw}{i_\pw}{1}$, then in case $(2)^F$ it is always possible
to start $h(\q)$ in $\qq{\uq}$, map $\ffirst$ to $\cfff{\uq}{\pw}=\bnode{\ffirst}{\pw}= \cfff{\pw}{i_\pw}=\nfff{\pw}{i_\pw}{1}$, 
and finish $h(\q)$ in the \final{} \legW{} of $\qqq{\pw}{i_\pw}$.

$(ii)$ Consider the $2$-CQ\\
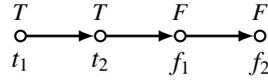
\centerline{
\begin{tikzpicture}[>=latex,line width=1pt,rounded corners,scale=.7]
			\begin{scope}
			\node[point,scale = 0.7,label=above:{\small $T$},label=below:{$\tfirst$}] (1) at (0,0) {};
			\node[point,scale = 0.7,label=above:{\small $T$},label=below:{$\tsec$}] (m) at (1.5,0) {};
			\node[point,scale = 0.7,label=above:{\small $F$},label=below:{$\ffirst$}] (2) at (3,0) {};
			\node[point,scale = 0.7,label=above:{\small $F$},label=below:{$\fsec$}] (3) at (4.5,0) {};
			\draw[->,right] (1) to node[below] {}  (m);
			\draw[->,right] (m) to node[below] {} (2);
			\draw[->,right] (2) to node[below] {} (3);
			\end{scope}
			\end{tikzpicture}
}\\
Then we must choose $\cff{\pw}=\ffirst$ and $\cff{\mw}=\fsec$, 
and so $\cfff{\uq}{\pw}=\bnode{\ffirst}{\uq}$ and $\cfff{\uq}{\mw}=\bnode{\fsec}{\uq}$.
If we choose $\cfff{\mw}{i_\mw}=\nfff{\mw}{i_\mw}{1}$, and $\I$ is such that 
all contacts of $\mathcal{W}_\pw$ are in $T^\I$
and all contacts of $\mathcal{W}_\mw$ are in $F^\I$,
then we do have the following $h\colon\q\to\I$ homomorphism (see case $(4)^F$ in Fig.~\ref{f:bikecases}):\\	
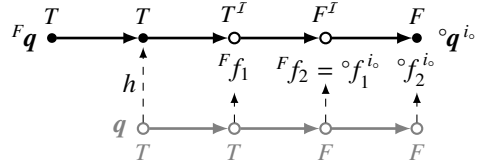
\centerline{
			\begin{tikzpicture}[>=latex,line width=1pt,rounded corners,scale=.6]
			\begin{scope}
			\node[point,gray,scale = 0.7,label=left: {\textcolor{gray}{$\q$}},label=below:{\textcolor{gray}{\small $T$}}] (1) at (-1,-.5) {};
			\node[point,gray,scale = 0.7,label=below:{\textcolor{gray}{\small $T$}}] (2) at (1,-.5) {};
			\node[point,gray,scale = 0.7,label=below:{\textcolor{gray}{\small $F$}}] (3) at (3,-.5) {};
			\node[point,gray,scale = 0.7,label=below:{\textcolor{gray}{\small $F$}}] (4) at (5,-.5) {};
			\draw[->,gray,right] (1) to node[below] {}  (2);
			\draw[->,gray,right] (2) to node[below] {} (3);
			\draw[->,gray,right] (3) to node[below] {} (4);
			\node[point,scale = 0.5,fill,label=above:{\small $T$},label=left:{$\qq{\uq}$}] (uuu1) at (-3,1.5) {};
			\node[point,scale = 0.5,fill,label=above:{\small $T$}] (uu1) at (-1,1.5) {};
			\node[point,scale = 0.7,label=above:{\small $T^{\I}$},label=below:{$\bnode{\ffirst}{\uq}$}] (u1) at (1,1.5) {};
			\node[point,scale = 0.7,label=above:{\small $F^{\I}$},label=below:\ {$\bnode{\fsec}{\uq}=\nfff{\mw}{i_\mw}{1}$}] (u3) at (3,1.5) {};
			\node[point,scale = 0.5,fill,label=above:{\small $F$},label=below:{$\nfff{\mw}{i_\mw}{2}$},label=right:{\ $\qqq{\mw}{i_\mw}$}] (u5) at (5,1.5) {};
			\draw[->,right] (uuu1) to node[below] {}  (uu1);
			\draw[->,right] (uu1) to node[below] {}  (u1);
			\draw[->,right] (u1) to node[below] {}  (u3);
			\draw[->,right] (u3) to node[below] {} (u5);
			\node at (-1.3,.5) {$h$};
			\draw[->,thin,dashed] (1) to  (-1,1.3);
			\draw[->,thin,dashed] (2) to  (1,.3);
			\draw[->,thin,dashed] (3) to  (3,.5);
			\draw[->,thin,dashed] (4) to  (5,.3);
			\end{scope}
			\end{tikzpicture}}	
\end{texample}

There are also examples showing that, unlike the \connections{\uq}, the \connections{\dq} cannot be chosen 
uniformly for all possible $2$-CQs $\q$.
(The problems we face are not `symmetric counterparts' of those with the \connections{\uq} because of our overall assumption that $\tfirst\prec\ffirst$; see \eqref{tprecf}.)
We hope that the above examples convince the reader that,
even after fixing (\mh{1}) and (\mh{2}) (or any other heuristics), finding a \emph{general} solution that satisfies \eqref{tprecfcontact} and \eqref{okcontact} but excludes all cases 
in Fig.~\ref{f:bikecases} for \emph{any} $2$-CQ $\q$ is quite a challenge. 
The following lemma describes such a solution found along the lines of (\mh{1}) and (\mh{2}): 

\begin{tlemma}\label{l:bike}
Let $\mathcal{B}$ be an $n$-\bike{}, for some $n\geq 4|\q|+2$, built up from the $n$-\wheels{} $\mathcal{W}_\pw$ and $\mathcal{W}_\mw$, each satisfying \eqref{tprecfcontact} and \eqref{okcontact}. Suppose $\mathcal{B}$ is such 	
that the following hold for its \connections{\dq}\textup{:} 
\[
\ctt{\pw}=\tfirst,\qquad\ctt{\mw}=\left\{
 \begin{array}{ll}
   \tStar, & \mbox{if $\ffirst\prec\tlast$,\quad where $\tStar$ is the first $T$-node succeeding $\ffirst$,}\\[3pt]
   \tlast, & \mbox{if $\tlast\prec\ffirst$;}
 \end{array}
\right.
\]
the following hold for its \neighbourhood{\dq}\textup{:} 
\begin{align*}
& \cttt{\pw}{j_\pw}=\nttt{\pw}{j_\pw}{\Box},\quad\mbox{where $\tBox$ is the last $T$-node preceding $\ffirst$,}\qquad
\cfff{\pw}{j_\pw}=\left\{
		\begin{array}{ll}
		\nfff{\pw}{j_\pw}{1}, & \mbox{if $\ffirst\prec\tlast$,}\\[3pt]
		\nfff{\pw}{j_\pw}{2}, & \mbox{if $\tlast\prec\ffirst$,}
		\end{array}
		\right.\\
& \cttt{\mw}{j_\mw}=\nttt{\mw}{j_\mw}{1},\qquad			
\cfff{\mw}{j_\mw}=\left\{
		\begin{array}{ll}
		\nfff{\mw}{j_\mw}{2}, & \mbox{if $\tlast\prec\ffirst$ and $\delta(\tlbo,\tlast)=\delta(\tlast,\ffirst)$,}\\[3pt]
		\nfff{\mw}{j_\mw}{1}, & \mbox{otherwise},
		\end{array}
		\right.\\[3pt]
& \cttt{\ast}{k}=\nttt{\ast}{k}{1},\qquad
\cfff{\ast}{k}=\nfff{\ast}{k}{1},\quad\mbox{for $\ast=\pw,\mw$ and for any other $k$ with $j_\ast-|\q|\le k\le j_\ast+|\q|$\textup{;}}
\end{align*}
the following hold for its \connections{\uq}\textup{:} 
\[
\cff{\pw}=\ffirst,\qquad\cff{\mw}=\fsec\textup{;} 
\]
and the following hold for its \neighbourhood{\uq}, for $\ast=\pw,\mw$\textup{:}
\begin{align*}
& \cttt{\ast}{i_\ast}=\nttt{\ast}{i_\ast}{1},\qquad
\cfff{\ast}{i_\ast}=\nfff{\ast}{i_\ast}{2},\\
& \cttt{\ast}{i_\ast-k}=\nttt{\ast}{i_\ast-k}{1},\qquad
\cfff{\ast}{i_\ast-k}=\nfff{\ast}{i_\ast-k}{1},\quad\mbox{for $0<k\le |\q|$,}\\
& \cttt{\ast}{i_\ast+\ell}=\left\{
 \begin{array}{ll}
   \nttt{\ast}{i_\ast+\ell}{1}, & \mbox{if $\tlast\prec\ffirst$ and $\delta(\ffirst,\fsec)<\delta(\tfirst,\ffirst)$,}\\[3pt]
   \nttt{\ast}{i_\ast+\ell}{\Box}, & \mbox{otherwise,\quad where $\tBox$ is the last $T$-node preceding $\ffirst$,}
 \end{array}
\right.\\[3pt]
& \cfff{\ast}{i_\ast+\ell}=\left\{
 \begin{array}{ll}
   \nfff{\ast}{i_\ast+\ell}{2}, & \mbox{if $\tlast\prec\ffirst$ and $\delta(\ffirst,\fsec)\geq\delta(\tfirst,\ffirst)$},\\[3pt]
   \nfff{\ast}{i_\ast+\ell}{1}, & \mbox{otherwise,}
 \end{array}
\right.\\
& \mbox{for $1\le\ell\le |\q|$.}
\end{align*}
Then, for any model $\I$ of $\TT$ and $\mathcal{B}$, we have
\begin{align*}
\I \not\models \q\quad\mbox{iff}\quad
& \mbox{either all contacts of $\mathcal{W}_\pw$ are in $T^\I$ and all contacts of $\mathcal{W}_\mw$ are in $F^\I$}\\
& \mbox{or all contacts of $\mathcal{W}_\pw$ are in $F^\I$ and all contacts of $\mathcal{W}_\mw$ are in $T^\I$.}
\end{align*}
	\end{tlemma}
It is straightforward to check  that
$n$-\bikes{} $\mathcal{B}$ satisfying the conditions of the lemma always exist:
The $\uq$- and \neighbourhoods{\dq} of $\mathcal{B}$ can be kept disjoint by taking	 $>2|\q|$
\cdist{} between the $\uq$- and \connections{\dq} of each of the \wheels{} in $\mathcal{B}$ and, by choosing, say, 
$\cttt{\ast}{k}=\nttt{\ast}{k}{1}$ and $\cfff{\ast}{k}=\nfff{\ast}{k}{2}$ 
for all other $k$ and $\ast=\pw,\mw$,  conditions \eqref{tprecfcontact} and \eqref{okcontact} hold in both \wheels.
%
\begin{proof}
		The implication $(\Rightarrow)$ of Lemma~\ref{l:bike} clearly holds for any $n$-\bike{} $\mathcal{B}$ 
		by the $(\Rightarrow)$ direction of Lemma~\ref{l:wheel}.
		To show $(\Leftarrow)$, suppose $\mathcal{B}$ is as above, and $\I$ is a model of $\TT$ and $\mathcal{B}$ such that \eqref{pmwsame} holds.
The proof is via excluding all possible locations in $\mathcal{B}$ for the image $h(\q)$ of a
		potential \shomo{} $h\colon\q\to\I$.
As we discussed above, we have the eight cases in Fig.~\ref{f:bikecases}.
In line with (\mh{2}), in each of these eight cases, we track the location of $h(\ffirst)$, and exclude all options for it.	
Throughout, our arguments will use the $h$-shift property in \eqref{hshift} without explicit reference.
		
First, we deal with the cases when $h(\q)\cap\qq{\dq}\ne\emptyset$:	
\begin{itemize}
\item[(1)$^\dq$]
$h(\q)$ starts in $\mathcal{W}_\pw$ and $h(\ctt{\pw})\prec_{h(\q)}\cttt{\dq}{\pw}$.\\
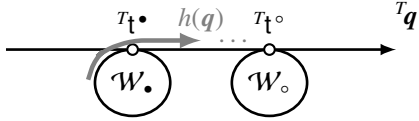
\begin{tikzpicture}[>=latex,line width=1pt,rounded corners,xscale=.9,yscale=.8]
	                 \draw (10,-0.55) circle [radius=.55];
	                 \draw (12,-0.55) circle [radius=.55];
			\node (1) at (8,0) {};
			\node[point,scale = 0.7,fill = white,label=above:{$\cttt{\dq}{\pw}$}] (m) at (10,0) {};
			\node[point,scale = 0.7,fill =  white,label=above:{$\cttt{\dq}{\mw}$}] (2) at (12,0) {};
			\node[label=above:{$\qq{\dq}$}] (4) at (14,0) {};
			\draw[-,right] (1) to node[below] {}  (m);
			\draw[-,right] (m) to node[below] {} (2);
			\draw[->,right] (2) to node[below] {} (4);
			\node[point, draw = white] (01) at (10,-0.55) {$\mathcal{W}_\pw$};
			\node[point, draw = white] (02) at (12,-0.55) {$\mathcal{W}_\mw$};
			\draw[->,line width=.6mm,gray] (9.35,-.5) -- (9.45,-.15) -- (9.6,0) -- (10,.15) -- (11,.15);
			\node at (11,.5) {\textcolor{gray}{$h(\q)$}};
			\node at (11.5,.15) {\textcolor{gray}{$\dots$}};
			\end{tikzpicture}\\
Then $h(\q)$ definitely properly intersects the \neighbourhood{\dq} of $\mathcal{W}_\pw$, and it might also 
properly intersect the \neighbourhood{\dq} of $\mathcal{W}_\mw$. 
It follows from $h(\ctt{\pw})\prec_{h(\q)}\cttt{\dq}{\pw}$ that $h(\ffirst)$ is in  $\qq{\dq}$ then $h(\ffirst)\prec_{\qq{\dq}}\bnode{\ffirst}{\dq}$.
As $\cttt{\pw}{j_\pw+1}=\nttt{\pw}{j_\pw+1}{1}$ and $\cttt{\pw}{j_\pw-k}\prec_{\qqq{\pw}{j_\pw-k}}\nfff{\pw}{j_\pw-k}{1}$ for all $k<|\q|$, 
$h(\ffirst)$ cannot be in the \initial{} \legW{} of neither $\qqq{\pw}{j_\pw+1}$ nor $\qqq{\pw}{j_\pw-k}$ for any $k\le|\q|$, otherwise there is not enough room for $h(\q)$ in that \legW.
As $\cttt{\pw}{j_\pw-k}=\nttt{\pw}{j_\pw-k}{1}$ and $\cfff{\pw}{j_\pw-k}=\nfff{\pw}{j_\pw-k}{1}$ for all $k$ with $0<k\le |\q|$, $h(\ffirst)$ cannot be a contact of 
$\mathcal{W}_\pw$ different from $\cttt{\dq}{\pw}$, otherwise $h(\tfirst)$ is also a contact of $\mathcal{W}_\pw$, contradicting \eqref{pmwsame}. For the remaining options,
we consider the two cases $\ffirst\prec\tlast$ and $\tlast\prec\ffirst$:

If $\ffirst\prec\tlast$ then $\ctt{\pw}=\tfirst$ and $\ffirst\prec\ctt{\mw}$, and so $\bnode{\ffirst}{\dq}\prec_{\qq{\dq}}\cttt{\dq}{\mw}$.
As there is no $F$-node preceding $\bnode{\ffirst}{\dq}$ in $\qq{\dq}$, $h(\ffirst)$ is in $\mathcal{W}_\pw$. 
As $\cttt{\pw}{j_\pw}=\nttt{\pw}{j_\pw}{\Box}$ and $\cfff{\pw}{j_\pw}=\nfff{\pw}{j_\pw}{1}$, $h(\ffirst)$ cannot be the contact $\cttt{\dq}{\mw}=\nfff{\mw}{j_\mw}{1}$, 
otherwise $h(\tBox)$ is also a contact of $\mathcal{W}_\pw$, contradicting \eqref{pmwsame}.
As there is no $F$-node preceding $\ffirst$ in $\q$, there are no other options for $h(\ffirst)$ in $\mathcal{W}_\pw$.\\
\centerline{			
	\begin{tikzpicture}[>=latex,line width=.75pt, rounded corners,xscale=.3,yscale=.3]
		\node[label=left:{$\qqq{\pw}{j_\pw}$}\!\!\!\!\!] (s3) at (-1,2) {};
		\node[label=left:{$\qqq{\pw}{j_\pw+1}$}\!\!\!\!\!] (s1) at (7,6) {};
		\node[] (e1) at (26,6) {};
		\node[point,scale=0.7,label=below:{$\nttt{\pw}{j_\pw}{\Box}\!\!\!$}] (l3) at (2,2) {};
		\node[point,scale=0.7,label=below right:{$\!\!\!\!\!\!\nfff{\pw}{j_\pw}{1}=\nttt{\pw}{j_\pw+1}{1}$},label=above:{\ $\cttt{\dq}{\pw}=\bnode{\tfirst}{\dq}$}] (m) at (10,6) {};
		\node[point,scale=0.7,label=above:{$\cttt{\dq}{\mw}$}] (m2) at (22,6) {};
		\node at (22,1.4) {$\mathcal{W}_\mw$};
		\draw[-,very thin, bend left=35] (m2) to (25,-1);
		\draw[-,very thin, bend right=35] (m2) to (19,-1);
		\draw[->] (s3) to  (l3);
		\draw[->] (s1) to  (m);
		\draw[->] (m2) to  (e1);
		\draw[-] (m) to  (m2);
		\draw[-] (l3) to  (m);
		\node at (15,7) {$\qq{\dq}$};	
		\draw[-] (19.5,5.8) to (19.5,6.2);
		\node at (19.5,7) {$\bnode{\ffirst}{\dq}$};	
		\node (q1) at (3.5,-2) {};
		\node[point,gray,scale=0.7,label=below:{\textcolor{gray}{\ \ $\tfirst\preceq$}}] (q2) at (6.5,-2) {};
		\node[point,gray,scale=0.7,label=below:{\textcolor{gray}{$\tBox$}}] (q3) at (8.5,-2) {};
		\node[point,gray,scale=0.7,label=below:{\textcolor{gray}{$\ffirst$}}] (q4) at (17,-2) {};
		\node[label=right:{\textcolor{gray}{$\q$}}]  (q5) at (22.5,-2) {};
		\draw[-,gray] (q1) to  (q2);
		\draw[-,gray] (q2) to  (q3);
		\draw[-,gray] (q3) to  (q4);
		\draw[->,gray] (q4) to  (q5);	
		\draw[-,thin] (3.8,-4) to (16.8,-4);
		\draw[-,thin] (3.8,-4) to (3.8,-3.8);
		\draw[-,thin] (16.8,-4) to (16.8,-3.8);
		\node at (10,-4.5) {{\scriptsize no $F$}};
\end{tikzpicture}	
		}

If $\tlast\prec\ffirst$ then $\ctt{\pw}=\tfirst$ and $\ctt{\mw}=\tlast=\tBox$, and so $\cttt{\dq}{\mw}=\bnode{\tlast}{\dq}\prec_{\qq{\dq}}\bnode{\ffirst}{\dq}$. As there is no $F$-node preceding $\bnode{\ffirst}{\dq}$ in $\qq{\dq}$,
$h(\ffirst)$ is either in $\mathcal{W}_\pw$ or in $\mathcal{W}_\mw$.\\
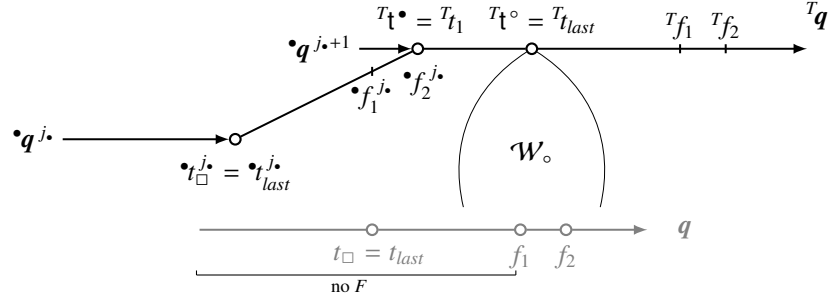
\centerline{	
	\begin{tikzpicture}[>=latex,line width=.75pt, rounded corners,scale=.3]
		\node[label=left:{$\qqq{\pw}{j_\pw}$}\!\!\!\!\!] (s3) at (-6,2) {};
		\node[label=left:{$\qqq{\pw}{j_\pw+1}$}\!\!\!\!\!] (s1) at (7,6) {};
		\node[label=above:{$\qq{\dq}$}] (e1) at (27.5,6) {};
		\node[point,scale=0.7,label=below:{$\nttt{\pw}{j_\pw}{\Box}=\nttt{\pw}{j_\pw}{\last}$}] (l3) at (2,2) {};
		\node[point,scale=0.7,label=below:{\ \ $\nfff{\pw}{j_\pw}{2}$},label=above:{\ $\cttt{\dq}{\pw}=\bnode{\tfirst}{\dq}$}] (m) at (10,6) {};
		\node[point,scale=0.7,label=above:{\ \ \ $\cttt{\dq}{\mw}=\bnode{\tlast}{\dq}$}] (m2) at (15,6) {};
		\node at (15,1.4) {$\mathcal{W}_\mw$};
		\draw[-,very thin, bend left=35] (m2) to (18,-1);
		\draw[-,very thin, bend right=35] (m2) to (12,-1);
		\draw[->] (s3) to  (l3);
		\draw[->] (s1) to  (m);
		\draw[->] (m2) to  (e1);
		\draw[-] (m) to  (m2);
		\draw[-] (l3) to  (m);
		\draw[-] (21.5,5.8) to (21.5,6.2);
		\node at (21.5,7) {$\bnode{\ffirst}{\dq}$};	
		\draw[-] (23.5,5.8) to (23.5,6.2);
		\node at (23.5,7) {$\bnode{\fsec}{\dq}$};
		\draw[-] (8,4.8) to (8,5.2);
		\node at (8,3.9) {$\nfff{\pw}{j_\pw}{1}$};
		\node (q1) at (0,-2) {};
		\node[point,gray,scale=0.7,label=below:{\textcolor{gray}{\ \ $\tBox=\tlast$}}] (q3) at (8,-2) {};
		\node[point,gray,scale=0.7,label=below:{\textcolor{gray}{$\ffirst$}}] (q4) at (14.5,-2) {};
		\node[point,gray,scale=0.7,label=below:{\textcolor{gray}{$\fsec$}}] (qq4) at (16.5,-2) {};
		\node[label=right:{\textcolor{gray}{$\q$}}]  (q5) at (20.5,-2) {};
		\draw[-,gray] (q1) to  (q3);
		\draw[-,gray] (q3) to  (q4);
		\draw[-,gray] (q4) to  (qq4);
		\draw[->,gray] (qq4) to  (q5);	
		\draw[-,thin] (.3,-4) to (14.3,-4);
		\draw[-,thin] (.3,-4) to (.3,-3.8);
		\draw[-,thin] (14.3,-4) to (14.3,-3.8);
		\node at (7,-4.5) {{\scriptsize no $F$}};
\end{tikzpicture}}	
		
%
\begin{itemize}
\item
First, we exclude the remaining options in $\mathcal{W}_\pw$.
As $\cttt{\pw}{j_\pw}=\nttt{\pw}{j_\pw}{\Box}$ and $\cfff{\pw}{j_\pw}=\nfff{\pw}{j_\pw}{2}$, we cannot have $h(\ffirst)=\nfff{\pw}{j_\pw}{1}$, otherwise both $h(\tBox)$ and $h(\fsec)$ are contacts of $\mathcal{W}_\pw$, contradicting \eqref{pmwsame}.
And if $h(\ffirst)$ is the contact $\cttt{\dq}{\pw}=\nfff{\pw}{j_\pw}{2}$, then we track the location of $h(\tlast)$.
As $h(\q)$ starts in $\mathcal{W}_\pw$ and $h(\tlast)\prec_{h(\q)}h(\ffirst)$, $h(\tlast)$ is in $\mathcal{W}_\pw$.
As $\cttt{\pw}{j_\pw+1}=\nttt{\pw}{j_\pw+1}{1}$ and $\tfirst\prec\tlast$, $h(\tlast)$ cannot be in the \initial{} \legW{} of 
$\qqq{\pw}{j_\pw+1}$, otherwise there is not enough room for $h(\q)$ in that \legW. Thus, we have
%
\[
\delta_{h(\q)}\bigl(h(\tlast),\nfff{\pw}{j_\pw}{2}\bigr)=
\delta_{h(\q)}\bigl(h(\tlast),h(\ffirst)\bigr)=
\delta(\tlast,\ffirst)<\delta(\tlast,\fsec)=
\delta_{\qqq{\pw}{j_\pw}}\bigl(\nttt{\pw}{j_\pw}{\last},\nfff{\pw}{j_\pw}{2}\bigr),
\]
and so $h(\tlast)$ is a node between $\nttt{\pw}{j_\pw}{\last}=\nttt{\pw}{j_\pw}{\tBox}$ and $\nfff{\pw}{j_\pw}{2}$. But there is no such $T$-node in $\qqq{\pw}{j_\pw}$.
As the only $F$-node preceding $\fsec$ in $\q$ is $\ffirst$, there are no other options for $h(\ffirst)$ in $\mathcal{W}_\pw$.\\
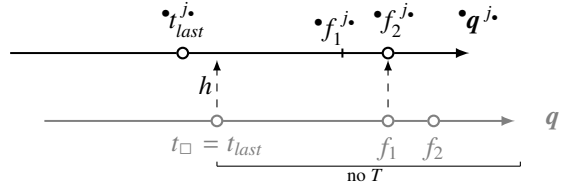
\centerline{
 \begin{tikzpicture}[>=latex,line width=.75pt, rounded corners,scale=.3]
			\node (1) at (-4,2) {};
			\node[point,scale=0.7,label=above:{$\nttt{\pw}{j_\pw}{\last}$}] (2) at (4,2) {};
			\node[point,scale=0.7,label=above:\ \ {$\nfff{\pw}{j_\pw}{2}$}] (3) at (13,2) {};
			\node [label=above:{$\qqq{\pw}{j_\pw}$}] (4) at (17,2) {};
			\node (q1) at (-2.5,-1) {};
			\node[point,gray,scale=0.7,label=below:{\textcolor{gray}{$\tBox=\tlast$}}] (q3) at (5.5,-1) {};
			\node[point,gray,scale=0.7,label=below:{\textcolor{gray}{$\ffirst$}}] (q4) at (13,-1) {};
			\node[point,gray,scale=0.7,label=below:{\textcolor{gray}{$\fsec$}}] (qq4) at (15,-1) {};
			\node[label=right:{\textcolor{gray}{$\q$}}]  (q5) at (19,-1) {};
			\draw[-] (1) to  (2);
			\draw[-] (2) to  (3);
			\draw[->] (3) to  (4);
			\draw[-] (11,1.8) to (11,2.2);
			\node at (10.7,3.1) {$\nfff{\pw}{j_\pw}{1}$};
			\draw[-,gray] (q1) to  (q3);
			\draw[-,gray] (q3) to  (q4);
			\draw[-,gray] (q4) to  (qq4);
			\draw[->,gray] (qq4) to  (q5);
			\node at (5,.6) {$h$};
			\draw[->,thin,dashed] (5.5,-.4) to  (5.5,1.7);
			\draw[->,thin,dashed] (13,-.4) to  (13,1.7);
			\draw[-,thin] (5.5,-3) to (18.8,-3);
			\draw[-,thin] (5.5,-3) to (5.5,-2.8);
			\draw[-,thin] (18.8,-3) to (18.8,-2.8);
			\node at (12,-3.4) {{\scriptsize no $T$}};
			\end{tikzpicture}
			}

\item
If $h(\ffirst)$ is in $\mathcal{W}_\mw$ then $\cttt{\dq}{\mw}=\bnode{\tlast}{\dq}\preceq_{h(\q)}h(\ffirst)$.
We track the location of $h(\tlast)$.
As $h(\ctt{\pw})\prec_{h(\q)}\cttt{\dq}{\pw}$ by our assumption, we have $h(\tlast)=h(\ctt{\mw})\prec_{h(\q)}\cttt{\dq}{\mw}$,
and so either $h(\tlast)$ is in $\mathcal{W}_\pw$ and $h(\tlast)\prec_{h(\q)}\cttt{\dq}{\pw}$, or $h(\tlast)$ is in $\qq{\dq}$.
In the former case,
as $\cttt{\pw}{j_\pw+1}=\nttt{\pw}{j_\pw+1}{1}$ and $\tfirst\prec\tlast$, $h(\tlast)$ cannot be in the \initial{} \legW{} of 
$\qqq{\pw}{j_\pw+1}$, otherwise there is not enough room for $h(\q)$ in that \legW. Thus, we have
\begin{multline*}
\delta_{h(\q)}\bigl(h(\tlast),\nfff{\pw}{j_\pw}{2}\bigr)=
\delta_{h(\q)}\bigl(h(\tlast),\cttt{\dq}{\pw}\bigr)<
\delta_{h(\q)}\bigl(h(\tlast),h(\ffirst)\bigr)=\\
\delta(\tlast,\ffirst)<\delta(\tlast,\fsec)=
\delta_{\qqq{\pw}{j_\pw}}\bigl(\nttt{\pw}{j_\pw}{\last},\nfff{\pw}{j_\pw}{2}\bigr),
\end{multline*}
and so $h(\tlast)$ is a node between $\nttt{\pw}{j_\pw}{\last}=\nttt{\pw}{j_\pw}{\tBox}$ and $\nfff{\pw}{j_\pw}{2}$. But there is no such $T$-node in $\qqq{\pw}{j_\pw}$.\\
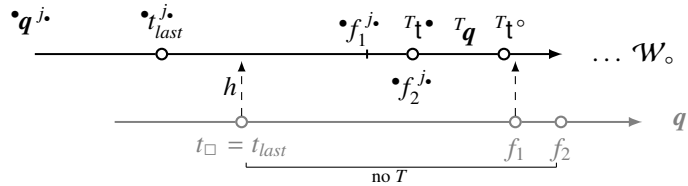
\centerline{
 \begin{tikzpicture}[>=latex,line width=.75pt, rounded corners,scale=.3]
			\node [label=above:{$\ \ \qqq{\pw}{j_\pw}$}] (1) at (-4,2) {};
			\node[point,scale=0.7,label=above:{$\nttt{\pw}{j_\pw}{\last}$}] (2) at (2,2) {};
			\node[point,scale=0.7,label=below:{$\nfff{\pw}{j_\pw}{2}$},label=above:\ \ {$\cttt{\dq}{\pw}$}] (3) at (13,2) {};
			\node at (15.3,3) {$\qq{\dq}$};
			\node [point,scale=0.7,label=above:\ \ {$\cttt{\dq}{\mw}$}] (3b) at (17,2) {};
			\node [label=right:{$\ldots\ \mathcal{W}_\mw$}] (4) at (20,2) {};
			\node (q1) at (-.5,-1) {};
			\node[point,gray,scale=0.7,label=below:{\textcolor{gray}{$\tBox=\tlast$}}] (q3) at (5.5,-1) {};
			\node[point,gray,scale=0.7,label=below:{\textcolor{gray}{$\ffirst$}}] (q4) at (17.5,-1) {};
			\node[point,gray,scale=0.7,label=below:{\textcolor{gray}{$\fsec$}}] (qq4) at (19.5,-1) {};
			\node[label=right:{\textcolor{gray}{$\q$}}]  (q5) at (23.5,-1) {};
			\draw[-] (1) to  (2);
			\draw[-] (2) to  (3);
			\draw[-] (3) to  (3b);
			\draw[->] (3b) to  (4);
			\draw[-] (11,1.8) to (11,2.2);
			\node at (10.7,3.1) {$\nfff{\pw}{j_\pw}{1}$};
			\draw[-,gray] (q1) to  (q3);
			\draw[-,gray] (q3) to  (q4);
			\draw[-,gray] (q4) to  (qq4);
			\draw[->,gray] (qq4) to  (q5);
			\node at (5,.6) {$h$};
			\draw[->,thin,dashed] (q3) to  (5.5,1.7);
			\draw[->,thin,dashed] (q4) to  (17.5,1.7);
			\draw[-,thin] (5.7,-3) to (19.3,-3);
			\draw[-,thin] (5.7,-3) to (5.7,-2.8);
			\draw[-,thin] (19.3,-3) to (19.3,-2.8);
			\node at (12,-3.4) {{\scriptsize no $T$}};
			\end{tikzpicture}
			}			

So suppose that $h(\ctt{\mw})=h(\tlast)$ is in $\qq{\dq}$. Now
we track the location of $h(\ctt{\pw})=h(\tfirst)$ in $\mathcal{W}_\pw$.
As $\cttt{\pw}{j_\pw+1}=\nttt{\pw}{j_\pw+1}{1}$, $h(\tfirst)$ cannot be in the \initial{} \legW{} of 
$\qqq{\pw}{j_\pw+1}$, otherwise there is not enough room for $h(\q)$ in that \legW.
Thus, we have
\begin{multline*}
\delta_{h(\q)}\bigl(h(\tfirst),\nfff{\pw}{j_\pw}{2}\bigr)=
\delta_{h(\q)}\bigl(h(\ctt{\pw}),\cttt{\dq}{\pw}\bigr)=
\delta_{\qq{\dq}}\bigl(h(\ctt{\mw}),\cttt{\dq}{\mw}\bigr)=
\delta_{\qq{\dq}}\bigl(h(\tlast),\bnode{\tlast}{\dq}\bigr)\leq\\
\delta_{h(\q)}\bigl(h(\tlast),h(\ffirst)\bigr)=
\delta(\tlast,\ffirst)\prec\delta(\tlast,\fsec)=
\delta_{\qqq{\pw}{j_\pw}}\bigl(\nttt{\pw}{j_\pw}{\last},\nfff{\pw}{j_\pw}{2}\bigr),
\end{multline*}
and so $h(\tfirst)$ is is a node between $\nttt{\pw}{j_\pw}{\last}$ and $\nfff{\pw}{j_\pw}{2}$.
But there is no such $T$-node in $\qqq{\pw}{j_\pw}$.\\
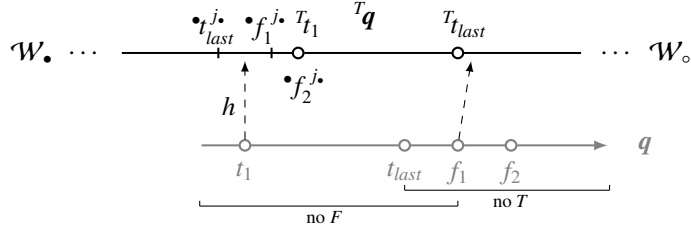
\centerline{
 \begin{tikzpicture}[>=latex,line width=.75pt, rounded corners,scale=.35]
			\node (1) at (6,4) {};
			\node[point,scale=0.7,label=above:\ \ \ {$\bnode{\tfirst}{\dq}$},label=below:\ \ {$\nfff{\pw}{j_\pw}{2}$}] (3) at (13,4) {};
			\node[point,scale=0.7,label=above:\ \ {$\bnode{\tlast}{\dq}$}] (4) at (19,4) {};
			\node (6) at (24,4) {};
			\draw[-] (1) to  (3);
			\draw[-] (3) to  (4);
			\draw[-] (4) to  (6);
			\draw[-] (10,3.8) to (10,4.2);
			\node at (9.8,5) {$\nttt{\pw}{j_\pw}{\last}$};
			\draw[-] (12,3.8) to (12,4.2);
			\node at (11.8,5) {$\nfff{\pw}{j_\pw}{1}$};
			\node at (15.5,5.4) {$\qq{\dq}$};
			\node at (5,4) {$\dots$};
			\node at (3,4) {$\mathcal{W}_\pw$};
			\node at (25,4) {$\dots$};
			\node at (27,4) {$\mathcal{W}_\mw$};
			\node (q1) at (9,.5) {};
			\node[point,gray,scale=0.7,label=below:{\textcolor{gray}{$\tfirst$}}] (q2) at (11,.5) {};
			\node[point,gray,scale=0.7,label=below:{\textcolor{gray}{$\tlast$}}] (q3) at (17,.5) {};
			\node[point,gray,scale=0.7,label=below:{\textcolor{gray}{$\ffirst$}}] (q4) at (19,.5) {};
			\node[point,gray,scale=0.7,label=below:{\textcolor{gray}{$\fsec$}}] (qq4) at (21,.5) {};
			\node[label=right:{\textcolor{gray}{$\q$}}]  (q5) at (25,.5) {};
			\draw[-,gray] (q1) to  (q2);
			\draw[-,gray] (q2) to  (q3);
			\draw[-,gray] (q3) to  (q4);
			\draw[-,gray] (q4) to  (qq4);
			\draw[->,gray] (qq4) to  (q5);
			\draw[-,thin] (17,-1.2) to (24.7,-1.2);
			\draw[-,thin] (17,-1.2) to (17,-1);
			\draw[-,thin] (24.7,-1.2) to (24.7,-1);
			\node at (21,-1.6) {{\scriptsize no $T$}};		
			\draw[-,thin] (9.3,-1.8) to (19,-1.8);	
			\draw[-,thin] (9.3,-1.8) to (9.3,-1.6);	
			\draw[-,thin] (19,-1.8) to (19,-1.6);	
			\node at (14,-2.2) {{\scriptsize no $F$}};		
			\draw[->,thin,dashed] (q2) to  (11,3.7);
			\draw[->,thin,dashed] (q4) to  (19.5,3.7);
			\node at (10.4,2) {$h$};			
			\end{tikzpicture}}
\end{itemize}
%


\item[(2)$^\dq$]
$h(\q)$ starts in $\qq{\dq}$ and ends in $\mathcal{W}_\pw$.\\
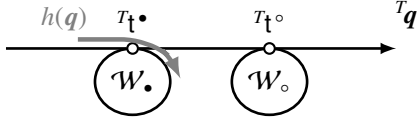
	\begin{tikzpicture}[>=latex,line width=1pt,rounded corners,xscale=.9,yscale=.8]
	                 \draw (10,-0.55) circle [radius=.55];
	                 \draw (12,-0.55) circle [radius=.55];
			\node (1) at (8,0) {};
			\node[point,scale = 0.7,fill = white,label=above:{$\cttt{\dq}{\pw}$}] (m) at (10,0) {};
			\node[point,scale = 0.7,fill =  white,label=above:{$\cttt{\dq}{\mw}$}] (2) at (12,0) {};
			\node[label=above:{$\qq{\dq}$}] (4) at (14,0) {};
			\draw[-,right] (1) to node[below] {}  (m);
			\draw[-,right] (m) to node[below] {} (2);
			\draw[->,right] (2) to node[below] {} (4);
			\node[point, draw = white] (01) at (10,-0.55) {$\mathcal{W}_\pw$};
			\node[point, draw = white] (02) at (12,-0.55) {$\mathcal{W}_\mw$};
			\draw[->,line width=.6mm,gray] (9.2,.15) -- (10,.15) -- (10.4,0) -- (10.55,-.15) -- (10.7,-.5);
			\node at (9,.5) {\textcolor{gray}{$h(\q)$}};
			\end{tikzpicture}\\
Then $h(\q)$ properly intersects the \neighbourhood{\dq} of $\mathcal{W}_\pw$ only.			
As $\ctt{\pw}=\tfirst$ and $\tfirst\prec\ffirst$ by \eqref{tprecf}, $h(\ffirst)$ is in $\mathcal{W}_\pw$ and $\cttt{\dq}{\pw}=\bnode{\tfirst}{\dq}\preceq_{h(\q)}h(\ffirst)$,
otherwise there is not enough room for $h(\q)$ in $\qq{\dq}$. Now we track the location of $h(\tfirst)$.
Again, $h(\tfirst)$ is in $\mathcal{W}_\pw$ and $\cttt{\dq}{\pw}=\bnode{\tfirst}{\dq}\preceq_{h(\q)}h(\tfirst)$ and 
otherwise there is not enough room for $h(\q)$ in $\qq{\dq}$.
As the part of $\q$ preceding $\tfirst$ is empty (containing no $F$- or $T$-nodes), if $h\colon\q\to\I$ is a \shomo{}, then we can modify it to obtain a \shomo{} from $\q$ to the restriction of $\I$ to $\mathcal{W}_\pw$, which contradicts  Lemma~\ref{l:wheel} by \eqref{pmwsame}.

\item[(3)$^\dq$]
$h(\q)$ starts in $\mathcal{W}_\mw$ and ends in $\qq{\dq}$.\\
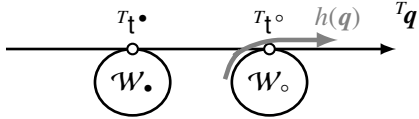
	\begin{tikzpicture}[>=latex,line width=1pt,rounded corners,xscale=.9,yscale=.8]
	                 \draw (10,-0.55) circle [radius=.55];
	                 \draw (12,-0.55) circle [radius=.55];
			\node (1) at (8,0) {};
			\node[point,scale = 0.7,fill = white,label=above:{$\cttt{\dq}{\pw}$}] (m) at (10,0) {};
			\node[point,scale = 0.7,fill =  white,label=above:{$\cttt{\dq}{\mw}$}] (2) at (12,0) {};
			\node[label=above:{$\qq{\dq}$}] (4) at (14,0) {};
			\draw[-,right] (1) to node[below] {}  (m);
			\draw[-,right] (m) to node[below] {} (2);
			\draw[->,right] (2) to node[below] {} (4);
			\node[point, draw = white] (01) at (10,-0.55) {$\mathcal{W}_\pw$};
			\node[point, draw = white] (02) at (12,-0.55) {$\mathcal{W}_\mw$};
			\draw[->,line width=.6mm,gray] (11.35,-.5) -- (11.45,-.15) -- (11.6,0) -- (12,.15) -- (13,.15);
			\node at (13,.5) {\textcolor{gray}{$h(\q)$}};
			\end{tikzpicture}\\
Then $h(\q)$ properly intersects the \neighbourhood{\dq} of $\mathcal{W}_\mw$ only.
As $\cttt{\mw}{j_\mw+1}\prec_{\qqq{\mw}{j_\mw+1}}\nfff{\mw}{j_\mw+1}{1}$,
$h(\ffirst)$ cannot be in the \initial{} \legW{} of 
$\qqq{\mw}{j_\mw+1}$, otherwise there is not enough room for $h(\q)$ in that \legW.
As $\cttt{\mw}{j_\mw-k}=\nttt{\mw}{j_\mw-k}{1}$ and $\cfff{\mw}{j_\mw-k}=\nfff{\mw}{j_\mw-k}{1}$ for all $k$ with $0<k\le |\q|$, $h(\ffirst)$ cannot be a contact of $\mathcal{W}_\mw$ different from $\cttt{\dq}{\mw}$, otherwise $h(\tfirst)$ is also a contact, contradicting \eqref{pmwsame}.
To exclude the remaining options,
we consider the two cases $\ffirst\prec\tlast$ and $\tlast\prec\ffirst$:

If $\ffirst\prec\tlast$ then $\bnode{\ffirst}{\dq}\prec_{\qq{\dq}}\cttt{\dq}{\mw}$, and so $h(\ffirst)$ cannot be in 
$\qq{\dq}$, otherwise  there is not enough room for $h(\q)$ in $\qq{\dq}$.
As $\cttt{\mw}{j_\mw}=\nttt{\mw}{j_\mw}{1}$ and $\cfff{\mw}{j_\mw}=\nfff{\mw}{j_\mw}{1}$, $h(\ffirst)$ cannot be the contact $\cttt{\dq}{\mw}=\nfff{\mw}{j_\mw}{1}$, otherwise $h(\tfirst)$ is also a contact of $\mathcal{W}_\mw$, contradicting \eqref{pmwsame}. As there is no $F$-node preceding $\ffirst$ in $\q$, there are no other options for $h(\ffirst)$ in $\mathcal{W}_\mw$.

If $\tlast\prec\ffirst$ then $\ctt{\mw}=\tlast$, and so $\cttt{\dq}{\mw}=\bnode{\tlast}{\dq}\prec_{\qq{\dq}}\bnode{\ffirst}{\dq}$.
As there is no $F$-node preceding $\bnode{\ffirst}{\dq}$ in $\qq{\dq}$, 
either $h(\ffirst)\preceq_{h(\q)}\cttt{\dq}{\mw}$ and $h(\ffirst)$ is in $\mathcal{W}_\mw$, 
or $h(\ffirst)=\bnode{\ffirst}{\dq}$, otherwise there is not enough room for $h(\q)$ in $\qq{\dq}$.
We will exclude both:
%
\begin{itemize}
\item
Suppose first that $h(\ffirst)$ is in $\mathcal{W}_\mw$.
As $\cttt{\mw}{j_\mw}=\nttt{\mw}{j_\mw}{1}$ and $\cfff{\mw}{j_\mw}=\cttt{\dq}{\mw}$ is either $\nfff{\mw}{j_\mw}{1}$
or $\nfff{\mw}{j_\mw}{2}$, we cannot have $h(\ffirst)=\nfff{\mw}{j_\mw}{1}$:
If $\cfff{\mw}{j_\mw}=\nfff{\mw}{j_\mw}{1}$ then because otherwise $h(\tfirst)$ is also a contact of $\mathcal{W}_\mw$,
and if $\cfff{\mw}{j_\mw}=\nfff{\mw}{j_\mw}{2}$ then because otherwise both $h(\tfirst)$ and $h(\fsec)$ are contacts 
of $\mathcal{W}_\mw$, contradicting \eqref{wsame} in both cases.
As the only $F$-node preceding $\fsec$ in $\q$ is $\ffirst$, the only remaining option
for $h(\ffirst)$ being in $\mathcal{W}_\mw$  is when $h(\ffirst)=\cfff{\mw}{j_\mw}=\nfff{\mw}{j_\mw}{2}$.
Now we track the location of $h(\tlast)$.\\
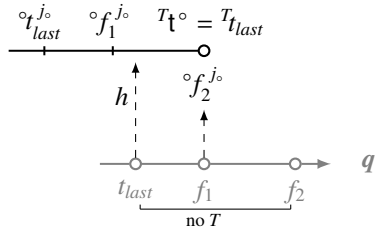
\centerline{
 \begin{tikzpicture}[>=latex,line width=.75pt, rounded corners,xscale=.3,yscale=.3]
			\node (1) at (2,4) {};
			\node[point,scale=0.7,label=above:{\ \ $\cttt{\dq}{\mw}=\bnode{\tlast}{\dq}$},label=below:{$\nfff{\mw}{j_\mw}{2}$}] (2) at (11,4) {};
			\draw[-] (1) to  (2);
			\draw[-] (4,3.8) to (4,4.2);
			\node at (3.8,5.3) {$\nttt{\mw}{j_\mw}{\last}$};
			\draw[-] (7,3.8) to (7,4.2);
			\node at (6.9,5.3) {$\nfff{\mw}{j_\mw}{1}$};
			\node (q1) at (6,-1) {};
			\node[point,gray,scale=0.7,label=below:{\textcolor{gray}{$\tlast$}}] (q2) at (8,-1) {};
			\node[point,gray,scale=0.7,label=below:{\textcolor{gray}{$\ffirst$}}] (q3) at (11,-1) {};
			\node[point,gray,scale=0.7,label=below:{\textcolor{gray}{$\fsec$}}] (q4) at (15,-1) {};
			\node[label=right:{\textcolor{gray}{$\q$}}]  (q5) at (17,-1) {};
			\draw[-,gray] (q1) to  (q2);
			\draw[-,gray] (q2) to  (q3);
			\draw[-,gray] (q3) to  (q4);
			\draw[->,gray] (q4) to  (q5);
			\draw[-,thin] (8.2,-3) to (14.8,-3);
			\draw[-,thin] (8.2,-3) to (8.2,-2.8);
			\draw[-,thin] (14.8,-3) to (14.8,-2.8);
			\node at (11,-3.5) {{\scriptsize no $T$}};		
			\draw[->,thin,dashed] (q2) to  (8,3.5);
			\draw[->,thin,dashed] (q3) to  (11,1.4);
			\node at (7.4,2) {$h$};			
			\end{tikzpicture}
			}\\
As $\cttt{\mw}{j_\mw+1}=\nttt{\mw}{j_\mw+1}{1}$, $h(\tlast)$ cannot
be in the \initial{} \legW{} of $\qqq{\mw}{j_\mw+1}$, otherwise there is not enough room for $h(\q)$ in that \legW.
Thus, we have
\[
\delta_{h(\q)}\bigl(h(\tlast),\nfff{\mw}{j_\mw}{2}\bigr)=
\delta_{h(\q)}\bigl(h(\tlast),h(\ffirst)\bigr)=
\delta(\tlast,\ffirst)<\delta(\tlast,\fsec)=
\delta_{\qqq{\mw}{j_\mw}}\bigl(\nttt{\mw}{j_\mw}{\last},\nfff{\mw}{j_\mw}{2}\bigr).
\]
As $\cttt{\mw}{j_\mw}=\nttt{\mw}{j_\mw}{1}\prec_{\qqq{\mw}{j_\mw}}\nttt{\mw}{j_\mw}{\last}$, 
it follows that $h(\tlast)$ is between $\nttt{\mw}{j_\mw}{\last}$ and $\nfff{\mw}{j_\mw}{2}$.
But there is no such $T$-node in $\qqq{\mw}{j_\mw}$, and so $h(\ffirst)$ cannot be in $\mathcal{W}_\mw$.

\item
If $h(\ffirst)=\bnode{\ffirst}{c}$ then $h(\tlast)=\bnode{\tlast}{\dq}=\cttt{\dq}{\mw}$.
We track the location of $h(\tlbo)$.
As $\cttt{\mw}{j_\mw+1}=\nttt{\mw}{j_\mw+1}{1}$, $h(\tlbo)$ cannot
be in the \initial{} \legW{} of $\qqq{\mw}{j_\mw+1}$, otherwise there is not enough room for $h(\q)$ in that \legW.
As $\nttt{\mw}{j_\mw}{\last}\prec_{\qqq{\mw}{j_\mw}}\nfff{\mw}{j_\mw}{1}\preceq_{\qqq{\mw}{j_\mw}}
\cfff{\mw}{j_\mw}=h(\tlast)$, we have 
$\nttt{\mw}{j_\mw}{1}\preceq_{\qqq{\mw}{j_\mw}}\nttt{\mw}{j_\mw}{\lbo}\prec_{\qqq{\mw}{j_\mw}}h(\tlbo)\prec_{\qqq{\mw}{j_\mw}}\cfff{\mw}{j_\mw}$, and so 
$h(\tlbo)=\nttt{\mw}{j_\mw}{\last}$ must hold (as $\tlast$ is the only $T$-node succeeding $\tlbo$ in $\q$).
Therefore,
\begin{multline*}
\delta(\tlbo,\tlast)=
\delta_{h(\q)}\bigl(h(\tlbo),h(\tlast)\bigr)=\\
=\left\{
\begin{array}{ll}
\delta_{\qqq{\mw}{j_\mw}}\bigl(\nttt{\mw}{j_\mw}{\last},\nfff{\mw}{j_\mw}{2}\bigr)=\delta(\tlast,\fsec), & \mbox{if $\delta(\tlbo,\tlast)=\delta(\tlast,\ffirst)$,}\\[3pt]
\delta_{\qqq{\mw}{j_\mw}}\bigl(\nttt{\mw}{j_\mw}{\last},\nfff{\mw}{j_\mw}{1}\bigr)=\delta(\tlast,\ffirst), & \mbox{if $\delta(\tlbo,\tlast)\ne\delta(\tlast,\ffirst)$,}
\end{array}
\right.
\end{multline*}
with both cases being impossible.\\
\centerline{
 \begin{tikzpicture}[>=latex,line width=.75pt, rounded corners,scale=.3]
			\node (1) at (4,4) {};
			\node[point,scale=0.7,label=above:{$\bnode{\tlast}{\dq}$},label=below:{$\cfff{\mw}{j_\mw}$}] (2) at (14,4) {};
			\node[label=above:{$\qq{\dq}$}] (3) at (23,4) {};
			\draw[-] (1) to  (2);
			\draw[->] (2) to  (3);
			\draw[-] (8,3.8) to (8,4.2);
			\node at (7.8,5.2) {$\nttt{\mw}{j_\mw}{\last}$};
			\draw[-] (18,3.8) to (18,4.2);
			\node at (17.8,5.2) {$\bnode{\ffirst}{\dq}$};
			\draw[-] (20,3.8) to (20,4.2);
			\node at (19.8,5.2) {$\bnode{\fsec}{\dq}$};
			\node (q1) at (6,0) {};
			\node[point,gray,scale=0.7,label=below:{\textcolor{gray}{$\tlbo$}}] (q2) at (8,0) {};
			\node[point,gray,scale=0.7,label=below:{\textcolor{gray}{$\tlast$}}] (q3) at (14,0) {};
			\node[point,gray,scale=0.7,label=below:{\textcolor{gray}{$\ffirst$}}] (q4) at (18,0) {};
			\node[point,gray,scale=0.7,label=below:{\textcolor{gray}{$\fsec$}}] (qq4) at (20,0) {};
			\node[label=right:{\textcolor{gray}{$\q$}}]  (q5) at (23,0) {};
			\draw[-,gray] (q1) to  (q2);
			\draw[-,gray] (q2) to  (q3);
			\draw[-,gray] (q3) to  (q4);
			\draw[-,gray] (q4) to  (qq4);
			\draw[->,gray] (qq4) to  (q5);
			\draw[-,thin] (8.2,-2) to (13.8,-2);
			\draw[-,thin] (8.2,-2) to (8.2,-1.8);
			\draw[-,thin] (13.8,-2) to (13.8,-1.8);
			\node at (11,-2.4) {{\scriptsize no $T$}};	
			\draw[-,thin] (14.2,-2) to (19.8,-2);	
			\draw[-,thin] (14.2,-2) to (14.2,-1.8);
			\draw[-,thin] (19.8,-2) to (19.8,-1.8);
			\node at (17,-2.4) {{\scriptsize no $T$}};	
			\draw[->,thin,dashed] (q2) to  (7,3.5);
			\draw[->,thin,dashed] (q2) to  (9,3.5);
			\draw[->,thin,dashed] (q3) to  (14,2.3);
			\draw[->,thin,dashed] (q4) to  (18,3.5);
			
			\node at (6.6,1.7) {$h$};			
			\end{tikzpicture}
			}

 \end{itemize}
%

	
\item[(4)$^\dq$]
$h(\q)$ ends in $\mathcal{W}_\mw$ and $\cttt{\dq}{\mw} \prec_{h(\q)}h(\ctt{\mw})$.\\
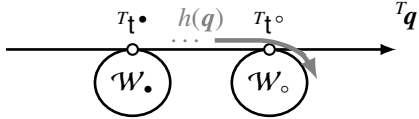
	\begin{tikzpicture}[>=latex,line width=1pt,rounded corners,xscale=.9,yscale=.8]
	                 \draw (10,-0.55) circle [radius=.55];
	                 \draw (12,-0.55) circle [radius=.55];
			\node (1) at (8,0) {};
			\node[point,scale = 0.7,fill = white,label=above:{$\cttt{\dq}{\pw}$}] (m) at (10,0) {};
			\node[point,scale = 0.7,fill =  white,label=above:{$\cttt{\dq}{\mw}$}] (2) at (12,0) {};
			\node[label=above:{$\qq{\dq}$}] (4) at (14,0) {};
			\draw[-,right] (1) to node[below] {}  (m);
			\draw[-,right] (m) to node[below] {} (2);
			\draw[->,right] (2) to node[below] {} (4);
			\node[point, draw = white] (01) at (10,-0.55) {$\mathcal{W}_\pw$};
			\node[point, draw = white] (02) at (12,-0.55) {$\mathcal{W}_\mw$};
			\draw[->,line width=.6mm,gray] (11.2,.15) -- (12,.15) -- (12.4,0) -- (12.55,-.15) -- (12.7,-.5);
			\node at (11,.5) {\textcolor{gray}{$h(\q)$}};
			\node at (10.8,.15) {\textcolor{gray}{$\dots$}};
			\end{tikzpicture}\\		
Then $h(\q)$ definitely properly intersects the \neighbourhood{\dq} of $\mathcal{W}_\mw$, and it might also 
properly intersect the \neighbourhood{\dq} of $\mathcal{W}_\pw$. 
As $\nfff{\mw}{j_\mw+\ell}{1}\preceq_{\qqq{\mw}{j_\mw+\ell}}\cfff{\mw}{j_\mw+\ell}$ for all $\ell\le|\q|$, 
$h(\ffirst)$ cannot be in the \final{} \legW{} of $\qqq{\mw}{j_\mw+\ell}$ for any $\ell\le|\q|$, otherwise there is not enough room for $h(\q)$ in that \legW. 
As $\cttt{\mw}{j_\mw+\ell}=\nttt{\mw}{j_\mw+\ell}{1}$ and $\cfff{\mw}{j_\mw+\ell}=\nfff{\mw}{j_\mw+\ell}{1}$ for all $\ell$ with $1\le\ell\le|\q|$, 
$h(\ffirst)$ cannot be a contact of $\mathcal{W}_\mw$ different from $\cttt{\dq}{\mw}$, otherwise $h(\tfirst)$ is also a contact, contradicting \eqref{pmwsame}. 
As there is no $F$-node preceding $\ffirst$ in $\q$, it follows that $h(\ffirst)$ must be in $\qq{\dq}$.	
To exclude the remaining options, we consider the two cases $\ffirst\prec\tlast$ and $\tlast\prec\ffirst$:

If $\ffirst\prec\tlast$ then $\ctt{\mw}=\tStar$, where $\tStar$ is the first $T$-node succeeding $\ffirst$.
As by our assumption $h(\q)$ ends in $\mathcal{W}_\mw$ and $\cttt{\dq}{\mw} \prec_{h(\q)}h(\ctt{\mw})$, 
it follows that
$\bnode{\ffirst}{\dq}\prec_{\qq{\dq}}h(\ffirst)\preceq_{\qq{\dq}}\cttt{\dq}{\mw}=\bnode{\tStar}{\dq}$.
Now we track the location of $h(\tBox)$ for the last $T$-node $\tBox$ preceding $\ffirst$. 
As $\tfirst\preceq\tBox\prec\ffirst$, it follows that $h(\tBox)$ is
between $\bnode{\tBox}{\dq}$ and $h(\ffirst)$, and so between $\bnode{\tBox}{\dq}$ and $\bnode{\tStar}{\dq}$.
But there is no such $T$-node in $\qq{\dq}$.\\
\centerline{
 \begin{tikzpicture}[>=latex,line width=.75pt, rounded corners,scale=.35]
			\node (1) at (0,4) {};
			\node[point,scale=0.7,label=above:{$\cttt{\dq}{\pw}=\bnode{\tfirst}{\dq}\qquad$}] (2) at (2,4) {};
			\node[point,scale=0.7,label=above:{\ \ $\bnode{\tStar}{\dq}=\cttt{\dq}{\mw}$}] (3) at (12,4) {};
			\node[label=above:{$\qq{\dq}$}] (4) at (16,4) {};
			\draw[-] (1) to  (2);
			\draw[-] (2) to  (3);
			\draw[->] (3) to  (4);
			\draw[-] (4,3.8) to (4,4.2);
			\node at (3.8,5.2) {$\bnode{\tBox}{\dq}$};
			\draw[-] (9,3.8) to (9,4.2);
			\node at (8.8,5.2) {$\bnode{\ffirst}{\dq}$};
			\node (q1) at (0,0) {};
			\node[point,gray,scale=0.7,label=below:{\textcolor{gray}{$\tBox$}}] (q2) at (4,0) {};
			\node[point,gray,scale=0.7,label=below:{\textcolor{gray}{$\ffirst$}}] (q3) at (9,0) {};
			\node[point,gray,scale=0.7,label=below:{\textcolor{gray}{$\tStar$}}] (q4) at (12,0) {};
			\node[label=right:{\textcolor{gray}{$\q$}}]  (q5) at (16,0) {};
			\draw[-,gray] (q1) to  (q2);
			\draw[-,gray] (q2) to  (q3);
			\draw[-,gray] (q3) to  (q4);
			\draw[->,gray] (q4) to  (q5);
			\draw[-,thin] (4.2,-2) to (11.8,-2);
			\draw[-,thin] (4.2,-2) to (4.2,-1.8);
			\draw[-,thin] (11.8,-2) to (11.8,-1.8);
			\node at (8,-2.4) {{\scriptsize no $T$}};		
			\draw[->,thin,dashed] (q2) to  (6,3.5);
			\draw[->,thin,dashed] (q3) to  (11,3.5);
			\node at (4.3,2) {$h$};			
			\end{tikzpicture}
			}

If $\tlast\prec\ffirst$ then $\ctt{\mw}=\tlast$. As $h(\q)$ ends in $\mathcal{W}_\mw$ and 
$\cttt{\dq}{\mw} \prec_{h(\q)}h(\ctt{\mw})$, it follows that $\cttt{\dq}{\mw} \prec_{h(\q)}h(\tlast)\prec_{h(\q)}h(\ffirst)$.
Thus, $h(\ffirst)$  cannot be in $\qq{\dq}$, leaving us no options.
%
\end{itemize}


Next, we deal with the cases when $h(\q)\cap\qq{\uq}\ne\emptyset$:
\begin{itemize}
\item[(1)$^\uq$]
$h(\q)$ starts in $\mathcal{W}_\pw$ and $h(\cff{\pw})\prec_{h(\q)}\cfff{\uq}{\pw}$.\\
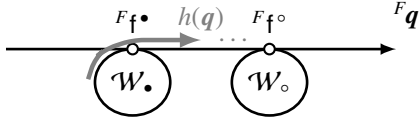
\begin{tikzpicture}[>=latex,line width=1pt,rounded corners,xscale=.9,yscale=.8]
	                 \draw (10,-0.55) circle [radius=.55];
	                 \draw (12,-0.55) circle [radius=.55];
			\node (1) at (8,0) {};
			\node[point,scale = 0.7,fill = white,label=above:{$\cfff{\uq}{\pw}$}] (m) at (10,0) {};
			\node[point,scale = 0.7,fill =  white,label=above:{$\cfff{\uq}{\mw}$}] (2) at (12,0) {};
			\node[label=above:{$\qq{\uq}$}] (4) at (14,0) {};
			\draw[-,right] (1) to node[below] {}  (m);
			\draw[-,right] (m) to node[below] {} (2);
			\draw[->,right] (2) to node[below] {} (4);
			\node[point, draw = white] (01) at (10,-0.55) {$\mathcal{W}_\pw$};
			\node[point, draw = white] (02) at (12,-0.55) {$\mathcal{W}_\mw$};
			\draw[->,line width=.6mm,gray] (9.35,-.5) -- (9.45,-.15) -- (9.6,0) -- (10,.15) -- (11,.15);
			\node at (11,.5) {\textcolor{gray}{$h(\q)$}};
			\node at (11.5,.15) {\textcolor{gray}{$\dots$}};
			\end{tikzpicture}\\
%
%
Then $h(\q)$ definitely properly intersects the \neighbourhood{\uq} of $\mathcal{W}_\pw$, and it might also 
properly intersect the \neighbourhood{\uq} of $\mathcal{W}_\mw$. 
As $\cff{\pw}=\ffirst$, we have $h(\ffirst)\prec_{h(\q)}\cfff{\uq}{\pw}$ and $h(\ffirst)$ is in $\mathcal{W}_\pw$.
As $\cttt{\pw}{i_\pw+1}\prec_{\qqq{\pw}{i_\pw+1}}\nfff{\pw}{i_\pw+1}{1}$ 
for either choice of $\cttt{\pw}{i_\pw+1}$, $h(\ffirst)$ cannot be in the \initial{} \legW{} of $\qqq{\pw}{i_\pw+1}$, otherwise there is not enough room for $h(\q)$ in that \legW.
As $\cttt{\pw}{j_\pw-k}=\nttt{\pw}{j_\pw-k}{1}$ and $\cfff{\pw}{j_\pw-k}=\nfff{\pw}{j_\pw-k}{1}$ for all $k$ with $0<k\le |\q|$, $h(\ffirst)$ cannot be a contact of 
$\mathcal{W}_\pw$ different from $\cfff{\uq}{\pw}$, otherwise $h(\tfirst)$ is also a contact of $\mathcal{W}_\pw$, contradicting \eqref{pmwsame}.
And as $\cttt{\pw}{i_\pw}=\nttt{\pw}{i_\pw}{1}$ and $\cfff{\pw}{i_\pw}=\nfff{\pw}{i_\pw}{2}$, we cannot have
$h(\ffirst)=\nfff{\pw}{i_\pw}{1}$, otherwise both $h(\tfirst)$ and $h(\fsec)$ are contacts of $\mathcal{W}_\pw$, again contradicting  \eqref{pmwsame}.
As the only $F$-node preceding $\fsec$ in $\q$ is $\ffirst$, there are no more options for $h(\ffirst)$.\\
			\centerline{
				\begin{tikzpicture}[>=latex,line width=.75pt, rounded corners,scale=.35]
				\node[label=right:{$\qqq{\pw}{i_\pw+1}$}] (8) at (14,4.5) {};
				\node[label=right:{$\qqq{\pw}{i_\pw}$}] (88) at (8,4.5) {};
				\node[label=right:{$\qqq{\pw}{i_\pw-1}$}] (888) at (4,4.5) {};
				\node[label=right:{$\qqq{\pw}{i_\pw-2}$}] (8888) at (0,4.5) {};
				\node (2) at (1,2) {};
				\node[point,scale=0.7,label=below:{$\nttt{\pw}{i_\pw-2}{1}$}] (3) at (3,2) {};
				\node[point,scale=0.7,label=above:{$\ \ \nfff{\pw}{i_\pw-2}{1}$},label=below:{$\nttt{\pw}{i_\pw-1}{1}$}] (4) at (7,2) {};
				\node[point,scale=0.7,label=above:{$\ \ \nfff{\pw}{i_\pw-1}{1}$},label=below:{$\nttt{\pw}{i_\pw}{1}$}] (5) at (11,2) {};
				\node[point,scale=0.7,label=above:{$\ \ \bnode{\ffirst}{\uq}$},label=below:{$\nfff{\pw}{i_\pw}{2}$}] (6) at (17,2) {};
				\node (7) at (24,2) {};
				\draw[-] (15,1.7) to  (15,2.3);
				\draw[->] (8) to  (6);
				\draw[->] (88) to  (5);
				\draw[->] (888) to  (4);
				\draw[->] (8888) to  (3);
				\draw[-] (2) to  (3);
				\draw[-] (3) to  (4);
				\draw[-] (4) to  (5);
				\draw[-] (5) to  (6);
				\draw[->] (6) to  (7);
				\node at (.2,2) {.};
				\node at (-.2,2) {.};
				\node at (-.6,2) {.};
				\node at (-2,2) {$\mathcal{W}_\pw$};
				\node at (15,.8) {$\nfff{\pw}{i_\pw}{1}$};
				\node at (22,2.8) {$\qq{\uq}$};
			\node[label=left:{\textcolor{gray}{$\q$}}] (q1) at (3,-2) {};
			\node[point,gray,scale=0.7,label=below:{\textcolor{gray}{$\tfirst$}}] (q2) at (7,-2) {};
			\node[point,gray,scale=0.7,label=below:{\textcolor{gray}{$\ffirst$}}] (q3) at (11,-2) {};
			\node[point,gray,scale=0.7,label=below:{\textcolor{gray}{$\fsec$}}] (q4) at (13,-2) {};
			\node (q5) at (18,-2) {};
			\draw[-,gray] (q1) to  (q2);
			\draw[-,gray] (q2) to  (q3);
			\draw[-,gray] (q3) to  (q4);
			\draw[->,gray] (q4) to  (q5);
			\draw[-,thin] (3.3,-3.8) to (10.7,-3.8);
			\draw[-,thin] (3.3,-3.8) to (3.3,-3.6);
			\draw[-,thin] (10.7,-3.8) to (10.7,-3.6);
			\node at (7,-4.2) {{\scriptsize no $F$}};
			\draw[-,thin] (11.3,-3.8) to (13,-3.8);
			\draw[-,thin] (11.3,-3.8) to (11.3,-3.6);
			\draw[-,thin] (13,-3.8) to (13,-3.6);
			\node at (12.2,-4.2) {{\scriptsize no $F$}};
			\node at (10.3,-1) {$h$};
			\draw[->,thin,dashed] (q3) to  (11,-.1);
				\end{tikzpicture}
			}

			\item[(2)$^\uq$]
			$h(\q)$ starts in $\qq{\uq}$ and ends in $\mathcal{W}_\pw$.\\
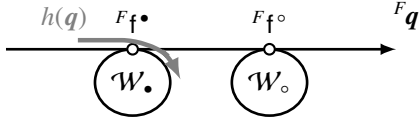
	\begin{tikzpicture}[>=latex,line width=1pt,rounded corners,xscale=.9,yscale=.8]
	                 \draw (10,-0.55) circle [radius=.55];
	                 \draw (12,-0.55) circle [radius=.55];
			\node (1) at (8,0) {};
			\node[point,scale = 0.7,fill = white,label=above:{$\cfff{\uq}{\pw}$}] (m) at (10,0) {};
			\node[point,scale = 0.7,fill =  white,label=above:{$\cfff{\uq}{\mw}$}] (2) at (12,0) {};
			\node[label=above:{$\qq{\uq}$}] (4) at (14,0) {};
			\draw[-,right] (1) to node[below] {}  (m);
			\draw[-,right] (m) to node[below] {} (2);
			\draw[->,right] (2) to node[below] {} (4);
			\node[point, draw = white] (01) at (10,-0.55) {$\mathcal{W}_\pw$};
			\node[point, draw = white] (02) at (12,-0.55) {$\mathcal{W}_\mw$};
			\draw[->,line width=.6mm,gray] (9.2,.15) -- (10,.15) -- (10.4,0) -- (10.55,-.15) -- (10.7,-.5);
			\node at (9,.5) {\textcolor{gray}{$h(\q)$}};
			\end{tikzpicture}\\
Then $h(\q)$ properly intersects the \neighbourhood{\uq} of $\mathcal{W}_\pw$ only.	
As $\cff{\pw}=\ffirst$, 	
we cannot have $h(\ffirst)\prec_{h(\q)}\cfff{\uq}{\pw}=\bnode{\ffirst}{\uq}$,
otherwise there is not enough room for $h(\q)$ in $\qq{\uq}$.
Thus, $\cfff{\uq}{\pw}\preceq_{h(\q)}h(\ffirst)$ and $h(\ffirst)$ is in $\mathcal{W}_\pw$.
As $\nfff{\pw}{i_\pw+\ell}{1}\preceq_{\qqq{\pw}{i_\pw+\ell}}\cfff{\pw}{i_\pw+\ell}$ for all $\ell\le |\q|$, $h(\ffirst)$ cannot be in the \final{} \legW{} of $\qqq{\pw}{i_\pw}$ for any $\ell\le |\q|$,
otherwise	there is not enough room for $h(\q)$ in that \legW.
To exclude the remaining options, we consider the two cases $\ffirst\prec\tlast$ and $\tlast\prec\ffirst$:

If $\ffirst\prec\tlast$ then $\cttt{\pw}{i_\pw+\ell}=\nttt{\pw}{i_\pw+\ell}{\Box}$ and $\cfff{\pw}{i_\pw+\ell}=\nfff{\pw}{i_\pw+\ell}{1}$ for all $\ell$ with $1\le\ell\le |\q|$, where $\tBox$ is the last $T$-node preceding $\ffirst$.
Thus, $h(\ffirst)$ cannot be a contact of $\mathcal{W}_\pw$ different from $\cfff{\uq}{\pw}$,
otherwise $h(\tBox)$ is also a contact, contradicting \eqref{pmwsame}.	
As there is no $F$-node preceding $\ffirst$ in $\q$, 
the only remaining option for $h(\ffirst)$ is $h(\ffirst)=\cfff{\uq}{\pw}$.
Then all contacts of $\mathcal{W}_\pw$ are in $F^{\I}$ by \eqref{pmwsame}.
We track the location of $h(\tStar)$ for the first $T$-node $\tStar$ succeeding $\ffirst$ in $\q$.\\
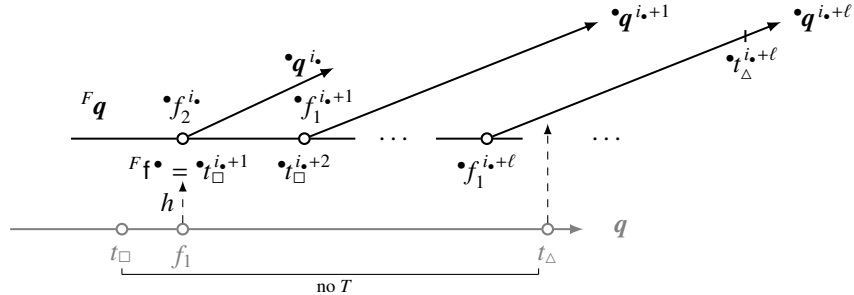
\centerline{			
		\begin{tikzpicture}[>=latex,line width=.75pt, rounded corners,scale=.4]
			\node[label=above right:{$\qq{\uq}$}] (1) at (0,2) {};
			\node[point,scale=0.7,label=above:{$\nfff{\pw}{i_\pw}{2}$},label=below:{\ \ $\cfff{\uq}{\pw}=\nttt{\pw}{i_\pw+1}{\Box}$}] (2) at (4,2) {};
			\node[point,scale=0.7,label=above:{$\quad\ \ \nfff{\pw}{i_\pw+1}{1}$},label=below:{$\nttt{\pw}{i_\pw+2}{\Box}$}] (3) at (8,2) {};
			\node (u1) at (10,2) {};
			\node (u2) at (12,2) {};
			\node[point,scale=0.7,label=below:{$\nfff{\pw}{i_\pw+\ell}{1}$}] (4) at (14,2) {};
			\node (5) at (15.5,2) {};
			\node[label=right:{\!\!\!\!$\qqq{\pw}{i_\pw+\ell}$}] (6) at (24,6) {};
			\node[label=right:{\!\!\!\!$\qqq{\pw}{i_\pw+1}$}] (66) at (18,6) {};
			\node[label=left:{$\qqq{\pw}{i_\pw}$}] (666) at (9.3,4.5) {};
			\node (q1) at (-2,-1) {};
			\node[point,gray,scale=0.7,label=below:{\textcolor{gray}{$\tBox$}}] (q2) at (2,-1) {};
			\node[point,gray,scale=0.7,label=below:{\textcolor{gray}{$\ffirst$}}] (q22) at (4,-1) {};
			\node[point,gray,scale=0.7,label=below:{\textcolor{gray}{$\tStar$}}] (q3) at (16,-1) {};
			\node[label=right:{\textcolor{gray}{$\q$}}] (q4) at (17.5,-1) {};
			\draw[-] (1) to  (2);
			\draw[-] (2) to  (3);
			\draw[-] (3) to  (u1);
			\draw[-] (u2) to  (4);
			\draw[-] (4) to  (5);
			\draw[->] (4) to  (6);
			\draw[->] (3) to  (66);
			\draw[->] (2) to  (666);
			\draw[-,gray] (q1) to  (q2);
			\draw[-,gray] (q2) to  (q22);
			\draw[-,gray] (q22) to  (q3);
			\draw[->,gray] (q3) to  (q4);
			\node at (11,2) {$\dots$};
			\node at (18,2) {$\dots$};
			\draw[-] (22.5,5.2) to (22.5,5.6);
			\node at (22.5,4.5) {$\ \ \nttt{\pw}{i_\pw+\ell}{\triangle}$};
			\node at (3.5,-.1) {$h$};
			\draw[->,thin,dashed] (q22) to  (4,.6);
			\draw[->,thin,dashed] (q3) to  (16,2.5);
			\draw[-,thin] (2,-2.5) to (15.7,-2.5);
			\draw[-,thin] (2,-2.5) to (2,-2.3);
			\draw[-,thin] (15.7,-2.5) to (15.7,-2.3);
			\node at (9,-2.9) {{\scriptsize no $T$}};
			\end{tikzpicture}}
	
As $\cfff{\pw}{i_\pw}=\nfff{\pw}{i_\pw}{2}$, $h(\tStar)$ cannot be in the \final{} \legW{} of $\qqq{\pw}{i_\pw}$,
otherwise $h(\fsec)$ is also in that \legW{} and there is not enough room for $h(\q)$ in that \legW{}
(unlike in Example~\ref{e:bike2}~$(i)$).
Further, $h(\tStar)$ cannot be a contact of $\mathcal{W}_\pw$, as all contacts of $\mathcal{W}_\pw$ are in
$F^{\I}$. As there is no $T$-node between $\tBox$ and $\ffirst$ in $\q$,
$h(\tStar)$ must be in the \final{} \legW{} of $\qqq{\pw}{i_\pw+\ell}$ for some $\ell$ with $1\le \ell\le |\q|$.
Then
			\[
			\delta_{\qqq{\pw}{i_\pw+\ell}}\bigl(\nfff{\pw}{i_\pw+\ell}{1},h(\tStar)\bigr)<
			\delta_{h(\q)}\bigl(\cfff{\uq}{\pw},h(\tStar)\bigr)=
			\delta_{h(\q)}\bigl(h(\ffirst),h(\tStar)\bigr)=
			\delta(\ffirst,\tStar)=
			\delta_{\qqq{\pw}{i_\pw+\ell}}\bigl(\nfff{\pw}{i_\pw+\ell}{1},\nttt{\pw}{i_\pw+\ell}{\triangle}\bigr),
			\]
			and so $h(\tStar)$ is between $\nfff{\pw}{i_\pw+\ell}{1}$ and $\nttt{\pw}{i_\pw+\ell}{\triangle}$. 
			But there is no such $T$-node in $\qqq{\pw}{i_\pw+\ell}$.

If $\tlast\prec\ffirst$ then there are two cases, depending on the relationship between 
$\delta(\ffirst,\fsec)$ and $\delta(\tfirst,\ffirst)$:
\begin{itemize}
\item
If $\delta(\ffirst,\fsec)<\delta(\tfirst,\ffirst)$ then $\cttt{\pw}{i_\pw+\ell}=\nttt{\pw}{i_\pw+\ell}{1}$ and $\cfff{\pw}{i_\pw+\ell}=\nfff{\pw}{i_\pw+\ell}{1}$, for all $\ell$ with $1\le\ell\le |\q|$.
Thus, $h(\ffirst)$ cannot be a contact of $\mathcal{W}_\pw$ different from $\cfff{\uq}{\pw}$, otherwise $h(\tfirst)$ is also a contact, contradicting \eqref{pmwsame}. 
As there is no $F$-node preceding $\ffirst$ in $\q$, 
the only remaining option for $h(\ffirst)$ is $h(\ffirst)=\cfff{\uq}{\pw}$.
Next, we track the location of $h(\fsec)$.
As $\cfff{\pw}{i_\pw}=\nfff{\pw}{i_\pw}{2}$,  $h(\fsec)$ cannot be in the \final{} \legW{} of $\qqq{\pw}{i_\pw}$,
otherwise there is not enough room for $h(\q)$ in that \legW{}
(unlike in Example~\ref{e:bike2}~$(i)$).
Thus,
\begin{multline*}
\delta_{h(\q)}\bigl(\nttt{\pw}{i_\pw+1}{1},h(\fsec)\bigr)=
\delta_{h(\q)}\bigl(\cfff{\uq}{\pw},h(\fsec)\bigr)=
\delta_{h(\q)}\bigl(h(\ffirst),h(\fsec)\bigr)=\\
\delta(\ffirst,\fsec)<\delta(\tfirst,\ffirst)=
\delta_{\qqq{\pw}{i_pw+1}}\bigl(\nttt{\pw}{i_\pw+1}{1},\nfff{\pw}{i_\pw+1}{1}\bigr),
\end{multline*}
and so $h(\fsec)$ is between $\nttt{\pw}{i_\pw+1}{1}$ and $\nfff{\pw}{i_\pw+1}{1}$. But there is no such $F$-node in 
$\qqq{\pw}{i_\pw+1}$.\\
				
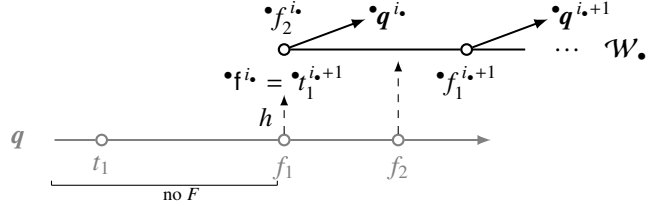
\centerline{
					\begin{tikzpicture}[>=latex,line width=.75pt, rounded corners,scale=.3]
					\node[point,scale=0.7,label=below:{$\cfff{\pw}{i_\pw}=\nttt{\pw}{i_\pw+1}{1}$},label=above:{$\nfff{\pw}{i_\pw}{2}$}] (2) at (5,2) {};
					\node[point,scale=0.7,label=below:{$\nfff{\pw}{i_\pw+1}{1}$}] (3) at (13,2) {};
					\node (4) at (16,2) {};
					\node[label=right:\!\!\!\!\!\!{$\qqq{\pw}{i_\pw+1}$}] (5) at (17,3.5) {};
					\node[label=right:\!\!\!\!\!\!{$\qqq{\pw}{i_\pw}$}] (55) at (9,3.5) {};
					\draw[-] (2) to  (3);
					\draw[-] (3) to  (4);
					\draw[->] (3) to  (5);
					\draw[->] (2) to  (55);
					\node at (17,2) {.};
					\node at (17.4,2) {.};
					\node at (17.8,2) {.};
					\node at (20,2) {$\mathcal{W}_\pw$};
			\node[label=left:{\textcolor{gray}{$\q$}}] (q1) at (-5.5,-2) {};
			\node[point,gray,scale=0.7,label=below:{\textcolor{gray}{$\tfirst$}}] (q2) at (-3,-2) {};
			\node[point,gray,scale=0.7,label=below:{\textcolor{gray}{$\ffirst$}}] (q3) at (5,-2) {};
			\node[point,gray,scale=0.7,label=below:{\textcolor{gray}{$\fsec$}}] (q4) at (10,-2) {};
			\node (q5) at (14.5,-2) {};
			\draw[-,gray] (q1) to  (q2);
			\draw[-,gray] (q2) to  (q3);
			\draw[-,gray] (q3) to  (q4);
			\draw[->,gray] (q4) to  (q5);
			\draw[-,thin] (-5.2,-4) to (4.7,-4);
			\draw[-,thin] (-5.2,-4) to (-5.2,-3.8);
			\draw[-,thin] (4.7,-4) to (4.7,-3.8);
			\node at (.5,-4.4) {{\scriptsize no $F$}};
			\node at (4.2,-1) {$h$};
			\draw[->,thin,dashed] (q3) to  (5,0);
			\draw[->,thin,dashed] (q4) to  (10,1.5);
\end{tikzpicture}
				}

\item
If $\delta(\ffirst,\fsec)\geq\delta(\tfirst,\ffirst)$ then $\cttt{\pw}{i_\pw+\ell}=\nttt{\pw}{i_\pw+\ell}{\Box}$ and $\cfff{\pw}{i_\pw+\ell}=\nfff{\pw}{i_\pw+\ell}{2}$, for all $\ell$ with $1\le\ell\le |\q|$,
where $\tBox$ is the last $T$-node preceding $\ffirst$.
Thus, $h(\ffirst)=\nfff{\pw}{i_\pw+\ell}{1}$ cannot hold for any $\ell$ with $1\le\ell\le |\q|$, otherwise both $h(\tBox)$ and
$h(\fsec)$ are contacts of $\mathcal{W}_\pw$, contradicting \eqref{pmwsame}. 
As the only $F$-node preceding $\fsec$ in $\q$ is $\ffirst$, 
the only remaining option for $h(\ffirst)$ is to be a contact of $\mathcal{W}_\pw$, that is, 
$h(\ffirst)=\nttt{\pw}{i_\pw+\ell}{\Box}$ for some $\ell$ with $1\le\ell\le |\q|$.
Again, we track the location of $h(\fsec)$.
As $\cfff{\pw}{i_\pw}=\nfff{\pw}{i_\pw}{2}$,  $h(\fsec)$ cannot be in the \final{} \legW{} of $\qqq{\pw}{i_\pw}$,
otherwise there is not enough room for $h(\q)$ in that \legW.
Thus,
\[
\delta_{h(\q)}\bigl(\nttt{\pw}{i_\pw+\ell}{\Box},h(\fsec)\bigr)=
\delta_{h(\q)}\bigl(h(\ffirst),h(\fsec)\bigr)=
\delta(\ffirst,\fsec)\geq\delta(\tfirst,\ffirst)>\delta(\tBox,\ffirst)=
\delta_{\qqq{\pw}{i_\pw+\ell}}\bigl(\nttt{\pw}{i_\pw+\ell}{\Box},\nfff{\pw}{i_\pw+\ell}{1}\bigr).
\]
On the other hand,
\[
\delta_{h(\q)}\bigl(\nttt{\pw}{i_\pw+\ell}{\Box},h(\fsec)\bigr)=
\delta_{h(\q)}\bigl(h(\ffirst),h(\fsec)\bigr)=
\delta(\ffirst,\fsec)<\delta(\tBox,\fsec)=
\delta_{\qqq{\pw}{i_\pw+\ell}}\bigl(\nttt{\pw}{i_\pw+\ell}{\Box},\nfff{\pw}{i_\pw+\ell}{2}\bigr),
\]
and so $h(\fsec)$ is between $\nfff{\pw}{i_\pw+\ell}{1}$ and $\nfff{\pw}{i_\pw+\ell}{2}$. But there is no such $F$-node in 
$\qqq{\pw}{i_\pw+\ell}$.\\
				
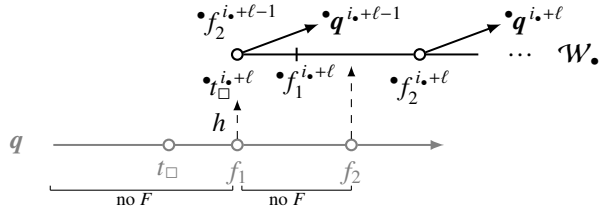
\centerline{
					\begin{tikzpicture}[>=latex,line width=.75pt, rounded corners,scale=.3]
					\node[point,scale=0.7,label=below:\!\!\!\!{$\nttt{\pw}{i_\pw+\ell}{\Box}$},label=above:{$\nfff{\pw}{i_\pw+\ell-1}{2}$}] (2) at (5,2) {};
					\node[point,scale=0.7,label=below:{$\nfff{\pw}{i_\pw+\ell}{2}$}] (3) at (13,2) {};
					\node (4) at (16,2) {};
					\node[label=right:\!\!\!\!\!\!{$\qqq{\pw}{i_\pw+\ell}$}] (5) at (17,3.5) {};
					\node[label=right:\!\!\!\!\!\!{$\qqq{\pw}{i_\pw+\ell-1}$}] (55) at (9,3.5) {};
					\draw[-] (2) to  (3);
					\draw[-] (3) to  (4);
					\draw[->] (3) to  (5);
					\draw[->] (2) to  (55);
					\node at (8,.8) {$\nfff{\pw}{i_\pw+\ell}{1}$};
					\draw[-] (7.6,1.7) to  (7.6,2.3);
					\node at (17,2) {.};
					\node at (17.4,2) {.};
					\node at (17.8,2) {.};
					\node at (20,2) {$\mathcal{W}_\pw$};
			\node[label=left:{\textcolor{gray}{$\q$}}] (q1) at (-3.5,-2) {};
			\node[point,gray,scale=0.7,label=below:{\textcolor{gray}{$\tBox$}}] (q2) at (2,-2) {};
			\node[point,gray,scale=0.7,label=below:{\textcolor{gray}{$\ffirst$}}] (q3) at (5,-2) {};
			\node[point,gray,scale=0.7,label=below:{\textcolor{gray}{$\fsec$}}] (q4) at (10,-2) {};
			\node (q5) at (14.5,-2) {};
			\draw[-,gray] (q1) to  (q2);
			\draw[-,gray] (q2) to  (q3);
			\draw[-,gray] (q3) to  (q4);
			\draw[->,gray] (q4) to  (q5);
			\draw[-,thin] (5.2,-4) to (10,-4);
			\draw[-,thin] (5.2,-4) to (5.2,-3.8);
			\draw[-,thin] (10,-4) to (10,-3.8);
			\node at (7.2,-4.4) {{\scriptsize no $F$}};
			\draw[-,thin] (-3.2,-4) to (4.8,-4);
			\draw[-,thin] (-3.2,-4) to (-3.2,-3.8);
			\draw[-,thin] (4.8,-4) to (4.8,-3.8);
			\node at (.5,-4.4) {{\scriptsize no $F$}};
			\node at (4.2,-1) {$h$};
			\draw[->,thin,dashed] (q3) to  (5,0);
			\draw[->,thin,dashed] (q4) to  (10,1.5);
\end{tikzpicture}
				}
\end{itemize}
%


\item[(3)$^\uq$]
$h(\q)$ starts in $\mathcal{W}_\mw$ and ends in $\qq{\uq}$.\\
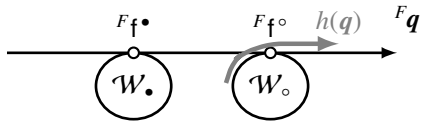
	\begin{tikzpicture}[>=latex,line width=1pt,rounded corners,xscale=.9,yscale=.8]
	                 \draw (10,-0.55) circle [radius=.55];
	                 \draw (12,-0.55) circle [radius=.55];
			\node (1) at (8,0) {};
			\node[point,scale = 0.7,fill = white,label=above:{$\cfff{\uq}{\pw}$}] (m) at (10,0) {};
			\node[point,scale = 0.7,fill =  white,label=above:{$\cfff{\uq}{\mw}$}] (2) at (12,0) {};
			\node[label=above:{$\qq{\uq}$}] (4) at (14,0) {};
			\draw[-,right] (1) to node[below] {}  (m);
			\draw[-,right] (m) to node[below] {} (2);
			\draw[->,right] (2) to node[below] {} (4);
			\node[point, draw = white] (01) at (10,-0.55) {$\mathcal{W}_\pw$};
			\node[point, draw = white] (02) at (12,-0.55) {$\mathcal{W}_\mw$};
			\draw[->,line width=.6mm,gray] (11.35,-.5) -- (11.45,-.15) -- (11.6,0) -- (12,.15) -- (13,.15);
			\node at (13,.5) {\textcolor{gray}{$h(\q)$}};
			\end{tikzpicture}\\
Then $h(\q)$ properly intersects the \neighbourhood{\uq} of $\mathcal{W}_\mw$ only.
We have $h(\cff{\mw})\preceq_{h(\q)}\cfff{\uq}{\mw}$, as otherwise there is no room for $h(\q)$ in $\qq{\uq}$.
As $\cff{\mw}=\fsec$, we have $h(\ffirst)\prec_{h(\q)}h(\fsec)\preceq_{h(\q)}\cfff{\uq}{\mw}$ and $h(\ffirst)$ is in $\mathcal{W}_\mw$.
We can exclude all possible locations for $h(\ffirst)$ by the same argument as in case $(1)^\uq$, 
with the \neighbourhood{\uq} of $\mathcal{W}_\mw$ in place of the  \neighbourhood{\uq} of $\mathcal{W}_\pw$.

			
\item[(4)$^\uq$]
$h(\q)$ ends in $\mathcal{W}_\mw$ and $\cfff{\uq}{\mw} \prec_{h(\q)}h(\cff{\mw})$.\\
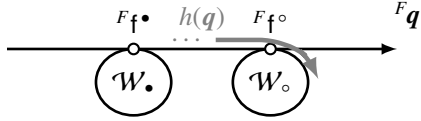
	\begin{tikzpicture}[>=latex,line width=1pt,rounded corners,xscale=.9,yscale=.8]
	                 \draw (10,-0.55) circle [radius=.55];
	                 \draw (12,-0.55) circle [radius=.55];
			\node (1) at (8,0) {};
			\node[point,scale = 0.7,fill = white,label=above:{$\cfff{\uq}{\pw}$}] (m) at (10,0) {};
			\node[point,scale = 0.7,fill =  white,label=above:{$\cfff{\uq}{\mw}$}] (2) at (12,0) {};
			\node[label=above:{$\qq{\uq}$}] (4) at (14,0) {};
			\draw[-,right] (1) to node[below] {}  (m);
			\draw[-,right] (m) to node[below] {} (2);
			\draw[->,right] (2) to node[below] {} (4);
			\node[point, draw = white] (01) at (10,-0.55) {$\mathcal{W}_\pw$};
			\node[point, draw = white] (02) at (12,-0.55) {$\mathcal{W}_\mw$};
			\draw[->,line width=.6mm,gray] (11.2,.15) -- (12,.15) -- (12.4,0) -- (12.55,-.15) -- (12.7,-.5);
			\node at (11,.5) {\textcolor{gray}{$h(\q)$}};
			\node at (10.8,.15) {\textcolor{gray}{$\dots$}};
			\end{tikzpicture}\\
Then $h(\q)$ definitely properly intersects the \neighbourhood{\uq} of $\mathcal{W}_\mw$, and it might also 
properly intersect the \neighbourhood{\uq} of $\mathcal{W}_\pw$. 
As $\cff{\pw}=\ffirst$ and $\cff{\mw}=\fsec$, 
there is no $F$-node between $\cfff{\uq}{\pw}$ and $\cfff{\uq}{\mw}$ in $\qq{\uq}$, and so  
$\cfff{\uq}{\mw} \preceq_{h(\q)}h(\ffirst)$ and  $h(\ffirst)$ is in $\mathcal{W}_\mw$.
We can exclude all possible locations for $h(\ffirst)$ by the same argument as in case $(2)^\uq$,
with the \neighbourhood{\uq} of $\mathcal{W}_\mw$ in place of the  \neighbourhood{\uq} of $\mathcal{W}_\pw$.
\end{itemize}			
We excluded all possible locations in $\mathcal{B}$ for the image $h(\q)$ of a potential \shomo{} $h\colon\q\to\I$, which 	completes the proof of Lemma~\ref{l:bike}.
\end{proof}


\subsection{Representing clauses with shared literals}\label{s:clauses}

Suppose $\psi$ is a 3CNF
	with $n_\psi$ clauses of the form $\lit_1 \lor \lit_2 \lor \lit_3$, where each $\lit_i$ is a literal. We build an ABox $\Apsi$ as follows.
	 We let $n\geq (n_\psi+2)(2|\q|+1)$ and,
	for each propositional variable $p$ in $\psi$, we take a fresh $n$-\bike{} $\mathcal{B}^p$ 
	having $n$-\wheels{} $\mathcal{W}_\pw^p$, $\mathcal{W}_\mw^p$ and satisfying the conditions in Lemma~\ref{l:bike}.
	We pick three nodes $\lc{1}$, $\lc{2}$ and $\lc{3}$ in $\q$ such that each $\lc{z}$ is a $T$-node or an $F$-node,
	and $\lc{1}\prec\lc{2}\prec\lc{3}$. We call these three nodes the \emph{\spect} of $\q$.
	Then, for every clause $c = (\lit_1^c \lor \lit_2^c \lor \lit_3^c)$ in $\psi$, we proceed as follows.
	We take a fresh copy $\qq{c}$ of $\q$, consider the copies $\lcc{1}$, $\lcc{2}$ and $\lcc{3}$ of the \spect{} in $\qq{c}$, and replace their $F$- or $T$-labels with $A$.
	Then, for $z=1,2,3$, we glue $\lcc{z}$ to a 
	contact 
	%
	\begin{itemize}
		\item[(p1)]
		in $\mathcal{W}_\pw^p$ iff  either  $\lit_z^c = p$ and $\lc{z}$ is an $F$-node in $\q$,
		or $\lit_z^c = \neg p$ and $\lc{z}$ is a $T$-node in $\q$; 
		
		\item[(p2)]
		in $\mathcal{W}_\mw^p$ iff  either  $\lit_z^c = p$ and $\lc{z}$ is an $T$-node in $\q$,
		or $\lit_z^c = \neg p$ and $\lc{z}$ is a $F$-node in $\q$.
	\end{itemize}
	For example, if $\q$ looks like on the left-hand side of the picture below and $c = (p \lor \neg q \lor r)$, then we obtain the graph shown on the right-hand side of the picture with the $n$-\wheels{} depicted as circles:\\ 
	
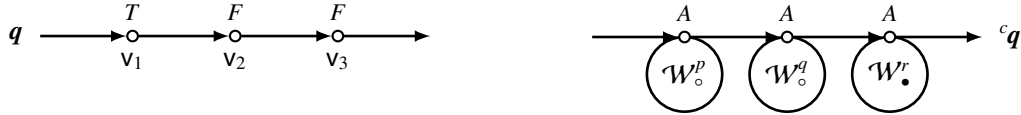
\centerline{
			\begin{tikzpicture}[>=latex,line width=1pt,rounded corners,scale=.9]
			\node at (-.5,-1) {\ };
			\node[label=left:{$\q$\!\!\!}] (1) at (0,0) {};
			\node[point,scale = 0.7,label=above:{\small $T$}, label=below:$\lc{1}$] (m) at (1.5,0) {};
			\node[point,scale = 0.7,label=above:{\small $F$}, label=below:$\lc{2}$] (2) at (3,0) {};
			\node[point,scale = 0.7,label=above:{\small $F$}, label=below:$\lc{3}$] (3) at (4.5,0) {};
			\node (4) at (6,0) {};
			\draw[->,right] (1) to node[below] {}  (m);
			\draw[->,right] (m) to node[below] {} (2);
			\draw[->,right] (2) to node[below] {} (3);
			\draw[->,right] (3) to node[below] {} (4);
			\end{tikzpicture}
			\hspace*{1cm}
			\begin{tikzpicture}[>=latex,line width=1pt,rounded corners,scale=.9]
			%
			\node at (7.5,0) {\ };
			\draw (9.5,-0.55) circle [radius=.55];
			\draw (11,-0.55) circle [radius=.55];
			\draw (12.5,-0.55) circle [radius=.55];
			\node (1) at (8,0) {};
			\node[point,scale = 0.7,label=above:{\small $A$}, fill =  white] (m) at (9.5,0) {};
			\node[point,scale = 0.7,label=above:{\small $A$}, fill =  white] (2) at (11,0) {};
			\node[point,scale = 0.7,label=above:{\small $A$}, fill =  white] (3) at (12.5,0) {};
			\node[label=right:{\!\!\!$\qq{c}$}] (4) at (14,0) {};
			\draw[->,right] (1) to node[below] {}  (m);
			\draw[->,right] (m) to node[below] {} (2);
			\draw[->,right] (2) to node[below] {} (3);
			\draw[->,right] (3) to node[below] {} (4);
			\node[point, draw = white] (01) at (9.5,-0.55) {$\mathcal{W}^p_\mw$};
			\node[point, draw = white] (02) at (11,-0.55) {$\mathcal{W}^q_\mw$};
			\node[point, draw = white] (03) at (12.5,-0.55) {$\mathcal{W}^r_\pw$};
			\end{tikzpicture}
	}\\
	We call $\lcc{1}$, $\lcc{2}$ and $\lcc{3}$
	$c$-\emph{connections\/}, while the \emph{\neighbourhood{c}} consists of those contacts in each of the three $n$-\wheels{} whose \cdist{} from its $c$-connection is $\leq |\q|$. For different clauses $c,c'$, we pick the $c$- and $c'$-connections  `sharing' the same $n$-\wheel{} $\mathcal{W}$ in such a way that the $c$- and \neighbourhoods{c'} are disjoint 
	from each other and from the $\uq$- and \neighbourhoods{\dq} in $\mathcal{W}$. 
	(We can do this as $n\geq (n_\psi+2)(2|\q|+1)$.)
	%
	We treat the resulting labelled graph as an ABox, call it a $(\psi,n)$-\emph{gadget \textup{(}for $\q$\textup{)}}, and denote it by $\Apsi$. Clearly, the
	size of $\Apsi$ is polynomial in the sizes of $\q$ and $\psi$.

	
	The following lemma is a consequence of the definition of $\Apsi$, and the `easy' 
	$(\Rightarrow)$ direction of Lemma~\ref{l:bike}.
	
	\begin{tlemma}\label{l:3cnfsound}
		If $\TT,\Apsi \not\models \q$, then $\psi$ is satisfiable. 
	\end{tlemma}
	
	\begin{proof}
		Suppose $\I$ is a model of $\TT$ and $\Apsi$ such that $\I \not\models \q$. 
		As for each variable $p$ in $\psi$, the $n$-bike $\mathcal{B}^p$ satisfies the conditions in Lemma~\ref{l:bike}, 
		either all contacts of the $n$-\wheel{} $\mathcal{W}^p_\pw$ are in $F^\I$ and all contacts
		of  $\mathcal{W}^p_\mw$ are in $T^\I$, or 
		all contacts of $\mathcal{W}^p_\pw$ are in $T^\I$ and all contacts
		of  $\mathcal{W}^p_\mw$ are in $F^\I$.
		As $\I \not\models \q$,
		for every clause $c = (\lit_1^c \lor \lit_2^c \lor \lit_3^c)$ in $\psi$, 
		there is $z=1,2,3$
		such that either $\lc{z}$ is a $T$-node in $\q$ but $\lcc{z} \in F^\I$,  or $\lc{z}$ is an $F$-node in $\q$ but $\lcc{z} \in T^\I$. Define an assignment $\mathfrak a$ by setting $\mathfrak a(\lit_z^c) = T$ for each clause $c$ in $\psi$ (and arbitrary otherwise). We claim that $\mathfrak a$ is well-defined in the sense that we never set both $\mathfrak a(p) = T$ and $\mathfrak a(\neg p) = T$. Indeed, suppose otherwise. Suppose also that the former is because of $\lit_{z_1}^{c_1}$ in a clause $c_1$ and the latter because of $\lit_{z_2}^{c_2}$ in a clause $c_2$.
		
		\emph{Case} 1: $\lc{z_1}$ is a $T$-node in $\q$ but $\lcco{z_1} \in F^\I$. As $\mathfrak a(p) = T$ implies that $\lit_{z_1}^{c_1} = p$, by (p2) of the construction $\lcco{z_1}$ is a contact in the $n$-\wheel{} $\mathcal{W}^p_\mw$. 
		So all contacts in  $\mathcal{W}^p_\mw$ are in $F^\I$. 
		On the other hand, $\mathfrak a(\neg p) = T$ implies that $\lit_{z_2}^{c_2} = \neg p$. 
		If $\lc{z_2}$ is a $T$-node in $\q$ but $\lcct{z_2} \in F^\I$, then  
		$\lcct{z_2}$ is a contact in $\mathcal{W}^p_\pw$ by (p1), and so 
		all contacts in  $\mathcal{W}^p_\pw$ are also in $F^\I$, a contradiction.
		And if $\lc{z_2}$ is an $F$-node in $\q$ but $\lcct{z_2} \in T^\I$, then 
		$\lcct{z_2}$ is a contact in $\mathcal{W}^p_\mw$ by (p2), 
		and so all contacts in  $\mathcal{W}^p_\mw$ are in $T^\I$, a contradiction again.
		
		\emph{Case} 2: $\lc{z_1}$ is an $F$-node in $\q$ but $\lcco{z_1} \in T^\I$. This case is similar and left to the reader.

		Thus, the assignment $\mathfrak a$ is well-defined and makes true at least one literal in every clause in $\psi$.
	\end{proof}

	It remains to find some conditions on $\Apsi$ that would guarantee that the converse of Lemma~\ref{l:3cnfsound} also holds.
	So suppose that $\psi$ is satisfiable under an assignment $\mathfrak a$. 
	We define a model $\Ia$ of $\TT$ and $\Apsi$  as follows: 
\begin{align}	
\nonumber
& \mbox{For every variable $p$ in $\psi$, we put}\\
\label{Iadef}
& \hspace*{1cm} \mbox{all contacts of $\mathcal{W}^p_\pw$ to $T^{\Ia}$ and all contacts of $\mathcal{W}^p_\mw$  to $F^{\Ia}$, whenever if $\mathfrak a(p) = T$; and}\\
\nonumber
& \hspace*{1cm} \mbox{all contacts of $\mathcal{W}^p_\pw$ to $F^{\Ia}$ and all contacts of $\mathcal{W}^p_\mw$  to $T^{\Ia}$, whenever if $\mathfrak a(p) = F$.}
\end{align}
%
%
We aim to find some conditions on $\Apsi$ that would imply $\Ia\not\models\q$. 
Just like in the case of other ABoxes built up from copies of $\q$ before, we are looking for conditions that exclude all possible locations in $\Apsi$ for the image $h(\q)$ of a
potential \shomo{} $h\colon\q\to\Ia$. 
The definition of $\Apsi$ allows flexibility 
\begin{itemize}
\item[--]
in the choice of the \spect{} $\lc{1},\lc{2},\lc{3}$ in $\q$, and
\item[--]
also in the choices of the contacts in the \neighbourhoods{c}, for each clause $c$.
\end{itemize}
If we choose all these contacts in such a way that \eqref{tprecfcontact}, \eqref{okcontact} and the conditions of Lemma~\ref{l:bike} hold then, by \eqref{Iadef} and Lemma~\ref{l:bike}, we know that $h(\q)$ must intersect with at least one $\qq{c}$ for some clause $c$.
Therefore, the intersection of $h(\q)$ with any of its $n$-\wheels{} cannot go beyond its \neighbourhoods{c}.
Further, we claim that,
		\begin{equation}\label{nofixh}
		\mbox{for any clause $c$ in $\psi$, there is no \shomo{} $h:\q\to\Ia$ such that $h(\lc{z})=\lcc{z}$ for all $z=1,2,3$.}
		\end{equation}
		(In particular, there is no $\q\to\Ia$ \shomo{} mapping $\q$ onto $\qq{c}$.)
		Indeed, 
		suppose on the contrary that there is such a \shomo{} $h$ for some $c$. Suppose $\mathfrak a(\lit^c_z) = T$ for some $a$. If $\lit^c_z = p$, then either $\lc{z}$ is an $F$-node in $\q$ but $\lcc{z}\in T^{\Ia}$ as it is in $\mathcal{W}^p_\pw$, or $\lc{z}$ is a $T$-node in $\q$ but $\lcc{z} \in F^{\Ia}$ as it is in $\mathcal{W}^p_\mw$, both are impossible when $h(\lc{z})=\lcc{z}$. The case of $\lit^c_z = \neg p$ is dually symmetric. It follows that $\mathfrak a(\lit^c_z) \ne T$ for any $z = 1,2,3$, contrary to $\mathfrak a$ satisfying $\psi$.
		
		By \eqref{nofixh}, 
		$h(\q)$ must \emph{properly intersect} with at least one of the three $n$-\wheels{} 
		$\mathcal{W}_1^c$, $\mathcal{W}_2^c$ and $\mathcal{W}_3^c$ glued to $\qq{c}$
		in the sense that $h(\q)\cap\mathcal{W}_z^c\not\subseteq\{\lcc{z}\}$ for some $z=1,2,3$.
		By \eqref{Iadef} and Lemma~\ref{l:wheel}, we may assume that $h(\q)\not\subseteq\mathcal{W}_z^c$ for any $z=1,2,3$.
		Also by Lemma~\ref{l:wheel}, we may assume that if $h(\q)$ properly intersects with 
		$\mathcal{W}_z^c$, then every node in $h(\q)\cap\mathcal{W}_z^c$ is in the \neighbourhood{c} of
		$\mathcal{W}_z^c$. As for $c\ne c'$ the $c$- and \neighbourhoods{c'} are disjoint,
		there is a unique $c$ 
		with $h(\q)$ properly  intersecting with one or two of 
		the $n$-\wheels{} $\mathcal{W}_1^c$, $\mathcal{W}_2^c$ and $\mathcal{W}_3^c$ glued to $\qq{c}$
		(it cannot properly intersect with all three).
		It is easy to check that, by \eqref{nofixh}, all options for such a $h(\q)$ are covered by the six cases (1)${}^c$--(6)${}^c$ 
		in Fig.~\ref{f:cases}.

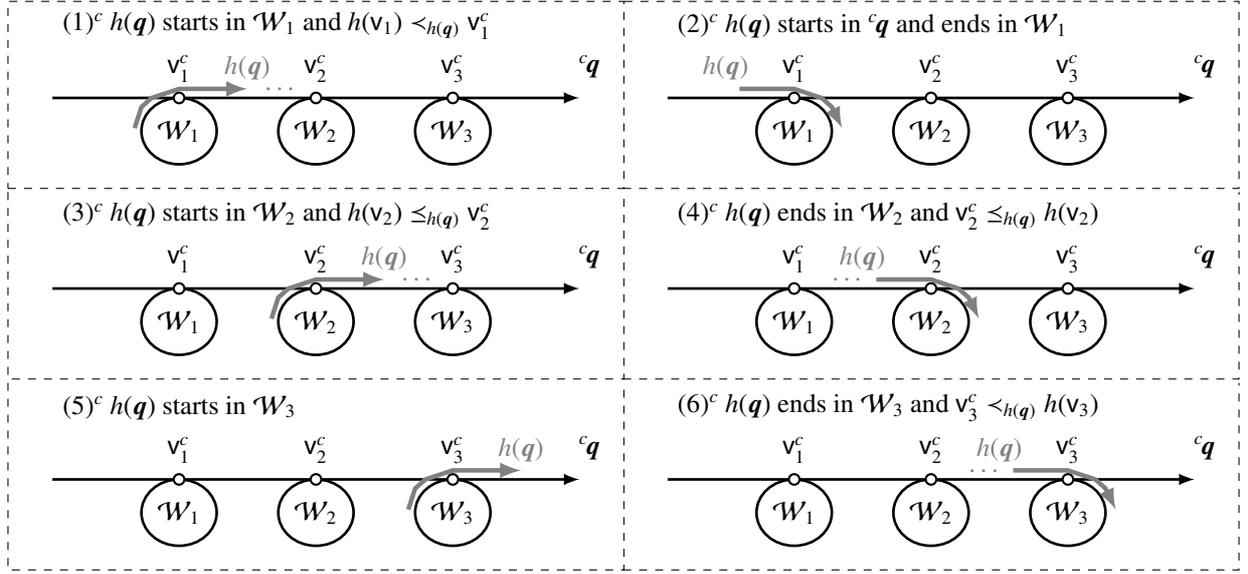
\begin{figure}[h]
\centering
\begin{tikzpicture}[>=latex,line width=1pt,xscale=.9,yscale=.8]
			%
			\node[label=right:{(1)${}^c$ $h(\q)$ starts in $\mathcal{W}_1$ and $h(\lc{1})\prec_{h(\q)}\lcc{1}$}] (t) at (8,1.2) {};
			\draw[-,thin,dashed] (7.5,1.6) -- (7.5,-1.5) -- (16.5,-1.5) -- (16.5,1.6) -- cycle;
	                 \draw (10,-0.55) circle [radius=.55];
	                 \draw (12,-0.55) circle [radius=.55];
			\draw (14,-0.55) circle [radius=.55];
			\node (1) at (8,0) {};
			\node[point,scale = 0.7,fill = white,label=above:{$\lcc{1}$}] (m) at (10,0) {};
			\node[point,scale = 0.7,fill =  white,label=above:{$\lcc{2}$}] (2) at (12,0) {};
			\node[point,scale = 0.7,fill =  white,label=above:{$\lcc{3}$}] (3) at (14,0) {};
			\node[label=above:{$\qq{c}$}] (4) at (16,0) {};
			\draw[-,right] (1) to node[below] {}  (m);
			\draw[-,right] (m) to node[below] {} (2);
			\draw[-,right] (2) to node[below] {} (3);
			\draw[->,right] (3) to node[below] {} (4);
			\node[point, draw = white] (01) at (10,-0.55) {$\mathcal{W}_1$};
			\node[point, draw = white] (02) at (12,-0.55) {$\mathcal{W}_2$};
			\node[point, draw = white] (03) at (14,-0.55) {$\mathcal{W}_3$};
			\draw[->,line width=.6mm,gray] (9.35,-.5) -- (9.45,-.15) -- (9.6,0) -- (10,.15) -- (11,.15);
			\node at (11,.5) {\textcolor{gray}{$h(\q)$}};
			\node at (11.5,.15) {\textcolor{gray}{$\dots$}};
			%
			\end{tikzpicture}	
\hspace*{-.2cm}
\begin{tikzpicture}[>=latex,line width=1pt,xscale=.9,yscale=.8]
			%
			\node[label=right:{(2)${}^c$ $h(\q)$ starts in $\qq{c}$ and ends in $\mathcal{W}_1$}] (t) at (8,1.2) {};
			\draw[-,thin,dashed]  (7.5,-1.5) -- (16.5,-1.5) -- (16.5,1.6) -- (7.5,1.6);
			\draw (10,-0.55) circle [radius=.55];
			\draw (12,-0.55) circle [radius=.55];
			\draw (14,-0.55) circle [radius=.55];
			\node (1) at (8,0) {};
			\node[point,scale = 0.7,fill =  white,label=above:{$\lcc{1}$}] (m) at (10,0) {};
			\node[point,scale = 0.7,fill =  white,label=above:{$\lcc{2}$}] (2) at (12,0) {};
			\node[point,scale = 0.7,fill =  white,label=above:{$\lcc{3}$}] (3) at (14,0) {};
			\node[label=above:{$\qq{c}$}] (4) at (16,0) {};
			\draw[-,right] (1) to node[below] {}  (m);
			\draw[-,right] (m) to node[below] {} (2);
			\draw[-,right] (2) to node[below] {} (3);
			\draw[->,right] (3) to node[below] {} (4);
			\node[point, draw = white] (01) at (10,-0.55) {$\mathcal{W}_1$};
			\node[point, draw = white] (02) at (12,-0.55) {$\mathcal{W}_2$};
			\node[point, draw = white] (03) at (14,-0.55) {$\mathcal{W}_3$};
			\draw[->,line width=.6mm,gray] (9.2,.15) -- (10,.15) -- (10.4,0) -- (10.55,-.15) -- (10.7,-.5);
			\node at (9,.5) {\textcolor{gray}{$h(\q)$}};
			\end{tikzpicture}\\	
\begin{tikzpicture}[>=latex,line width=1pt,xscale=.9,yscale=.8]
			%
			\node[label=right:{(3)${}^c$ $h(\q)$ starts in $\mathcal{W}_2$ and $h(\lc{2})\preceq_{h(\q)}\lcc{2}$}] (t) at (8,1.2) {};
			\draw[-,thin,dashed] (7.5,1.6) -- (7.5,-1.5) -- (16.5,-1.5) -- (16.5,1.6);
			\draw (10,-0.55) circle [radius=.55];
			\draw (12,-0.55) circle [radius=.55];
			\draw (14,-0.55) circle [radius=.55];
			\node (1) at (8,0) {};
			\node[point,scale = 0.7,fill =  white,label=above:{$\lcc{1}$}] (m) at (10,0) {};
			\node[point,scale = 0.7,fill =  white,label=above:{$\lcc{2}$}] (2) at (12,0) {};
			\node[point,scale = 0.7,fill =  white,label=above:{$\lcc{3}$}] (3) at (14,0) {};
			\node[label=above:{$\qq{c}$}] (4) at (16,0) {};
			\draw[-,right] (1) to node[below] {}  (m);
			\draw[-,right] (m) to node[below] {} (2);
			\draw[-,right] (2) to node[below] {} (3);
			\draw[->,right] (3) to node[below] {} (4);
			\node[point, draw = white] (01) at (10,-0.55) {$\mathcal{W}_1$};
			\node[point, draw = white] (02) at (12,-0.55) {$\mathcal{W}_2$};
			\node[point, draw = white] (03) at (14,-0.55) {$\mathcal{W}_3$};
			\draw[->,line width=.6mm,gray] (11.35,-.5) -- (11.45,-.15) -- (11.6,0) -- (12,.15) -- (13,.15);
			\node at (13,.5) {\textcolor{gray}{$h(\q)$}};
			\node at (13.5,.15) {\textcolor{gray}{$\dots$}};
			\end{tikzpicture}
\hspace*{-.2cm}
\begin{tikzpicture}[>=latex,line width=1pt,xscale=.9,yscale=.8]
			%
			\node[label=right:{(4)${}^c$ $h(\q)$ ends in $\mathcal{W}_2$ and $\lcc{2}\preceq_{h(\q)} h(\lc{2})$}] (t) at (8,1.2) {};
			\draw[-,thin,dashed]  (7.5,-1.5) -- (16.5,-1.5) -- (16.5,1.6);
			\draw (10,-0.55) circle [radius=.55];
			\draw (12,-0.55) circle [radius=.55];
			\draw (14,-0.55) circle [radius=.55];
			\node (1) at (8,0) {};
			\node[point,scale = 0.7,fill =  white,label=above:{$\lcc{1}$}] (m) at (10,0) {};
			\node[point,scale = 0.7,fill =  white,label=above:{$\lcc{2}$}] (2) at (12,0) {};
			\node[point,scale = 0.7,fill =  white,label=above:{$\lcc{3}$}] (3) at (14,0) {};
			\node[label=above:{$\qq{c}$}] (4) at (16,0) {};
			\draw[-,right] (1) to node[below] {}  (m);
			\draw[-,right] (m) to node[below] {} (2);
			\draw[-,right] (2) to node[below] {} (3);
			\draw[->,right] (3) to node[below] {} (4);
			\node[point, draw = white] (01) at (10,-0.55) {$\mathcal{W}_1$};
			\node[point, draw = white] (02) at (12,-0.55) {$\mathcal{W}_2$};
			\node[point, draw = white] (03) at (14,-0.55) {$\mathcal{W}_3$};
			\draw[->,line width=.6mm,gray] (11.2,.15) -- (12,.15) -- (12.4,0) -- (12.55,-.15) -- (12.7,-.5);
			\node at (11,.5) {\textcolor{gray}{$h(\q)$}};
			\node at (10.8,.15) {\textcolor{gray}{$\dots$}};
			\end{tikzpicture}\\
\begin{tikzpicture}[>=latex,line width=1pt,xscale=.9,yscale=.8]
			%
			\node[label=right:{(5)${}^c$ $h(\q)$ starts in $\mathcal{W}_3$}] (t) at (8,1.2) {};
			\draw[-,thin,dashed] (7.5,1.6) -- (7.5,-1.5) -- (16.5,-1.5) -- (16.5,1.6);
			\draw (10,-0.55) circle [radius=.55];
			\draw (12,-0.55) circle [radius=.55];
			\draw (14,-0.55) circle [radius=.55];
			\node (1) at (8,0) {};
			\node[point,scale = 0.7,fill =  white,label=above:{$\lcc{1}$}] (m) at (10,0) {};
			\node[point,scale = 0.7,fill =  white,label=above:{$\lcc{2}$}] (2) at (12,0) {};
			\node[point,scale = 0.7,fill =  white,label=above:{$\lcc{3}$}] (3) at (14,0) {};
			\node[label=above:{$\qq{c}$}] (4) at (16,0) {};
			\draw[-,right] (1) to node[below] {}  (m);
			\draw[-,right] (m) to node[below] {} (2);
			\draw[-,right] (2) to node[below] {} (3);
			\draw[->,right] (3) to node[below] {} (4);
			\node[point, draw = white] (01) at (10,-0.55) {$\mathcal{W}_1$};
			\node[point, draw = white] (02) at (12,-0.55) {$\mathcal{W}_2$};
			\node[point, draw = white] (03) at (14,-0.55) {$\mathcal{W}_3$};
			\draw[->,line width=.6mm,gray] (13.35,-.5) -- (13.45,-.15) -- (13.6,0) -- (14,.15) -- (15,.15);
			\node at (15,.5) {\textcolor{gray}{$h(\q)$}};
			\end{tikzpicture}	
\hspace*{-.2cm}
\begin{tikzpicture}[>=latex,line width=1pt,xscale=.9,yscale=.8]
			%
			\node[label=right:{(6)${}^c$ $h(\q)$ ends in $\mathcal{W}_3$ and $\lcc{3}\prec_{h(\q)} h(\lc{3})$}] (t) at (8,1.2) {};
			\draw[-,thin,dashed]  (7.5,-1.5) -- (16.5,-1.5) -- (16.5,1.6);
			\draw (10,-0.55) circle [radius=.55];
			\draw (12,-0.55) circle [radius=.55];
			\draw (14,-0.55) circle [radius=.55];
			\node (1) at (8,0) {};
			\node[point,scale = 0.7,fill =  white,label=above:{$\lcc{1}$}] (m) at (10,0) {};
			\node[point,scale = 0.7,fill =  white,label=above:{$\lcc{2}$}] (2) at (12,0) {};
			\node[point,scale = 0.7,fill =  white,label=above:{$\lcc{3}$}] (3) at (14,0) {};
			\node[label=above:{$\qq{c}$}] (4) at (16,0) {};
			\draw[-,right] (1) to node[below] {}  (m);
			\draw[-,right] (m) to node[below] {} (2);
			\draw[-,right] (2) to node[below] {} (3);
			\draw[->,right] (3) to node[below] {} (4);
			\node[point, draw = white] (01) at (10,-0.55) {$\mathcal{W}_1$};
			\node[point, draw = white] (02) at (12,-0.55) {$\mathcal{W}_2$};
			\node[point, draw = white] (03) at (14,-0.55) {$\mathcal{W}_3$};
			\draw[->,line width=.6mm,gray] (13.2,.15) -- (14,.15) -- (14.4,0) -- (14.55,-.15) -- (14.7,-.5);
			\node at (13,.5) {\textcolor{gray}{$h(\q)$}};
			\node at (12.8,.15) {\textcolor{gray}{$\dots$}};
			\end{tikzpicture}
\caption{Possible locations for $h(\q)$ intersecting $\qq{c}$.}\label{f:cases}
\end{figure}

We aim to show that for every $2$-CQ suitable contact choices always exist by actually providing an
\emph{algorithm} that, given  \emph{any} $2$-CQ $\q$, describes contact choices that, for large enough $n$, are suitable for any 3CNF $\psi$ and any $(\psi,n)$-gadget constructed from copies of $\q$.
Just like in case of \bikes,
the different potential locations of a homomorphic image place different constraints on our choices.
By following our heuristics choices  (\mh{1}) and (\mh{2}) above, we will be able to use the same techniques
as for \bikes{} in the proof of Lemma~\ref{l:bike}. 
In light of (\mh{1}), our algorithm chooses $\lc{1}=\tfirst$.
This is because our assumption throughout is that $\tfirst\prec\ffirst$ (cf.\ \eqref{tprecf}), and so $\tfirst$ is the only node that is followed
by at least two other $F$- or $T$-nodes ($\tlast$ and $\ffirst$) in every $2$-CQ $\q$, even if $\q$ contains only two $F$-nodes and two $T$-nodes.
Similarly, $\lc{3}$ is chosen to be  $\fsec$, as in general $\fsec$ is the only node that is preceded
by at least two other $F$- or $T$-nodes ($\tfirst$ and $\ffirst$).
And then $\lc{2}$ is chosen from the two `middle' nodes that are always present, either $\tlast$ or $\ffirst$.
The choice of the $\prec$-smaller of $\tlast$ and $\ffirst$ as $\lc{2}$ is motivated by (\mh{2}).

However, now the 3CNF $\psi$ introduces some more `variables' into our constraint system.
In order to reduce the search space, we made some further choices in our heuristics:
%
\begin{itemize}
\item[(\mh{3})]
Given any $2$-CQ $\q$ and any assignment $\mathfrak a$ satisfying some 3CNF $\psi$, we give an algorithm describing choices suitable for achieving $\Ia\not\models\q$ for any $\q$ such that the choices do not depend on $\psi$ and 
$\mathfrak a$, only on $\q$.

\item[(\mh{4})]
The algorithm chooses the contacts in the \neighbourhoods{c} of $\Apsi$ uniformly, not depending on the particular clause $c$,
but only on $\q$.

\end{itemize}

Yet another difficulty is that \eqref{nofixh} is weaker than \eqref{nofixhpm}: 
It does not exclude cases when $h$ `fixes' two (but not all three)
\connections{c}. Say, in case (3)${}^c$ it can happen that $h(\q)$ intersects $\mathcal{W}_2$ and $\mathcal{W}_3$, 
at least one of them properly, it does not intersect $\mathcal{W}_1$, and both $h(\lc{2})=\lcc{2}$ and $h(\lc{3})=\lcc{3}$ hold.
The following example shows how the need for excluding such a situation might `force' particular contact choices not only for the \connection{c} of the `middle' \wheel, but also \emph{throughout} `half' of its \neighbourhood{c}:

\begin{texample}\label{e:gadget}
\em
Consider again the $2$-CQ from Example~\ref{ex:cog}.\\
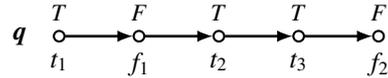
\centerline{
			\begin{tikzpicture}[>=latex,line width=1pt,rounded corners,scale=.7]
			\begin{scope}
			\node at (-.75,0) {$\q$};
			\node[point,scale = 0.7,label=above:{\small $T$},label=below:{$\tfirst$}] (1) at (0,0) {};
			\node[point,scale = 0.7,label=above:{\small $F$},label=below:{$\ffirst$}] (m) at (1.5,0) {};
			\node[point,scale = 0.7,label=above:{\small $T$},label=below:{$\tsec$}] (2) at (3,0) {};
			\node[point,scale = 0.7,label=above:{\small $T$},label=below:{$t_3$}] (3) at (4.5,0) {};
			\node[point,scale = 0.7,label=above:{\small $F$},label=below:{$\fsec$}] (4) at (6,0) {};
			\draw[->,right] (1) to node[below] {}  (m);
			\draw[->,right] (m) to node[below] {} (2);
			\draw[->,right] (2) to node[below] {} (3);
			\draw[->,right] (3) to node[below] {} (4);
			\end{scope}
			\end{tikzpicture}}\\
According to the above, as $\ffirst\prec\tlast=t_3$, we choose $\lc{2}=\ffirst$ 
(and $\lc{1}=\tfirst$, $\lc{3}=\fsec$).
Suppose that, for some clause $c$, the contact in the `middle' \wheel{} $\mathcal{W}_2$ glued together with
$\lcc{2}=\bnode{\ffirst}{c}$ is $\cfff{2}{x_2}$, in some copy $\qqq{2}{x_2}$ of $\q$.
Then the argument in Example~\ref{e:bike2}~$(i)$ shows that we cannot choose 
$\cfff{2}{x_2}=\nfff{2}{x_2}{1}$, and so we must have $\cfff{2}{x_2}=\nfff{2}{x_2}{2}$.
Now we have three choices for $\cttt{2}{x_2}$. However, if we choose 
either $\cttt{2}{x_2}=\nttt{2}{x_2}{1}$ or $\cttt{2}{x_2}=\nttt{2}{x_2}{2}$,
and $\Ia$ is such that all contacts of $\mathcal{W}_1$, $\mathcal{W}_2$ and $\mathcal{W}_3$ are in $F^{\Ia}$, 
then we do have the following $h\colon\q\to\Ia$ homomorphism
(see case $(3)^c$ in Fig.~\ref{f:cases}):\\
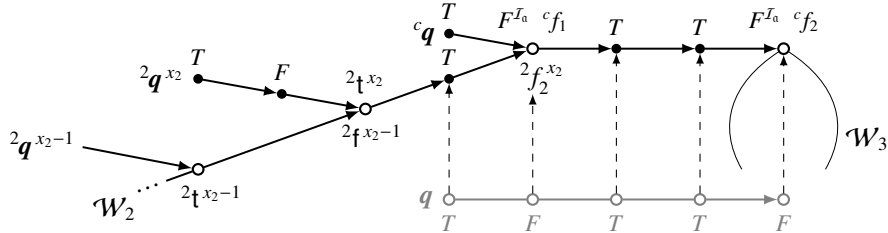
\centerline{		
	\begin{tikzpicture}[>=latex,line width=.75pt, rounded corners,xscale=.55,yscale=.4]
		\node[point,scale=0.5,fill,label=above:{\small $T$},label=left:{$\qq{c}$}] (s1) at (8,6.5) {};
		\node[label=left:{$\qqq{2}{x_2-1}$}\!\!\!\!\!] (s3) at (-1,2.8) {};
		\node[point,scale=0.7,label=below:{\ \ \ \ $\cttt{2}{x_2-1}$}] (l3) at (2,2) {};
		\node[] (l4) at (1,1.5) {};
		\node[point,scale=0.7,label=below right:{\!\!\!\!\!\!$\nfff{2}{x_2}{2}$},label=above:{\small $F^{\Ia}\ \ \bnode{\ffirst}{c}$}] (m) at (10,6) {};
		\node[point,scale=0.5,fill,label=above:{\small $T$}] (mm1) at (8,5) {};
		\node[point,scale=0.7,label=above:{$\cttt{2}{x_2}$},label=below:{\ \ $\cfff{2}{x_2-1}$}] (mm2) at (6,4) {};
		\node[point,scale=0.5,fill,label=above:{\small $F$}] (mm3) at (4,4.5) {};
		\node[point,scale=0.5,fill,label=left:{$\qqq{2}{x_2}$},label=above:{\small $T$}] (mm4) at (2,5) {};
		\node[point,scale=0.5,fill,label=above:{\small $T$}] (m1) at (12,6) {};
		\node[point,scale=0.5,fill,label=above:{\small $T$}] (m2) at (14,6) {};
		\node[point,scale=0.7,label=above:{\small $F^{\Ia}\ \ \bnode{\fsec}{c}$}] (m3) at (16,6) {};
		\node at (18,3) {$\mathcal{W}_3$};
		\node at (.6,1.3) {.};
		\node at (.8,1.4) {.};
		\node at (1,1.5) {.};
		\node at (0,.7) {$\mathcal{W}_2$};
		\draw[-,very thin, bend left=35] (m3) to (17,2);
		\draw[-,very thin, bend right=35] (m3) to (15,2);
		\draw[->] (s1) to  (m);
		\draw[->] (s3) to  (l3);
		\draw[->] (m) to  (m1);
		\draw[->] (m1) to  (m2);
		\draw[->] (m2) to  (m3);
		\draw[->] (mm1) to  (m);
		\draw[->] (mm2) to  (mm1);
		\draw[->] (mm3) to  (mm2);
		\draw[->] (mm4) to  (mm3);
		\draw[->] (l3) to  (mm2);
		\draw[-] (l4) to  (l3);
		\node[point,gray,scale=0.7,label=left:{\textcolor{gray}{$\q$}\ },label=below:{\textcolor{gray}{\small $T$}}] (q1) at (8,1) {};
		\node[point,gray,scale=0.7,label=below:{\textcolor{gray}{\small $F$}}] (q2) at (10,1) {};
		\node[point,gray,scale=0.7,label=below:{\textcolor{gray}{\small $T$}}] (q3) at (12,1) {};
		\node[point,gray,scale=0.7,label=below:{\textcolor{gray}{\small $T$}}] (q4) at (14,1) {};
		\node[point,gray,scale=0.7,label=below:{\textcolor{gray}{\small $F$}}] (q5) at (16,1) {};
		\draw[-,gray] (q1) to  (q2);
		\draw[-,gray] (q2) to  (q3);
		\draw[-,gray] (q3) to  (q4);
		\draw[->,gray] (q4) to  (q5);
		\draw[->,thin,dashed] (q1) to  (mm1);	
		\draw[->,thin,dashed] (q2) to  (10,4.5);	
		\draw[->,thin,dashed] (q3) to  (m1);	
		\draw[->,thin,dashed] (q4) to  (m2);	
		\draw[->,thin,dashed] (q5) to  (m3);	
		\end{tikzpicture}	}\\
Therefore, $\cttt{2}{x_2}=\nttt{2}{x_2}{3}$ must hold. Let us continue with some other contact choices in 
the \neighbourhood{c} of $\mathcal{W}_2$.
In light of Remark~\ref{r:f1}, we might want to stick to the `default' contact choice for $\qqq{2}{x_2-1}$, and choose
$\cfff{2}{x_2-1}=\nfff{2}{x_2-1}{1}$. Then, as $\cttt{2}{x_2-1}\prec_{\qqq{2}{x_2-1}} \cfff{2}{x_2-1}$ by \eqref{tprecfcontact},
we must choose $\cttt{2}{x_2-1}=\nttt{2}{x_2-1}{1}$.
However, in this case \eqref{okcontact} fails, and there \emph{is} a $h\colon\q\to\Ia$ homomorphism, as shown in Example~\ref{ex:cog}.
In fact, by repeating the above argument, we obtain that we must choose 
$\cfff{2}{x_2-k}=\nfff{2}{x_2-k}{2}$ and $\cttt{2}{x_2-k}=\nttt{2}{x_2-k}{3}$ , for \emph{every} 
$k\le |\q|$.
\end{texample}

In Lemma~\ref {l:3cnfcomp}
below, we describe a general algorithmic solution to the constraint system along the lines of (\mh{1})--(\mh{4}), and show that
for this solution the converse of Lemma~\ref{l:3cnfsound} holds.
In order to formulate our solution, we need to fix some notation for \neighbourhoods{c}.
With a slight abuse of notation in light of (\mh{4}), for any given clause $c$ in $\psi$,  we denote by 
$\mathcal{W}_1,\mathcal{W}_2,\mathcal{W}_3$ the three $n$-\wheels{} the node $\lcc{z}$ of $\qq{c}$ is glued to.
For each $z=1,2,3$, $\mathcal{W}_z$ is built up from the $\q$-copies ${}^{z\!}\q^1,\dots,{}^{z\!}\q^n$, and
the $c$-connection of $\mathcal{W}_z$ is obtained by glueing together
node $\lcc{z}$ of $\qq{c}$ with the contact  ${}^{z}\cf^{x_z}={}^{z}\ct^{x_z+1}$ of $\mathcal{W}_z$ 
(throughout, as before, $\pm$ is modulo $n$).\\
	
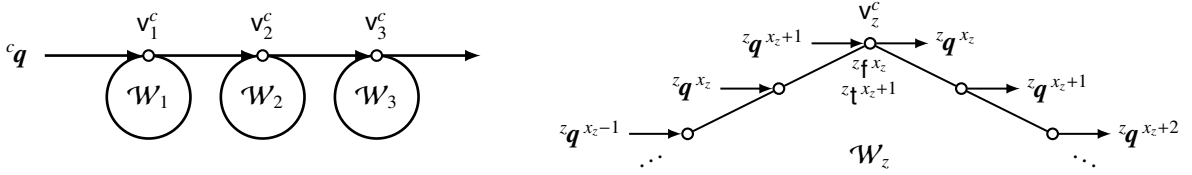
\centerline{
	\begin{tikzpicture}[>=latex,line width=1pt,rounded corners]
			\begin{scope}
			%
			\node at (7.5,-1.5) {\ };
			\draw (9.5,-0.55) circle [radius=.55];
			\draw (11,-0.55) circle [radius=.55];
			\draw (12.5,-0.55) circle [radius=.55];
			\node[label=left:{$\qq{c}$\!\!\!}] (1) at (8,0) {};
			\node[point,scale = 0.7,label=above:{$\lcc{1}$}, fill =  white] (m) at (9.5,0) {};
			\node[point,scale = 0.7,label=above:{$\lcc{2}$}, fill =  white] (2) at (11,0) {};
			\node[point,scale = 0.7,label=above:{$\lcc{3}$}, fill =  white] (3) at (12.5,0) {};
			\node (4) at (14,0) {};
			\draw[->,right] (1) to node[below] {}  (m);
			\draw[->,right] (m) to node[below] {} (2);
			\draw[->,right] (2) to node[below] {} (3);
			\draw[->,right] (3) to node[below] {} (4);
			\node[point, draw = white] (01) at (9.5,-0.55) {$\mathcal{W}_1$};
			\node[point, draw = white] (02) at (11,-0.55) {$\mathcal{W}_2$};
			\node[point, draw = white] (03) at (12.5,-0.55) {$\mathcal{W}_3$};
			\end{scope}
			\end{tikzpicture}
			\hspace*{.5cm}
		\begin{tikzpicture}[>=latex,line width=.75pt, rounded corners,scale=.3]
		\node[label=left:{$\qqq{z}{x_z-1}$}\!\!\!\!\!] (s3) at (-1,2) {};
		\node[label=left:{$\qqq{z}{x_z}$}\!\!\!\!\!] (s2) at (3,4) {};
		\node[label=left:{$\qqq{z}{x_z+1}$}\!\!\!\!\!] (s1) at (7,6) {};
		\node[label=right:\!\!\!\!\!{$\qqq{z}{x_z}$}] (e1) at (13,6) {};
		\node[label=right:\!\!\!\!\!{$\qqq{z}{x_z+1}$}] (e2) at (17,4) {};
		\node[label=right:\!\!\!\!\!{$\qqq{z}{x_z+2}$}] (e3) at (21,2) {};
		\node[point,scale=0.7] (l3) at (2,2) {};
		\node[point,scale=0.7] (l2) at (6,4) {};
		\node[point,scale=0.7] (r2) at (14,4) {};
		\node[point,scale=0.7] (r3) at (18,2) {};
		\node[point,scale=0.7,label=below:{$\cfff{z}{x_z}$},label=above:{$\lcc{z}$}] (m) at (10,6) {};
		\node at (.8,1) {.};
		\node at (.4,.8) {.};
		\node at (0,.6) {.};
		\node at (19,1) {.};
		\node at (19.4,.8) {.};
		\node at (19.8,.6) {.};
		\node at (10,3.8) {$\cttt{z}{x_z+1}$};
		\node at (10,1) {$\mathcal{W}_z$};
		\draw[->] (s3) to  (l3);
		\draw[->] (s2) to  (l2);
		\draw[->] (s1) to  (m);
		\draw[->] (m) to  (e1);
		\draw[->] (r2) to  (e2);
		\draw[->] (r3) to  (e3);
		\draw[-] (l3) to  (l2);
		\draw[-] (l2) to  (m);
		\draw[-] (m) to  (r2);
		\draw[-] (r2) to  (r3);
		\end{tikzpicture}
	}\\
For any node $x$ in $\q$, we denote by $\bnode{x}{c}$ the copy of $x$ in $\qq{c}$; and
for $i=1,\dots,n$ and $z=1,2,3$, we denote by $\bnode{x^i}{z}$ the copy of $x$ in $\qqq{z}{i}$.
Recall that for any $k$, we let $\tkth$ ($\fkth$) denote the $k$th $T$-node ($F$-node) in $\q$.
	In particular, $\tlbo$ denotes the last but one $T$-node in $\q$, and $\tlast$ the last $T$-node. We again assume that $\tfirst\prec\ffirst$ (cf.\ \eqref{tprecf}), and let $\tBox$ denote the last $T$-node preceding $\ffirst$.

\begin{tlemma}\label{l:3cnfcomp}
Given a 3CNF $\psi$, let $\Apsi$ be a $(\psi,n)$-gadget, for some $n\geq (n_\psi+2)(2|\q|+1)$, built up from $n$-\bikes,
each satisfying the conditions of Lemma~\ref{l:bike}. Suppose $\Apsi$ is such 
that the following hold for its \spect\textup{:}
\[
\lc{1}=\tfirst,\qquad
\lc{2}=\left\{
		\begin{array}{ll}
		\ffirst, & \mbox{if $\ffirst\prec\tlast$,}\\[3pt]
		\tlast, & \mbox{if $\tlast\prec\ffirst$,}
		\end{array}
		\right.
\qquad  
\lc{3}=\fsec\textup{;}
\]
and for every clause $c$ in $\psi$, the following hold for the \neighbourhood{c} in $\mathcal{W}_1$\textup{:}	
\begin{align*}
& \cttt{1}{x_1}=\nttt{1}{x_1}{\Box},\qquad
\cfff{1}{x_1}=\left\{
		\begin{array}{ll}
		\nfff{1}{x_1}{1}, & \mbox{if $\ffirst\prec\tlast$,}\\[3pt]
		\nfff{1}{x_1}{2}, & \mbox{if $\tlast\prec\ffirst$,}
		\end{array}
		\right.\\
& \cttt{1}{k}=\nttt{1}{k}{1},\qquad
\cfff{1}{k}=\nfff{1}{k}{1},\quad\mbox{for any other $k$ with $x_1-|\q|\le k\le x_1+|\q|$\textup{;}}
\end{align*}
the following hold for the \neighbourhood{c} in $\mathcal{W}_2$, for $k\le|\q|$ and $1\le\ell\le |\q|$\textup{:}	
\begin{align*}
& \cttt{2}{x_2-k}=\left\{
		\begin{array}{ll}
		\nttt{2}{x_2-k}{\Diamond}, & \mbox{if $\ffirst\prec\tlast$ and there is a $T$-node $\tDiamond$ with $\tDiamond\prec\fsec$ and $\delta(\tDiamond,\fsec)=\delta(\tfirst,\ffirst)$,}\\[3pt]
		\nttt{2}{x_2-k}{1}, & \mbox{otherwise,} 
		\end{array}
\right.\\
& \cfff{2}{x_2-k}=\left\{
		\begin{array}{ll}
		\nfff{2}{x_2-k}{2}, & \mbox{if $\ffirst\prec\tlast$,}\\[3pt]
		\nfff{2}{x_2}{2}, & \mbox{if $k=0$, $\tlast\prec\ffirst$ and $\delta(\tlbo,\tlast)=\delta(\tlast,\ffirst)$,}\\[3pt]
		\nfff{2}{x_2-k}{1}, & \mbox{otherwise},
		\end{array}
		\right.\\
& \cttt{2}{x_2+\ell}=\left\{
                 \begin{array}{ll}
                 \nttt{2}{x_2+\ell}{\Box}, & \mbox{if $\ffirst\prec\tlast$,}\\[3pt]
                 \nttt{2}{x_2+\ell}{1} & \mbox{if $\tlast\prec\ffirst$,}
                 \end{array}
                 \right.\\
& \cfff{2}{x_2+\ell}=\nfff{2}{x_2+\ell}{1}\textup{;}		
\end{align*}
the following hold for the \neighbourhood{c} in $\mathcal{W}_3$\textup{:}	
\begin{align*}
& \cttt{3}{x_3}=\nttt{3}{x_3}{1},\qquad
\cfff{3}{x_3}=\nfff{3}{x_3}{2},\\
& \cttt{3}{x_3-k}=\nttt{3}{x_3-k}{1},\qquad
\cfff{3}{x_3-k}=\nfff{3}{x_3-k}{1},\quad\mbox{for $0<k\le |\q|$,}\\
& \cttt{3}{x_3+\ell}=\left\{
 \begin{array}{ll}
   \nttt{3}{x_3+\ell}{1}, & \mbox{if $\tlast\prec\ffirst$ and $\delta(\ffirst,\fsec)<\delta(\tfirst,\ffirst)$,}\\[3pt]
   \nttt{3}{x_3+\ell}{\Box}, & \mbox{otherwise,}
 \end{array}
\right.\\
& \cfff{3}{x_3+\ell}=\left\{
 \begin{array}{ll}
   \nfff{3}{x_3+\ell}{2}, & \mbox{if $\tlast\prec\ffirst$ and $\delta(\ffirst,\fsec)\geq\delta(\tfirst,\ffirst)$},\\[3pt]
   \nfff{3}{x_3+\ell}{1}, & \mbox{otherwise,}
 \end{array}
\right.\quad\mbox{for $1\le\ell\le |\q|$.}
\end{align*}
%
%
%
Then $\Ia\not\models \q$, for any assignment $\mathfrak a$ satisfying $\psi$.
	\end{tlemma}
	%

It is straightforward to check that 
$(\psi,n)$-gadgets $\Apsi$ satisfying the conditions of the lemma always exist:
As $\psi$ has $n_\psi$-many clauses and $n\geq (n_\psi+2)(2|\q|+1)$,
for different clauses $c,c'$, the $c$- and \neighbourhoods{c'} of the same $n$-\wheel{} $\mathcal{W}$ can be kept disjoint from each other and from the $\uq$- and \neighbourhoods{\dq} of $\mathcal{W}$.
Thus, choices for the present lemma do not interfere with the choices for Lemma~\ref{l:bike}. Also, by choosing (the corresponding copies of) $\tfirst$ and $\fsec$ as contacts outside the $\uq$-, $\dq$- and \neighbourhoods{c},
conditions \eqref{tprecfcontact}, \eqref{okcontact} hold for all \wheels{} in $\Apsi$.

\begin{proof}
Suppose $\mathfrak a$ is an assignment  satisfying $\psi$, and take the model of $\TT$ and $\Apsi$ defined in \eqref{Iadef}.
In light of (\mh{4}),
we do not use any specifics about the clause $c$, and so we do not have explicit information about the particular labelings of 
the $c$-connections $\lcc{1}$, $\lcc{2}$ and $\lcc{3}$ in $\Ia$. However, \eqref{Iadef} still implies that each of the attached \wheels{} $\mathcal{W}_1$, $\mathcal{W}_2$ and $\mathcal{W}_3$ `represents' a truth-value:
\begin{equation}\label{wsame}
\mbox{for each $z=1,2,3$, the contacts of $\mathcal{W}_z$ are either all in $T^{\Ia}$ or all in $F^{\Ia}$.}
\end{equation}
%
%
Now the proof of Lemma~\ref{l:3cnfcomp} 
		is via excluding all possible locations in $\Apsi$ for the image $h(\q)$ of a
		potential \shomo{} $h\colon\q\to\Ia$.
		As explained above, by \eqref{nofixh}, Lemmas~\ref{l:wheel} and \ref{l:bike}, only the cases (1)${}^c$--(6)${}^c$ in Fig.~\ref{f:cases} remain for the location of $h(\q)$, and we need to show that none of them is possible.
In light of 	(\mh{2}), we always track the location of $h(\ffirst)$ and, whenever possible, try to reduce the cases to cases in the proof of Lemma~\ref{l:bike} for \bikes:

\begin{itemize}
\item[(1)${}^c$]
$h(\q)$ starts in $\mathcal{W}_1$ and $h(\lc{1})\prec_{h(\q)}\lcc{1}$.\\
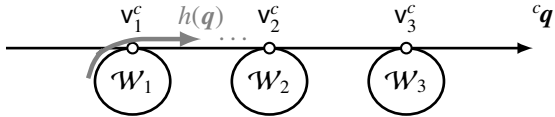
			\begin{tikzpicture}[>=latex,line width=1pt,rounded corners,xscale=.9,yscale=.8]
			%
	                 \draw (10,-0.55) circle [radius=.55];
	                 \draw (12,-0.55) circle [radius=.55];
			\draw (14,-0.55) circle [radius=.55];
			\node (1) at (8,0) {};
			\node[point,scale = 0.7,fill = white,label=above:{$\lcc{1}$}] (m) at (10,0) {};
			\node[point,scale = 0.7,fill =  white,label=above:{$\lcc{2}$}] (2) at (12,0) {};
			\node[point,scale = 0.7,fill =  white,label=above:{$\lcc{3}$}] (3) at (14,0) {};
			\node[label=above:{$\qq{c}$}] (4) at (16,0) {};
			\draw[-,right] (1) to node[below] {}  (m);
			\draw[-,right] (m) to node[below] {} (2);
			\draw[-,right] (2) to node[below] {} (3);
			\draw[->,right] (3) to node[below] {} (4);
			\node[point, draw = white] (01) at (10,-0.55) {$\mathcal{W}_1$};
			\node[point, draw = white] (02) at (12,-0.55) {$\mathcal{W}_2$};
			\node[point, draw = white] (03) at (14,-0.55) {$\mathcal{W}_3$};
			\draw[->,line width=.6mm,gray] (9.35,-.5) -- (9.45,-.15) -- (9.6,0) -- (10,.15) -- (11,.15);
			\node at (11,.5) {\textcolor{gray}{$h(\q)$}};
			\node at (11.5,.15) {\textcolor{gray}{$\dots$}};
			\end{tikzpicture}\\
%
%
%
%
Then $h(\q)$ definitely properly intersects the \neighbourhood{c} of $\mathcal{W}_1$, and it might also 
properly intersect the \neighbourhoods{c} of $\mathcal{W}_2$ or $\mathcal{W}_3$.
It follows from $h(\lc{1})\prec_{h(\q)}\lcc{1}$ that if $h(\ffirst)$ is in $\qq{c}$ then $h(\ffirst)\prec_{\qq{c}}\bnode{\ffirst}{c}$.
We can exclude all possible locations for $h(\ffirst)$ by the same argument as in case $(1)^\dq$ in the proof of
Lemma~\ref{l:bike},
with the \neighbourhood{c} of $\mathcal{W}_1$ in place of the  \neighbourhood{\dq} of $\mathcal{W}_\pw$.


\item[(2)${}^c$]
			$h(\q)$ starts in $\qq{c}$ and ends in $\mathcal{W}_1$.\\
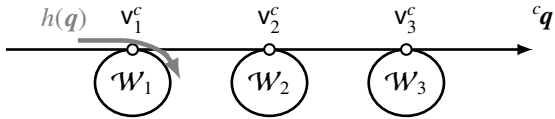
			\begin{tikzpicture}[>=latex,line width=1pt,rounded corners,xscale=.9,yscale=.8]
			\draw (10,-0.55) circle [radius=.55];
			\draw (12,-0.55) circle [radius=.55];
			\draw (14,-0.55) circle [radius=.55];
			\node (1) at (8,0) {};
			\node[point,scale = 0.7,fill =  white,label=above:{$\lcc{1}$}] (m) at (10,0) {};
			\node[point,scale = 0.7,fill =  white,label=above:{$\lcc{2}$}] (2) at (12,0) {};
			\node[point,scale = 0.7,fill =  white,label=above:{$\lcc{3}$}] (3) at (14,0) {};
			\node[label=above:{$\qq{c}$}] (4) at (16,0) {};
			\draw[-,right] (1) to node[below] {}  (m);
			\draw[-,right] (m) to node[below] {} (2);
			\draw[-,right] (2) to node[below] {} (3);
			\draw[->,right] (3) to node[below] {} (4);
			\node[point, draw = white] (01) at (10,-0.55) {$\mathcal{W}_1$};
			\node[point, draw = white] (02) at (12,-0.55) {$\mathcal{W}_2$};
			\node[point, draw = white] (03) at (14,-0.55) {$\mathcal{W}_3$};
			\draw[->,line width=.6mm,gray] (9.2,.15) -- (10,.15) -- (10.4,0) -- (10.55,-.15) -- (10.7,-.5);
			\node at (9,.5) {\textcolor{gray}{$h(\q)$}};
			\end{tikzpicture}\\
Then $h(\q)$ properly intersects the \neighbourhood{c} of $\mathcal{W}_1$ only.
We can exclude all possible locations for $h(\ffirst)$ by the same argument as in case $(2)^\dq$ in the proof of
Lemma~\ref{l:bike},
with the \neighbourhood{c} of $\mathcal{W}_1$ in place of the  \neighbourhood{\dq} of $\mathcal{W}_\pw$.

\item[(3)${}^c$]
			$h(\q)$ starts in $\mathcal{W}_2$ and $h(\lc{2})\preceq_{h(\q)}\lcc{2}$.\\
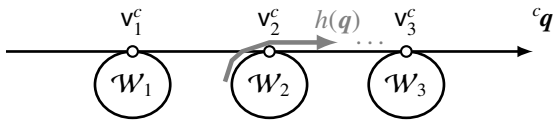
			\begin{tikzpicture}[>=latex,line width=1pt,xscale=.9,yscale=.8]
			%
			\draw (10,-0.55) circle [radius=.55];
			\draw (12,-0.55) circle [radius=.55];
			\draw (14,-0.55) circle [radius=.55];
			\node (1) at (8,0) {};
			\node[point,scale = 0.7,fill =  white,label=above:{$\lcc{1}$}] (m) at (10,0) {};
			\node[point,scale = 0.7,fill =  white,label=above:{$\lcc{2}$}] (2) at (12,0) {};
			\node[point,scale = 0.7,fill =  white,label=above:{$\lcc{3}$}] (3) at (14,0) {};
			\node[label=above:{$\qq{c}$}] (4) at (16,0) {};
			\draw[-,right] (1) to node[below] {}  (m);
			\draw[-,right] (m) to node[below] {} (2);
			\draw[-,right] (2) to node[below] {} (3);
			\draw[->,right] (3) to node[below] {} (4);
			\node[point, draw = white] (01) at (10,-0.55) {$\mathcal{W}_1$};
			\node[point, draw = white] (02) at (12,-0.55) {$\mathcal{W}_2$};
			\node[point, draw = white] (03) at (14,-0.55) {$\mathcal{W}_3$};
			\draw[->,line width=.6mm,gray] (11.35,-.5) -- (11.45,-.15) -- (11.6,0) -- (12,.15) -- (13,.15);
			\node at (13,.5) {\textcolor{gray}{$h(\q)$}};
			\node at (13.5,.15) {\textcolor{gray}{$\dots$}};
			\end{tikzpicture}\\
%
Then $h(\q)$ definitely properly intersects the \neighbourhood{c} of $\mathcal{W}_2$, and it may also 
properly intersect the \neighbourhood{c} of $\mathcal{W}_3$. 
We consider the two cases $\ffirst\prec\tlast$ and $\tlast\prec\ffirst$:

If $\ffirst\prec\tlast$ then $\lc{2}=\ffirst$, and so $h(\ffirst)\preceq_{h(\q)}\bnode{\ffirst}{c}=\lcc{2}$.
As $\cfff{2}{x_2-k}=\nfff{2}{x_2-k}{2}$ for all $k\le |\q|$, $h(\ffirst)=\nfff{2}{x_2-k}{1}$ cannot hold for any such $k$,
otherwise both $h(t)$ and $h(\fsec)$ are contacts of $\mathcal{W}_2$ for the $T$-node $t$ with $\cttt{2}{x_2-k}=\nttt{2}{x_2-k}{\ }$,
contradicting \eqref{wsame}. Since the only $F$-node preceding $\fsec$ in $\q$ is $\ffirst$, the only remaining option
for $h(\ffirst)$  is when $h(\ffirst)=\nfff{2}{x_2-k}{2}$ is a contact of $\mathcal{W}_2$ for some $k\le |\q|$. 
Now we track the location of $h(\tfirst)$. We have
\begin{equation}\label{yloc}
\delta_{h(\q)}\bigl(h(\tfirst),\nfff{2}{x_2-k}{2}\bigr)=
\delta_{h(\q)}\bigl(h(\tfirst),h(\ffirst)\bigr)=
\delta(\tfirst,\ffirst)=\delta(y,\fsec)=
\delta_{\qqq{2}{x_2-k}}\bigl(\bnode{y^{x_2-k}}{2},\nfff{2}{x_2-k}{2}\bigr),
\end{equation}
where $y$ is the node in $\q$ with $y\prec\fsec$ and $\delta(y,\fsec)=\delta(\tfirst,\ffirst)$. 
Consider two cases, depending on whether $y$ is a $T$-node or not:
\begin{itemize}
\item
If $y$ is a $T$-node $\tDiamond$, then $\cttt{2}{x_2-k}=\nttt{2}{x_2-k}{\Diamond}$, and so 
$h(\tfirst)=\cttt{2}{x_2-k}$ by \eqref{yloc}. Thus, $h(\tfirst)$ is a contact,
contradicting \eqref{wsame} and the fact that $h(\ffirst)$ is also a contact of $\mathcal{W}_2$.

\item
If $y$ is not a $T$-node $\tDiamond$ then $\cttt{2}{x_2-k}=\nttt{2}{x_2-k}{1}$. 
While $y\preceq\ffirst$ and $\ffirst\prec y$ are both possible, we surely have $\tfirst\prec y$,
as $\delta(y,\fsec)=\delta(\tfirst,\ffirst)<\delta(\tfirst,\fsec)$.
Then $h(\tfirst)=\bnode{y^{x_2-k}}{2}$ follows by \eqref{yloc}. But $\bnode{y^{x_2-k}}{2}$ 
is not a $T$-node.
\end{itemize}

If $\tlast\prec\ffirst$ then $\lc{2}=\tlast$.
 If follows from $h(\lc{2})\preceq_{h(\q)}\lcc{2}$ that if $h(\ffirst)$ is in $\qq{c}$ then $h(\ffirst)\preceq_{\qq{c}}\bnode{\ffirst}{c}$.
As there is no $F$-node preceding $\bnode{\ffirst}{c}$ in $\qq{c}$,
either $h(\ffirst)$ is in $\mathcal{W}_2$,
or $h(\ffirst)=\bnode{\ffirst}{c}$.
We can exclude all possible locations for $h(\ffirst)$ by the same argument as in case $(3)^\dq$ in the proof of
Lemma~\ref{l:bike},
with the \neighbourhood{c} of $\mathcal{W}_2$ in place of the  \neighbourhood{\dq} of $\mathcal{W}_\mw$.

\item[(4)${}^c$]
			$h(\q)$ ends in $\mathcal{W}_2$ and $\lcc{2}\preceq_{h(\q)} h(\lc{2})$.\\
%
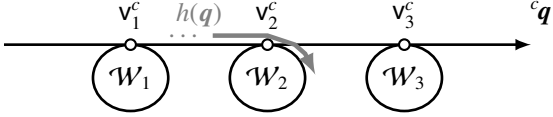
			\begin{tikzpicture}[>=latex,line width=1pt,xscale=.9,yscale=.8]
			%
			\draw (10,-0.55) circle [radius=.55];
			\draw (12,-0.55) circle [radius=.55];
			\draw (14,-0.55) circle [radius=.55];
			\node (1) at (8,0) {};
			\node[point,scale = 0.7,fill =  white,label=above:{$\lcc{1}$}] (m) at (10,0) {};
			\node[point,scale = 0.7,fill =  white,label=above:{$\lcc{2}$}] (2) at (12,0) {};
			\node[point,scale = 0.7,fill =  white,label=above:{$\lcc{3}$}] (3) at (14,0) {};
			\node[label=above:{$\qq{c}$}] (4) at (16,0) {};
			\draw[-,right] (1) to node[below] {}  (m);
			\draw[-,right] (m) to node[below] {} (2);
			\draw[-,right] (2) to node[below] {} (3);
			\draw[->,right] (3) to node[below] {} (4);
			\node[point, draw = white] (01) at (10,-0.55) {$\mathcal{W}_1$};
			\node[point, draw = white] (02) at (12,-0.55) {$\mathcal{W}_2$};
			\node[point, draw = white] (03) at (14,-0.55) {$\mathcal{W}_3$};
			\draw[->,line width=.6mm,gray] (11.2,.15) -- (12,.15) -- (12.4,0) -- (12.55,-.15) -- (12.7,-.5);
			\node at (11,.5) {\textcolor{gray}{$h(\q)$}};
			\node at (10.8,.15) {\textcolor{gray}{$\dots$}};
			\end{tikzpicture}\\
%
Then $h(\q)$ definitely properly intersects the \neighbourhood{c} of $\mathcal{W}_2$, and it might also 
properly intersect the \neighbourhood{c} of $\mathcal{W}_1$. 
We consider the two cases $\ffirst\prec\tlast$ and $\tlast\prec\ffirst$:

If $\ffirst\prec\tlast$ then $\lc{2}=\ffirst$, and so $\lcc{2}=\bnode{\ffirst}{c}\preceq_{h(\q)}h(\ffirst)$ and $h(\ffirst)$ is in $\mathcal{W}_2$. 
 We can exclude all possible locations for $h(\ffirst)$ by the same argument as in case $(2)^\uq$ in the proof of
Lemma~\ref{l:bike},
with the \neighbourhood{c} of $\mathcal{W}_2$ in place of the  \neighbourhood{\uq} of $\mathcal{W}_\pw$.
 
If $\tlast\prec\ffirst$ then $\lc{2}=\tlast$. 
As by our assumption $h(\q)$ ends in $\mathcal{W}_2$ and 
$\lcc{2}\preceq_{h(\q)} h(\lc{2})$, if follows that $\lcc{2}\preceq_{h(\q)}h(\tlast)\prec_{h(\q)}h(\ffirst)$.
We can exclude all possible locations for $h(\ffirst)$ by the same argument as in case $(4)^\dq$ in the proof of
Lemma~\ref{l:bike},
with the \neighbourhood{c} of $\mathcal{W}_2$ in place of the  \neighbourhood{\dq} of $\mathcal{W}_\mw$.

\item[(5)${}^c$]
			$h(\q)$ starts in $\mathcal{W}_3$.\\
%
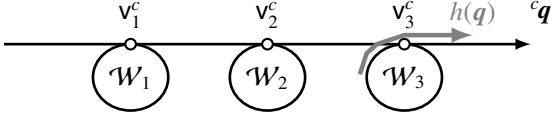
			\begin{tikzpicture}[>=latex,line width=1pt,xscale=.9,yscale=.8]
			%
			\draw (10,-0.55) circle [radius=.55];
			\draw (12,-0.55) circle [radius=.55];
			\draw (14,-0.55) circle [radius=.55];
			\node (1) at (8,0) {};
			\node[point,scale = 0.7,fill =  white,label=above:{$\lcc{1}$}] (m) at (10,0) {};
			\node[point,scale = 0.7,fill =  white,label=above:{$\lcc{2}$}] (2) at (12,0) {};
			\node[point,scale = 0.7,fill =  white,label=above:{$\lcc{3}$}] (3) at (14,0) {};
			\node[label=above:{$\qq{c}$}] (4) at (16,0) {};
			\draw[-,right] (1) to node[below] {}  (m);
			\draw[-,right] (m) to node[below] {} (2);
			\draw[-,right] (2) to node[below] {} (3);
			\draw[->,right] (3) to node[below] {} (4);
			\node[point, draw = white] (01) at (10,-0.55) {$\mathcal{W}_1$};
			\node[point, draw = white] (02) at (12,-0.55) {$\mathcal{W}_2$};
			\node[point, draw = white] (03) at (14,-0.55) {$\mathcal{W}_3$};
			\draw[->,line width=.6mm,gray] (13.35,-.5) -- (13.45,-.15) -- (13.6,0) -- (14,.15) -- (15,.15);
			\node at (15,.5) {\textcolor{gray}{$h(\q)$}};
			\end{tikzpicture}\\
%
Then $h(\q)$ properly intersects the \neighbourhood{c} of $\mathcal{W}_3$ only.
We have $h(\lc{3})\preceq_{h(\q)}\lcc{3}$, as otherwise there is no room for $h(\q)$ in $\qq{c}$.
			As $\lc{3}=\fsec$, we have $h(\ffirst)\prec_{h(\q)}h(\fsec)\preceq_{h(\q)}\lcc{3}$ and $h(\ffirst)$ is in $\mathcal{W}_3$. 
We can exclude all possible locations for $h(\ffirst)$ by the same argument as in case $(1)^\uq$ in the proof of
Lemma~\ref{l:bike},
with the \neighbourhood{c} of $\mathcal{W}_3$ in place of the  \neighbourhood{\uq} of $\mathcal{W}_\pw$.
			
			
\item[(6)${}^c$]
$h(\q)$ ends in $\mathcal{W}_3$ and $\lcc{3}\prec_{h(\q)} h(\lc{3})$.\\
%
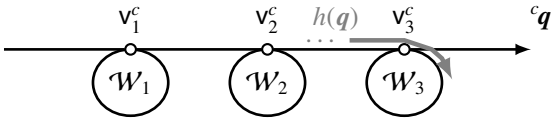
			\begin{tikzpicture}[>=latex,line width=1pt,xscale=.9,yscale=.8]
			%
			\draw (10,-0.55) circle [radius=.55];
			\draw (12,-0.55) circle [radius=.55];
			\draw (14,-0.55) circle [radius=.55];
			\node (1) at (8,0) {};
			\node[point,scale = 0.7,fill =  white,label=above:{$\lcc{1}$}] (m) at (10,0) {};
			\node[point,scale = 0.7,fill =  white,label=above:{$\lcc{2}$}] (2) at (12,0) {};
			\node[point,scale = 0.7,fill =  white,label=above:{$\lcc{3}$}] (3) at (14,0) {};
			\node[label=above:{$\qq{c}$}] (4) at (16,0) {};
			\draw[-,right] (1) to node[below] {}  (m);
			\draw[-,right] (m) to node[below] {} (2);
			\draw[-,right] (2) to node[below] {} (3);
			\draw[->,right] (3) to node[below] {} (4);
			\node[point, draw = white] (01) at (10,-0.55) {$\mathcal{W}_1$};
			\node[point, draw = white] (02) at (12,-0.55) {$\mathcal{W}_2$};
			\node[point, draw = white] (03) at (14,-0.55) {$\mathcal{W}_3$};
			\draw[->,line width=.6mm,gray] (13.2,.15) -- (14,.15) -- (14.4,0) -- (14.55,-.15) -- (14.7,-.5);
			\node at (13,.5) {\textcolor{gray}{$h(\q)$}};
			\node at (12.8,.15) {\textcolor{gray}{$\dots$}};
			\end{tikzpicture}\\
%
%
			Then $h(\q)$ definitely properly intersects the \neighbourhood{c} of $\mathcal{W}_3$, and it may also 
properly intersect the \mbox{\neighbourhood{c}} of $\mathcal{W}_1$ or $\mathcal{W}_2$.
			%
As $\lc{3}=\fsec$, we have $\bnode{\fsec}{c}=\lcc{3}\prec_{h(\q)} h(\fsec)$.
Therefore, if $h(\ffirst)$ is in $\qq{c}$ then $\bnode{\ffirst}{c}\prec_{\,\qq{c}} h(\ffirst)$.
As $\lcc{2}\preceq_{\,\qq{c}}\bnode{\ffirst}{c}$ and there is no $F$-node between $\bnode{\ffirst}{c}$ and
$\bnode{\fsec}{c}$ in $\qq{c}$, it follows that $\lcc{3}\preceq_{h(\q)} h(\ffirst)$ and $h(\ffirst)$ is in
$\mathcal{W}_3$.
We can exclude all possible locations for $h(\ffirst)$ by the same argument as in case $(2)^\uq$ in the proof of
Lemma~\ref{l:bike},
with the \neighbourhood{c} of $\mathcal{W}_3$ in place of the  \neighbourhood{\uq} of $\mathcal{W}_\pw$.
\end{itemize}
We excluded all possible locations in $\Apsi$ for the image $h(\q)$ of a potential \shomo{} $h\colon\q\to\Ia$, which completes the proof of Lemma~\ref{l:3cnfcomp}.
\end{proof}
	
To complete the proof of Theorem~\ref{t:coNPhard}, given a 3CNF $\psi$ with $n_\psi$ clauses, we set $n= (n_\psi+2)(2|\q|+1)$ and take some $(\psi,n)$-gadget $\Apsi$ satisfying the conditions of Lemma~\ref{l:3cnfcomp}. By Lemmas~\ref{l:3cnfsound} and \ref{l:3cnfcomp}, we then obtain: $\TT,\Apsi \not\models \q$ iff $\psi$ is satisfiable. 
	

\section{Conclusion}\label{conclude}

This article contributes to the non-uniform approach to ontology-based data access, which---broadly conceived---also includes optimisation of datalog and disjunctive datalog programs. There are three distinctive directions of research in this area (for detailed references, see Section~\ref{related}): 
\begin{itemize}
\item[\textbf{(I)}] Finding general automata-theoretic, model-theoretic or algebraic characterisations of OMQs with a given data complexity or rewritability type and investigating the computational complexity of checking those characterisations. As it turned out, for many standard DL ontology languages and monadic (disjunctive) datalog programs, the complexity of deciding FO- and datalog-rewritability ranges between \ExpTime{} and 3\ExpTime.

\item[\textbf{(II)}] Designing practical (possibly incomplete) rewriting and approximation algorithms. For example, the algorithm from~\cite{DBLP:journals/ai/KaminskiNG16} either successfully rewrites a given disjunctive datalog program into an equivalent plain datalog program or fails to decide whether the input is datalog rewritable or not.

\item[\textbf{(III)}] Obtaining explicit classifications of `natural' restricted families of OMQs such as, for instance, binary chain datalog sirups~\cite{DBLP:journals/jacm/AfratiP93}. Apart from supplementing \textbf{(II)}, results in this direction help  pinpoint key sources of the high complexity in \textbf{(I)} and thereby identify interesting and better behaved classes of OMQs, as well as develop fine methods of establishing data complexity bounds for OMQ answering.
\end{itemize}
This article contributes to directions \textbf{(I)} and \textbf{(III)}. We introduce two classes of rudimentary OMQs, called d- and dd-sirups, and show that they capture many difficulties of both general OMQs with a disjunctive DL ontology and general monadic (plain and disjunctive) datalog queries. Indeed, the syntactically very simple and seemingly inexpressive d-sirups reveal rather complex and unexpected behaviour: $(i)$ answering them is $\Pi^p_2$-complete for combined complexity and requires finding exponential-size resolution proofs in general; $(ii)$ deciding their FO-rewritability turns out to be 2\ExpTime-hard~\cite{PODS21}---as hard as deciding FO-rewritability of arbitrary monadic datalog queries---with $(iii)$ nonrecursive datalog, positive existential, and UCQ rewritings being of at least single-, double- and triple-exponential size in the worst case, respectively. Thus, understanding the behaviour of d-sirups is challenging yet fundamental for developing OBDA with expressive ontologies (note that d-sirups also constitute a new interesting class of CSPs).

The proofs of the `negative' results mentioned above point to two `culprits': possibly intersecting classes $F$ and $T$ in the covering axiom $F(x) \lor T(x) \leftarrow A(x)$, and multiple binary relations between the same pair of variables in a query.
We demonstrate that elimination of these culprits can lead to non-trivial OMQ classes that admit complete explicit classifications, though may need the development of new methods and quite tricky, laborious proofs. Our main achievement here is an explicit $\ACz$\,/\,\NL\,/\,\PTime\,/\,\coNP-tetrachotomy of path-shaped dd-sirups (with disjoint $F$ and $T$), which required new techniques for establishing membership in \NL{} and for proving \PTime- and especially \coNP-hardness. (Incidentally, the bike technique for proving \coNP-hardness shows that the algorithm from~\cite{DBLP:journals/ai/KaminskiNG16} mentioned in \textbf{(II)} is complete for path-shaped dd-sirups.) We believe that these  techniques can also be used for wider classes of OMQs, which is witnessed by the $\ACz$\,/\,L\,/\,\NL-hardness trichotomy of ditree-shaped dd-sirups in~\cite{PODS21}.

\subsection{Next steps}

Interesting and challenging problems arising from our research are abundant; here are some of them.
\begin{enumerate}

\item Find complete explicit classifications of the following families of OMQs: $(i)$ d-sirups with path CQs (that may contain $FT$-twins), $(ii)$ undirected path-shaped, $(iii)$ ditree- and $(iv)$ undirected tree-shaped dd- and d-sirups. Also, consider (d)d-sirups $(\dis_\top,\q)$ and $(\dis_\top^\bot,\q)$ with total covering $\forall x \, (F(x) \lor T(x))$.

\item Settle the tight complexity of deciding FO- and other types of rewritability for arbitrary $(i)$ d-sirups and $(ii)$ dd-sirups. We conjecture that $(i)$ is harder than $(ii)$ in general.

\item 
Identify the complexity of deciding FO- and other types of rewritability to ontologies in $(i)$ Schema.org and $(ii)$ $\DLk$ and $\DLb$~\cite{ACKZ09}. Ontologies in $(i)$ allow multiple disjunctions (and so covering by any number of classes); those in $(ii)$ allow restricted existential quantification on the right-hand side of implications. 

\item Analyse the size of FO-rewritings for OMQs with disjunctive axioms (starting with d- and dd-sirups). Could FO-rewritings be substantially more succinct than NDL- and PE-rewritings (cf.~\cite[Theorem 6.1]{DBLP:journals/jacm/Rossman08})? (Note that the  succinctness problem for OMQ rewritings is closely related to circuit complexity~\cite{DBLP:journals/ai/GottlobKKPSZ14,DBLP:journals/jacm/BienvenuKKPZ18}.)

\item Consider the {\bf (data complexity)} and {\bf (rewritability)} problems for d- and dd-sirups with multiple answer variables (which could lead to simpler classifications as indicated by~\cite{DBLP:conf/ijcai/HernichLOW15}). 

\item Investigate interconnections between (d)d-sirups and CSPs (starting from those in~\cite{DBLP:journals/tods/BienvenuCLW14,DBLP:journals/lmcs/FeierKL19}) with the aim of transferring results from one formalism to the other.

\item Using the techniques developed in this article for establishing lower data complexity bounds, identify classes of OMQs for which rewriting algorithms such as the ones in~\cite{DBLP:conf/dlog/KaminskiG13,DBLP:journals/ai/KaminskiNG16} are complete.
\end{enumerate}

	
\section*{\bf Acknowledgements} 
The work of O.~Gerasimova was funded by RFBR, project number 20-31-90123. 
The work of V.~Podolskii was supported by the HSE University Basic Research Program. The work of M.~Zakharyaschev was supported by the EPSRC U.K.\ grant EP/S032282. We are grateful to Frank Wolter for his remarks that helped us improve the article. Thanks are also due to the anonymous referees for their careful reading, valuable comments and constructive suggestions.


\bibliographystyle{elsarticle-num-names}

\end{document}